\documentclass[twoside,11pt]{article}

\usepackage[preprint]{jmlr2e}

\usepackage{amsmath}

\usepackage{subfigure}
\usepackage{multirow}
\usepackage{algorithm}
\usepackage{algorithmic}
\usepackage{bbm}

\usepackage{lastpage}
\jmlrheading{25}{2025}{1-\pageref{LastPage}}{2/25; Revised 5/22}{9/22}{21-0000}{Haoran Li, Zicheng Zhang, Wang Luo, Congying Han, Jiayu Lv, Tiande Guo, Yudong Hu}

\ShortHeadings{Consistent Adversarial Robust Reinforcement Learning}{Li, Zhang, Luo, Han, Lv, Guo, Hu}
\firstpageno{1}

\begin{document}

\title{Towards Optimal Adversarial Robust Reinforcement Learning with Infinity Measurement Error\thanks{Preprint. Under review.}\thanks{A preliminary version of this paper was accepted for an oral presentation at ICML 2024~\citep{li2024towards}.}}

\author{\name Haoran Li \email lihaoran21@mails.ucas.ac.cn 
       \AND
       \name Zicheng Zhang \email zhangzicheng19@mails.ucas.ac.cn 
       \AND
       \name Wang Luo \email luowang21@mails.ucas.ac.cn 
       \AND 
       \name Congying Han\thanks{Corresponding author} \email hancy@ucas.ac.cn 
       \AND
       \name Jiayu Lv \email lvjiayu24@mails.ucas.ac.cn 
       \AND
       \name Tiande Guo \email tdguo@ucas.ac.cn 
       \AND 
       \name Yudong Hu \email huyudong201@mails.ucas.ac.cn
       \\
       \addr School of Mathematical Sciences\\
       University of Chinese Academy of Sciences\\
       Beijing 100049, P. R. China}
\editor{My editor}

\maketitle

\begin{abstract}
Ensuring the robustness of deep reinforcement learning~(DRL) agents against adversarial attacks is critical for their trustworthy deployment. 
Recent research highlights the challenges of achieving state-adversarial robustness and suggests that an optimal robust policy~(ORP) does not always exist, complicating the enforcement of strict robustness constraints. 
In this paper, we further explore the concept of ORP. 
We first introduce the Intrinsic State-adversarial Markov Decision Process~(ISA-MDP), a novel formulation where adversaries cannot fundamentally alter the intrinsic nature of state observations. ISA-MDP, supported by empirical and theoretical evidence, universally characterizes decision-making under state-adversarial paradigms.
We rigorously prove that within ISA-MDP, a deterministic and stationary ORP exists, aligning with the Bellman optimal policy. 
Our findings theoretically reveal that improving DRL robustness does not necessarily compromise performance in natural environments.
Furthermore, we demonstrate the necessity of infinity measurement error~(IME) in both $Q$-function and probability spaces to achieve ORP, unveiling vulnerabilities of previous DRL algorithms that rely on $1$-measurement errors.
Motivated by these insights, we develop the Consistent Adversarial Robust Reinforcement Learning~(CAR-RL) framework, which optimizes surrogates of IME. We apply CAR-RL to both value-based and policy-based DRL algorithms, achieving superior performance and validating our theoretical analysis.
\end{abstract}


\section{Introduction}

Deep reinforcement learning (DRL) has achieved remarkable success in addressing complex problems~\citep{mnih2015human, lillicrap2015continuous, silver2016mastering}, showcasing its potential in various practical domains, such as robots~\citep{ibarz2021train}, autonomous driving~\citep{kiran2021deep}, healthcare~\citep{yu2021reinforcement} and news recommendation~\citep{zheng2018drn}. 
Reinforcement learning~(RL) algorithms generally fall into value-based and policy-based categories. Value-based methods are well-suited for small action spaces, while policy-based methods excel in large and continuous action spaces~\citep{sutton2018reinforcement}. Notable examples include the Deep Q-network~(DQN)~\citep{mnih2015human} and the Proximal Policy Optimization~(PPO)~\citep{schulman2017proximal}, which are the crown jewels and considered as baselines in their respective domains.
Despite these achievements, DRL agents remain vulnerable to subtle perturbations in state observations, which can significantly impair their performance~\citep{huang2017adversarial, behzadan2017vulnerability, lin2017tactics, weng2019toward, ilahi2021challenges}. This vulnerability restricts their reliable deployment in real-world scenarios and underscores the critical need to develop robust DRL algorithms for withstanding adversarial attacks.

Pioneering work by~\cite{zhang2020robust} introduced the state-adversarial paradigm in DRL by formulating a modified Markov Decision Process (MDP), called SA-MDP. In this framework, the underlying true state remains invariant while the observed state is subjected to disturbances. They also pointed out the uncertain existence of an optimal robust policy~(ORP) within SA-MDP, suggesting a potential conflict between robustness and optimality of policies. Consequently, current methods based on SA-MDP often seek a balance between robust and optimal policies using various regularizations~\citep{zhang2020robust, oikarinen2021robust, liang2022efficient} or alternating training with learned adversaries~\citep{zhang2021robust, sun2021strongest}.  While these approaches enhance robustness, they lack theoretical guarantees and completely neglect the study of ORP.

In this paper, we focus on investigating the existence of ORP. We identify the states lacking an ORP as a subset of two special sets and find that only a few exceptional states fall into this category. For theoretical clarity, excluding these states from our analysis is essential. Therefore, our investigation begins with the concept of the intrinsic state neighborhood, which depicts the set of states where optimal actions within the MDP remain consistent despite adversarial disturbances. Based on this, we introduce the Intrinsic State-adversarial Markov Decision Process~(ISA-MDP), a novel formulation within which adversaries cannot fundamentally change the intrinsic nature of state observations. Although this formulation may seem idealistic, we demonstrate through theoretical analysis that its difference from SA-MDP is negligible. Moreover, we also validate its rationality by empirical experiments against strong adversarial attacks such as FGSM~\citep{goodfellow2014explaining} and PGD~\citep{madry2017towards}. Both theoretical and empirical supports showcase that the ISA-MDP can be universally applicable to describe state-adversarial decision scenarios.

Within the ISA-MDP, we demonstrate that a stationary and deterministic adversarial ORP always exists and coincides with the Bellman optimal policy derived from the Bellman optimality equations. This objective has been widely employed in previous value-based and policy-based DRL algorithms~\citep{silver2014deterministic, schulman2015trust, wang2016dueling, mnih2016asynchronous, schulman2017proximal} to maximize natural returns, despite lacking robust capabilities~\citep{huang2017adversarial, behzadan2017vulnerability}. Remarkably, our findings reveal that \textit{the Bellman optimal policy also serves as the ORP}. This indicates that enhancing the robustness of DRL agents does not necessitate sacrificing their performance in natural environments, aligning with prior experiment results~\citep{shen2020deep, zhang2020robust, oikarinen2021robust, liang2022efficient}. This insight is vital for deploying DRL agents in real-world scenarios where strong adversarial attacks are relatively rare.

In pursuit of the ORP, we further explore \textit{why conventional DRL algorithms, which target the Bellman optimal policy, fail to ensure adversarial robustness.} We address this challenge by examining the measurements used in action-value function spaces for value-based DRL methods and in probability spaces for policy-based DRL algorithms. 
For value-based DRL agents trained according to the Bellman optimality equations, we analyze the theoretical properties of the distance $\| Q_\theta - Q^* \|_p$ and the Bellman error $\| \mathcal{T}_{B} Q_{\theta} - Q_{\theta} \|_p$ across various Banach spaces, where $1\le p \le \infty$. We identify the significant impact of the parameter $p$ on adversarial robustness. Specifically, achieving an ORP corresponds to minimizing the Bellman Infinity-error (i.e., $ p = \infty $), whereas conventional algorithms typically relate to $p=1$. 
For policy-based DRL agents trained with policy gradient methods, we introduce the concept of measurements $\mathcal{D}_{k, f}$ in probability spaces, where $1\le k \le \infty$ and $f$ represents an $f$-divergence. We also confirm that optimizing $\mathcal{D}_{\infty, f}$ is necessary for adversarial robustness, while previous methods are vulnerable due to their focus on $\mathcal{D}_{1, f}$.

Motivated by our theoretical findings, we develop the Consistent Adversarial Robust Reinforcement Learning~(CAR-RL) framework, taking the infinity measurement error as the optimization objective to attain an ORP. 
To address the computational challenges associated with the $L^{\infty}$-norm, we propose the Consistent Adversarial Robust Deep Q-network~(CAR-DQN), which utilizes a surrogate objective of the Bellman Infinity-error for robust policy learning. 
Additionally, we develop the Consistent Adversarial Robust Proximal Policy Optimization~(CAR-PPO) to approximate the gradient of $\mathcal{D}_{\infty, f}$. CAR-PPO updates the policy based on the gradient of an infinity measurement error surrogate objective for enhanced robustness. We validate the natural and robust performance of our methods across discrete video games in Atari and continuous control tasks in Mujoco.

\paragraph{Contributions}
To summarize, our paper makes the following key contributions:
\begin{itemize}
    \item We propose a universal ISA-MDP formulation for state-adversarial decision, confirm the existence of a deterministic and stationary ORP, and demonstrate its strict alignment with the Bellman optimal policy. This theoretically indicates that improving the robustness of DRL agents need not sacrifice their performance in natural environments, offering a significant advancement over previous research.
    \item We emphasize the necessity of utilizing the infinity measurement error in both action-value function and probability spaces as the minimization objective for achieving theoretical ORP. This contrasts with conventional DRL algorithms, which suffer from a lack of robustness due to their reliance on a $1 $-measurement error.
    \item We develop the CAR-RL framework, which employs a surrogate objective based on the infinity measurement error to learn both natural returns and robustness. We further apply this framework to both value-based and policy-based DRL algorithms, resulting in CAR-DQN and CAR-PPO. We conduct extensive comparative and ablation evaluations across various benchmarks, substantiating the practical effectiveness of CAR-RL and validating the theoretical foundation of our approaches.
\end{itemize}

Some preliminary results of this paper have been accepted for an oral presentation at ICML 2024~\citep{li2024towards}. Compared to the conference version, this paper replaces the assumption of consistency~(CAP) with the ISA-MDP formulation for a more general mathematical description. Additionally, we expand our analysis to include measurements in probability spaces, introduce the CAR-PPO framework, and conduct experiments on continuous control tasks in Mujoco.

\section{Related Work}

\paragraph{Adversarial Attacks on DRL Agents}
The vulnerability of DRL agents to adversarial attacks was first highlighted by \cite{huang2017adversarial}, who demonstrated the susceptibility of DRL policies to Fast Gradient Sign Method (FGSM) attacks~\citep{goodfellow2014explaining} in Atari games. This foundational work sparked further research into various attack methods and robust policies.
Following this, \cite{lin2017tactics, kos2017delving} introduced limited-step attacks to deceive DRL policies,
while \cite{pattanaik2017robust} further explored these vulnerabilities by employing a critic action-value function and gradient descent to undermine DRL performance.
Additionally, \cite{behzadan2017vulnerability} proposed black-box attacks on DQN and verified the transferability of adversarial examples across different models.
\cite{inkawhich2019snooping} showed that even adversaries with restricted access to only action and reward signals could execute highly effective and damaging attacks. 
For continuous control agents, \cite{weng2019toward} developed a two-step attack algorithm based on learned model dynamics.
\cite{zhang2021robust, sun2021strongest} developed learned adversaries by training attackers as RL agents, resulting in SA-RL and PA-AD attacks. 

Research by~\cite{kiourti2020trojdrl, wang2021backdoorl, bharti2022provable, guo2023policycleanse} further explored backdoor attacks in reinforcement learning, uncovering significant vulnerabilities.
In a novel approach, \cite{lu2023adversarial} introduced an adversarial cheap talk setting and trained an adversary through meta-learning. 
\cite{korkmaz2023adversarial} analyzed adversarial directions in the Arcade Learning Environment and found that even state-of-the-art robust agents~\citep{zhang2020robust, oikarinen2021robust} remain vulnerable to policy-independent sensitivity directions.
\cite{franzmeyerillusory} used dual ascent to learn an illusory attack end-to-end.
\cite{gleave2019adversarial} further studied the impact of adversarial policies in multi-agent scenarios.
Lastly, \cite{liang2023game} proposed a temporally-coupled attack, further degrading the performance of robust agents.
This body of work underscores the ongoing challenge of enhancing the adversarial robustness of DRL agents and highlights the need for continued research in this critical area.

\paragraph{Adversarial Robust Policy for DRL Agents}
Earlier studies by~\cite{kos2017delving, behzadan2017whatever} incorporated adversarial states into the replay buffer during training in Atari environments, resulting in limited robustness. 
\cite{fischer2019online} proposed separating the DQN architecture into a Q-network and a policy network, robustly training the policy network with generated adversarial states and provably robust bounds.
\cite{zhang2020robust} characterized state-adversarial RL as SA-MDP and revealed the potential non-existence of the ORP. They addressed this challenge by balancing robustness and natural returns through a KL-based regularization.
\cite{oikarinen2021robust} leveraged robustness certification bounds to design the adversarial loss and combined it with the vanilla training loss.
\cite{liang2022efficient} improved training efficiency by estimating the worst-case value estimation and combining it with classic Temporal Difference~(TD)-target~\citep{sutton1988learning} or Generalized Advantage Estimation~(GAE)~\citep{schulman2015high}.
\cite{nie2023improve} built the DRL architecture for discrete action spaces upon SortNet~\citep{zhang2022rethinking}, enabling global Lipschitz continuity and reducing the need for training extra attackers or finding adversaries. 
Recently, \cite{sun2024belief} proposed learning a pessimistic discrete policy combined with belief state inference and diffusion-based purification.
Prior methods often constrained local smoothness or invariance heuristically to achieve commendable robustness, sometimes compromising natural performance. In contrast, our approach seeks optimal robust policies with strict theoretical guarantees, simultaneously improving both natural and robust performance.

\cite{shen2020deep} found that smooth regularization can enhance both natural performance and robustness for TRPO~\citep{schulman2015trust} and DDPG~\citep{silver2014deterministic}.
\cite{wu2021crop, kumar2021policy} used Randomized Smoothing~(RS) to enable certifiable robustness.
The latest work by \cite{sun2024breaking} introduced a novel smoothing strategy to address the overestimation of robustness.
Moreover, \cite{liu2024beyond} proposed an adaptive defense based on a family of non-dominated policies during the testing phase. 
In multi-agent settings, \cite{he2023robust} analyzed state adversaries in a Markov Game and proposed robust multi-agent Q-learning and actor-critic methods to solve the robust equilibrium.
\cite{bukharin2024robust} extended robustness regularization~\citep{shen2020deep, zhang2020robust} to multi-agent environments by considering a sub-optimal Lipschitz policy in smooth environments. 
\cite{liu2023rethinking} proposed adversarial training with two timescales for effective convergence to a robust policy.
Another line of research focuses on alternated training for agents with learned adversaries~\citep{zhang2021robust, sun2021strongest}, further developed in a game-theoretic framework by~\citep{liang2023game}.
This body of work underscores the importance of developing robust DRL policies and highlights the progress and challenges in enhancing the adversarial robustness of DRL agents.

\section{Preliminaries}

\paragraph{Markov Decision Process (MDP)}
A Markov Decision Process (MDP) is characterized by a tuple $\left( \mathcal{S}, \mathcal{A}, r, \mathbb{P}, \gamma, \mu_0 \right)$, where $\mathcal{S}$ represents the state space, $\mathcal{A}$ denotes the action space, $r: \mathcal{S} \times \mathcal{A} \rightarrow \mathbb{R}$ is the reward function, and $\mathbb{P}: \mathcal{S} \times \mathcal{A} \rightarrow \Delta\left(\mathcal{S}\right)$ describes the transition dynamics with $\Delta\left(\mathcal{S}\right)$ being the probability space over $\mathcal{S}$. The discount factor $\gamma \in [0,1)$ determines the present value of future rewards, and $\mu_0 \in \Delta\left( \mathcal{S} \right)$ specifies the initial state distribution. 
In the following theoretical analysis, we consider MDPs with a continuous state space $\mathcal{S} \subset \mathbb{R}^d$ that is a compact set, and a finite action space $\mathcal{A}$.
Given an MDP, the state value function is defined as $V^\pi(s) = \mathbb{E}_{\pi,\mathbb{P}}\left[ \sum_{t=0}^{\infty} \gamma^t r(s_t,a_t) | s_0=s\right]$, and the action-value function, or \(Q\)-function, is $Q^\pi(s,a) = \mathbb{E}_{\pi,\mathbb{P}}\left[ \sum_{t=0}^{\infty} \gamma^t r(s_t,a_t) | s_0=s, a_0=a\right]$ for any policy $\pi$. 
Denote the function family $\Pi$ as the set of all non-stationary and randomized policies.
A key property of MDPs is the existence of a stationary, deterministic policy that maximizes both $V^\pi(s)$ and $Q^\pi(s, a)$ for all states $s\in\mathcal{S}$ and actions $a\in\mathcal{A}$. Additionally, the optimal $Q$-function, $Q^*(s,a)=\sup_{\pi\in\Pi}Q^\pi(s,a)$, satisfies the Bellman optimality equations:
\begin{equation} \label{eq: bellman optimality equation} \notag
    Q^*(s, a)=r(s, a)+\gamma \mathbb{E}_{s^{\prime} \sim \mathbb{P}(\cdot \mid s, a)}\left[\max _{a^{\prime} \in \mathcal{A}} Q^*\left(s^{\prime}, a^{\prime}\right)\right] \quad \forall s\in\mathcal{S},\ a\in\mathcal{A}.
\end{equation}

\paragraph{State-adversarial Markov Decision Process (SA-MDP)}
A State-adversarial Markov Decision Process (SA-MDP) is defined by the tuple $\left( \mathcal{S}, \mathcal{A}, r, \mathbb{P}, \gamma, \mu_0, B, \nu \right)$, which extends the standard MDP by introducing the definitions of adversaries and perturbation neighborhoods. 
The set $B: \mathcal{S} \rightarrow 2^{\mathcal{S}}$ specifies the allowable perturbation for each state, where the power set $2^{\mathcal{S}}$ represents the set of all subsets of the state space $\mathcal{S}$.
In this framework, an adversary $\nu: \mathcal{S}\rightarrow \mathcal{S}$ can perturb the observed state $s$ to a state $s_\nu:= \nu(s) \in B(s)$. The policy under perturbations is denoted by $\pi \circ \nu$.
The adversarial value function is given by
$V^{\pi\circ \nu}(s) = \mathbb{E}_{\pi\circ \nu,\mathbb{P}}\left[ \sum_{t=0}^{\infty} \gamma^t r(s_t,a_t) | s_0=s\right]$,
and the adversarial action-value function (Q-function) is
$Q^{\pi\circ \nu}(s,a) = \mathbb{E}_{\pi\circ \nu,\mathbb{P}}\left[ \sum_{t=0}^{\infty} \gamma^t r(s_t,a_t) | s_0=s, a_0=a\right]$.
For any policy $\pi$, there exists the strongest adversary $\nu^*(\pi)$ that minimizes the value function for all states, defined as $\nu^*(\pi) = \mathop{\arg\min}_\nu V^{\pi\circ \nu}$.
An optimal robust policy (ORP) $\pi^*$ should maximize the value function under the strongest adversary for all states, satisfying $V^{\pi^*\circ \nu^*(\pi^*)}(s) = \max_\pi V^{\pi\circ \nu^*(\pi)}(s)$ for all $s\in \mathcal{S}$. 
This framework emphasizes the interaction between the policy and adversary, highlighting the importance of developing robust policies in adversarial environments.

\paragraph{Deep Q-network (DQN)}
DQN leverages a neural network $Q_{\theta}$ to parameterize the action-value function. The policy $\pi_{\theta}$ is deterministic, selecting actions based on the highest $Q$-value. Following the baseline work~\citep{zhang2020robust, oikarinen2021robust, liang2022efficient}, we consider Double DQN~\citep{van2016deep} and Dueling DQN~\citep{wang2016dueling} variations. Double DQN uses two Q-networks to alleviate overestimation of the target value $Q_{\Bar{\theta}}$. Dueling DQN enhances learning efficiency by splitting the Q-network output into two separate heads: one representing the state value and the other representing the advantage function.
DQN optimizes the Q-network by minimizing the Bellman error derived from Bellman optimality equations. The Bellman error can be formulated by:
\begin{equation}\label{eq: dqn loss} \notag
    \mathcal{L}_{\text{DQN}}(\theta) = \mathbb{E}_{(s,a,s^\prime,r)}  \left| r + \gamma \max_{a^\prime} Q_{\Bar{\theta}}(s^\prime, a^\prime) - Q_\theta(s,a) \right| .
\end{equation}

\paragraph{Proximal Policy Optimization (PPO)}
PPO is an actor-critic method that employs a policy gradient approach. It consists of a policy network $\pi_\theta$ as the actor and a state value network $V_{\theta_v}$ as the critic. To estimate the advantage function, PPO employs Generalized Advantage Estimation~(GAE)~\citep{schulman2015high}, defined as $A_t(s_t, a_t) = \sum_{i=0}^{k-1} \gamma^i r_{t+i} + \gamma^k V_{\theta_v}(s_{t+k}) - V_{\theta_v}(s_t)$, with $k$ as a hyperparameter.
PPO, through the clipping function, ensures that the new policy does not deviate significantly from the old one. The critic network can typically be trained by regression on the mean-square error. The policy loss for training the actor is defined as the following formulation:
\begin{equation} \notag
    \mathcal{L}_{\text{PPO}}(\theta) = \mathbb{E}_{(s_t,a_t,r_t)} \left[ - \min \left( \frac{\pi_\theta(a_t|s_t)}{\pi_{\theta_{\text{old}}}(a_t|s_t)}A_t, \operatorname{clip} \left( \frac{\pi_\theta(a_t|s_t)}{\pi_{\theta_{\text{old}}}(a_t|s_t)}, 1-\eta, 1+\eta \right) A_t \right) \right],
\end{equation}
where $\eta$ is the clipping hyperparameter. 
Additionally, the PPO loss typically includes an entropy penalty to encourage further exploration.

\section{Optimal Adversarial Robustness}

In this section, we explore the concept of Optimal Robust Policy~(ORP). While \citet{zhang2020robust} pointed out that ORP does not universally exist in all adversarial scenarios, our findings indicate that ORP is absent in only a few states. Importantly, the measure of these exceptional states is nearly zero in complex tasks.
To address this, we introduce the concept of intrinsic states to eliminate the impact of these exceptional states, leading to the formulation of a new framework, Intrinsic State-adversarial MDP~(ISA-MDP).
We further demonstrate that the difference between ISA-MDP and SA-MDP in complex environments is negligible, meaning that ISA-MDP can be applied to any scenarios where SA-MDP is applicable. 
Subsequently, we propose a novel Consistent Adversarial Robust~(CAR) Operator~$\mathcal{T}_{car}$ for computing the adversarial $Q$-function. Within the ISA-MDP framework, we identify that the fixed point of the CAR operator corresponds exactly to the optimal $Q$-function $Q^*$, thereby proving the existence of a deterministic and stationary ORP.

\subsection{Intrinsic State-adversarial Markov Decision Process (ISA-MDP)}

\begin{figure}
    \centering
\includegraphics[width=0.85\columnwidth]{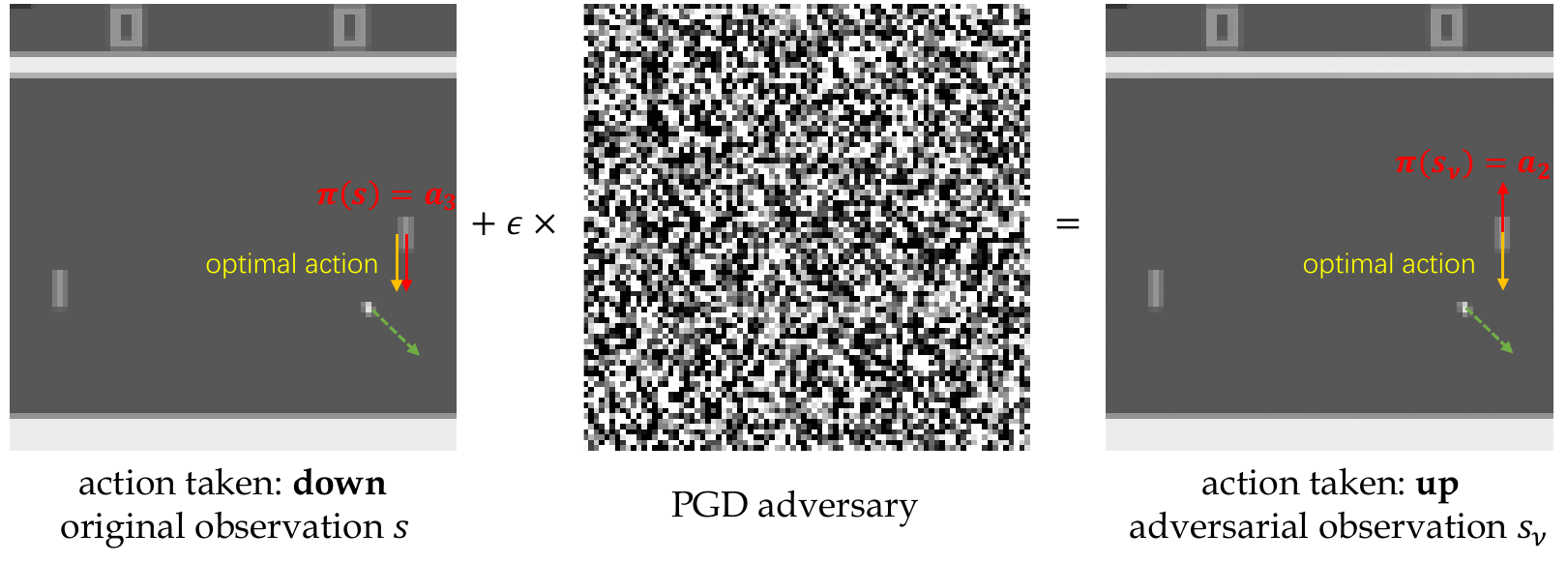} 
    \vspace{-1em}
    \caption{An example of state adversary in DQN. While the adversary disrupts the policy executed by DQN, it does not affect the optimal action prescribed by the Bellman optimal policy. This observation leads us to examine two critical issues: whether the Bellman optimal policy serves as the ORP, and why vanilla DQN trained with Bellman error fails to achieve robustness.
    }
    \label{fig:intrisic state}
\end{figure}

Given a general adversary, we observe that the true state $s$ and the perturbed observation~$s_\nu$ typically share the same optimal action in practice. This observation, illustrated in Figure \ref{fig:intrisic state} and Appendix \ref{app: instrinsic state}, suggests that the Bellman optimal policy is inherently robust. We explore this by theoretically analyzing the effects of adversarial perturbations on the optimal action.
We then define the intrinsic state neighborhood as the set of states for which the optimal action remains consistent.
\begin{definition}[Intrinsic State Neighborhood]
    \label{defn: intrinsic state neighborhood}
    Given an SA-MDP, let $Q^*$ denote the Bellman optimal $Q$ function derived from the Bellman optimality equations. The intrinsic state neighborhood for any state $s$ is defined as the following:
\begin{equation}\notag
    \begin{aligned}
        B^*(s) := \left\{ s^{\prime} \in \mathcal{S} | s^\prime \in B(s),\ \mathop{\arg\max}_a Q^*(s^{\prime},a) =   \mathop{\arg\max}_a Q^*(s,a) \right\}.
    \end{aligned}
\end{equation} 
\end{definition}
Without loss of generality, we assume that the adversary perturbation set is an $\epsilon$-neighborhood, defined as $B_\epsilon(s) = \left\{ \|s^\prime - s\| \le \epsilon \right\}$ for convenience. Note that our following theorems and proofs can be extended to general cases.

Furthermore, we characterize the states where the standard state neighborhood differs from the intrinsic neighborhood and show that such states are rare in real environments. This finding underpins the ISA-MDP formulation that we present later.
\begin{theorem}[Sparse Difference Between Intrinsic and Standard Neighborhood] \label{thm: consistency assumption reasonable}
    For any MDP $\mathcal{M}$, let $\mathcal{S}_{nu}$ denote the state set where the optimal action is not unique, i.e., $\mathcal{S}_{nu} = \left\{ s\in\mathcal{S} | \mathop{\arg\max}_a Q^*(s,a) \text{ is not a singleton} \right\}$. Given $\epsilon > 0$, let $\mathcal{S}_{nin}$ denote the set of states where \textit{the intrinsic state $\epsilon$-neighborhood} is not the same as \textit{the $\epsilon$-neighborhood}, i.e., $\mathcal{S}_{nin} = \left\{ s\in\mathcal{S} |  B_\epsilon(s) \neq B^*_\epsilon(s) \right\}$. Then, we have that
    $$\mathcal{S}_{nin} \subseteq \mathcal{S}_{nu} \cup \mathcal{S}_{0} + B_\epsilon ,$$
    where $\mathcal{S}_{0}$ is the set of discontinuous points that cause the optimal action to change, i.e.,   
    $\mathcal{S}_{0}=\{s \in S|\forall\  \epsilon_1 > 0, \exists \  s' \in B_{\epsilon_1}(s), \text{s.t. } \arg\max_a Q^*(s', a) \neq \arg\max_a Q^*(s, a)\} \cap \{s \in S|\exists\  a \in \mathcal{A}, \text{s.t. } Q^*(s, a)\  \text{is not continuous at}\  s\}$.
\end{theorem}
The proof of Theorem~\ref{thm: consistency assumption reasonable} is provided in Appendix~\ref{app: reasonable of C assumption}. In practice, the set $\mathcal{S}_{nu}$ is nearly empty in most complex environments, and $\mathcal{S}_{0}$ consists of rare and special discontinuous points of $Q^*$. Theorem~\ref{thm: consistency assumption reasonable} essentially shows that for intricate tasks, $\mathcal{S}_{nin}$ is a quite small set, with the measure $m(\mathcal{S}_{nin})$ being approximately $O(\epsilon^d)$, where $m(A)$ represents the measure of set $A$ and $d$ is the dimension of the state space. 
For example, in certain natural environments, the reward function and transition dynamics are smooth, especially in continuous control tasks where the transition dynamics come from some physical laws. In these scenarios, the value and action-value functions are continuous, making $\mathcal{S}_0$ an empty set.
Beyond smooth environments, many tasks can be modeled with sparse rewards. In these cases, the value and action-value functions are almost everywhere continuous, indicating that $\mathcal{S}_0$ is a set of zero measures.
Further insights and stronger conclusions are provided in Appendix~\ref{app: reasonable of C assumption} under more stringent conditions.

These findings and analysis motivate the development of the ISA-MDP formulation. Firstly, we introduce the concept of intrinsic adversaries.
\begin{definition}[Intrinsic Adversary] \label{ass: intrinsic adversary assumption}
    The intrinsic adversary $\nu_I: \mathcal{S} \rightarrow \mathcal{S},\ s\mapsto \nu_I(s)\in B^*(s)$ can perturb any state $s$ to a state $s_\nu:=\nu_I(s)$ within its intrinsic state neighborhood.
\end{definition}
Building upon this definition, we formulate the intrinsic state-adversarial MDP.
\begin{definition}[Intrinsic State-adversarial Markov Decision Process~(ISA-MDP)] \label{def: isamdp}
    An intrinsic state-adversarial MDP is defined by the tuple $\left( \mathcal{S}, \mathcal{A}, r, \mathbb{P}, \gamma, \mu_0, B^*, \nu_I \right)$, where the allowable perturbation set $B^*: \mathcal{S} \rightarrow 2^{\mathcal{S}}$ is the intrinsic state neighborhood as characterized in Definition~\ref{defn: intrinsic state neighborhood}, and $\nu_I$ is the intrinsic adversary defined in Definition~\ref{ass: intrinsic adversary assumption}.
\end{definition}
While ISA-MDP can be seen as a more special characterization of SA-MDP with the intrinsic adversary, Theorem~\ref{thm: consistency assumption reasonable} inherently showcases that the difference between ISA-MDP and SA-MDP is negligible in complex environments, making ISA-MDP applicable in any scenario where SA-MDP is used. 

\subsection{Consistent Optimal Robust Policy}

To establish the relation between the optimal $Q$-function before and after the perturbation, we propose a consistent adversarial robust (CAR) operator.
\begin{definition}[Consistent Adversarial Robust~(CAR) Operator $\mathcal{T}_{car}$]\label{thm: equivalence}
    Given an SA-MDP, the CAR operator is   $\mathcal{T}_{car}: L^p\left( \mathcal{S}\times\mathcal{A} \right) \rightarrow L^p\left( \mathcal{S}\times\mathcal{A} \right)$,
    \begin{equation} \notag
        \begin{aligned}
            \left( \mathcal{T}_{car} Q \right) (s,a) = r(s,a) + \gamma \mathbb{E}_{ s^\prime \sim \mathbb{P}(\cdot|s,a)} \left[ \min _{s^\prime_\nu \in B(s^\prime)} Q \left(s^\prime,\mathop{\arg\max}_{a_{s^\prime_\nu}} Q\left(s^\prime_\nu, a_{s^\prime_\nu}\right)\right) \right].
        \end{aligned}
    \end{equation}
\end{definition}

However, $\mathcal{T}_{car}$ is not contractive (shown in Appendix \ref{app: not a contraction}), so we cannot directly ensure the existence of a fixed point through contraction mapping. Fortunately, Theorem~\ref{thm: fixed point} demonstrates that within the ISA-MDP, $\mathcal{T}_{car}$ has a fixed point, which corresponds to the optimal adversarial action-value function $Q^{\pi^*\circ \nu^*(\pi^*)}$. 
\begin{theorem}[Relation between $Q^*$ and $Q^{\pi^*\circ \nu^*(\pi^*)}$]\label{thm: fixed point} \
        \begin{itemize}
            \item If the optimal adversarial action-value function $Q^{\pi^*\circ \nu^*(\pi^*)}$ under the strongest adversary exists for all $s\in\mathcal{S}$ and $a\in\mathcal{A}$, then it is the fixed point of CAR operator.
            \item Within the ISA-MDP, $Q^*$ is the fixed point of CAR operator $\mathcal{T}_{car}$. Furthermore, $Q^*$ is the optimal adversarial action-value function under the strongest adversary, \textit{i.e.}, $Q^*(s,a) = Q^{\pi^*\circ \nu^*(\pi^*)}(s,a)$, for all $s\in\mathcal{S}$ and $a\in\mathcal{A}$.
        \end{itemize}
\end{theorem}
\begin{remark}
    Note that the outer minimization and inner maximization operations in Definition~\ref{thm: equivalence} do not constitute a standard minimax problem because their objectives differ, resulting in a bilevel optimization problem. Generally, the minimization and maximization operations cannot be swapped. However, they can be swapped if $\mathop{\arg\max}_{a_{s^\prime_\nu}} Q\left(s^\prime_\nu, a_{s^\prime_\nu}\right)$ is a singleton for all $s^\prime_\nu \in B(s^\prime)$, which is a mild condition in our training. Additionally, we verify that $Q^*$ is still the fixed point of the operator after swapping.
\end{remark}
On this basis, it can be derived from Theorem~\ref{thm: fixed point} that within the ISA-MDP, the greedy policy $\pi^*(s):=\mathop{\arg\max}_a Q^*(s,a)$, for all $s\in \mathcal{S}$, is exactly the ORP.
\begin{corollary}[Existence of ORP]\label{cor: existence of ORP}
    Within the ISA-MDP, there exists a deterministic and stationary policy $\pi^*$ which satisfies $V^{\pi^*\circ \nu^*(\pi^*)}(s) \ge V^{\pi\circ \nu^*(\pi)}(s)$ and $Q^{\pi^*\circ \nu^*(\pi^*)}(s, a) \ge Q^{\pi\circ \nu^*(\pi)}(s, a)$, for all $\pi\in\Pi$, $s\in\mathcal{S}$ and $a\in\mathcal{A}$.
\end{corollary}
Theorem~\ref{thm: fixed point} and Corollary~\ref{cor: existence of ORP}, whose proofs are provided in Appendix~\ref{app: fixed point}, demonstrate that within the ISA-MDP, the ORP against the strongest adversary is equivalent to Bellman optimal policy derived from Bellman optimality equations. This finding indicates that the objectives in both natural and adversarial environments are aligned, as further supported by our experiment results in Figure~\ref{fig:consitency}, whose detailed illustrations are provided in Section~\ref{sec: Comparison}. 

\begin{figure}[t]
    \centering
    \begin{subfigure}
        \centering
\includegraphics[width=0.43\textwidth]{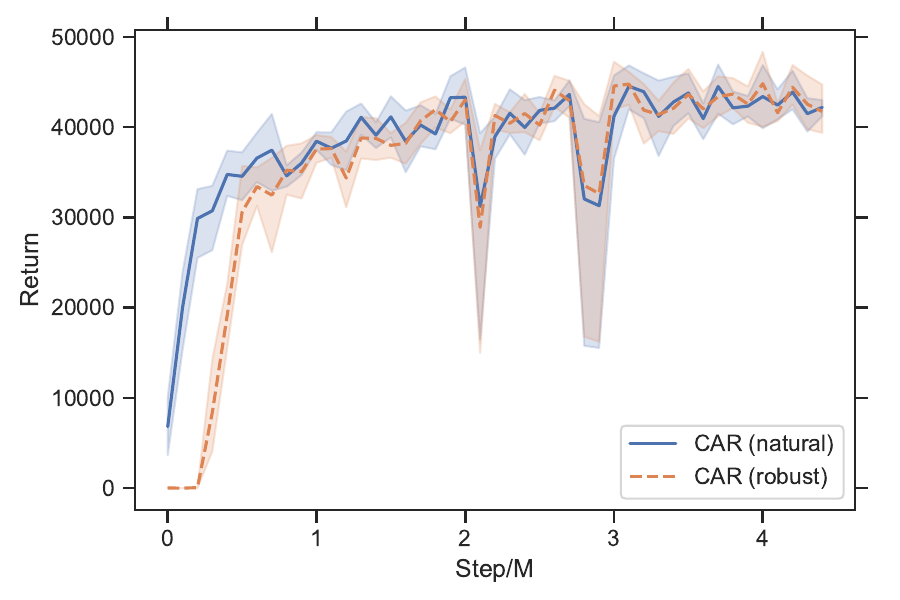} 
    \end{subfigure}
    \begin{subfigure}
        \centering
\includegraphics[width=0.43\textwidth]{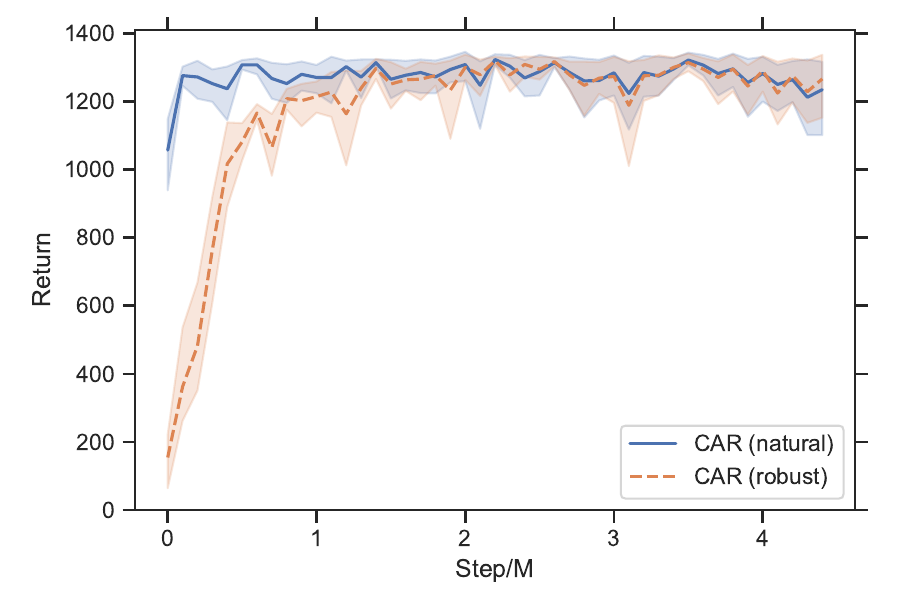}
    \end{subfigure}
    \\
    \vskip -0.1in
    \begin{minipage}{0.43\textwidth}
        \centering
        \quad \scriptsize{RoadRunner} 
    \end{minipage}
    \begin{minipage}{0.43\textwidth}
        \centering
        \quad\quad \scriptsize{BankHeist}
    \end{minipage}
    \\
    \vskip 0.05in
    \begin{subfigure}
        \centering
\includegraphics[width=0.43\textwidth]{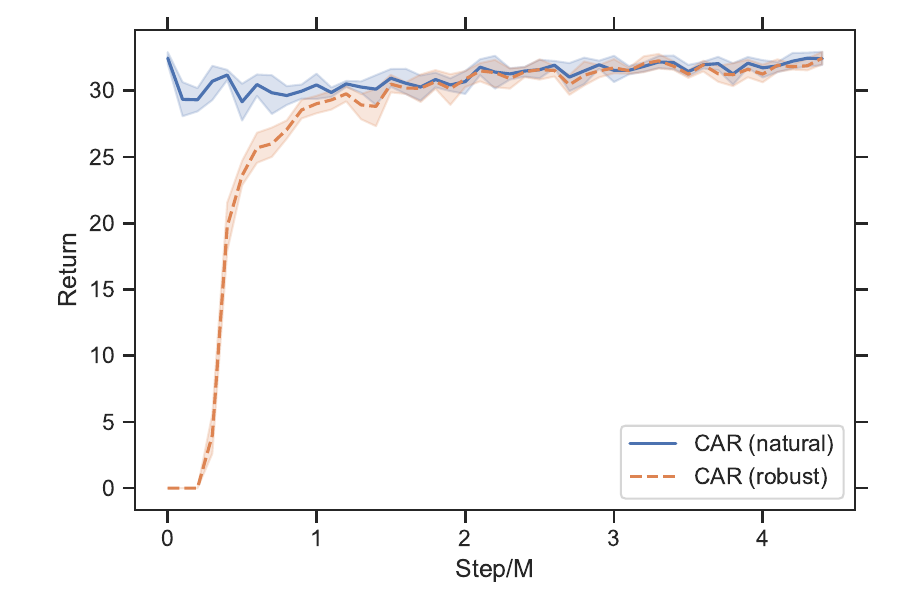}
    \end{subfigure}
    \begin{subfigure}
        \centering
\includegraphics[width=0.43\textwidth]{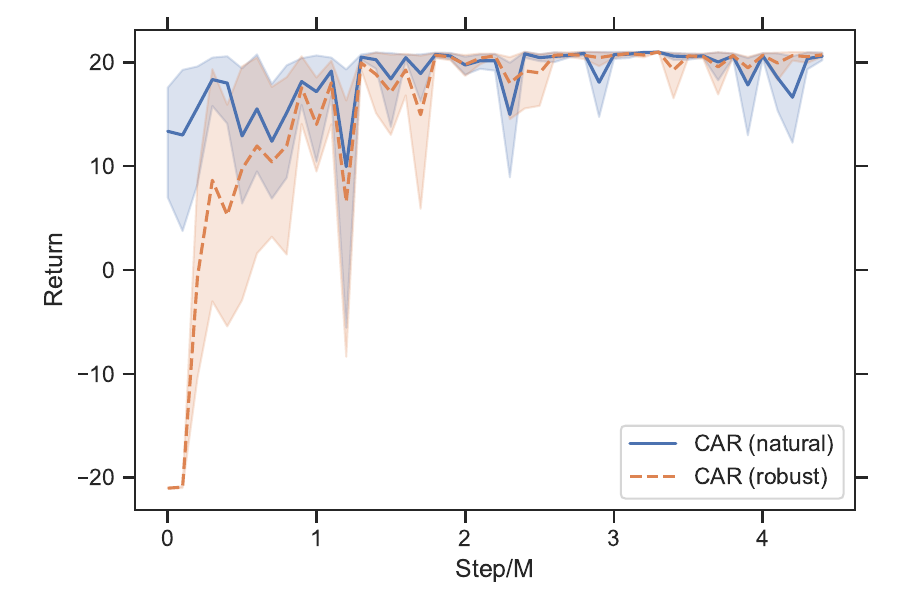}
    \end{subfigure}
    \\
    \vskip -0.1in
    \begin{minipage}{0.43\textwidth}
        \centering
        \quad \scriptsize{Freeway}
    \end{minipage}
    \begin{minipage}{0.43\textwidth}
        \centering
        \setlength{\parindent}{0.5em}
        \quad\quad \scriptsize{Pong}
    \end{minipage}
    \vskip -0.05in
    \caption{Episode rewards of CAR-DQN agents with and without 10-step PGD attacks on 4 Atari games and 5 random seeds. As evidenced by the overlap of the two curves,  CAR-DQN achieves the consistency between Bellman optimal policy and ORP. }
    \label{fig:consitency}
    \vskip -0.1in
\end{figure}

Furthermore, we investigate the convergence of $\mathcal{T}_{car}$ in a smooth environment. We find that within the ISA-MDP, the fixed point iterations of $\mathcal{T}_{car}$ at least converge to a sub-optimal solution that is close to $Q^*$. 
The detailed proof of Theorem~\ref{thm:convergence} and the formal definition of $\left(L_r, L_{\mathbb{P}}\right)$-smooth environments are presented in Appendix~\ref{app:convergence}. 
\begin{theorem}[Convergence of CAR Operator $\mathcal{T}_{car}$]\label{thm:convergence}
    Suppose the environment is $\left(L_r, L_{\mathbb{P}}\right)$-smooth and suppose $Q_0$ and $r$ are uniformly bounded, i.e. $\exists\ M_{Q_0},M_r >0$ such that $\left|Q_0(s,a)\right| \le M_{Q_0},\ \left|r(s,a)\right| \le M_r,\ \forall s\in\mathcal{S}, a\in\mathcal{A}$. Let $Q^*$ denote the Bellman optimal Q-function and $Q_{k+1} = \mathcal{T}_{car} Q_{k} = \mathcal{T}_{car}^{k+1} Q_0$ for all $k\in\mathbb{N}$. Within the ISA-MDP, we have
    \begin{equation} \notag
        \|Q_{k+1} - Q^*\|_\infty \le  \gamma^{k+1} \| Q_0 - Q^*\|_\infty + \gamma^{k+1} D_{Q_0} + \frac{2 \gamma \epsilon }{1-\gamma}L_{\mathcal{T}_{car}},
    \end{equation}
    where $D_{Q_0} = 2 \max_{s,a} \max_{s_\nu \in B_\epsilon(s)} \max_{a} \left| Q_0 \left(s,a\right) - Q_0 \left(s_\nu,a\right) \right|$ is a constant relating to the local continuity of initial $Q_0$, $L_{\mathcal{T}_{car}} =  L_r + \gamma C_{Q_0} L_{\mathbb{P}}$ and $C_{Q_0} = \max\left\{ M_{Q_0}, \frac{M_r}{1-\gamma} \right\}$. 
\end{theorem}

Given that the Bellman optimal policy doubles as the ORP, we further investigate the reasons behind the vulnerability of DRL agents. 
Despite aiming for the Bellman optimal policy, why do traditional reinforcement learning methods exhibit poor robustness? 
In the following section, we address this issue by examining the natural performance and robustness of policies across various measurement errors in action-value function and probability spaces.

\section{Robustness of Policy with Tiny k-measurement Error}
In this section, we examine the adversarial robustness of value-based and policy-based reinforcement learning methods with different measurements, identifying the significance of minimizing infinity measurement error to achieve the optimal robust policy.

\subsection{Policy Robustness under Bellman p-error in Action-value Function Space}

We investigate the stability of policies across various Banach spaces for value-based reinforcement learning methods and highlight the necessity of optimizing Bellman infinity-error to achieve the optimal robust policy. Firstly, in Subsection~\ref{sec: necessity of infinity norm for value function space}, we demonstrate the necessity of the $L^\infty$-norm for measuring $\| Q_\theta - Q^*\|$. Next, in Subsection~\ref{sec: stability of Bellman optimality equations}, we identify the critical role of Bellman infinity-error in ensuring the stability of Bellman optimality equations. Finally, in Subsection~\ref{sec: stability of deep q learning}, we extend these theoretical results to deep Q-learning, accounting for the practical sampling process.

\subsubsection{Necessity of Infinity Norm in Action-value Function Space for Adversarial Robustness}
\label{sec: necessity of infinity norm for value function space}

\begin{figure}[t]
\centering
\includegraphics[width=0.6\columnwidth]{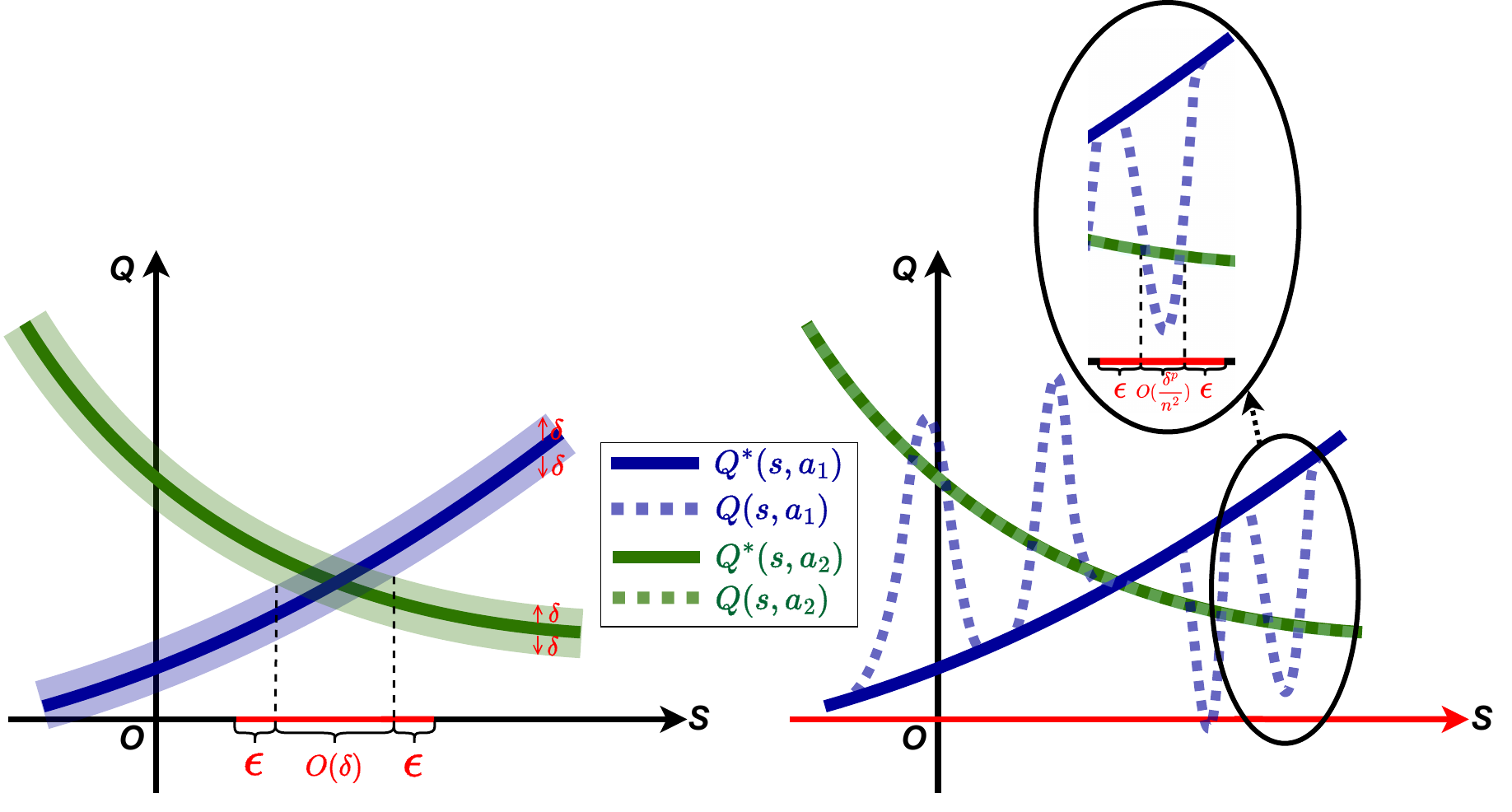}
\caption{Examples of adversarial robustness for $Q$ satisfying $\|Q-Q^*\|_p\le\delta$.
Given a perturbation radius $\epsilon$, the red line represents the set $\mathcal{S}_{adv}^Q$, in which states have adversarial states. The left panel depicts the case of $p=\infty$. In this scenario, all such $Q$ functions are distributed within the shadow area, with the measure of $\mathcal{S}_{adv}^Q$ being a small value, approximately $2 \epsilon + O\left( \delta \right)$, indicating good robustness. 
In contrast, the right panel shows that for $1\le p<\infty$, there always exists a $Q$ function such that $\mathcal{S}_{adv}^Q = \mathcal{S}$, indicating poor robustness.}
\label{fig:example}
\end{figure}

Let $Q_\theta$ denote a parameterized $Q$-function. Theoretically, value-based RL training requires minimizing $\left\|Q_\theta - Q^*\right\|_{\mathcal{B}}$, where $\mathcal{B}$ is a Banach space. 
When $\| Q_\theta - Q^*\|_{\mathcal{B}}$ equals zero, $Q_\theta$ perfectly matches $Q^*$. 
However, in practice, achieving this perfect match is challenging due to limitations such as the representation capacity of neural networks and the convergence issues in optimization algorithms. 
Given these practical constraints, we examine the properties of $Q_\theta$ when the $\left\|Q_\theta - Q^*\right\|_{\mathcal{B}}$ is a small positive value across different Banach spaces $\mathcal{B}$. 
This analysis underscores the importance of selecting appropriate norm spaces in ensuring adversarial robustness for value-based reinforcement learning methods.

We study the adversarial robustness of the greedy policy derived from $Q$ when $0 < \left\|Q - Q^*\right\|_{p}=\delta \ll 1$ for different $L^p$ spaces.
Given a function $Q$, let $\mathcal{S}^Q_{sub}$ denote the set of states where the greedy policy according to $Q$ is suboptimal, \textit{i.e.},
\begin{equation}\notag
    \mathcal{S}^Q_{sub} = \{ s | Q^*(s,\mathop{\arg\max}_a Q(s, a)) < \max_a Q^*(s, a) \}.
\end{equation}
Note that the smaller set $\mathcal{S}^Q_{sub}$ indicates better performance in natural environments without adversarial attacks.
For a given perturbation budget $\epsilon$, let $\mathcal{S}^{Q,\epsilon}_{adv}$ denote the set of states within whose $\epsilon$-neighborhood there exists the adversarial state, \textit{i.e.},
\begin{equation}\notag
    \begin{aligned}
        \mathcal{S}^{Q,\epsilon}_{adv} = \{ s | \exists s_\nu \in B_\epsilon(s),\text{s.t. } Q^*(s,\mathop{\arg\max}_a Q(s_\nu,a)) < \max_a Q^*(s,a) \}.        
    \end{aligned}
\end{equation}
The smaller the set $\mathcal{S}^{Q,\epsilon}_{adv}$, the stronger the adversarial robustness in the worst case.
In addition, note that $\mathcal{S}_{sub}^Q = \mathcal{S}_{adv}^{Q,0}$. Furthermore, our investigation can be understood as a mathematical study of when it holds that $\lim_{\epsilon \rightarrow 0} \mathcal{S}_{adv}^{Q,\epsilon} = \mathcal{S}_{adv}^{Q,0}$, whose establishment indicates that the natural performance is aligned with the adversarial robustness in the worst case.

\begin{theorem}[Necessity of the $L^\infty$ Space for Adversarial Robustness]\label{thm:necessity of infty norm}
    There exists an MDP instance such that the following statements hold. 
        \begin{itemize}
            \item[\textbf{(1).}] For any $1\le p<\infty$ and $\delta>0$, there exists a function $Q\in L^p\left( \mathcal{S}\times\mathcal{A} \right)$ satisfying $\|Q-Q^*\|_{p} \leq \delta$ such that $m\left(\mathcal{S}^Q_{sub}\right) = O(\delta)$ yet $m\left( \mathcal{S}^{Q,\epsilon}_{adv} \right) =m \left(\mathcal{S}\right)$.
            \item[\textbf{(2).}] There exists a $\bar{\delta}>0$ such that for any $0< \delta \le \bar{\delta}$, and any function $Q\in L^\infty\left( \mathcal{S}\times\mathcal{A} \right)$ satisfying $\|Q-Q^*\|_{\infty} \leq  \delta$, we have that $m\left(\mathcal{S}^Q_{sub}\right) = O(\delta)$ and $m\left( \mathcal{S}^{Q,\epsilon}_{adv} \right) = 2\epsilon + O\left( \delta \right)$.
        \end{itemize}
\end{theorem}
The proof of Theorem~\ref{thm:necessity of infty norm} is provided in Appendix~\ref{app: infinity is necessary in value function space}. 
The first statement indicates that when $\|Q-Q^*\|_p$ is small for $1\le p<\infty$, there always exist adversarial examples near almost all states, resulting in quite poor robustness, even though the policy might exhibit excellent performance in a natural environment without adversarial attacks.
This observation sheds light on the vulnerability of DRL agents, consistent with findings from various studies~\citep{huang2017adversarial, ilahi2021challenges}.
Importantly, the second statement points out that minimizing $\|Q-Q^*\|$ in the $L^\infty$-norm space can mitigate this vulnerability, enabling the policy to achieve both natural performance and robustness. This insight motivates the optimization of $\|Q_{\theta}-Q^*\|_{\infty}$ in DRL algorithms. Intuitive examples of Theorem~\ref{thm:necessity of infty norm} are illustrated in Figure~\ref{fig:example}.

\subsubsection{Stability of Bellman Optimality Equations}
\label{sec: stability of Bellman optimality equations}

Unfortunately, it is infeasible to directly measure $\|Q_{\theta}-Q^*\|$ in practical DRL procedures due to the unknown nature of $Q^*$. Most value-based methods train $Q_\theta$ by optimizing the Bellman error $\| \mathcal{T}_B Q_\theta - Q_\theta\|_{\mathcal{B}'}$, where $\mathcal{T}_B$ is the Bellman optimal operator:
$$\left(\mathcal{T}_B Q\right) (s, a)=r(s, a)+\gamma \mathbb{E}_{s^{\prime} \sim \mathbb{P}(\cdot \mid s, a)}\left[\max _{a^{\prime} \in \mathcal{A}} Q\left(s^{\prime}, a^{\prime}\right)\right] .$$

Similar to Theorem~\ref{thm:necessity of infty norm}, we need to determine which Banach space $\mathcal{B}^\prime$ is the best for training DRL agents to minimize adversarial states. 
When $\| \mathcal{T}_B Q_\theta - Q_\theta\|_{\mathcal{B}^\prime}$ is zero, $Q_\theta$ is precisely $Q^*$. 
However, in practice, we can only optimize $\| \mathcal{T}_B Q_\theta - Q_\theta\|_{\mathcal{B}^\prime}$ to a small nonzero quantity. In this scenario, we discuss the conditions under which $\left\|Q_\theta - Q^*\right\|_{\mathcal{B}}$ can be controlled. 
To analyze this issue, we introduce the concept of functional equations stability, drawing on relevant research about physics-informed neural networks for partial differential equations~\citep{wang20222}.
\begin{definition}[Stability of Functional Equations]\label{def: stability}
    Given two Banach spaces $\mathcal{B}_1$ and $\mathcal{B}_2$, if there exist $\delta>0$ and $C>0$ such that for all $Q\in \mathcal{B}_1 \cap \mathcal{B}_2$ satisfying $\|\mathcal{T}Q - Q\|_{\mathcal{B}_1} < \delta$, we have that $\|Q - Q^*\|_{\mathcal{B}_2} < C \|\mathcal{T}Q - Q\|_{\mathcal{B}_1}$, where $Q^*$ is the exact solution of this functional equation, then we say a nonlinear functional equation $\mathcal{T}Q = Q$ is $\left( \mathcal{B}_1, \mathcal{B}_2 \right)$-stable. For simplicity, we call that functional $\mathcal{T}$ is $\left( \mathcal{B}_1, \mathcal{B}_2 \right)$-stable.
\end{definition}
\begin{remark}
    This definition indicates that if $\mathcal{T}Q = Q$ is $\left( \mathcal{B}_1, \mathcal{B}_2 \right)$-stable, then $\|Q - Q^*\|_{\mathcal{B}_2} = O\left(  \|\mathcal{T}Q - Q\|_{\mathcal{B}_1} \right)$, as $ \|\mathcal{T}Q - Q\|_{\mathcal{B}_1} \longrightarrow 0$, $\forall\ Q\in \mathcal{B}_1 \cap \mathcal{B}_2$.    
\end{remark}

The above Definition implies that if there exists a space $\mathcal{B}^\prime$ such that  $\mathcal{T}_{B}$ is $\left(  \mathcal{B}^\prime,  L^\infty\left( \mathcal{S}\times\mathcal{A} \right) \right)$-stable,
then $\left\|Q_\theta - Q^*\right\|_{\infty}$ will be controlled when minimizing the Bellman error in $\mathcal{B}^\prime$ space. This will make DRL agents robust according to Theorem~\ref{thm:necessity of infty norm}.(2).  
The following theorems illustrate the conditions that affect the stability and instability of $\mathcal{T}_{B}$ and provide guidance for selecting a suitable $\mathcal{B}^\prime$.
\begin{theorem}[Stable and Unstable Properties of $\mathcal{T}_{B}$ in $L^p$ Spaces]\label{thm: stability} \
\begin{itemize}
    \item[\textbf{(1).}] For any MDP $\mathcal{M}$, let $C_{\mathbb{P},p}:= \sup_{(s,a)\in\mathcal{S}\times \mathcal{A}} \left\| \mathbb{P}(\cdot \mid s, a) \right\|_{L^{\frac{p}{p-1}}\left( \mathcal{S} \right)}$. Assume $p$ and $q$ satisfy the following conditions:
    \begin{equation}\label{eq: stability condition} \notag
        C_{\mathbb{P},p}< \frac{1}{\gamma};\quad
        p \ge \max\left\{1, \frac{\log \left( \left| \mathcal{A}\right|\right) + \log \left( \mu\left( \mathcal{S} \right) \right)}{\log \frac{1}{\gamma C_{\mathbb{P},p}} } \right\}; \quad p \le q \le \infty.
    \end{equation}
    Then, Bellman optimality equation $\mathcal{T}_B Q = Q$ is $\left(  L^q\left( \mathcal{S}\times\mathcal{A} \right),  L^p\left( \mathcal{S}\times\mathcal{A} \right) \right)$-stable.
    \item[\textbf{(2).}] There exists an MDP such that for all $1 \le q < p\le \infty$, the Bellman optimality equation $\mathcal{T}_B Q = Q$ is not $\left(  L^q\left( \mathcal{S}\times\mathcal{A} \right),  L^p \left( \mathcal{S}\times\mathcal{A} \right) \right)$-stable.
\end{itemize}
\end{theorem}
\begin{remark}
    Note that we have proved a stronger conclusion than stability because the stable property holds for all $Q$ rather than for $Q$ satisfying $ \|\mathcal{T}Q - Q\|_{\mathcal{B}_1} \longrightarrow 0$.
\end{remark}
The proofs of Theorem~\ref{thm: stability} are shown in Appendices~\ref{app: stability of BOE} and~\ref{app: instability of BOE}.
We note that $\lim_{p\rightarrow\infty} C_{\mathbb{P},p} =1 < \frac{1}{\gamma}$, so the first condition holds when $p$ is sufficiently large. 
The second condition indicates that $p$ is relevant to the size of the state and action spaces, and the third condition reveals that stability is achievable when $q$ is larger than $p$.  
Consequently, we have that $\mathcal{T}_{B}$ is $\left(L^\infty\left( \mathcal{S}\times\mathcal{A} \right),  L^\infty\left( \mathcal{S}\times\mathcal{A} \right) \right)$-stable. Therefore, to achieve adversarial robustness, we can optimize DRL agents in space $\mathcal{B}^\prime = L^\infty\left( \mathcal{S}\times\mathcal{A} \right)$. Moreover, $\mathcal{B}^\prime$ cannot be $L^q\left( \mathcal{S}\times\mathcal{A} \right)$ for any $1\le q < \infty$, as stated in Theorem~\ref{thm: stability}.(2).

\subsubsection{Stability of Deep Q-learning}
\label{sec: stability of deep q learning}

Our theoretical analysis has shown that training a deep Q-network (DQN) by minimizing the Bellman error in $L^{\infty}$ space is feasible for achieving the ORP. In this subsection, we further examine the stability of DQN by considering the practical sampling process.

Given an MDP $\mathcal{M}$, define the state-action visitation distribution $d_{\mu_0}^\pi$ as the following:
$$d_{\mu_0}^\pi(s,a)=\mathbb{E}_{s_0 \sim \mu_0(\cdot), s_{t} \sim \mathbb{P}(\cdot|s_{t-1}, a_{t-1}), a_{t} \sim \pi(\cdot|s_t)} \left[ (1-\gamma)\sum_{t=0}^\infty \gamma^t \mathbb{I}(s_t = s, a_t = a) \right].$$
Deep Q-learning algorithms, such as DQN, utilize the following objective $\mathcal{L}(Q_{\theta};\pi_{\theta})$ based on interactions with the environment:
\begin{equation}\notag
    \begin{aligned}
    \mathcal{L}(Q_{\theta};\pi_{\theta})=\mathbb{E}_{(s,a)\sim d_{\mu_0}^{\pi_{\theta}}} \left| \mathcal{T}_B Q_{\theta}(s,a) - Q_{\theta}(s,a)  \right|. 
\end{aligned}
\end{equation}
The theoretical analysis of functional equations stability can be extended to $\mathcal{L}(Q_{\theta};\pi_{\theta})$ by incorporating the sampling probability into a seminorm. 

\begin{definition}[$(p,d_{\mu_0}^\pi)$-seminorm]
    Given a policy $\pi$, a function $f:\mathcal{S}\times\mathcal{A}\rightarrow \mathbb{R}$ and $1\le p\le\infty$, if $d_{\mu_0}^\pi$ is a probability density function, we define the $(p,d_{\mu_0}^\pi)$-seminorm as the following, which satisfies the absolute homogeneity and triangle inequality:
    \begin{equation}\notag
        \begin{aligned}
            \left\| f \right\|_{p,d_{\mu_0}^\pi} := \left\| d_{\mu_0}^\pi f \right\|_{p} = \left(\int_{(s,a)\in\mathcal{S}\times\mathcal{A}} \left| d_{\mu_0}^\pi(s,a)  f(s,a) \right|^p d \mu(s,a) \right)^{\frac{1}{p}}.
        \end{aligned}
    \end{equation}
\end{definition}
We further analysis the properties of this seminorm in Appendix~\ref{app:seminorm}. 
Note that the $(p,d_{\mu_0}^\pi)$-seminorm becomes a norm if $d_{\mu_0}^\pi(s,a)>0$ almost everywhere for $(s,a)\in\mathcal{S}\times\mathcal{A}$. Based on this definition, the deep Q-learning objective $\mathcal{L}(Q_{\theta};\pi_{\theta})$ can be expressed as 
$$\mathcal{L}(Q_{\theta};\pi_{\theta}) = \left\|\mathcal{T}_B Q_{\theta} - Q_{\theta} \right\|_{1,d_{\mu_0}^{\pi_{\theta}}}.$$ 
Similar to Theorem~\ref{thm: stability}, we prove that the objective $\mathcal{L}(Q_{\theta};\pi_{\theta})$ cannot ensure robustness, while $(\infty,d_{\mu_0}^\pi)$-seminorm is necessary for adversarial robustness. 
\begin{theorem}[Stable and Unstable Properties of $\mathcal{T}_{B}$ in $(p,d_{\mu_0}^\pi)$-seminorm Spaces] \label{thm: stable of seminorm} \
\begin{itemize}
    \item For any MDP $\mathcal{M}$ and fixed policy $\pi$, assume $C_{d_{\mu_0}^\pi}:=\inf_{(s,a)\in\mathcal{S}\times\mathcal{A}} d_{\mu_0}^\pi(s,a)  >0$ and assume $p$ and $q$ satisfy the following conditions:
    \begin{equation}\notag
        C_{\mathbb{P},p}< \frac{1}{\gamma};\quad
        p \ge \max\left\{1, \frac{\log \left( \left| \mathcal{A}\right|\right) + \log \left( \mu\left( \mathcal{S} \right) \right)}{\log \frac{1}{\gamma C_{\mathbb{P},p}} } \right\}; \quad p \le q \le \infty,
    \end{equation}
    where $C_{\mathbb{P},p}:= \sup_{(s,a)\in\mathcal{S}\times \mathcal{A}} \left\| \mathbb{P}(\cdot \mid s, a) \right\|_{L^{\frac{p}{p-1}}\left( \mathcal{S} \right)}$.
    Then, Bellman optimality equation $\mathcal{T}_B Q = Q$ is $\left(  L^{q,d_{\mu_0}^\pi}\left( \mathcal{S}\times\mathcal{A} \right),  L^p\left( \mathcal{S}\times\mathcal{A} \right) \right)$-stable.
    \item There exists an MDP $\mathcal{M}$ such that for all $\pi$ satisfying $M_{d_{\mu_0}^\pi}:=\sup_{(s,a)\in\mathcal{S}\times\mathcal{A}} d_{\mu_0}^\pi(s,a) <\infty$, Bellman optimality equation $\mathcal{T}_B Q = Q$ is not $\left(  L^{q,d_{\mu_0}^\pi}\left( \mathcal{S}\times\mathcal{A} \right),  L^p \left( \mathcal{S}\times\mathcal{A} \right) \right)$-stable, for all $1 \le q < p\le \infty$.
\end{itemize}
\end{theorem}
\begin{remark}
    Note that in a practical Q-learning scheme, we take the $\epsilon$-greedy policy for exploration. As a result, for any state-action pair $(s, a)$, we can visit it with positive probability, and thus the condition $C_{d_{\mu_0}^\pi}>0$ holds.
\end{remark}
The proofs of Theorem~\ref{thm: stable of seminorm} are provided in Appendix~\ref{app: Stability of DQN: the Good}, where a theorem with tighter bounds yet stringent conditions is also shown.
We also investigate the stability when $d_{\mu_0}^\pi$ is a probability mass function in Appendices~\ref{app:seminorm} and~\ref{app:Stability of DQN: the Bad}.

\subsection{Policy Robustness under k-measurement Error in Probability Space}

We examine the adversarial robustness of policy-based reinforcement learning methods across different measurements in probability space, emphasizing the importance of optimizing infinity measurement error to attain the optimal robust policy. In Subsection~\ref{sec: probability measurement formulation for policy-based rl}, we reformulate the optimization objectives of REINFORCE and PPO as $1$-measurement errors, drawing on the operator view of policy gradient methods presented by~\cite{ghosh2020operator}. In Subsection~\ref{sec: necessity of infinity measurement in probability space}, we demonstrate the vulnerability of non-infinity measurement errors and establish robustness guarantee under infinity measurement error.

\subsubsection{Probability Measurement Formulation for Policy-based Reinforcement Learning}
\label{sec: probability measurement formulation for policy-based rl}

We reformulate the optimization objectives of two key policy-based reinforcement learning methods, REINFORCE~\citep{williams1992simple} and PPO~\citep{schulman2017proximal}, as the discrepancy between two probability distributions under $1$-measurement. 

Firstly, we define a $k$-measurement in probability space to describe the discrepancy between two probability distributions.
\begin{definition}[$k$-measurement in Probability Space]
    Given any probability distributions $\omega(\cdot) \in \Delta(\mathcal{S})$ and $p(\cdot|s),q(\cdot|s) \in \Delta(\mathcal{A})$ for all $s \in \mathcal{S}$, the $k$-measurement with f-divergence between $p$ and $q$ under $\omega$ is defined as the following:
    \begin{equation} \notag
        \mathcal{D}_{k, f}^{\omega} \left( p \| q \right) := \left\| \omega \mathcal{D}_f \left( p \| q \right) \right\|_{k} = \left( \int_{s \in \mathcal{S}} \left| \omega(s) \mathcal{D}_f \left( p(\cdot|s) \| q(\cdot|s) \right) \right|^k d \mu(s) \right)^{\frac{1}{k}},
    \end{equation}
    where $\mathcal{D}_f$ represents the f-divergence in policy space.
\end{definition}
\begin{remark}
    Note that $\mathcal{D}_{k, f}^{\omega}$ is not a distance due to the asymmetrical nature of f-divergence.
\end{remark}

Given an MDP $\mathcal{M}$, define the state visitation distribution $d_{\mu_0}^\pi$ as the following:
$$d_{\mu_0}^\pi(s)=\mathbb{E}_{s_0 \sim \mu_0(\cdot), a_{t-1} \sim \pi(\cdot|s_{t-1}), s_{t} \sim \mathbb{P}(\cdot|s_{t-1}, a_{t-1})} \left[ (1-\gamma)\sum_{t=0}^\infty \gamma^t \mathbb{I}(s_t = s) \right].$$
In the following text, if there is no ambiguity, we will abbreviate $d_{\mu_0}^\pi$ to $d^\pi$.

\paragraph{Probability Measurement Formulation of REINFORCE.}
REINFORCE is one of the most fundamental policy-based methods and updates according to the policy gradient theorem. Its update can be formulated as the following:
\begin{equation} \label{eq: reinforce update step} \notag
    \theta_{t+1} = \theta_t + \eta_t \mathbb{E}_{s \sim d^{\pi_t}(\cdot), a \sim \pi_t(\cdot | s)} \left[ Q^{\pi_t}(s,a) \nabla_\theta \log \pi_\theta(a|s) \Big|_{\theta=\theta_t} \right],
\end{equation}
where $\pi_{\theta_t}$ is abbreviated as $\pi_t$.
If $Q^\pi(s,a)$ are positive for all $(s,a)\in \mathcal{S}\times\mathcal{A}$, the above equation can be interpreted as performing a gradient step for the optimization problem:
\begin{equation} \notag
    \min_{\pi_\theta \in \Pi} 
    \mathcal{D}_{1, \operatorname{KL}}^{\mu_t} \left( \varphi_t \| \pi_\theta \right),
\end{equation}
where probability distributions $\varphi_t$ and $\mu_t$ are as follows, respectively:
\begin{align} \notag
    \varphi_t(a|s) = \frac{1}{V^{\pi_t}(s)} Q^{\pi_t}(s,a) \pi_t(a|s),\quad \mu_t(s) = \frac{1}{\mathbb{E}_{s \sim d^{\pi_t}(\cdot)} \left[ V^{\pi_t}(s) \right]} d^{\pi_t}(s) V^{\pi_t}(s).
\end{align}
It can be explained as iteratively approximating an improved policy $\varphi_t$ based on value functions by optimizing a $1$-measurement with the forward KL-divergence.

\paragraph{Probability Measurement Formulation of PPO.}
PPO is one of the most widely used policy-based methods and often incorporates an entropy bonus with a coefficient $\beta > 0$. PPO can be viewed as performing a gradient step for the following optimization problem:
\begin{equation} \notag
    \min_{\pi_\theta \in \Pi} 
    \mathcal{D}_{1, \operatorname{KL}}^{\mu_t} \left( \operatorname{clip}(\pi_\theta) \| \varphi_t \right),
\end{equation}
where probability distributions $\varphi_t$ and $\mu_t$ are as follows, respectively:
\begin{align} \notag
    \varphi_t(a|s) = \frac{\exp \left( \beta A^{\pi_t}(s,a) \right)}{\sum_{a^\prime} \exp \left( \beta A^{\pi_t}(s,a^\prime) \right)}, \quad \mu_t(s) = d^{\pi_t}(s).
\end{align}
Similarly, it can be understood as iteratively approximating an improved policy $\varphi_t$ based on advantage functions by minimizing a $1$-measurement with the reversed KL-divergence.
This formulation results in the update:
\begin{equation} \notag
    \pi_{t+1} = \mathop{\arg \max}_{\pi_\theta} \mathbb{E}_{s \sim d^{\pi_t}(\cdot)} \left[ \sum_a \pi_\theta(a|s) A^{\pi_t}(s,a) - \frac{1}{\beta} \sum_a \pi_\theta(a|s) \log \pi_\theta(a|s) \right],
\end{equation}
where the clipping operator is omitted for readability.

\subsubsection{Necessity of Infinity Measurement in Probability Space for Adversarial Robustness}
\label{sec: necessity of infinity measurement in probability space}

When the value of $\mathcal{D}_{k, \operatorname{KL}}^{\mu_t}$ vanishes, $\pi_\theta$ is exactly the improved policy $\varphi_t$. However, this is not achievable due to practical constraints. Therefore, we investigate the approximation properties when the $\mathcal{D}_{k, \operatorname{KL}}^{\mu_t}$ is a small but nonzero value for different $k$. This examination highlights the significance of utilizing suitable measurements for guaranteeing adversarial robustness in policy-based reinforcement learning methods.

\begin{figure}[t]
    \centering
    \includegraphics[width=0.6\columnwidth]{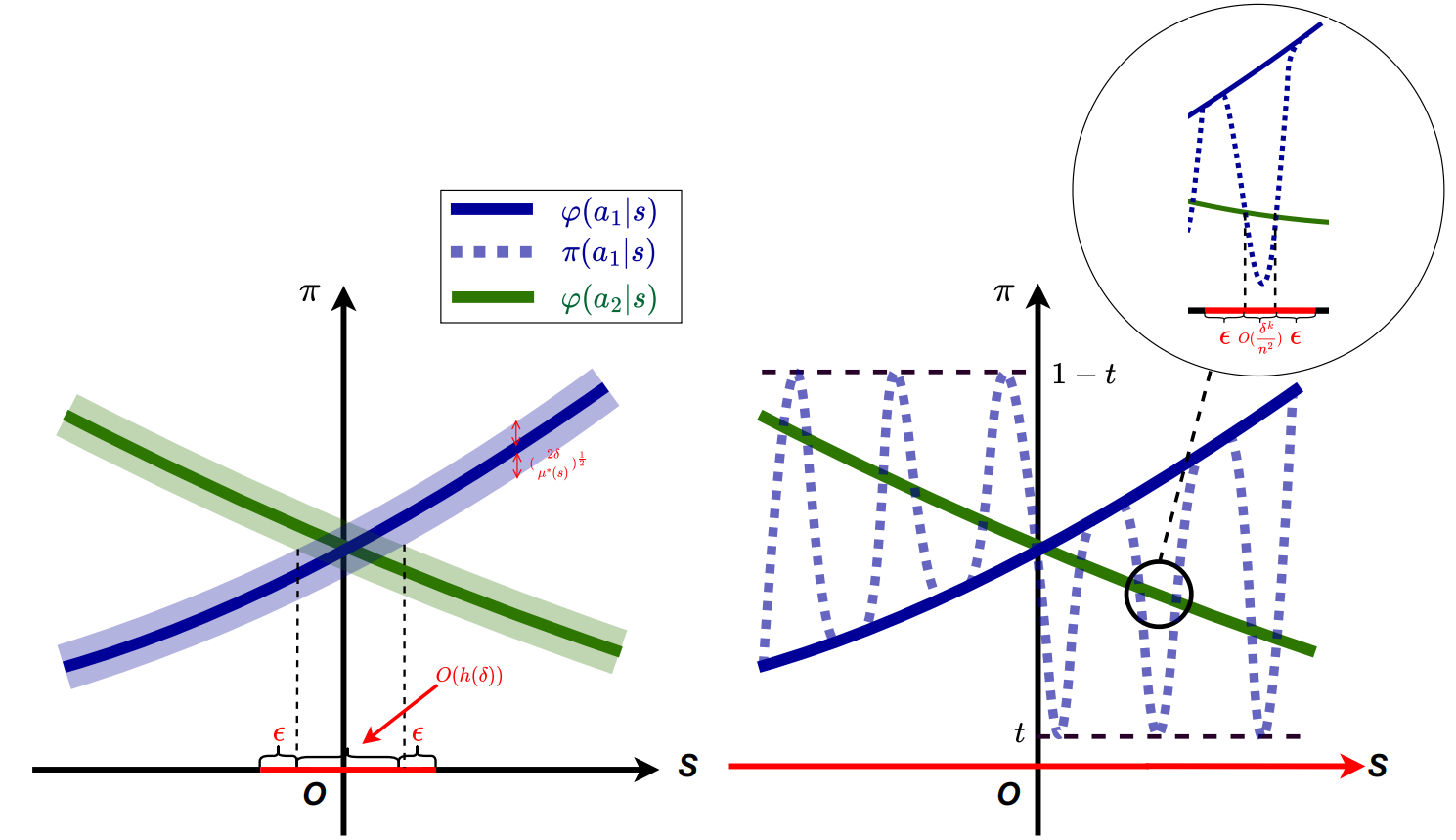}
    \caption{Examples of adversarial robustness for the policy $\pi$ satisfying $\mathcal{D}_{k, \operatorname{KL}}^{\mu} \left( \varphi \| \pi \right) \le \delta$. Given a perturbation radius $\epsilon$, the red line represents the set $\mathcal{S}_{adv}^{\pi,\epsilon}$, which consists of states with adversarial states. In the left panel, we depict the case of $k = \infty$. In this scenario, all such policies $\pi$ are distributed within the shadow area, with the measure of $\mathcal{S}_{adv}^{\pi,\epsilon}$ being a small value, approximately $2\epsilon + O\left(h(\delta)\right)$, indicating good robustness. In contrast, the right panel illustrates that for $1 \leq k < \infty$, there always exists $\pi$ such that $\mathcal{S}_{adv}^{\pi,\epsilon} = \mathcal{S}$ in the worst case, indicating poor robustness.}
    \label{fig: measurement in probability space}
\end{figure}

Given a policy $\pi$ and a perturbation budget $\epsilon$, let $\mathcal{S}_{sub}^\pi$ represent the set of states where the greedy policy according to $\pi$ is suboptimal:
$$\mathcal{S}_{sub}^\pi = \left\{ s \left| Q^*(s, \mathop{\arg\max}_a \pi(a|s) ) < \max_a Q^*(s,a) \right. \right\}.$$
The smaller set $\mathcal{S}_{sub}^\pi$ signifies better natural performance.
Given a perturbation budget $\epsilon$, $\mathcal{S}_{adv}^{\pi,\epsilon}$ denotes the state set within whose $\epsilon$-neighborhood there exists the adversarial state:
$$
\mathcal{S}_{adv}^{\pi,\epsilon} = \left\{ s \left| \exists\ s_\nu \in B_\epsilon(s), \text{s.t. } Q^*(s, \mathop{\arg\max}_a \pi(a|s_\nu) ) < \max_a Q^*(s,a) \right. \right\}.
$$
The smaller set $\mathcal{S}_{adv}^{\pi,\epsilon}$ means stronger worst-case adversarial robustness.
In addition, we note that $\mathcal{S}_{sub}^\pi = \mathcal{S}_{adv}^{\pi,0}$. From this perspective, our analysis can be understood as a mathematical study of when it holds that $\lim_{\epsilon \rightarrow 0} \mathcal{S}_{adv}^{\pi,\epsilon} = \mathcal{S}_{adv}^{\pi,0}$, showcasing that the natural performance and the adversarial robustness in the worst case are consistent.
Denote that $\pi^*$ as the greedy policy according to the Bellman optimal $Q$-function and $\mu^*(s):=d^{\pi^*}(s),\ \forall\ s\in\mathcal{S}$.

We first demonstrate that small $k$-measurement errors for any $1\le k < \infty$ may ensure the natural performance but can lead to complete vulnerability in the worst case.
\begin{theorem}[Vulnerability of Non-infinity Measurement Errors] \label{thm: Vulnerability of Non-infinity Measurement Errors}
    There exists an MDP such that the following statement holds. Let $\varphi$ be a policy from the policy family $\mathcal{F} = \left\{ \varphi \left| \mathop{\arg\max}_a \varphi(a|s) = \mathop{\arg\max}_a \pi^*(a|s) \right. \right\}$. For any $\epsilon>0$, $1 \le k < \infty$, $\delta > 0$ and any state distribution $\mu$, there exists a policy $\pi$ satisfying $\mathcal{D}_{k, \operatorname{KL}}^{\mu} \left( \varphi \| \pi \right) \le \delta$ and $\mathcal{D}_{k, \operatorname{KL}}^{\mu} \left( \pi \| \varphi \right) \le \delta$, such that $m\left( \mathcal{S}_{sub}^\pi \right) = O(\delta)$ but $m\left( \mathcal{S}_{adv}^{\pi,\epsilon} \right) = m(\mathcal{S})$.
\end{theorem}

Fortunately, we find that policies with small infinity measurement errors can guarantee both natural and robust performance.
\begin{theorem}[Robustness Guarantee under Infinity Measurement Error] \label{thm: robustness of infinity measurement}
    For any MDP and state distribution $\mu>0$, let $S_{\delta} = \{s \mid \exists \ a,a' \in \mathcal{A},\ \text{s.t.}\ |\varphi(a|s)-\varphi(a'|s)| \leq 2\sqrt{\frac{2\delta}{\mu(s)}}  \}$ and $h(\delta) = \mu(S_{\delta})$. Then $h(\delta)$ is a monotonic function with $h(0) = 0$. Furthermore, let $\varphi$ be a policy from the policy class $\mathcal{F} = \left\{ \varphi \mid \mathop{\arg\max}_a \varphi(a|s) = \mathop{\arg\max}_a \pi^*(a|s)  \right\}$. If $S_\delta $ is the union of finite connected subsets, then for any $\epsilon,\delta > 0$ and any policy $\pi$ satisfying $\mathcal{D}_{\infty, \operatorname{KL}}^{\mu} \left( \varphi \| \pi \right) \le \delta$ or $\mathcal{D}_{\infty, \operatorname{KL}}^{\mu} \left( \pi \| \varphi \right) \le \delta$, we have that $m\left( \mathcal{S}_{sub}^\pi \right) = O(h(\delta))$ and $m\left( \mathcal{S}_{adv}^{\pi,\epsilon} \right) = 2\epsilon +  O(h(\delta))$.
\end{theorem}
The proofs of Theorem~\ref{thm: Vulnerability of Non-infinity Measurement Errors} and~\ref{thm: robustness of infinity measurement} are presented in Appendix~\ref{app: infinity measurement is necessary in probability space}, with intuitional examples illustrated in Figure~\ref{fig: measurement in probability space}. These findings underscore the necessity of using $\mathcal{D}_{\infty, \operatorname{KL}}^{\mu_t}$ as the objective for ensuring adversarial robustness in the worst case.

\section{Consistent Adversarial Robust Reinforcement Learning}

Our theoretical analysis reveals the potential of employing the infinity measurement error as the optimization objective to achieve the optimal robust policy.
However, the exact computation of the infinity measurement is intractable due to the unknown environment dynamics and continuous state space. Therefore, we introduce the surrogate objective of the infinity measurement error and develop the Consistent Adversarial Robust Reinforcement Learning~(CAR-RL) framework. This framework enhances both the natural and robust performance of agents.
Furthermore, we apply CAR-RL to both the value-based DQN and the policy-based PPO algorithms, leading to CAR-DQN and CAR-PPO.

\subsection{Consistent Adversarial Robust Deep Q-network}

Inspired by Theorem~\ref{thm: stable of seminorm}, we propose the CAR-DQN to train robust DQN by minimizing the Bellman infinity-error $\| \mathcal{T}_B Q_{\theta} - Q_{\theta}\|_{\infty,d^{\pi_{\theta}}_{\mu_0}}$. This objective can be minimized using the following loss function (as derived in Appendix~\ref{app:derive car-dqn}):
$$\mathcal{L}_{car}(\theta) = \sup_{(s,a)\in\mathcal{S}\times\mathcal{A}} d_{\mu_0}^{\pi_\theta}(s,a) \max_{s_\nu \in B_\epsilon(s)}  \left| \mathcal{T}_{B} Q_\theta(s_\nu,a) - Q_\theta(s_\nu,a) \right|,$$
where $\pi_\theta$ is the behavior policy  associated with $Q_\theta$, typically an $\epsilon$-greedy policy derived from $Q_\theta$. Since interactions with the environment in an SA-MDP are based on the true state $s$ rather than the perturbed state $s_\nu$, it is not feasible to directly estimate $\mathcal{T}_{B}Q_\theta(s_\nu, a)$. To address this, we exploit $\mathcal{T}_{B}Q_\theta(s,a)$ as a substitute, leading to the training objective:
$$\mathcal{L}_{car}^{train}(\theta) = \sup_{(s,a)\in\mathcal{S}\times\mathcal{A}} d_{\mu_0}^{\pi_\theta}(s,a) \max_{s_\nu \in B_\epsilon(s)}  \left| \mathcal{T}_{B} Q_\theta(s,a) - Q_\theta(s_\nu,a) \right|.$$
As shown in Theorem~\ref{thm:bound car objective}, this surrogate objective $\mathcal{L}_{car}^{train}$ effectively bounds  $\mathcal{L}_{car}$, especially in smooth environments.
We also define $\mathcal{L}_{car}^{diff}(\theta)$ as the following:
$$\mathcal{L}_{car}^{diff}(\theta) = \sup_{(s,a)\in\mathcal{S}\times\mathcal{A}} d_{\mu_0}^{\pi_\theta}(s,a) \max_{s_\nu \in B_\epsilon(s)} \left| \mathcal{T}_{B} Q_\theta(s_\nu,a) - \mathcal{T}_{B}Q_\theta(s,a) \right|.$$
\begin{theorem}[Bounding $\mathcal{L}_{car}$ with $\mathcal{L}_{car}^{train}$]\label{thm:bound car objective} \
We have that
    \begin{equation}\notag
        \left| \mathcal{L}_{car}^{train}(\theta) -  \mathcal{L}_{car}^{diff}(\theta)\right| \le \mathcal{L}_{car}(\theta) \le \mathcal{L}_{car}^{train}(\theta) + \mathcal{L}_{car}^{diff}(\theta).
    \end{equation}    
Further, suppose the environment is $\left(L_r, L_{\mathbb{P}}\right)$-smooth and suppose $Q_\theta$ and $r$ are uniformly bounded, i.e. $\exists\ M_Q,M_r >0$ such that $\left|Q_\theta(s,a)\right| \le M_Q,\ \left|r(s,a)\right| \le M_r\ \forall s\in\mathcal{S}, a\in\mathcal{A}$. If $M:=\sup_{\theta,(s,a)\in\mathcal{S}\times\mathcal{A}} d_{\mu_0}^{\pi_\theta}(s,a) <\infty$, then we have that
    \begin{equation}\notag
        \mathcal{L}_{car}^{diff}(\theta) \le C_{\mathcal{T}_{B}} \epsilon,
    \end{equation}
    where $C_{\mathcal{T}_{B}}=L_{\mathcal{T}_{B}} M$, $L_{\mathcal{T}_{B}} =  L_r + \gamma C_Q L_{\mathbb{P}}$ and $C_Q = \max\left\{ M_Q, \frac{M_r}{1-\gamma} \right\}$. The definition of $\left(L_r, L_{\mathbb{P}}\right)$-smooth environment is shown in Appendix \ref{app:convergence}.
\end{theorem}
The proofs of Theorem~\ref{thm:bound car objective} are provided in Appendix~\ref{app: surrogate objective of car-dqn}, which confirm that $\mathcal{L}_{car}^{train}(\theta)$ is a valid surrogate objective from the optimization perspective. This theorem also highlights a potential source of instability during robust training:
if $\mathcal{L}_{car}^{train}(\theta)$ is minimized to a small value but remains less than $\mathcal{L}_{car}^{diff}(\theta)$, then the primary objective $\mathcal{L}_{car}(\theta)$ may tend to increase, indicating a potential training overfitting.

To make better use of batch samples and improve training efficiency, we introduce a soft version $\mathcal{L}_{car}^{soft}(\theta)$ of the CAR objective, denoted as $\mathcal{L}_{car}^{soft}(\theta)$ (derivation in Appendix~\ref{app: soft objective of car-dqn}):
\begin{align}\notag
    \mathcal{L}_{car}^{soft}(\theta) = \sum_{i\in \mathcal{\left|B\right|}} \alpha_i \max_{s_\nu \in B_\epsilon(s_i)} \left|r_i + \gamma \max_{a^\prime} Q_{\bar{\theta}}(s_i^\prime,a^\prime) - Q_{\theta}(s_\nu,a_i)  \right|, 
\end{align}
where the weighting $\alpha_i$ is defined by the following probability distribution over a batch:
$$\alpha_i = \frac{e^{\frac{1}{\lambda} \max_{s_\nu } \left|r_i + \gamma \max_{a^\prime} Q_{\bar{\theta}}(s_i^\prime,a^\prime) - Q_{\theta}(s_\nu,a_i)  \right|}}{\sum_{i\in \mathcal{\left|B\right|}} e^{\frac{1}{\lambda} \max_{s_\nu } \left|r_i + \gamma \max_{a^\prime} Q_{\bar{\theta}}(s_i^\prime,a^\prime) - Q_{\theta}(s_\nu,a_i)  \right|}}. $$
Here, $\mathcal{B}$ represents a batch of transition pairs sampled from the replay buffer, $\bar{\theta}$ is the parameter of the target network, and $\lambda$ is the coefficient to control the level of softness.

\subsection{Consistent Adversarial Robust Proximal Policy Optimization}

Motivated by Theorems~\ref{thm: Vulnerability of Non-infinity Measurement Errors} and~\ref{thm: robustness of infinity measurement}, we develop the CAR-PPO method to enhance the robustness of PPO. The objective of CAR-PPO is to optimize the infinity measurement error $\mathcal{D}_{\infty, \operatorname{KL}}^{\mu_t} \left( \operatorname{clip}(\pi_\theta) \| \varphi_t \right)$. This is equivalent to minimizing the following loss function (as derived in Appendix~\ref{app: derivation of car-ppo}):
$$ 
\mathcal{L}_{car}^{train}(\theta) =  \sup_{s\in\mathcal{S}}\ d^{\pi_t}(s) \left( -\frac{1}{\beta} \mathcal{H}\left( \pi_\theta(\cdot | s) \right) - \min_{s_\nu \in B(s)} \mathbb{E}_{a \sim \pi_t(\cdot | s)}\left[ g\left(\frac{\pi_\theta (a | s_\nu)}{\pi_t (a | s)}, A^{\pi_t}(a | s) \right) \right] \right),
$$
where $\mathcal{H}(\cdot)$ represents the entropy, and the function $g(x,y)$ is defined as:
$$
g(x,y) = \min \left\{ x\cdot y, \operatorname{clip}\left( x, 1-\eta, 1+\eta \right) \cdot y \right\},
$$
with $\eta$ being the clipping hyperparameter.

Furthermore, we approximate the objective $\mathcal{L}_{car}^{train}(\theta)$ by considering the practice sampling process. The approximate objective is defined as:
\begin{equation} \notag
    \mathcal{L}_{car}^{app}(\theta) = \frac{1}{|\mathcal{B}|} \max_{(s,a) \in \mathcal{B}} \left( -\frac{1}{\beta} \mathcal{H}\left( \pi_\theta(\cdot | s) \right) - \min_{s_\nu \in B(s)} g\left(\frac{\pi_\theta (a | s_\nu)}{\pi_t (a | s)}, A^{\pi_t}(a | s) \right) \right).
\end{equation}
To effectively utilize each sample in a minibatch and improve training efficiency, we introduce a soft version $\mathcal{L}_{car}^{soft}(\theta)$ of the CAR objective $\mathcal{L}_{car}^{app}(\theta)$:
\begin{equation}\notag
    \mathcal{L}_{car}^{soft}(\theta) = \frac{1}{|\mathcal{B}|} \sum_{(s_i, a_i) \in \mathcal{B}} \alpha_i \left( -\frac{1}{\beta} \mathcal{H}\left( \pi_\theta(\cdot | s_i) \right) - \min_{s_\nu \in B(s_i)} g\left(\frac{\pi_\theta (a_i | s_\nu)}{\pi_t (a_i | s_i)}, A^{\pi_t}(a_i | s_i) \right) \right),
\end{equation}
where the sample weighting $\alpha_i$ is defined by the following distribution over a minibatch:
$$
\alpha_i = \frac{e^{\frac{1}{\lambda} \left( -\frac{1}{\beta} \mathcal{H}\left( \pi_\theta(\cdot | s_i) \right) - \min_{s_\nu \in B(s_i)} g\left(\frac{\pi_\theta (a_i | s_\nu)}{\pi_t (a_i | s_i)}, A^{\pi_t}(a_i | s_i) \right) \right)}}{\sum_{i\in \mathcal{\left|B\right|}} e^{\frac{1}{\lambda} \left( -\frac{1}{\beta} \mathcal{H}\left( \pi_\theta(\cdot | s_i) \right) - \min_{s_\nu \in B(s_i)} g\left(\frac{\pi_\theta (a_i | s_\nu)}{\pi_t (a_i | s_i)}, A^{\pi_t}(a_i | s_i) \right) \right)}}.
$$
Here, $\mathcal{B}$ represents a sampled minibatch, and the coefficient $\lambda$ controls the level of softness. Detailed derivations are provided in Appendix~\ref{app: derivation of car-ppo}.

\section{Experiments}\label{sec:exp}

In this section, we conduct extensive comparisons and ablation experiments to validate the rationality of our theoretical analysis and the effectiveness of CAR-DQN and CAR-PPO. Our source code and models are available at~\href{https://github.com/RyanHaoranLi/CAR-RL}{https://github.com/RyanHaoranLi/CAR-RL}.

\subsection{Implementation Details}

\paragraph{Environments.}
Following recent works~\citep{zhang2020robust, oikarinen2021robust, liang2022efficient}, we conduct experiments on four Atari video games~\citep{brockman2016openai}, including Pong, Freeway, BankHeist, and RoadRunner with DQN agents to validate CAR-DQN. These environments feature high-dimensional pixel inputs and discrete action spaces. 
For PPO agents, we conduct experiments on four MuJoCo tasks~\citep{todorov2012mujoco}, including Hopper, Walker2d, Halfcheetah, and Ant, which have continuous action spaces.

\begin{table*}[t]
\caption{Average natural and robust episode rewards over 50 episodes for baselines and CAR-DQN. The best results within the same solver type are highlighted in bold. CAR-DQN with the PGD solver outperforms SA-DQN with the PGD solver in almost all metrics and achieves remarkably better robustness in the more complex BankHeist and RoadRunner environments. CAR-DQN with the convex relaxation solver outperforms baselines in a majority of cases.}
\label{table: compare2}
\resizebox{\textwidth}{!}{%
\begin{tabular}{c|cc|c|ccc|ccc|c}
\hline \hline
\multirow{2}{*}{\textbf{Environment}} & \multicolumn{2}{c|}{\multirow{2}{*}{\textbf{Model}}}                                                                         & \multirow{2}{*}{\textbf{\begin{tabular}[c]{@{}c@{}}Natural\\  Reward\end{tabular}}} & \multicolumn{3}{c|}{\textbf{PGD}}                                               & \multicolumn{3}{c|}{\textbf{MinBest}}                                                      & \multicolumn{1}{c}{\textbf{ACR}}                               \\
                             & \multicolumn{2}{c|}{}                                                                                               &                              & $\epsilon=1/255$          & $\epsilon=3/255$          & $\epsilon=5/255$          & $\epsilon=1/255$          & $\epsilon=3/255$          & $\epsilon=5/255$          & $\epsilon=1/255$  \\ \hline
\multirow{7}{*}{\textbf{Pong}}        & \multicolumn{1}{c|}{Standard}                                                                      & DQN            & $21.0 $                  & $-21.0$           & $-21.0$           & $-20.8$           & $-21.0$           & $-21.0 $           & $-21.0$           & $0$                            \\ \cline{2-11} 
                             & \multicolumn{1}{c|}{\multirow{2}{*}{PGD}}                                                          & SA-DQN         & \textbf{$\bf 21.0$}         & \textbf{$\bf 21.0$}   & $-19.4 $           & $-21.0$           & \textbf{$\bf 21.0$}   & $-19.4$           & $-21.0$           & $0$                            \\
                             & \multicolumn{1}{c|}{}                                                                              & CAR-DQN (Ours) & \textbf{$\bf 21.0$}         & \textbf{$\bf 21.0$}   & $\bf 16.8$            & $-21.0$           & \textbf{$\bf 21.0$}   & $\bf 20.7$            & $\bf -0.8$            & $0$                            \\ \cline{2-11} 
                             & \multicolumn{1}{c|}{\multirow{4}{*}{\begin{tabular}[c]{@{}c@{}}Convex \\ Relaxation\end{tabular}}} & SA-DQN         & \textbf{$\bf 21.0$}         & \textbf{$\bf 21.0$}   & \textbf{$\bf 21.0$}   & $-19.6 $           & \textbf{$\bf 21.0$}   & \textbf{$\bf 21.0$}   & $-9.5 $            & $1.000$                        \\
                             & \multicolumn{1}{c|}{}                                                                              & RADIAL-DQN     & \textbf{$\bf 21.0 $}         & \textbf{$\bf 21.0$}   & \textbf{$\bf 21.0$}   & \textbf{$\bf 21.0$}   & \textbf{$\bf 21.0$}   & \textbf{$\bf 21.0$}   & $4.9$             & $0.898$                        \\
                             & \multicolumn{1}{c|}{}                                                                              & WocaR-DQN      & \textbf{$\bf 21.0 $}         & \textbf{$\bf 21.0 $}   & $20.5$            & $20.6 $            & \textbf{$\bf 21.0 $}   & $20.7 $            & $20.9$            & $0.979$                        \\
                             & \multicolumn{1}{c|}{}                                                                              & CAR-DQN (Ours) & \textbf{$\bf 21.0 $}         & \textbf{$\bf 21.0 $}   & \textbf{$\bf 21.0$}   & \textbf{$\bf 21.0$}   & \textbf{$\bf 21.0$}   & \textbf{$\bf 21.0$}   & \textbf{$\bf 21.0$}   & $0.986$                        \\ \hline \hline
\multirow{7}{*}{\textbf{Freeway}}     & \multicolumn{1}{c|}{Standard}                                                                      & DQN            & $33.9$                  & $0.0 $             & $0.0$             & $0.0$             & $0.0 $             & $0.0 $             & $0.0 $             & $0$                            \\ \cline{2-11} 
                             & \multicolumn{1}{c|}{\multirow{2}{*}{PGD}}                                                          & SA-DQN         & $33.6$                  & $23.4$            & $20.6$            & \textbf{$\bf 7.6$}    & $21.1$            & $21.3$            & $21.8$            & $0.250$                    \\
                             & \multicolumn{1}{c|}{}                                                                              & CAR-DQN (Ours) & \textbf{$\bf 34.0$}         & \textbf{$\bf 33.7$}   & \textbf{$\bf 25.8$}   & $3.8$             & \textbf{$\bf 33.7$}   & \textbf{$\bf 30.0$}   & \textbf{$\bf 26.2$}   & $0$                            \\ \cline{2-11} 
                             & \multicolumn{1}{c|}{\multirow{4}{*}{\begin{tabular}[c]{@{}c@{}}Convex \\ Relaxation\end{tabular}}} & SA-DQN         & $30.0 $                  & $30.0 $            & $30.2$            & $27.7$            & $30.0$            & $30.0$            & $29.2$            & $1.000$                       \\
                             & \multicolumn{1}{c|}{}                                                                              & RADIAL-DQN     & $33.1$                  & \textbf{$\bf 33.3$}   & \textbf{$\bf 33.3$}   & \textbf{$\bf 29.0$}   & \textbf{$\bf 33.3$}   & \textbf{$\bf 33.3$}   & \textbf{$\bf 31.2$}   & $0.998$                    \\
                             & \multicolumn{1}{c|}{}                                                                              & WocaR-DQN      & $30.8 $                  & $31.0$            & $30.6$            & $29.0 $            & $31.0$            & $31.1$            & $29.0$            & $0.992$                       \\
                             & \multicolumn{1}{c|}{}                                                                              & CAR-DQN (Ours) & \textbf{$\bf 33.2$}         & $33.2$            & $32.3$            & $27.6$            & $33.2$            & $32.8$            & $31.0$            & $0.981$                  \\ \hline \hline
\multirow{7}{*}{\textbf{BankHeist}}   & \multicolumn{1}{c|}{Standard}                                                                      & DQN            & $1317.2$                & $22.2$            & $0.0$             & $0.0$             & $0.0$             & $0.0$             & $0.0$             & $0$                           \\ \cline{2-11} 
                             & \multicolumn{1}{c|}{\multirow{2}{*}{PGD}}                                                          & SA-DQN         & $1248.8 $                & $965.8$          & $35.6$            & $0.6$             & $1118.0$          & $50.8$            & $4.8$             & $0$                           \\
                             & \multicolumn{1}{c|}{}                                                                              & CAR-DQN (Ours) & \textbf{$\bf 1307.0$}       & \textbf{$\bf 1243.2$} & \textbf{$\bf 908.2$} & \textbf{$\bf 83.0$}   & \textbf{$\bf 1242.6$} & \textbf{$\bf 970.8$}  & \textbf{$\bf 819.4$}  & $0$                            \\ \cline{2-11} 
                             & \multicolumn{1}{c|}{\multirow{4}{*}{\begin{tabular}[c]{@{}c@{}}Convex \\ Relaxation\end{tabular}}} & SA-DQN         & $1236.0 $                & $1232.2$          & $1208.8$          & $1029.8 $         & $1232.2$          & $1214.8$          & $1051.0$         & $0.991$                       \\
                             & \multicolumn{1}{c|}{}                                                                              & RADIAL-DQN     & $1341.8$                & $1341.8$          & \textbf{$\bf 1346.4$} & $1092.6$         & $1341.8$          & $1328.6$          & $732.6$          & $0.982$                        \\
                             & \multicolumn{1}{c|}{}                                                                              & WocaR-DQN      & $1315.0 $                & $1312.0$          & $1323.4$          & $1094.0$         & $1312.0$          & $1301.6$          & $1041.4$         & $0.987$                   \\
                             & \multicolumn{1}{c|}{}                                                                              & CAR-DQN (Ours) & \textbf{$\bf 1349.6 $}       & \textbf{$\bf 1347.6 $} & $1332.0 $          & \textbf{$\bf 1191.0$} & \textbf{$\bf 1347.4 $} & \textbf{$\bf 1338.0$} & \textbf{$\bf 1233.6$} & $0.974$                    \\ \hline \hline
\multirow{7}{*}{\textbf{RoadRunner}}  & \multicolumn{1}{c|}{Standard}                                                                      & DQN            & $41492 $                 & $0$                 & $0$                 & $0$                 & $0$                 & $0$                 & $0$                 & $0$                        \\ \cline{2-11} 
                             & \multicolumn{1}{c|}{\multirow{2}{*}{PGD}}                                                          & SA-DQN         & $33380 $                 & $20482 $          & $0$                 & $0 $                 & $24632$           & $614$              & $214$              & $0$                        \\
                             & \multicolumn{1}{c|}{}                                                                              & CAR-DQN (Ours) & \textbf{$\bf 49700$}       & \textbf{$\bf 43286$}  & \textbf{$\bf 25740 $} & \textbf{$\bf 2574 $}   & \textbf{$\bf 48908 $} & \textbf{$\bf 35882 $}  & \textbf{$\bf 23218 $}  & $0$               \\ \cline{2-11} 
                             & \multicolumn{1}{c|}{\multirow{4}{*}{\begin{tabular}[c]{@{}c@{}}Convex \\ Relaxation\end{tabular}}} & SA-DQN         & $46372 $                 & $44960$           & $20910$           & $3074$            & $45226$          & $25548$           & $12324$           & $0.819$                    \\
                             & \multicolumn{1}{c|}{}                                                                              & RADIAL-DQN     & $46224$                 & $45990 $          & \textbf{$\bf 42162 $} & \textbf{$\bf 23248 $}  & $46082$          & \textbf{$\bf 42036 $} & \textbf{$\bf 25434 $}  & $0.994$                    \\
                             & \multicolumn{1}{c|}{}                                                                              & WocaR-DQN      & $43686 $                & $45636 $           & $19386$           & $6538$            & $45636 $           & $21068$          & $15050$           & $0.956$                    \\
                             & \multicolumn{1}{c|}{}                                                                              & CAR-DQN (Ours) & \textbf{$\bf 49398 $}       & \textbf{$\bf 49456 $}  & $28588 $          & $15592 $           & \textbf{$\bf 47526 $} & $32878$          & $21102 $          & $0.760$                       \\ \hline \hline
\end{tabular}%
}
\end{table*}

\paragraph{Baselines.}
We compare CAR-RL with several state-of-the-art robust training methods. SA-DQN/SA-PPO~\citep{zhang2020robust} incorporates a KL-based regularization and solves the inner maximization problem using PGD~\citep{madry2017towards} and CROWN-IBP~\citep{zhang2019towards}, respectively. RADIAL-DQN/RADIAL-PPO~\citep{oikarinen2021robust} applies adversarial regularizations based on robustness verification bounds from IBP~\citep{gowal2018effectiveness}. WocaR-DQN/WocaR-PPO~\citep{liang2022efficient} employs a worst-case value estimation and incorporates the KL-based regularization. For DQN agents, we utilize the officially released models of SA-DQN and RADIAL-DQN, and replicate WocaR-DQN, as its official implementation uses a different environment wrapper from SA-DQN and RADIAL-DQN. For PPO agents, we utilize the official code to train 17 agents using different methods in the same setting, adjust some parameters appropriately, and report the medium performance for reproducibility due to the high variance in RL training.

\paragraph{Evaluations of DQN.}
For DQN agents, we evaluate their robustness using three metrics on Atari games: (1) episode return under a 10-step untargeted PGD attack~\citep{madry2017towards}, (2) episode return under the MinBest~\citep{huang2017adversarial} attack, both of which minimize the probability of selecting the learned optimal action, and (3) Action Certification Rate~(ACR)~\citep{zhang2020robust}, which employs relaxation bounds to estimate the percentage of frames where the learned optimal action is guaranteed to remain unchanged during rollouts under attacks. 

\begin{table}[t]
\caption{Average episode rewards over 50 episodes for baselines and CAR-PPO on MuJoCo tasks. Results include natural rewards, rewards under six types of attacks, and the worst rewards under these attacks. The best results of the algorithm under natural environments and various attacks are highlighted in bold. CAR-PPO significantly outperforms the baselines in worst-case robustness across all four tasks, improves the natural performance of vanilla PPO on Hopper, Walker2d, and Halfcheetah, and achieves comparable natural performance to vanilla PPO on Ant.}
\label{table: mujoco compare}
\resizebox{\columnwidth}{!}{%
\begin{tabular}{c|c|c|c|ccccccc}
\hline \hline
\multirow{2}{*}{\textbf{Env}}         & \multirow{2}{*}{\textbf{$\epsilon$}} & \multirow{2}{*}{\textbf{Method}} & \multirow{2}{*}{\textbf{\begin{tabular}[c]{@{}c@{}}Natural\\  Reward\end{tabular}}} & \multicolumn{7}{c}{\textbf{Attack Reward}}                                                                           \\
                                      &                                      &                                  &                                                                                     & \textbf{Random} & \textbf{Critic} & \textbf{MAD}  & \textbf{RS}   & \textbf{SA-RL} & \textbf{PA-AD} & \textbf{Worst} \\ \hline 
\multirow{6}{*}{\textbf{Hopper}}      & \multirow{6}{*}{0.075}               & PPO                              & 3081                                                                                & 2923            & 2035            & 1763          & 756           & 79             & 823            & 79             \\
                                      &                                      & SA-PPO                           & 3518                                                                                & 2835            & 3662            & 3045          & 1407          & 1476           & 1286           & 1286           \\
                                      &                                      & RADIAL-PPO                       & 3254                                                                                & 3170            & \textbf{3706}   & 2558          & 1307          & 993            & 1696           & 993            \\
                                      &                                      & WocaR-PPO                        & 3629                                                                                & 3546            & 3657            & 3048          & 1171          & 1452           & 2124           & 1171           \\
                                      &                                      & CAR-PPO SGLD (ours)              & 3566                                                                                & 3537            & 3480            & \textbf{3484} & \textbf{1990} & \textbf{2977}  & \textbf{3232}  & \textbf{1990}  \\
                                      &                                      & CAR-PPO PGD (ours)               & \textbf{3711}                                                                       & \textbf{3702}   & 3692            & 3473          & 1652          & 2430           & 2640           & 1652           \\ \hline \hline
\multirow{6}{*}{\textbf{Walker2d}}    & \multirow{6}{*}{0.05}                & PPO                              & 4622                                                                                & 4628            & 4584            & 4507          & 1062          & 719            & 336            & 336            \\
                                      &                                      & SA-PPO                           & \textbf{4875}                                                                       & \textbf{4907}   & 5029            & \textbf{4833} & 2775          & 3356           & 997            & 997            \\
                                      &                                      & RADIAL-PPO                       & 2531                                                                                & 2170            & 2063            & 2316          & 1239          & 426            & 1353           & 426            \\
                                      &                                      & WocaR-PPO                        & 4226                                                                                & 4347            & 4342            & 4373          & 3358          & 2385           & 1064           & 1064           \\
                                      &                                      & CAR-PPO SGLD (ours)              & 4622                                                                                & 4609            & 4684            & 4498          & 4242          & \textbf{4397}  & 3134           & 3134           \\
                                      &                                      & CAR-PPO PGD (ours)               & 4755                                                                                & 4848            & \textbf{5044}   & 4637          & \textbf{4379} & 4307           & \textbf{4303}  & \textbf{4303}  \\ \hline \hline
\multirow{6}{*}{\textbf{Halfcheetah}} & \multirow{6}{*}{0.15}                & PPO                              & 5048                                                                                & 4463            & 3281            & 918           & 1049          & -213           & -69            & -213           \\
                                      &                                      & SA-PPO                           & 4780                                                                                & 4983            & 5035            & 3759          & 2727          & 1443           & 1511           & 1443           \\
                                      &                                      & RADIAL-PPO                       & 4739                                                                                & 4642            & 4546            & 2961          & 1327          & 1522           & 1968           & 1327           \\
                                      &                                      & WocaR-PPO                        & 4723                                                                                & 4798            & 4846            & 4543          & 3302          & 2270           & 2498           & 2270           \\
                                      &                                      & CAR-PPO SGLD (ours)              & 4599                                                                                & 4574            & 4731            & 4348          & 3888          & 3908           & 4032           & 3888           \\
                                      &                                      & CAR-PPO PGD (ours)               & \textbf{5053}                                                                       & \textbf{5058}   & \textbf{5065}   & \textbf{5051} & \textbf{5140} & \textbf{4860}  & \textbf{4942}  & \textbf{4860}  \\ \hline \hline
\multirow{6}{*}{\textbf{Ant}}         & \multirow{6}{*}{0.15}                & PPO                              & \textbf{5381}                                                                       & \textbf{5329}   & 4696            & 1768          & 1097          & -1398          & -3107          & -3107          \\
                                      &                                      & SA-PPO                           & 5367                                                                                & 5217            & \textbf{5012}   & \textbf{5114} & \textbf{4396} & \textbf{4227}  & 2355           & 2355           \\
                                      &                                      & RADIAL-PPO                       & 4358                                                                                & 4309            & 3628            & 4205          & 3742          & 2364           & 3261           & 2364           \\
                                      &                                      & WocaR-PPO                        & 4069                                                                                & 3911            & 3978            & 3689          & 3176          & 1868           & 1830           & 1830           \\
                                      &                                      & CAR-PPO SGLD (ours)              & 5056                                                                               & 5007            & 4864            & 4468          & 3755 & 3088  & 3763  & 3088           \\
                                      &                                      & CAR-PPO PGD (ours)               & 5029                                                                                & 5006            & 4786            & 4549          & 3553          & 3099           & \textbf{3911}  & \textbf{3099}  \\ \hline \hline
\end{tabular}%
}
\end{table}

\paragraph{Evaluations of PPO.}
We assess the robustness of PPO agents using six attacks on MuJoCo tasks: (1) random attack, adding uniform random noise to state observations; (2) critic attack~\citep{pattanaik2017robust}, conducted based on the action-value functions; (3) MAD~(maximal action difference) attack~\citep{zhang2020robust}, maximizing the discrepancy between policies in clean and perturbed states; (4) RS~(robust sarsa) attack~\citep{zhang2020robust}, training a robust action-value function and then performing critic-based attacks based on it; (5) SA-RL~\citep{zhang2021robust}, which employs a learned adversarial agent to perturb the state; (6) PA-AD~\citep{sun2021strongest}, which trains an adversarial agent to select a perturbed direction and then uses FGSM to attack along that direction.

\begin{figure*}[t]
    \centering
        \begin{subfigure}
        \centering
        \includegraphics[width=0.45\textwidth]{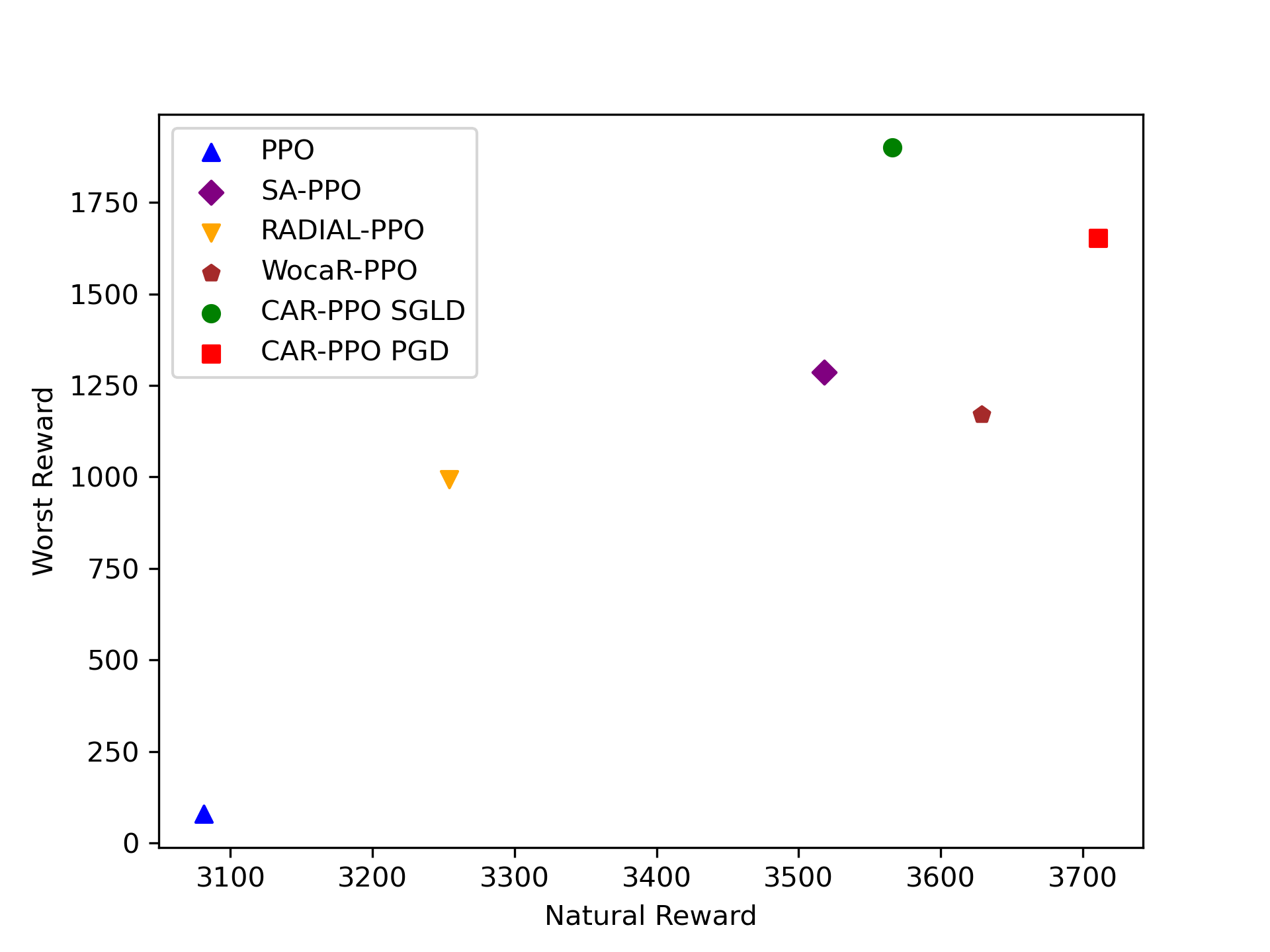}
    \end{subfigure}
    \begin{subfigure}
        \centering
        \includegraphics[width=0.45\textwidth]{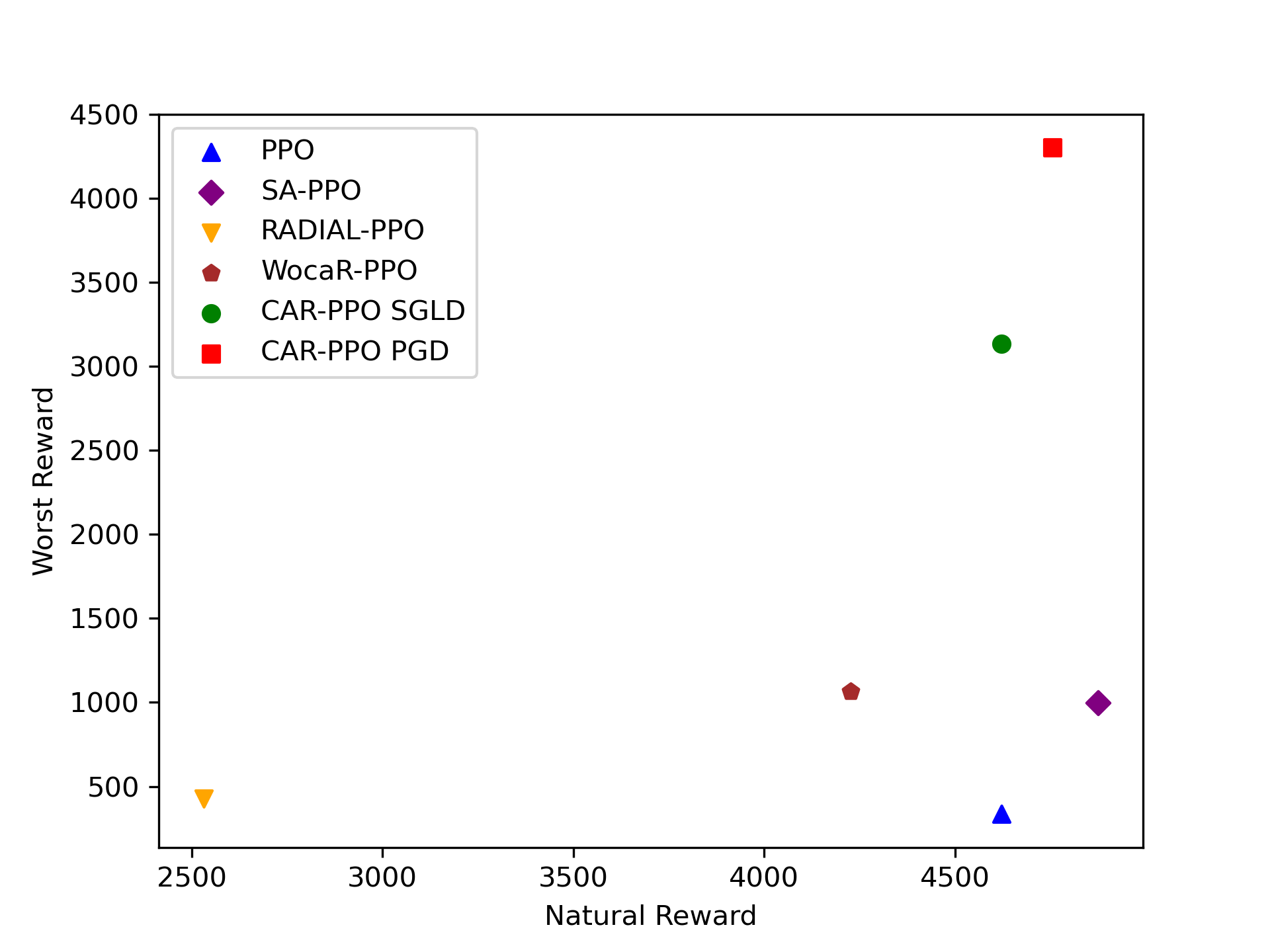}
    \end{subfigure}
    \\
    \vskip -0.1in
    \begin{minipage}{0.48\textwidth}
        \centering
        \quad \scriptsize{Hopper}
    \end{minipage}
    \begin{minipage}{0.48\textwidth}
        \centering
        \scriptsize{Walker2d}
    \end{minipage}
    \\
    \centering
        \begin{subfigure}
        \centering
        \includegraphics[width=0.45\textwidth]{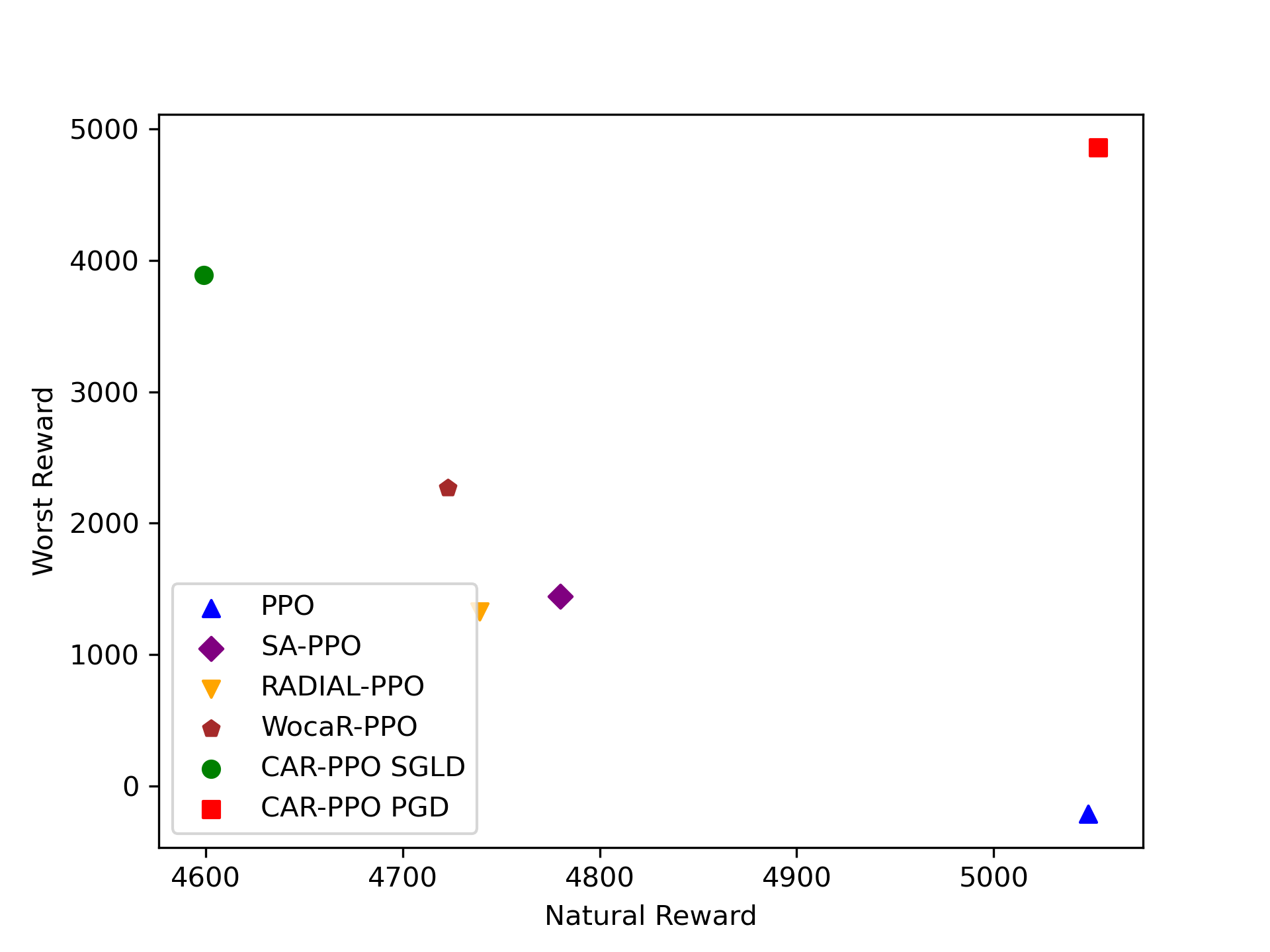}
    \end{subfigure}
    \begin{subfigure}
        \centering
        \includegraphics[width=0.45\textwidth]{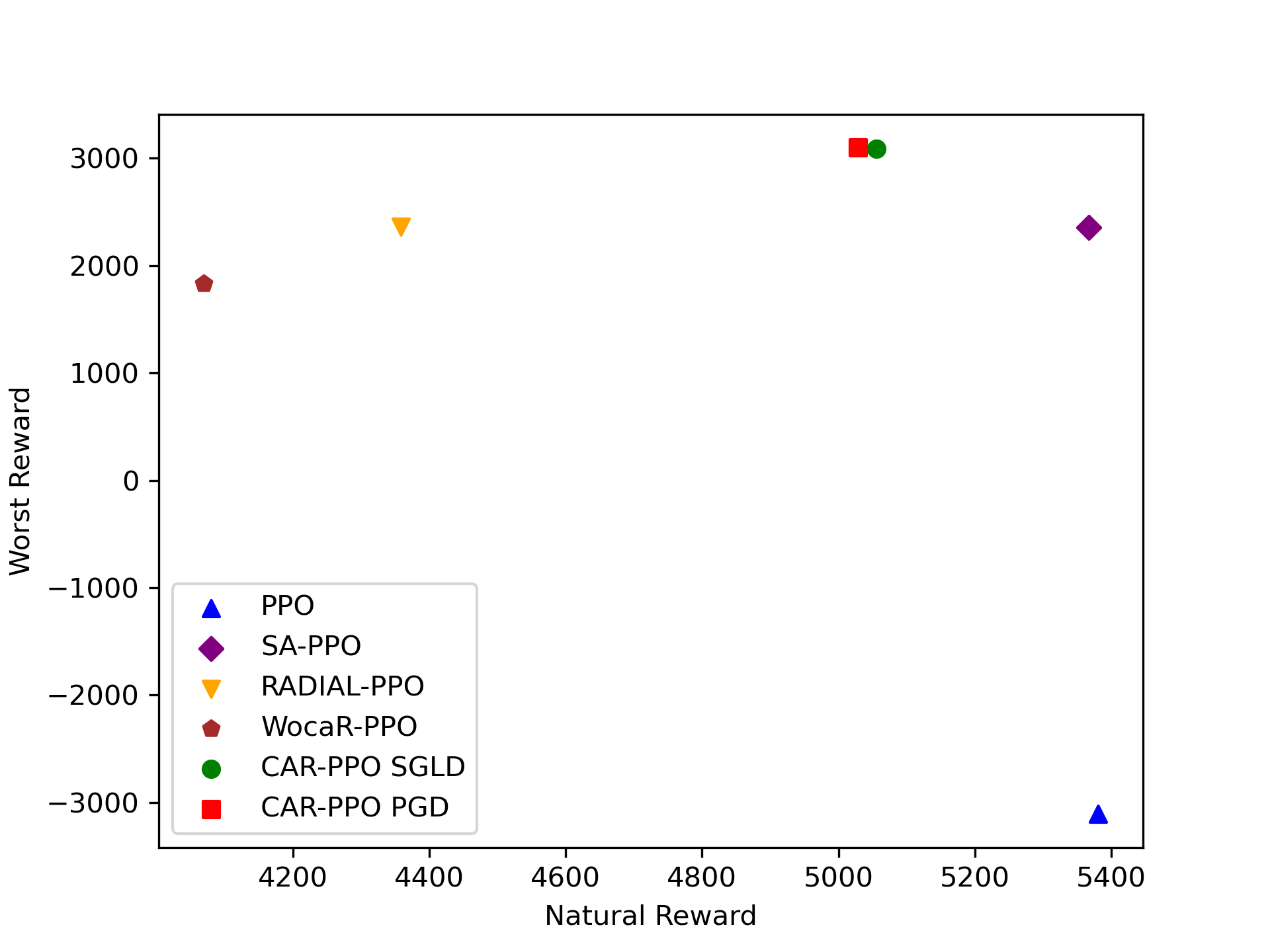}
    \end{subfigure}
    \\
    \vskip -0.1in
    \begin{minipage}{0.48\textwidth}
        \centering
        \quad \scriptsize{Halfcheetah}
    \end{minipage}
    \begin{minipage}{0.48\textwidth}
        \centering
        \scriptsize{Ant}
    \end{minipage}
    \caption{Natural reward and worst-case robustness under various attacks in MuJoCo.}
    \label{fig: natural and worst rewards}
\end{figure*}

\paragraph{CAR-DQN.}
CAR-DQN is implemented based on Double Dueling DQN~\citep{van2016deep,wang2016dueling}, and all baselines and CAR-DQN are trained for 4.5 million steps, based on the same standard model released by~\cite{zhang2020robust}, which has been trained for 6 million steps. We increase the attack $\epsilon$ from $0$ to $1/255$ during the first 4 million steps, using the same smoothed schedule as in~\cite{zhang2020robust, oikarinen2021robust, liang2022efficient}, and then continue training with a fixed $\epsilon$ for the remaining 0.5 million steps. We use Huber loss to replace the absolute value function and separately apply classic gradient-based methods (PGD) and cheap convex relaxation (IBP) for resolving the inner optimization in $\mathcal{L}_{car}^{soft}(\theta)$.
For CAR-DQN with the PGD solver, hyperparameters are the same as those of SA-DQN~\citep{zhang2020robust}. For CAR-DQN with the IBP solver, we update the target network every 2000 steps, set the learning rate to $1.25\times 10^{-4}$, use a batch size of $32$, and set the exploration $\epsilon_{exp}$-end to $0.01$, soft coefficient $\lambda=1.0$ and discount factor to $0.99$. The replay buffer has a capacity of $2\times 10^{5}$, and we use the Adam optimizer~\citep{kingma2014adam} with $\beta_1=0.9$ and $\beta_2=0.999$. The overall CAR-DQN algorithm process and additional details are shown in Appendix~\ref{app: additional algorithm details}.

\begin{figure*}[t]
    \centering
    \begin{subfigure}
        \centering
        \includegraphics[width=0.45\textwidth]{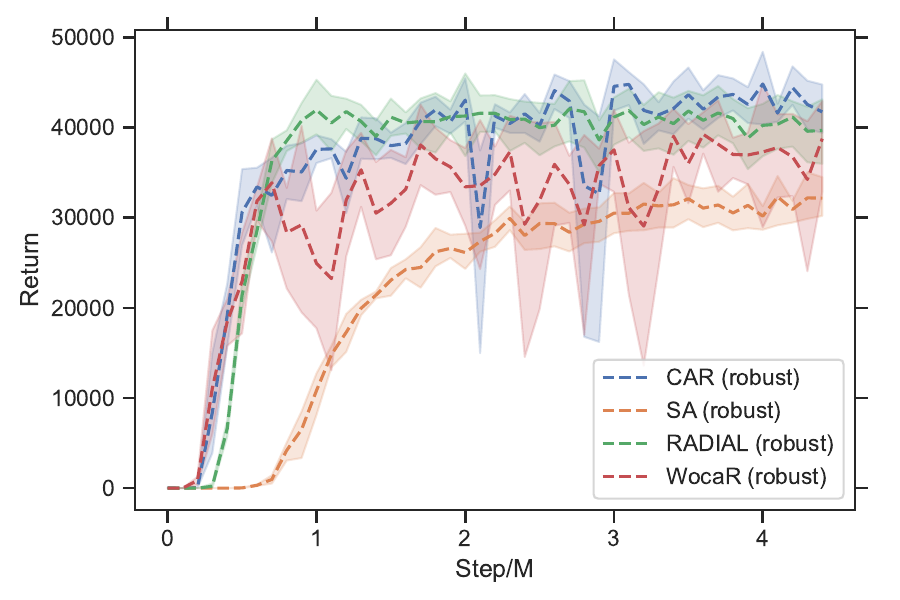}
    \end{subfigure}
    \begin{subfigure}
        \centering
        \includegraphics[width=0.45\textwidth]{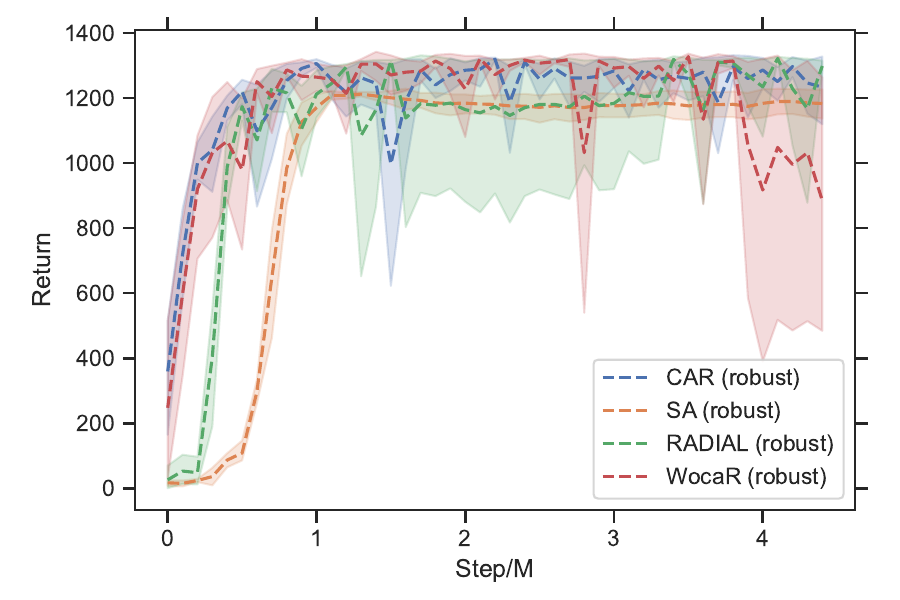}
    \end{subfigure}
    \\
    \vskip -0.1in
    \begin{minipage}{0.48\textwidth}
        \centering
        \qquad \scriptsize{RoadRunner}
    \end{minipage}
    \begin{minipage}{0.48\textwidth}
        \centering
        \quad \scriptsize{BankHeist}
    \end{minipage}
    \\
    \vskip 0.05in
    \begin{subfigure}
        \centering
        \includegraphics[width=0.45\textwidth]{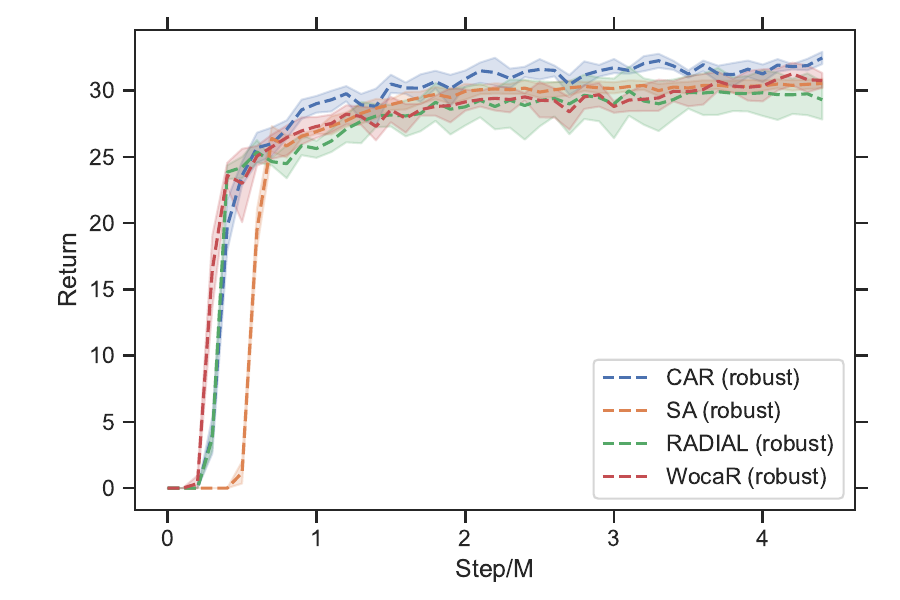}
    \end{subfigure}
    \begin{subfigure}
        \centering
        \includegraphics[width=0.45\textwidth]{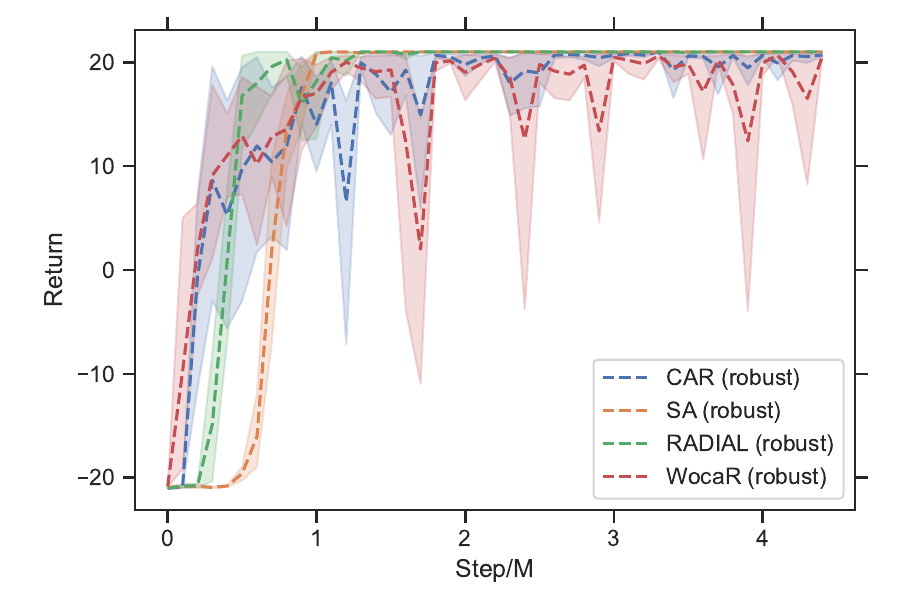}
    \end{subfigure}
    \\
    \vskip -0.1in
    \begin{minipage}{0.48\textwidth}
        \centering
        \setlength{\parindent}{1.3em}
        \quad \scriptsize{Freeway}
    \end{minipage}
    \begin{minipage}{0.48\textwidth}
        \centering
        \setlength{\parindent}{0.1em}
        \quad \scriptsize{Pong}
    \end{minipage}
    \caption{Robust episode rewards of baselines and CAR-DQN under strong PGD attacks on 4 Atari games.
    Shaded regions are computed over 5 random seeds. CAR-DQN demonstrates superior robust performance in all environments.}
    \label{fig: robustness rewards during training}
    \vskip -0.05in
\end{figure*}

\paragraph{CAR-PPO.}
SA-PPO, WocaR-PPO, and CAR-PPO are trained for 2 million steps (976 iterations) on Hopper, Walker2d, and Halfcheetah, and 10 million steps (4882 iterations) on Ant for convergence. RADIAL-PPO are trained for 4 million steps (2000 iterations) on Hopper, Walker2d, and Halfcheetah following the official implementation and 10 million steps (4882 iterations) on Ant. We run 2048 simulation steps per iteration and use a simple MLP network for all PPO agents. The attack budget $\epsilon$ is linearly increased from $0$ to the target value during the first 732 iterations on Hopper, Walker2d, and Halfcheetah, and 3662 iterations on Ant, before continuing with the target value for the remaining iterations. This scheduler is aligned with~\cite{zhang2020robust, liang2022efficient}. We combine vanilla PPO with CAR loss for efficient training, using a regularization weight $\kappa$. The weight $\kappa$ is chosen from $\{0.1, 0.3, 1.0\}$. The soft coefficient $\lambda$ in $\mathcal{L}_{car}^{soft}(\theta)$ is chosen from $\{ 10, 100, 1000 \}$ on Hopper, Walker2d and Halfcheetah and $\{ 100, 1000, 10000 \}$ on Ant to ensure stable training. We separately apply PGD and SGLD~\citep{gelfand1991recursive} to resolve the inner optimization in $\mathcal{L}_{car}^{soft}(\theta)$. We run 10 iterations with step size $\epsilon/10$ for both methods and set the temperature parameter $\beta=1\times 10^{-5}$ for SGLD. The overall CAR-PPO algorithm process and additional implementation details are shown in Appendix~\ref{app: additional algorithm details}.

\subsection{Comparison Results}\label{sec: Comparison}

\begin{figure*}[t]
    \centering
    \begin{subfigure}
        \centering
        \includegraphics[width=0.45\textwidth]{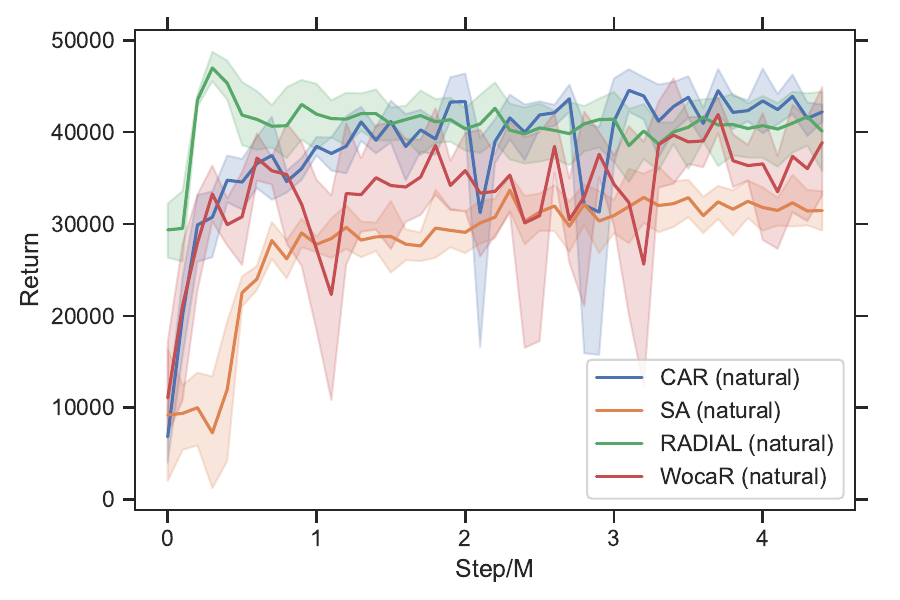}
    \end{subfigure}
    \begin{subfigure}
        \centering
        \includegraphics[width=0.45\textwidth]{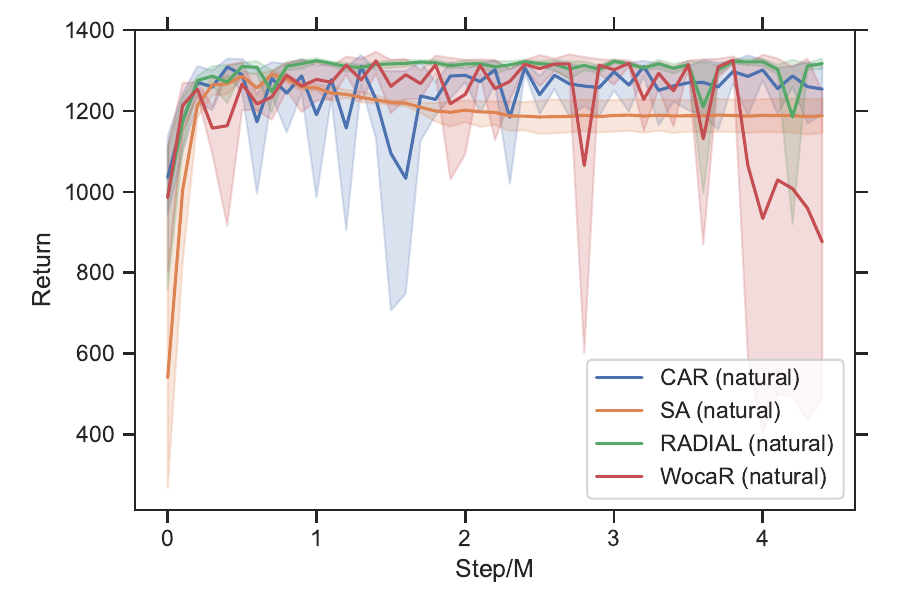}
    \end{subfigure}
    \\
    \vskip -0.1in
    \begin{minipage}{0.48\textwidth}
        \centering
        \qquad \scriptsize{RoadRunner}
    \end{minipage}
    \begin{minipage}{0.48\textwidth}
        \centering
        \quad \scriptsize{BankHeist}
    \end{minipage}
    \\
    \vskip 0.05in
    \begin{subfigure}
        \centering
        \includegraphics[width=0.45\textwidth]{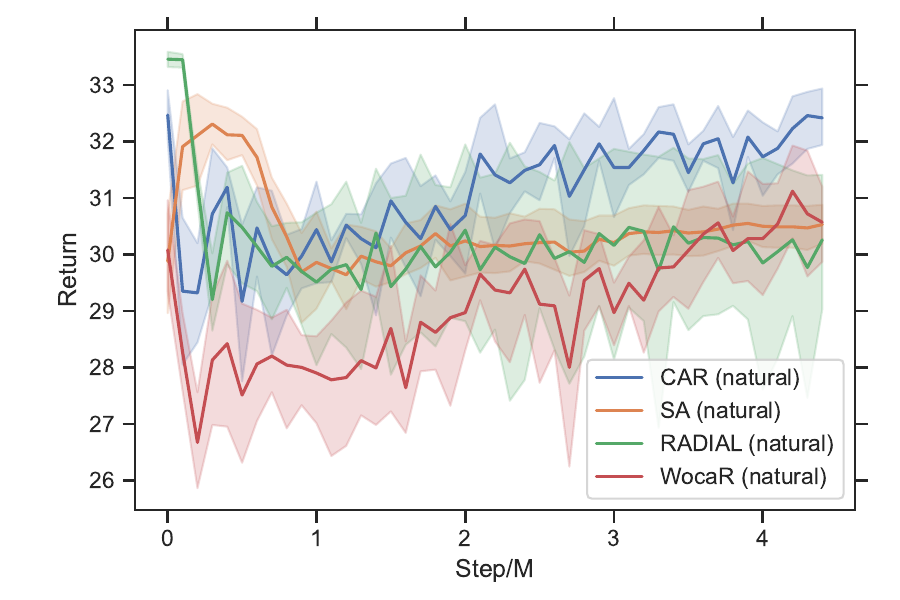}
    \end{subfigure}
    \begin{subfigure}
        \centering
        \includegraphics[width=0.45\textwidth]{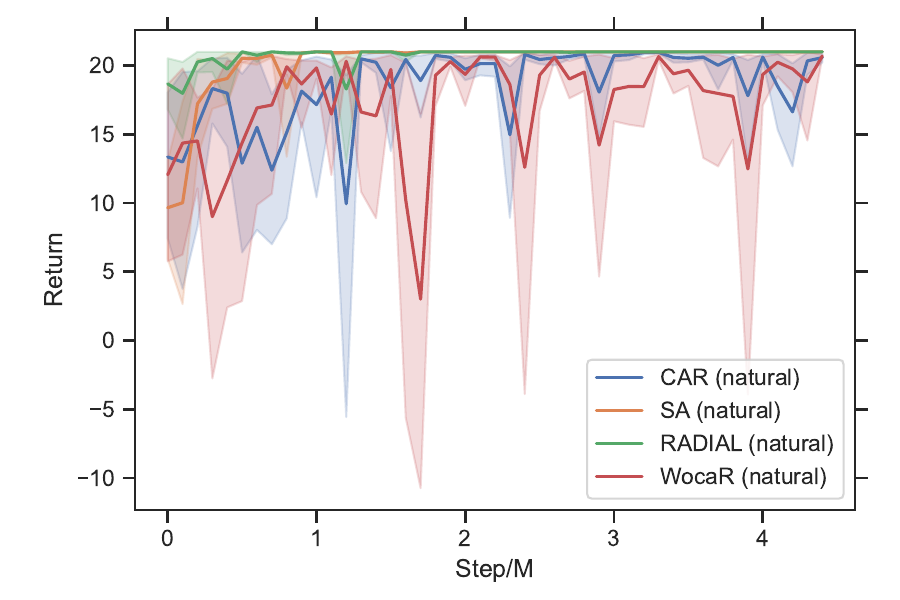}
    \end{subfigure}
    \\
    \vskip -0.1in
    \begin{minipage}{0.48\textwidth}
        \centering
        \setlength{\parindent}{1.3em}
        \quad \scriptsize{Freeway}
    \end{minipage}
    \begin{minipage}{0.48\textwidth}
        \centering
        \setlength{\parindent}{0.1em}
        \quad \scriptsize{Pong}
    \end{minipage}
    \caption{Natural episode rewards of baselines and CAR-DQN without attacks on 4 Atari games.
    Shaded regions are computed over 5 random seeds. CAR-DQN demonstrates superior natural performance in all environments.}
    \label{fig: natural rewards during training}
    \vskip -0.05in
\end{figure*}

\paragraph{Evaluation on Atari.} 
Table \ref{table: compare2} presents the natural and robust performance of DQN agents, all of which are trained using a perturbation radius of $\epsilon=1/255$ and evaluated under attacks with $\epsilon=1/255,\ 3/255,\ 5/255$. For the $\epsilon=1/255$ attacks, it is worth noting that CAR-DQN agents exhibit superior performance compared to baselines in the most challenging RoadRunner environment, achieving significant improvements in both natural and robust rewards. In the other three games, CAR-DQN matches the performance of the state-of-the-art baseline well. Our CAR-DQN loss function coupled with the PGD solver, achieves an impressive return of around 45,000 on the RoadRunner environment, significantly surpassing SA-DQN with the PGD approach. It also attains $60\%$ higher robust rewards under the MinBest attack on the Freeway game.
For attacks with increasing perturbation radius $\epsilon=3/255,\ 5/255$, We can see that CAR-DQN agents achieve superior performance in Pong and BankHeist, and attain comparable performance in Freeway. Notably, in more complex environments like BankHeist and RoadRunner, SA-DQN with the PGD solver fails to maintain robustness under larger perturbations. In contrast, CAR-DQN with the PGD solver, trained under the small perturbation budget, maintains strong robust performance even with larger attack budgets. 
In addition, we also notice that RADIAL-DQN outperforms CAR-DQN under large perturbation budgets on RoadRunner, for which extensive analysis and comparisons are provided in Appendix~\ref{app: exp with increasing perturbation radius}. 
We also compare the two solvers and find that PGD shows weaker robustness than convex relaxation, particularly failing to ensure the ACR computed with relaxation bounds. This discrepancy can be attributed to that the PGD solver offers a lower bound surrogate function of the loss, while the IBP solver provides an upper bound.

\paragraph{Evaluation on MuJoCo.}
Table~\ref{table: mujoco compare} showcases the natural and robust performance of CAR-PPO and baselines. Hopper agents are trained and attacked with a perturbation radius of $\epsilon=0.075$, Walker2d agents are with $\epsilon=0.05$, and Halfcheetah and Ant agents are with $\epsilon=0.15$. \emph{Notably, CAR-PPO agents achieve the best robustness in the worst case across all four environments, with significant 55\%,\ 304\%,\ 114\%,\ 31\% improvement on Hopper, Walker2d, Halfcheetah, and Ant, respectively.}
Meanwhile, the natural performance of CAR-PPO agents also matches the best level in each environment. CAR-PPO agents with the PGD solver achieve the highest natural performance on Hopper and Halfcheetah, outperform vanilla PPO agents on Walker2d, and attain comparable natural rewards on Ant. 
While CAR-PPO is slightly $2.5\%,\ 6.3\%$ lower in natural performance than SA-PPO on Walker2d and Ant, it respectively achieves $332\%$ and $32\%$ higher worst-case robustness. The overall performance of CAR-PPO with PGD is better than CAR-PPO with SGLD, underscoring the importance of the PGD method for training robust PPO agents. More intuitive comparisons of natural and worst-case robust performance are shown in Fig.~\ref{fig: natural and worst rewards}.


\begin{table}[t]
\caption{Performance of DQN with different Bellman $p$-error.}
\label{table: p-error}
\vskip 0.15in
\centering
\resizebox{0.85\columnwidth}{!}{%
\begin{tabular}{c|c|c|ccc}
\hline
Environment                 & Norm   & Natural   & PGD              & MinBest          & ACR     \\ \hline
\multirow{3}{*}{Pong}       & $L^1$      & $\bf 21.0 \pm 0.0$   & $-21.0 \pm 0.0$  & $-21.0 \pm 0.0$  & $0$     \\
                            & $L^2$      & $\bf 21.0 \pm 0.0$   & $-21.0 \pm 0.0$  & $-20.8 \pm 0.1$  & $0$     \\
                            & $L^\infty$ & $\bf 21.0 \pm 0.0$   & $\bf 21.0 \pm 0.0$   & $\bf 21.0 \pm 0.0$   & $0.985$ \\ \hline
\multirow{3}{*}{Freeway}    & $L^1$      & $\bf 33.9 \pm 0.1$   & $0.0 \pm 0.0$        & $0.0 \pm 0.0$        & $0$     \\
                            & $L^2$      & $21.8 \pm 0.3$   & $21.7 \pm 0.3$   & $22.1 \pm 0.3$   & $0$     \\
                            & $L^\infty$ & $33.3 \pm 0.1$   & $\bf 33.2 \pm 0.1$   & $\bf 33.2 \pm 0.1$   & $0.981$ \\ \hline
\multirow{3}{*}{BankHeist}  & $L^1$      & $1325.5 \pm 5.7$ & $27.0 \pm 2.0$   & $0.0 \pm 0.0$        & $0$     \\
                            & $L^2$      & $1314.5 \pm 4.0$ & $18.5 \pm 1.5$   & $22.5 \pm 2.6$   & $0$     \\
                            & $L^\infty$ & $\bf 1356.0 \pm 1.7$  & $\bf 1356.5 \pm 1.1$  & $\bf 1356.5 \pm 1.1$  & $0.969$ \\ \hline
\multirow{3}{*}{RoadRunner} & $L^1$      & $43795 \pm 1066$ & $0 \pm 0$        & $0 \pm 0$        & $0$     \\
                            & $L^2$      & $30620 \pm 990$  & $0 \pm 0$        & $0 \pm 0$        & $0$     \\
                            & $L^\infty$ & $\bf 49500 \pm 2106$ & $\bf 48230 \pm 1648$ & $\bf 48050 \pm 1642$ & $0.947$ \\ \hline
\end{tabular}%
}
\vskip -0.1in
\end{table}

\begin{figure*}[t]
    \centering
    \begin{subfigure}
        \centering
        \includegraphics[width=0.45\textwidth]{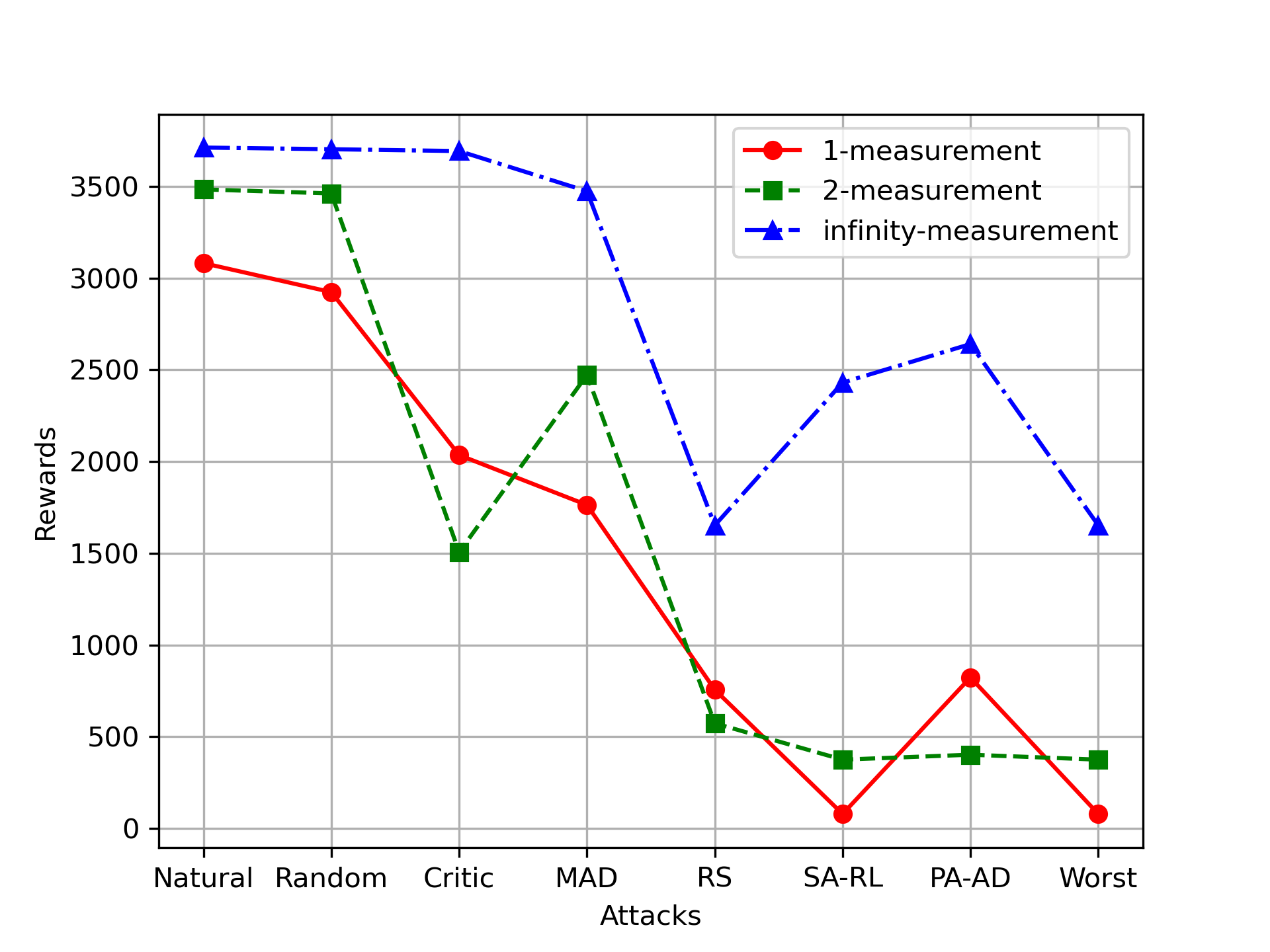}
    \end{subfigure}
    \begin{subfigure}
        \centering
        \includegraphics[width=0.45\textwidth]{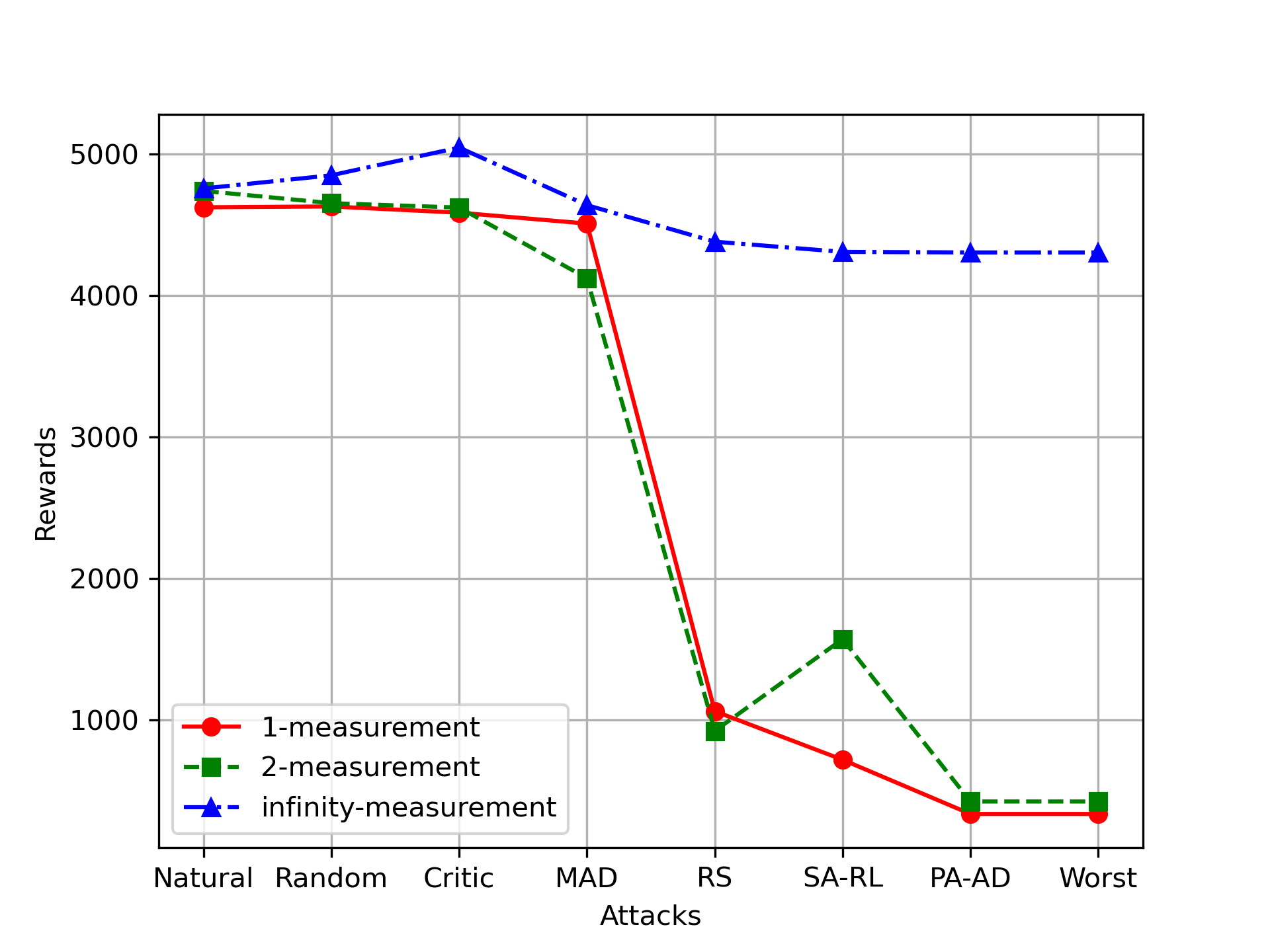}
    \end{subfigure}
    \\
    \vskip -0.1in
    \begin{minipage}{0.48\textwidth}
        \centering
        \quad \scriptsize{Hopper}
    \end{minipage}
    \begin{minipage}{0.48\textwidth}
        \centering
        \scriptsize{Walker2d}
    \end{minipage}
    \\
    \vskip 0.05in
    \begin{subfigure}
        \centering
        \includegraphics[width=0.45\textwidth]{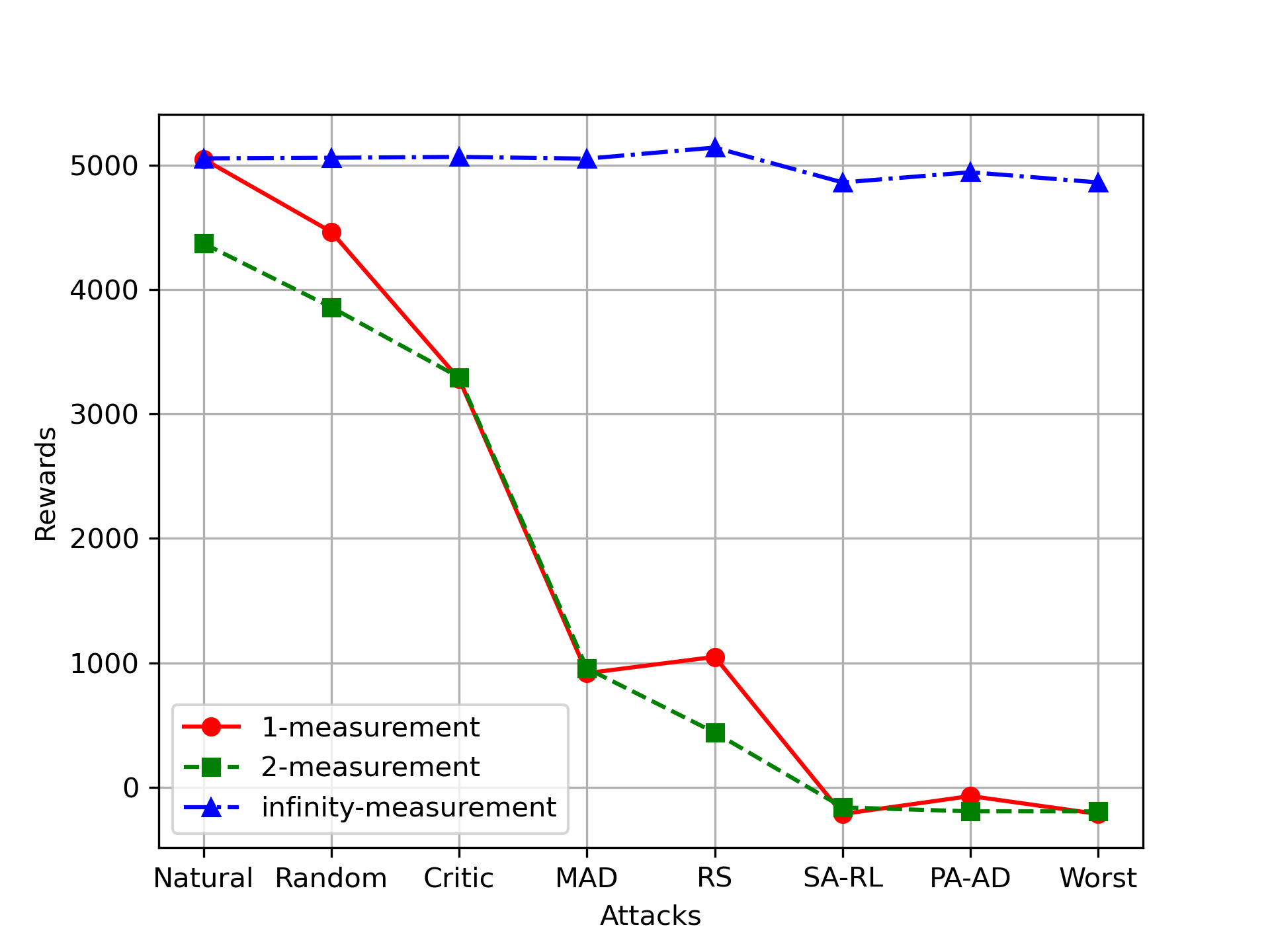}
    \end{subfigure}
    \begin{subfigure}
        \centering
        \includegraphics[width=0.45\textwidth]{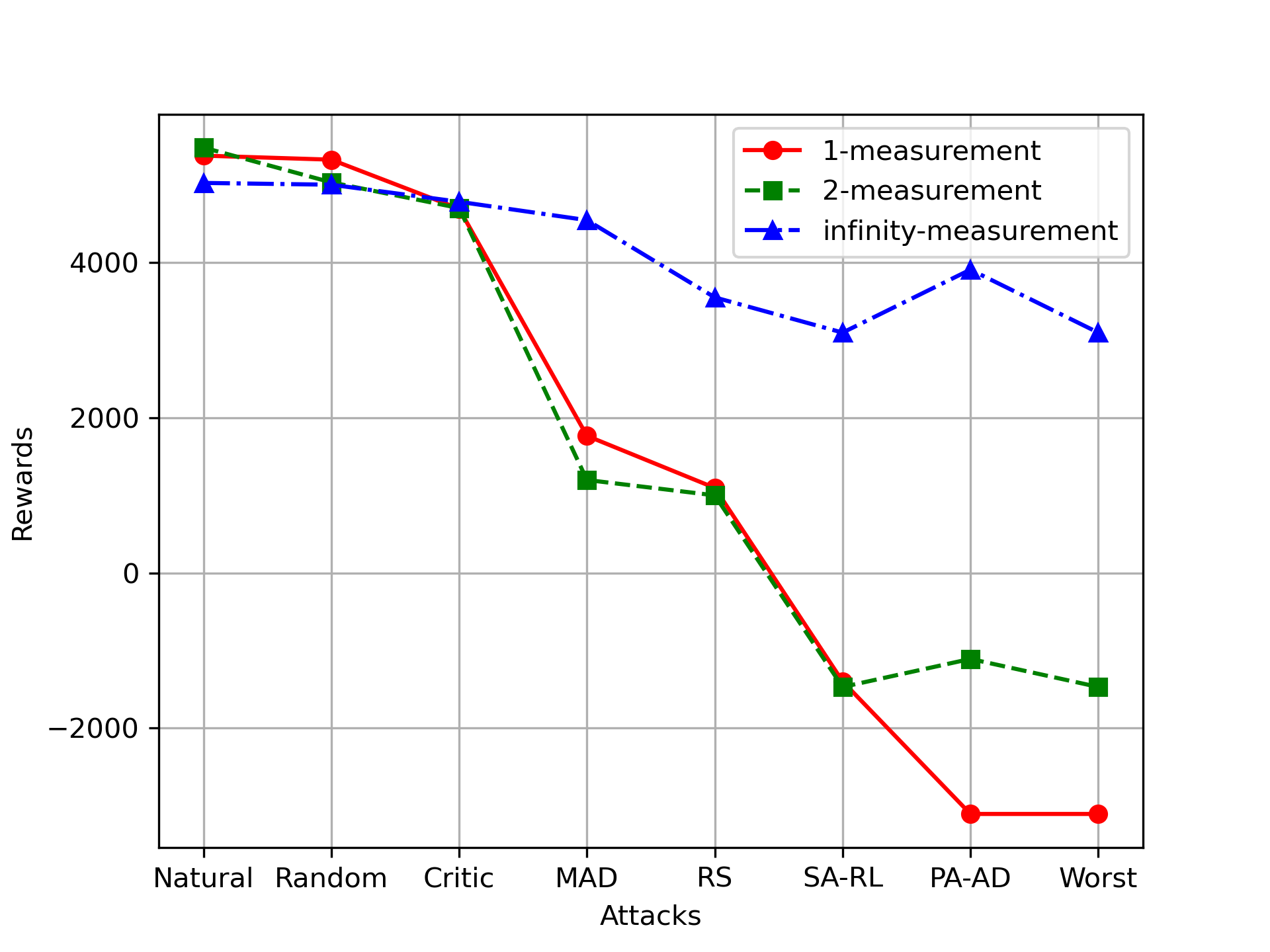}
    \end{subfigure}
    \\
    \vskip -0.1in
    \begin{minipage}{0.48\textwidth}
        \centering
        \quad \scriptsize{Halfcheetah}
    \end{minipage}
    \begin{minipage}{0.48\textwidth}
        \centering
        \setlength{\parindent}{-0.6em}
        \scriptsize{Ant}
    \end{minipage}
    
    \caption{Performance of PPO with different $k$-measurement errors.}
    \label{fig: k measurement}
    \vskip -0.05in
\end{figure*}

\paragraph{Consistency in Natural and PGD Attack Returns on Atari.} 
Figure~\ref{fig:consitency} records the natural and PGD attack returns of CAR-DQN agents during training, showcasing strong alignment between natural performance and robustness across all environments. This consistency supports our theory that the ORP is aligned with the Bellman optimal policy and confirms the rationality of the proposed ISA-MDP framework. 
Additionally, Figures~\ref{fig: natural rewards during training} and~\ref{fig: robustness rewards during training} illustrate the natural episode returns and robustness during training for various algorithms. It is worth noting that CAR-DQN agents quickly and stably converge to peak robustness and natural performance across all environments, while other algorithms exhibit unstable trends. 
For instance, the natural reward curves of SA-DQN and WocaR-DQN on BankHeist and RADIAL-DQN on RoadRunner show notable declines, and the robust curves of SA-DQN and WocaR-DQN on BankHeist also tend to decrease. 
These discrepancies primarily arise from their robustness objectives that deviate from the standard training loss, leading to learning sub-optimal actions. In contrast, the consistent objective of CAR-DQN ensures that it always learns optimal actions in both natural and robust contexts.

\paragraph{Training Efficiency.} 
Training times for SA-DQN, RADIAL-DQN, WocaR-DQN, and CAR-DQN are approximately 27, 12, 20, and 14 hours, respectively, all trained for 4.5 million frames on identical hardware with a GeForce RTX 3090 GPU. 
All robust PPO agents require around 2 hours to train on Hopper, Walker2d, and Halfcheetah, and 9 hours on Ant, on the same device with an AMD EPYC 7742 CPU. Additionally, our CAR loss incurs no additional memory consumption compared to vanilla training.

\subsection{Ablation Studies}

\paragraph{Necessity of Infinity Measurement Errors.} 
To verify the necessity of the $(\infty, d_{\mu_0}^\pi)$-norm in action-value function spaces for adversarial robustness, we train DQN agents using the Bellman error under the $(1, d_{\mu_0}^\pi)$-norm and $(2, d_{\mu_0}^\pi)$-norm, respectively. We then compare their performance with our CAR-DQN, which approximates the Bellman error under the $(\infty, d_{\mu_0}^\pi)$-norm. All agents are trained and evaluated with a perturbation budget $\epsilon=1/255$. As shown in Table~\ref{table: p-error}, all agents perform well without attacks across the four games. However, agents using the $(1, d_{\mu_0}^\pi)$-norm and $(2, d_{\mu_0}^\pi)$-norm experience significant performance degradation under strong attacks, with episode rewards dropping near the lowest values for each game. These results are highly consistent with Theorem \ref{thm:necessity of infty norm}, confirming the importance of the $(\infty, d_{\mu_0}^\pi)$-norm for robust performance.

To validate the necessity of infinity measurement errors in probability spaces, we train PPO agents that rely on $1$-measurement and $2$-measurement errors, respectively. We then compare their performance to our CAR-PPO agents, which are trained with surrogates of infinity measurement errors. As depicted in Figure~\ref{fig: k measurement} and Table~\ref{app table: k measurement}, all agents achieve good performance in natural environments and under weak attacks across the four continuous control tasks. However, agents trained with $1$-measurement and $2$-measurement errors exhibit similar and quite poor performance in environments with strong attacks. In contrast, agents using infinity measurement errors maintain strong robustness across all tasks. These experimental phenomenons align with Theorems~\ref{thm: Vulnerability of Non-infinity Measurement Errors} and~\ref{thm: robustness of infinity measurement}, underscoring the critical role of infinity measurement errors for robust DRL agents.

\begin{figure}[t]
    \centering
    \includegraphics[width=0.7\columnwidth]{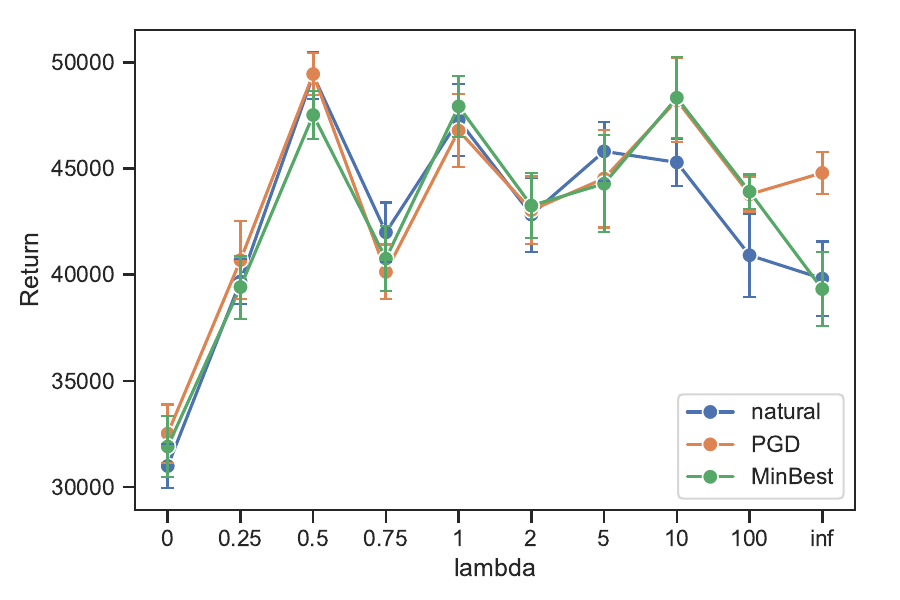}
    \vspace{-1em}
    \caption{Natural, PGD attack, and MinBest attack rewards of CAR-DQN with different soft coefficients $\lambda$ on the RoadRunner game.}
    \vspace{-1em}
    \label{fig:soft roadrunner}
\end{figure}

\begin{table}[t]
\caption{Ablation studies for soft coefficients on 4 Atari games.}
\label{table:soft coeff}
\vskip 0.15in
\centering
\resizebox{0.85\columnwidth}{!}{%
\begin{tabular}{c|c|c|ccc}
\hline
Environment                 & $\lambda$ & Natural    & PGD               & MinBest           & ACR     \\ \hline
\multirow{3}{*}{Pong}       & $0$       & $\bf 21.0 \pm 0.0$    & $\bf 21.0 \pm 0.0$    & $\bf 21.0 \pm 0.0$    & $0.972$ \\
                            & $1$       & $\bf 21.0 \pm 0.0$    & $\bf 21.0 \pm 0.0$    & $\bf 21.0 \pm 0.0$    & $0.985$ \\
                            & $\infty$  & $20.6 \pm 0.1$   & $20.7 \pm 0.1$    & $20.7 \pm 0.1$    & $0.980$ \\ \hline
\multirow{3}{*}{Freeway}    & $0$       & $31.6 \pm 0.2$    & $31.5 \pm 0.1$    & $31.5 \pm 0.1$    & $0.966$ \\
                            & $1$       & $\bf 33.3 \pm 0.1$   & $\bf 33.2 \pm 0.1$   & $\bf 33.2 \pm 0.1$   & $0.981$ \\
                            & $\infty$  & $31.5 \pm 0.1$    & $30.9 \pm 0.3$    & $31.2 \pm 0.2$    & $0.967$ \\ \hline
\multirow{3}{*}{BankHeist}  & $0$       & $1307.5 \pm 11.0$ & $1288.5 \pm 14.0$ & $1284.0 \pm 13.8$ & $0.980$ \\
                            & $1$       & $\bf 1356.0 \pm 1.7$  & $\bf 1356.5 \pm 1.1$  & $\bf 1356.5 \pm 1.1$  & $0.969$ \\
                            & $\infty$  & $1326.0 \pm 4.8$  & $1316.0 \pm 6.8$  & $1314.0 \pm 6.6$  & $0.979$ \\ \hline
\multirow{3}{*}{RoadRunner} & $0$       & $25160 \pm 802$   & $24540 \pm 760$   & $26785 \pm 617$   & $0.007$ \\
                            & $1$       & $\bf 49500 \pm 2106$ & $\bf 48230 \pm 1648$ & $\bf 48050 \pm 1642$ & $0.947$ \\
                            & $\infty$  & $40890 \pm 2075$  & $36760 \pm 1874$  & $36740 \pm 2098$  & $0.940$ \\ \hline
\end{tabular}%
}
\vskip -0.1in
\end{table}

\paragraph{Effects of Soft Coefficient $\lambda$.} 
We assess the impact of the soft coefficient $\lambda$ on the CAR-DQN loss function by training agents in the RoadRunner environment with varying values of $\lambda$ ranging from $0$ to $\infty$, validating the effectiveness of soft CAR loss.  When $\lambda=0$, the agent utilizes only the sample with the largest adversarial TD-error within the batch, whereas $\lambda=\infty$ corresponds to averaging over all samples in the batch. It is important to note that small values of $\lambda$ can lead to numerical instability. As depicted in Figure~\ref{fig:soft roadrunner}, agents exhibit similar capabilities when $0.5 \leq \lambda \leq 10$, indicating that the learned policies are relatively insensitive to the soft coefficient within this range. Table~\ref{table:soft coeff} presents the performance of CAR-DQN agents with $\lambda=0,1,\infty$ across four Atari environments. In the RoadRunner, $\lambda=0$ results in poor performance, achieving returns of around 25,000 due to insufficient utilization of the samples. Interestingly, utilizing only the sample with the largest adversarial TD error in the batch achieves good robustness in the other three simpler games. The case $\lambda=\infty$ results in weaker robustness compared to other cases with differentiated weights. This suggests that each sample in the batch plays a distinct role in robust training, and assigning appropriate weights to the samples in a batch enhances robust performance. These results further validate the efficacy of our CAR-DQN loss function.

Additionally, we evaluate the effects of the soft coefficient $\lambda$ through training CAR-PPO with various solvers and soft coefficients. As shown in Table~\ref{table:soft coeff mujoco}, in Hopper and Ant tasks, CAR-PPO agents show similar performance under selected coefficients and solvers, indicating that trained policies are relatively insensitive to the soft coefficient under these settings. In Walker2d and Halfcheetah environments, CAR-PPO agents are more sensitive to different $\lambda$, further confirming the effectiveness of the CAR-PPO loss function.

\begin{table}[t]
\caption{Ablation studies for soft coefficients on 4 Mujoco tasks.}
\label{table:soft coeff mujoco}
\vskip 0.15in
\centering
\resizebox{\columnwidth}{!}{%
\begin{tabular}{c|c|c|c|ccccccc}
\hline\hline
\multirow{2}{*}{\textbf{Env}}                                                            & \multirow{2}{*}{\textbf{Solver}}                                               & \multirow{2}{*}{\textbf{$\lambda$}} & \multirow{2}{*}{\textbf{\begin{tabular}[c]{@{}c@{}}Natural\\ Reward\end{tabular}}} & \multicolumn{7}{c}{\textbf{Attack Reward}}                                                                           \\
                                                                                         &                                                                                &                                     &                                                                                    & \textbf{Random} & \textbf{Critic} & \textbf{MAD}  & \textbf{RS}   & \textbf{SA-RL} & \textbf{PA-AD} & \textbf{Worst} \\ \hline
\multirow{6}{*}{\begin{tabular}[c]{@{}c@{}}Hopper\\ ($\epsilon$=0.075)\end{tabular}}     & \multirow{3}{*}{\begin{tabular}[c]{@{}c@{}}SGLD\\ ($\kappa$=1.0)\end{tabular}} & 10                                  & 3148                                                                               & 3090            & 3552            & 3211          & \textbf{2292} & 1554           & 2492           & 1554           \\
                                                                                         &                                                                                & 100                                 & \textbf{3566}                                                                      & \textbf{3537}   & 3480            & \textbf{3484} & 1990          & \textbf{2977}  & \textbf{3232}  & \textbf{1990}  \\
                                                                                         &                                                                                & 1000                                & 3540                                                                               & 3487            & \textbf{3622}   & 3324          & 1079          & 1079           & 2133           & 1079           \\ \cline{2-11}
                                                                                         & \multirow{3}{*}{\begin{tabular}[c]{@{}c@{}}PGD\\ ($\kappa$=0.3)\end{tabular}}  & 10                                  & 3660                                                                               & 3604            & 3283            & 3046          & 1426          & 601            & 2444           & 601            \\
                                                                                         &                                                                                & 100                                 & 3608                                                                               & \textbf{3734}   & \textbf{3745}   & 2946          & 1610          & 1103           & 1888           & 1103           \\
                                                                                         &                                                                                & 1000                                & \textbf{3711}                                                                      & 3702            & 3692            & \textbf{3473} & \textbf{1652} & \textbf{2430}  & \textbf{2640}  & \textbf{1652}  \\ \hline\hline
\multirow{6}{*}{\begin{tabular}[c]{@{}c@{}}Walker2d\\ ($\epsilon$=0.05)\end{tabular}}    & \multirow{3}{*}{\begin{tabular}[c]{@{}c@{}}SGLD\\ ($\kappa$=0.1)\end{tabular}} & 10                                  & \textbf{4622}                                                                      & \textbf{4609}   & \textbf{4684}   & \textbf{4498} & 4242          & \textbf{4397}  & \textbf{3134}  & \textbf{3134}  \\
                                                                                         &                                                                                & 100                                 & 3649                                                                               & 3252            & 3947            & 3607          & 1979          & 650            & 689            & 650            \\
                                                                                         &                                                                                & 1000                                & 4323                                                                               & 4326            & 4308            & 4340          & \textbf{4323} & 4137           & 2073           & 2073           \\ \cline{2-11}
                                                                                         & \multirow{3}{*}{\begin{tabular}[c]{@{}c@{}}PGD\\ ($\kappa$=1.0)\end{tabular}}  & 10                                  & 3247                                                                               & 3262            & 3771            & 3216          & 1449          & 1098           & 1416           & 1098           \\
                                                                                         &                                                                                & 100                                 & \textbf{4755}                                                                      & \textbf{4848}   & \textbf{5045}   & \textbf{4637} & \textbf{4379} & \textbf{4308}  & \textbf{4303}  & \textbf{4303}  \\
                                                                                         &                                                                                & 1000                                & 3400                                                                               & 3296            & 3136            & 3469          & 2937          & 1516           & 1362           & 1362           \\ \hline\hline
\multirow{6}{*}{\begin{tabular}[c]{@{}c@{}}Halfcheetah\\ ($\epsilon$=0.15)\end{tabular}} & \multirow{3}{*}{\begin{tabular}[c]{@{}c@{}}SGLD\\ ($\kappa$=0.1)\end{tabular}} & 10                                  & 3953                                                                               & 3967            & 3862            & 4008          & 3792          & 3273           & 3334           & 3273           \\
                                                                                         &                                                                                & 100                                 & \textbf{5131}                                                                      & \textbf{5059}   & \textbf{4803}   & 3131          & 3360          & 3032           & 1507           & 1507           \\
                                                                                         &                                                                                & 1000                                & 4599                                                                               & 4574            & 4731            & \textbf{4348} & \textbf{3888} & \textbf{3908}  & \textbf{4032}  & \textbf{3888}  \\ \cline{2-11}
                                                                                         & \multirow{3}{*}{\begin{tabular}[c]{@{}c@{}}PGD\\ ($\kappa$=1.0)\end{tabular}}  & 10                                  & 4751                                                                               & 4741            & 4676            & 4712          & 4354          & 4389           & 4501           & 4354           \\
                                                                                         &                                                                                & 100                                 & 3291                                                                               & 3298            & 3528            & 3341          & 3126          & 2991           & 3159           & 2991           \\
                                                                                         &                                                                                & 1000                                & \textbf{5053}                                                                      & \textbf{5058}   & \textbf{5065}   & \textbf{5053} & \textbf{5140} & \textbf{4860}  & \textbf{4942}  & \textbf{4860}  \\ \hline\hline
\multirow{6}{*}{\begin{tabular}[c]{@{}c@{}}Ant\\ ($\epsilon$=0.15)\end{tabular}}         & \multirow{3}{*}{\begin{tabular}[c]{@{}c@{}}SGLD\\ ($\kappa$=0.1)\end{tabular}} & 100                                 & 5056                                                                               & 5007            & 4864            & 4468          & \textbf{3755} & \textbf{3088}  & \textbf{3763}  & \textbf{3088}  \\
                                                                                         &                                                                                & 1000                                & \textbf{5283}                                                                      & \textbf{5180}   & 4798            & 4469          & 3292          & 2387           & 3442           & 2387           \\
                                                                                         &                                                                                & 10000                               & 5187                                                                               & 5008            & \textbf{4997}   & \textbf{4589} & 3715          & 2322           & 3527           & 2322           \\ \cline{2-11}
                                                                                         & \multirow{3}{*}{\begin{tabular}[c]{@{}c@{}}PGD\\ ($\kappa$=0.3)\end{tabular}}  & 100                                 & 4901                                                                               & 4807            & 4401            & 4429          & \textbf{3876} & \textbf{3197}  & \textbf{3917}  & \textbf{3197}  \\
                                                                                         &                                                                                & 1000                                & \textbf{5074}                                                                      & 4937            & \textbf{4838}   & 4536          & 2779          & 2961           & 3889           & 2776           \\
                                                                                         &                                                                                & 10000                               & 5029                                                                               & \textbf{5006}   & 4787            & \textbf{4549} & 3553          & 3099           & 3911           & 3099         \\ \hline\hline 
\end{tabular}%
}
\end{table}

\section{Conclusion and Future Work}
In this paper, we demonstrate the alignment of the optimal robust policy with the Bellman optimal policy within the universal Intrinsic State-adversarial MDP framework. We prove that optimizing different $k$-measurement errors yields varied performance, underscoring the necessity of infinity measurement errors for achieving adversarial robustness. Our findings are validated through experiments with CAR-DQN and CAR-PPO, which optimize surrogate objectives of infinity measurement errors. We believe this work contributes significantly to unveiling the nature of adversarial robustness in reinforcement learning. 

Our theoretical focus primarily addresses the necessity of infinity measurements in the worst-case scenario. In future work, we plan to explore how DRL agents converge to vulnerable models from the perspective of feature learning theory.
Additionally, the CAR operator we introduced to characterize the adversarial robust training involves a bilevel optimization problem, and robust training itself can be framed as a minimax problem. We aim to investigate these aspects further, leveraging bilevel and minimax optimization theories and techniques to improve the efficiency of robust training.

\section*{Acknowledgements}
This paper is supported by the National Key R\&D Program of China (2021YFA1000403) and the National Natural Science Foundation of China (Nos. 12431012). We would like to express our gratitude to Yongyuan Liang and Chung-En Sun for their enthusiastic assistance with the experimental setup. We also thank Shichen Liao for his early contributions to the conference version of this work. Additionally, we appreciate the valuable feedback on the paper drafts provided by Anqi Li and Wenzhao Liu.

\newpage

\appendix

\tableofcontents

\section{Examples of Intrinsic States}\label{app: instrinsic state}
In Figure \ref{app fig:intrinsic states pong}, \ref{app fig:intrinsic states free}, \ref{app fig:intrinsic states road}, we show some examples in 3 Atari games (Pong, Freeway, and RoadRunner), indicating that the state observation $s$ and adversarial observation $s_\nu$ share the same intrinsic state.

\begin{figure}[H]
    \centering
\includegraphics[width=0.75\textwidth]{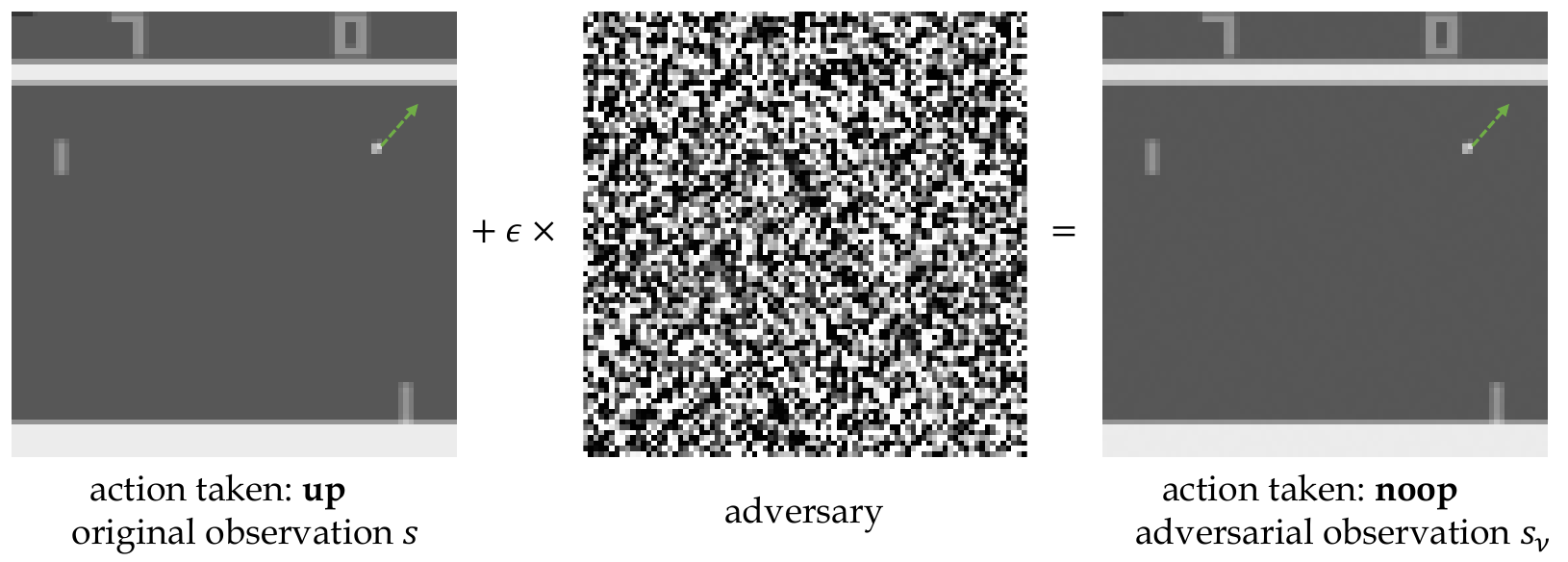} \includegraphics[width=0.75\textwidth]{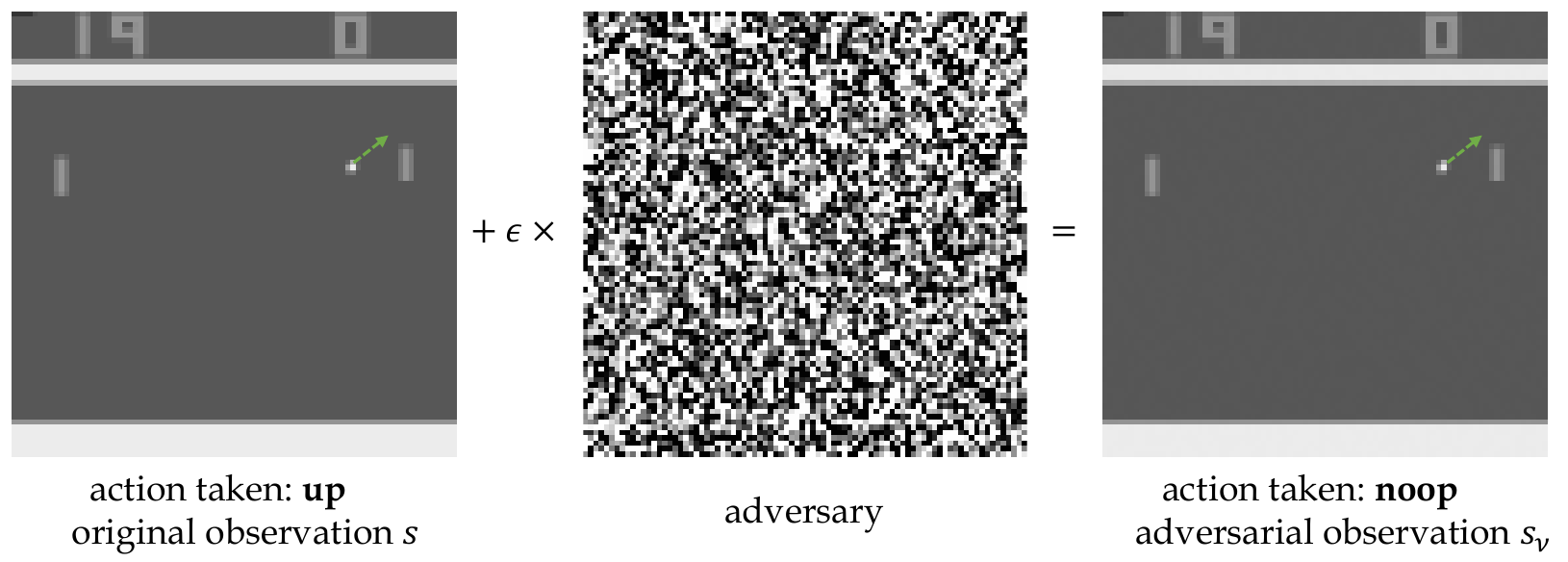}
\includegraphics[width=0.75\textwidth]{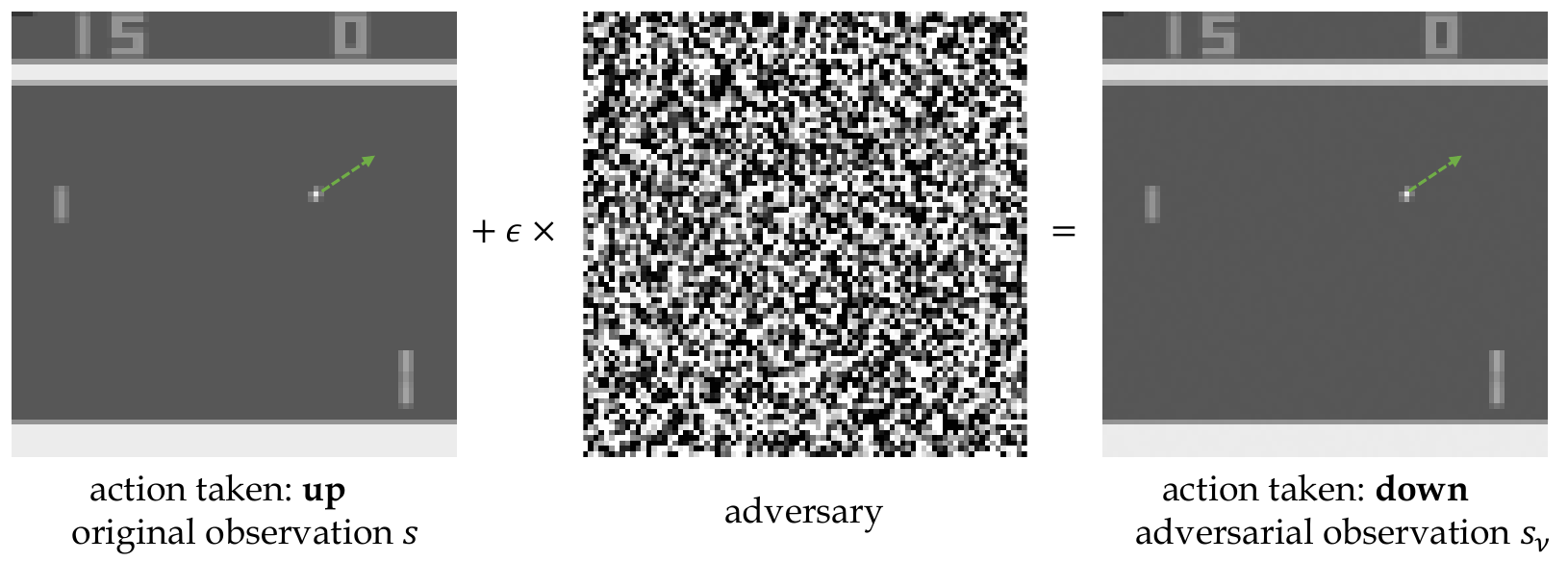}
\includegraphics[width=0.75\textwidth]{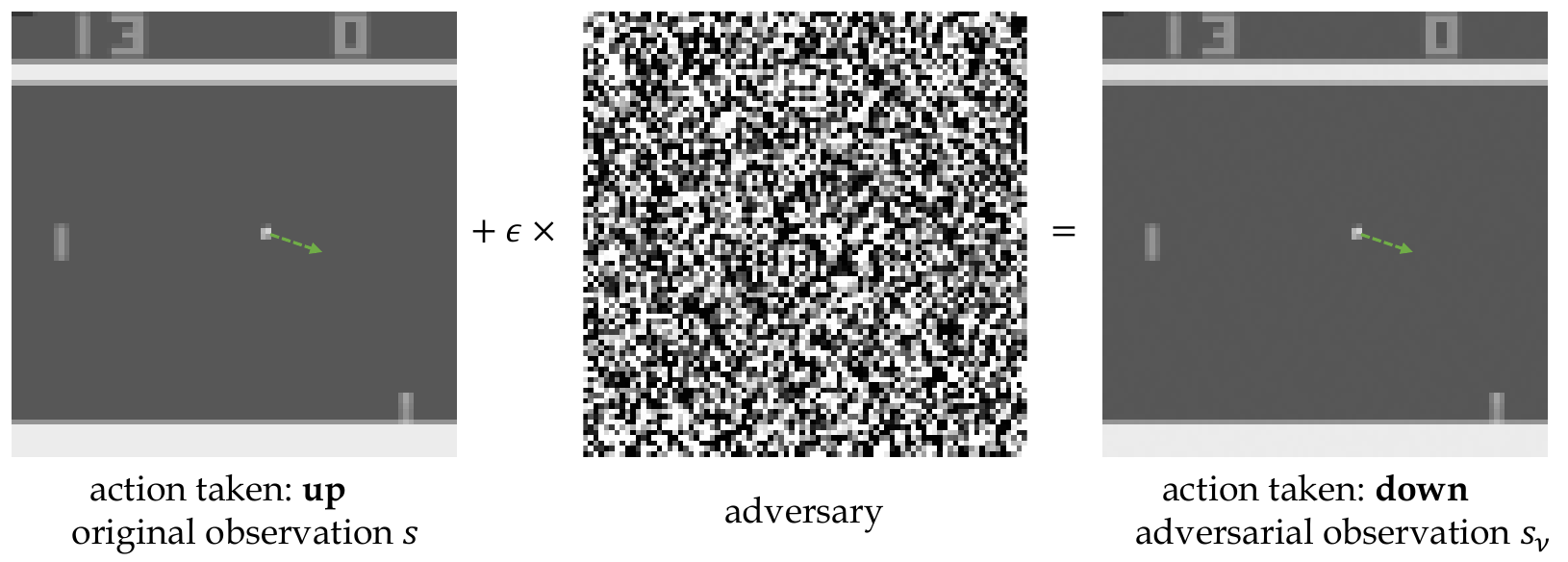}
\includegraphics[width=0.75\textwidth]{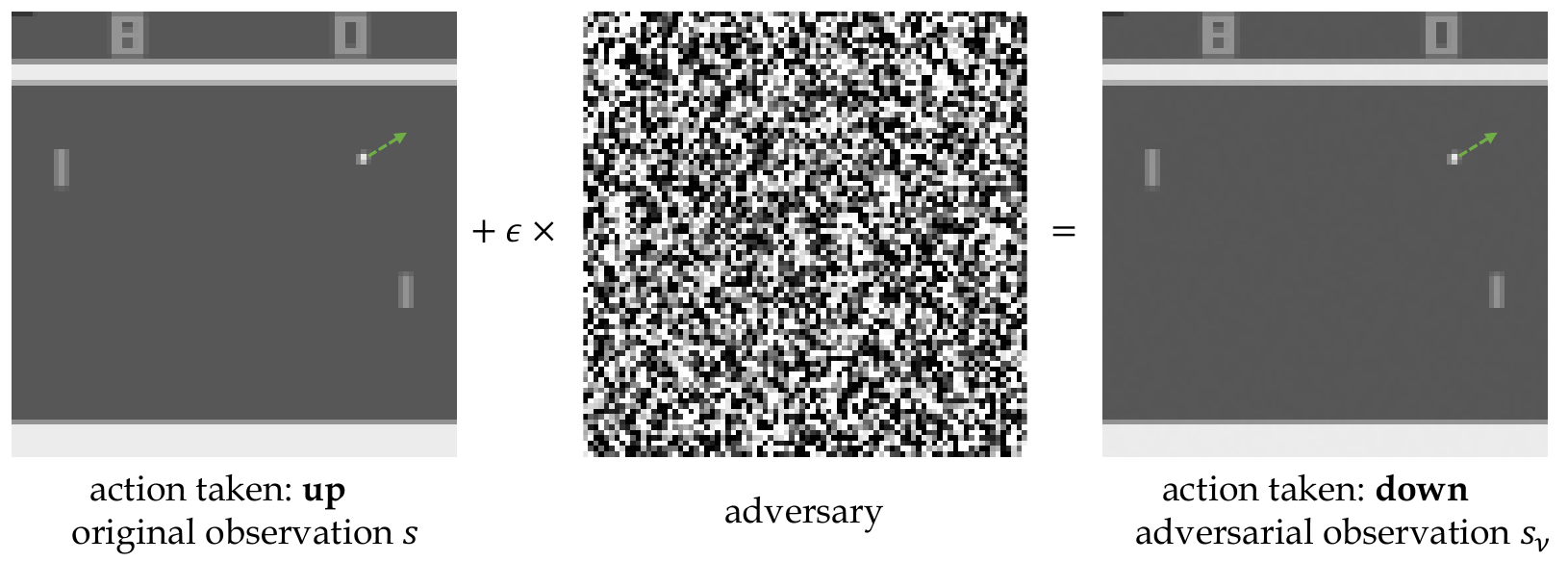}
    \caption{Examples of intrinsic states in Pong games. The direction for the movement of the ball is marked.}
    \label{app fig:intrinsic states pong}
\end{figure}

\begin{figure}[H]
    \centering
\includegraphics[width=0.75\textwidth]{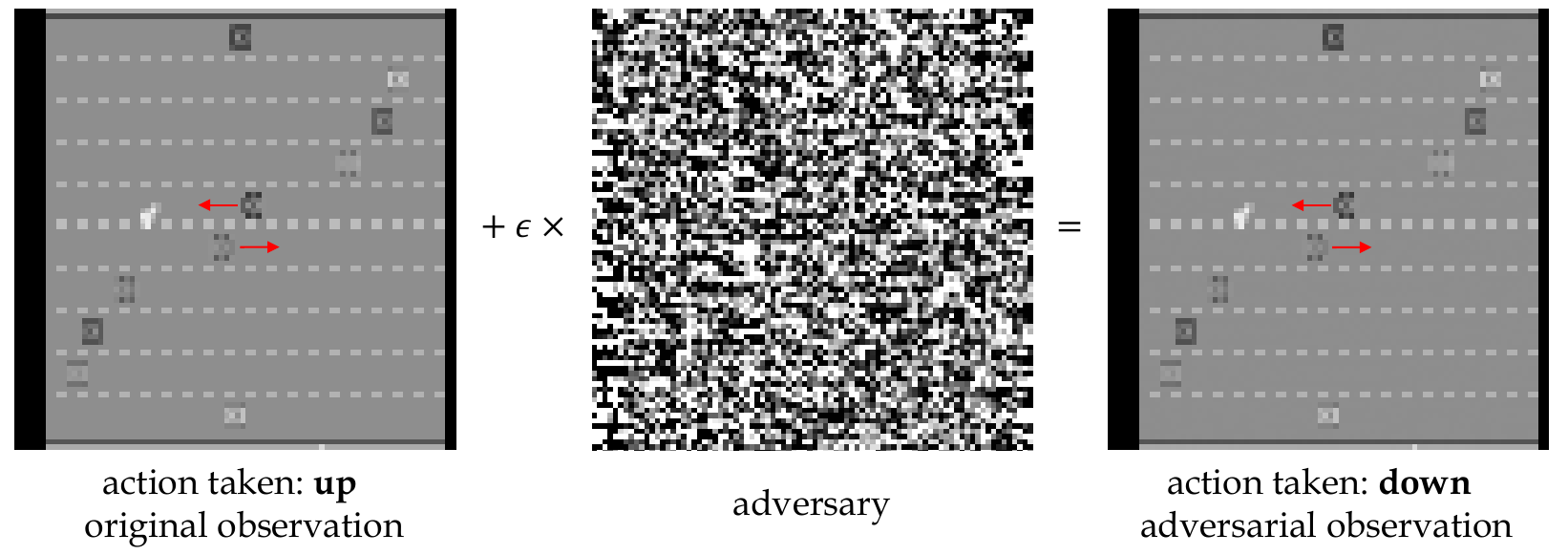}
\includegraphics[width=0.75\textwidth]{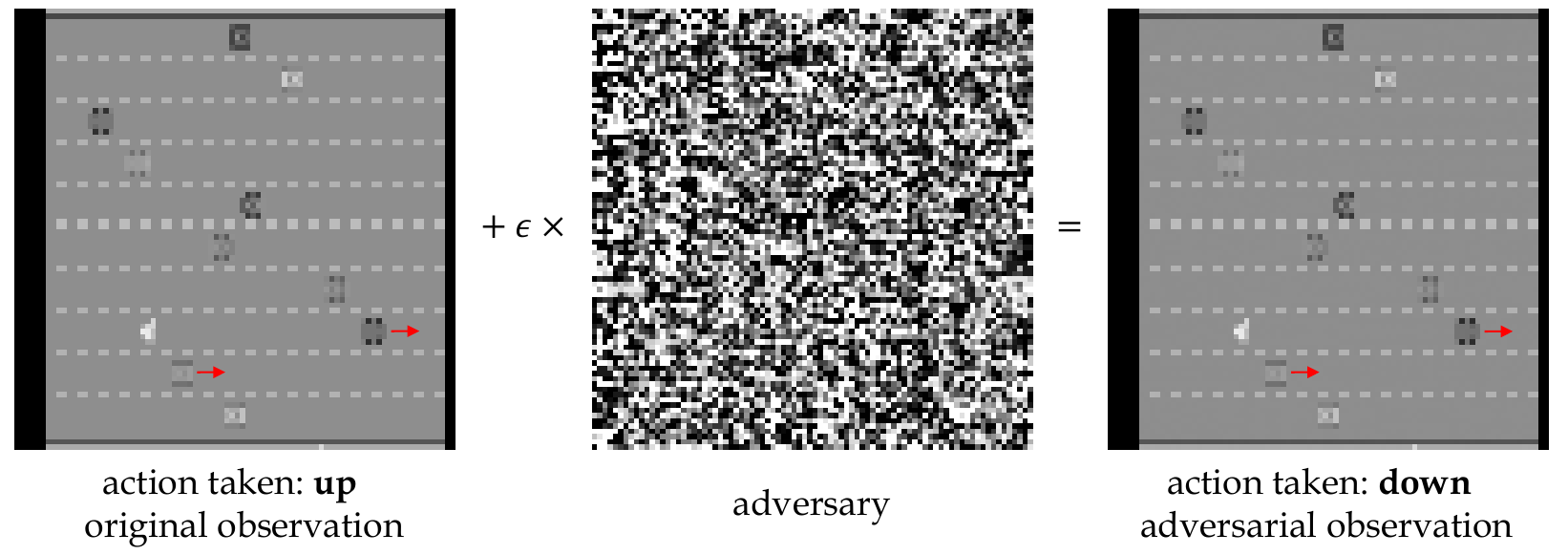}
\includegraphics[width=0.75\textwidth]{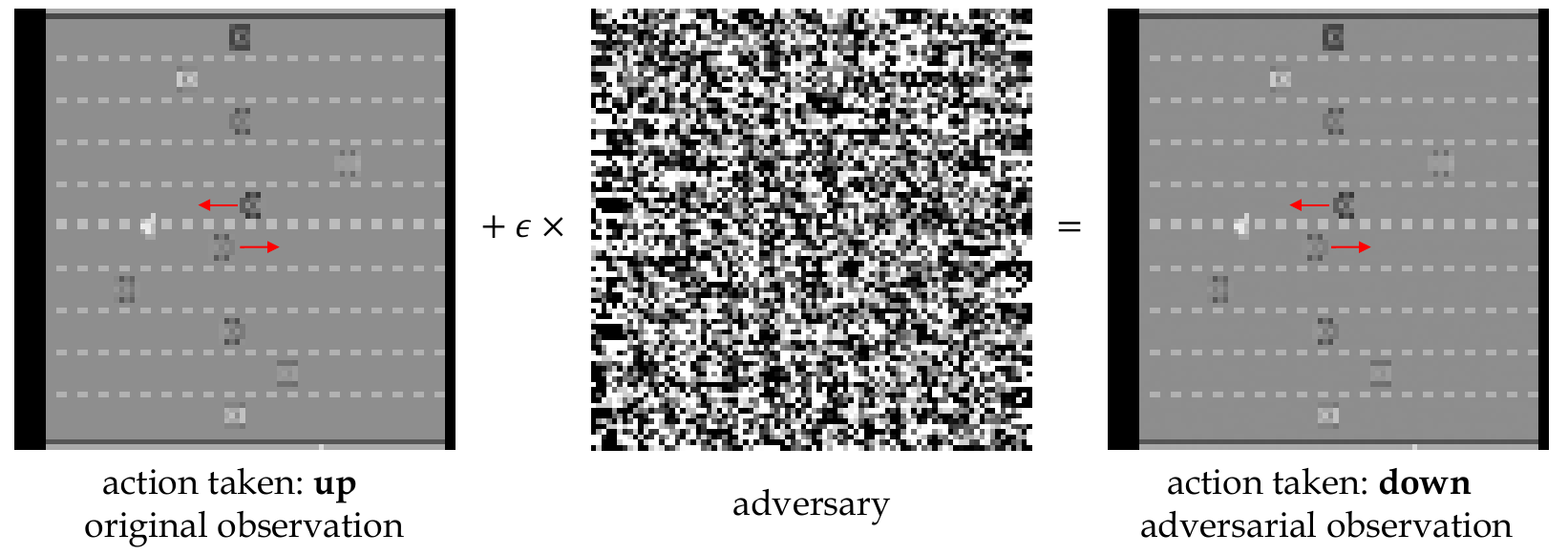}
\includegraphics[width=0.75\textwidth]{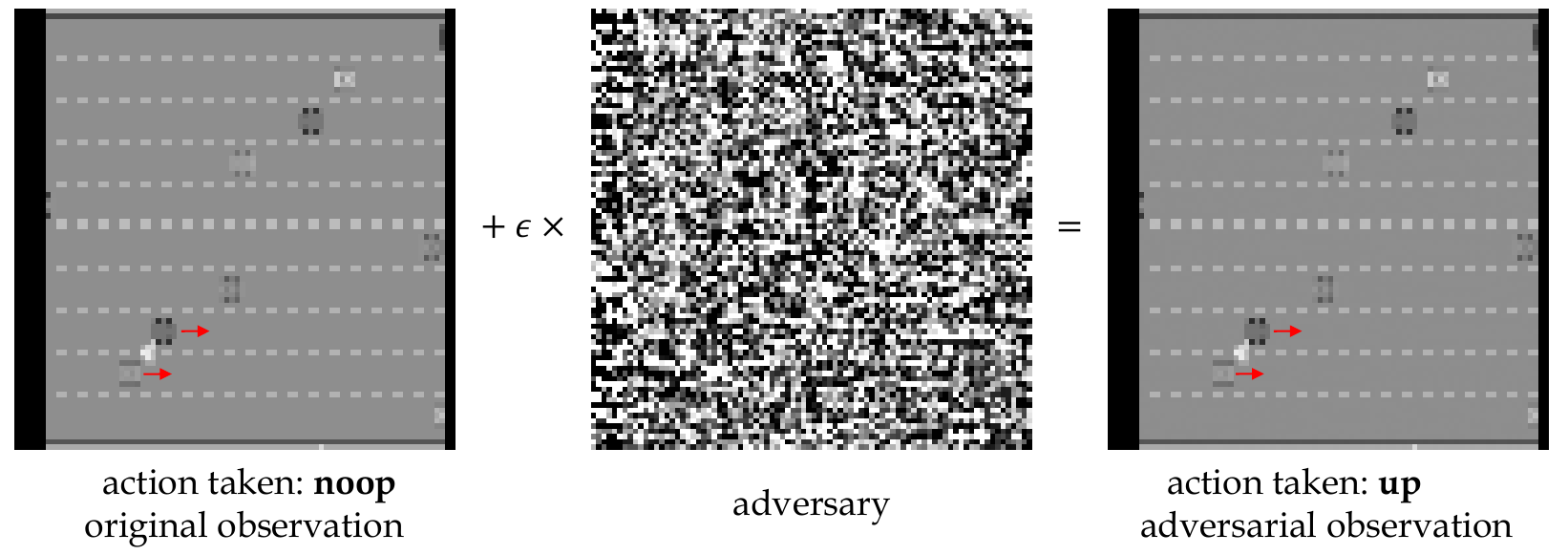}
\includegraphics[width=0.75\textwidth]{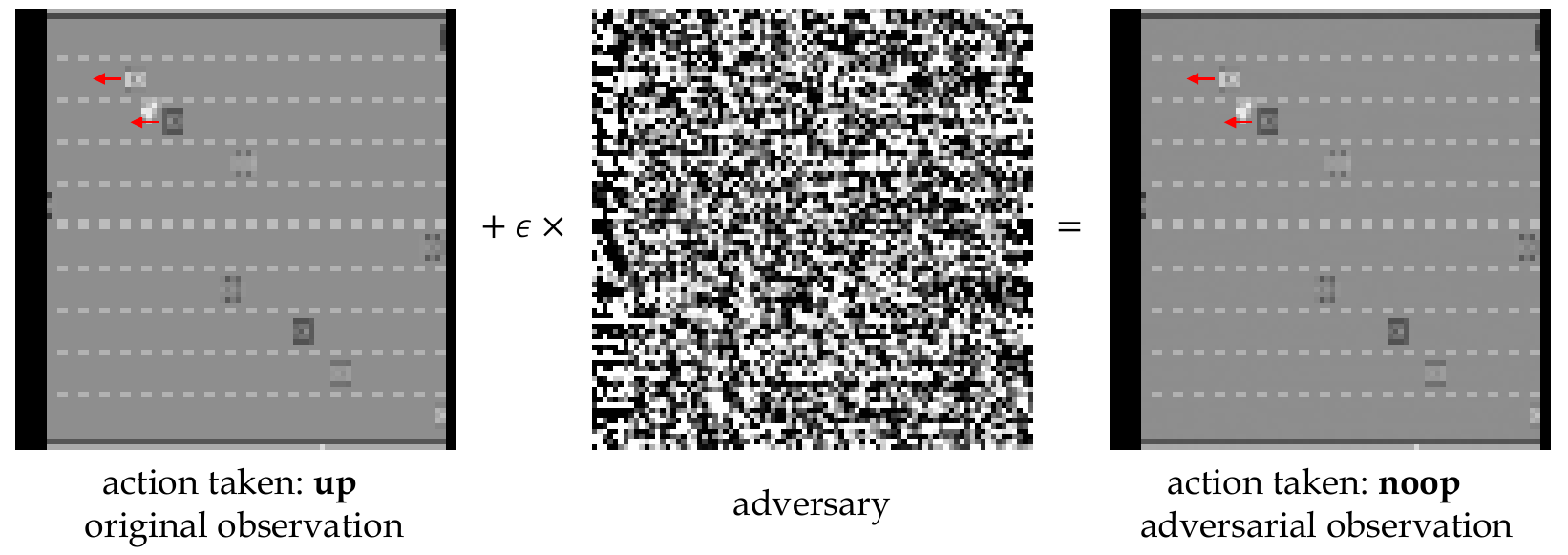}
    \caption{Examples of intrinsic states in Freeway games. The directions for the movement of cars around the chicken are marked.}
    \label{app fig:intrinsic states free}
\end{figure}

\begin{figure}[H]
    \centering
\includegraphics[width=0.75\textwidth]{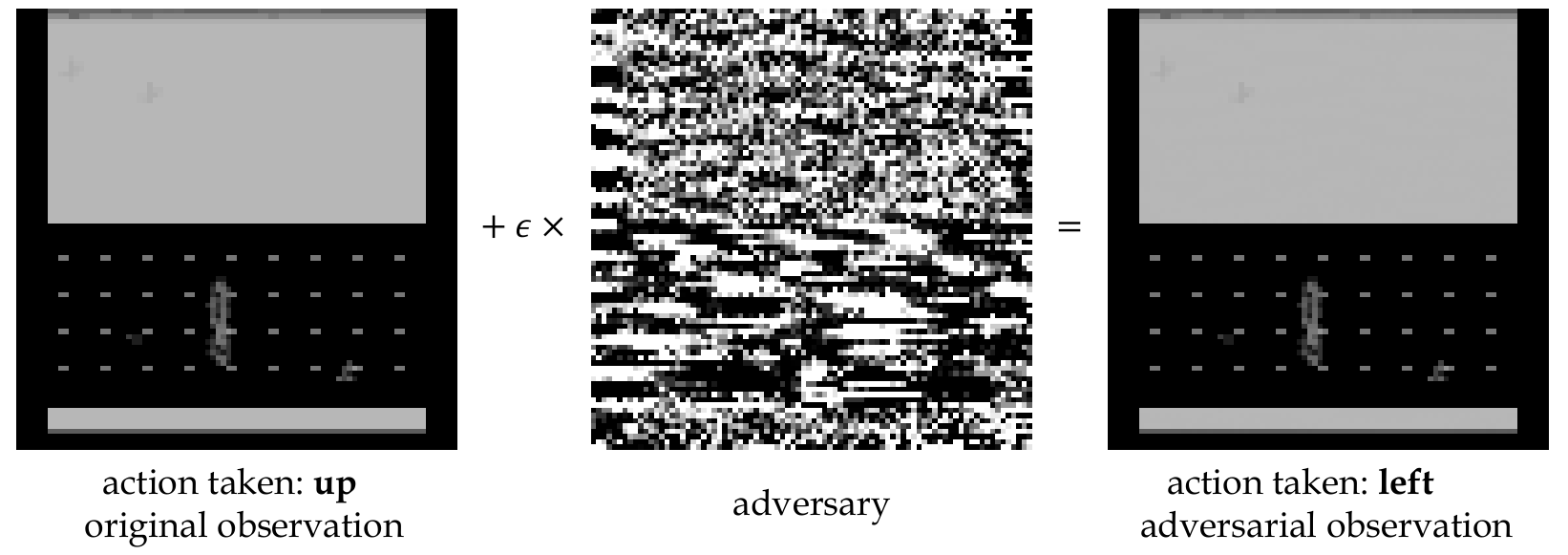}
\includegraphics[width=0.75\textwidth]{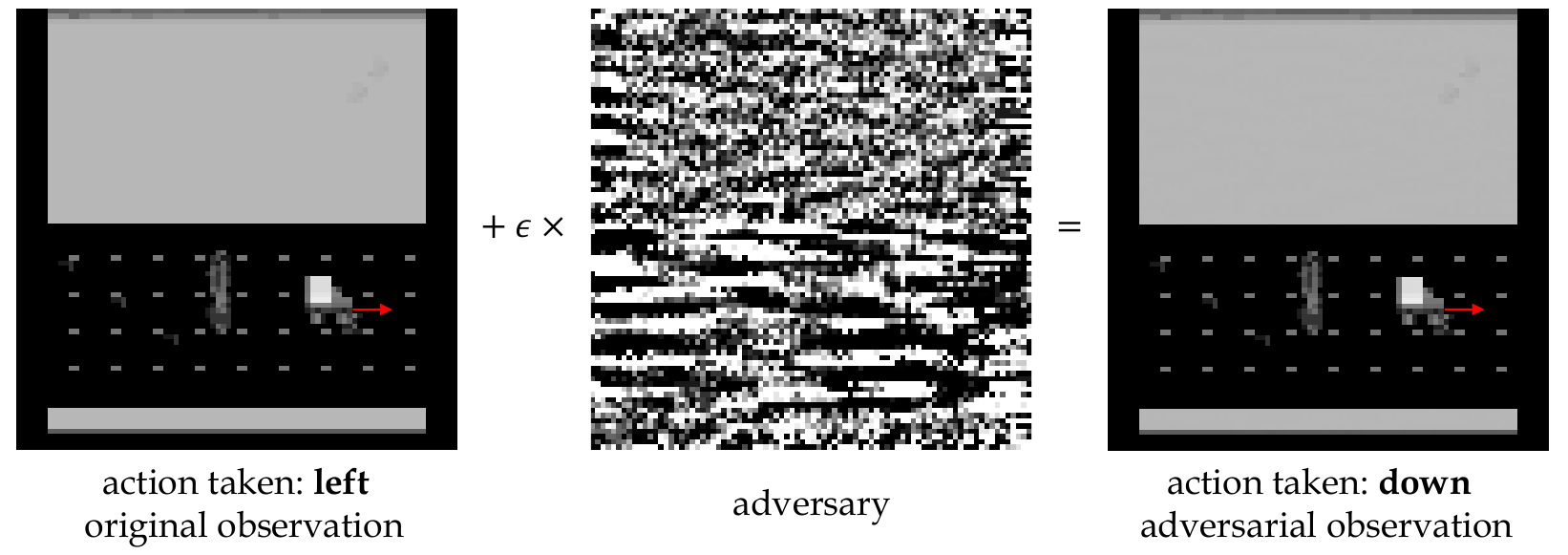}
\includegraphics[width=0.75\textwidth]{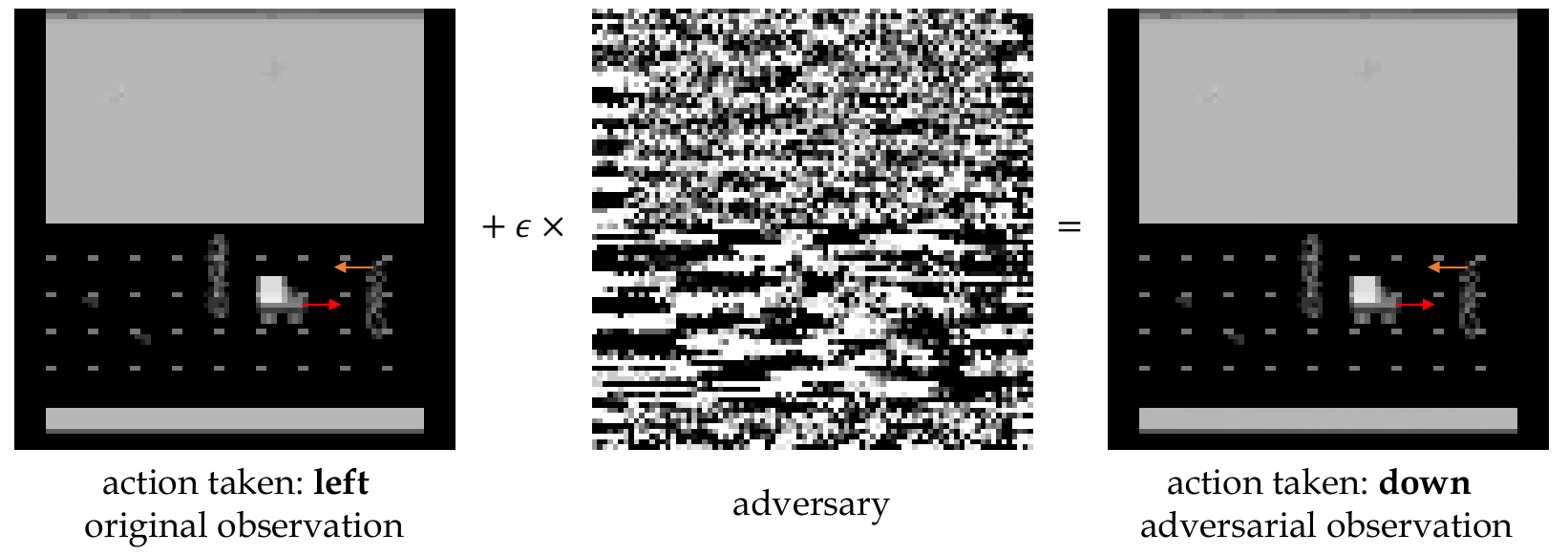}
\includegraphics[width=0.75\textwidth]{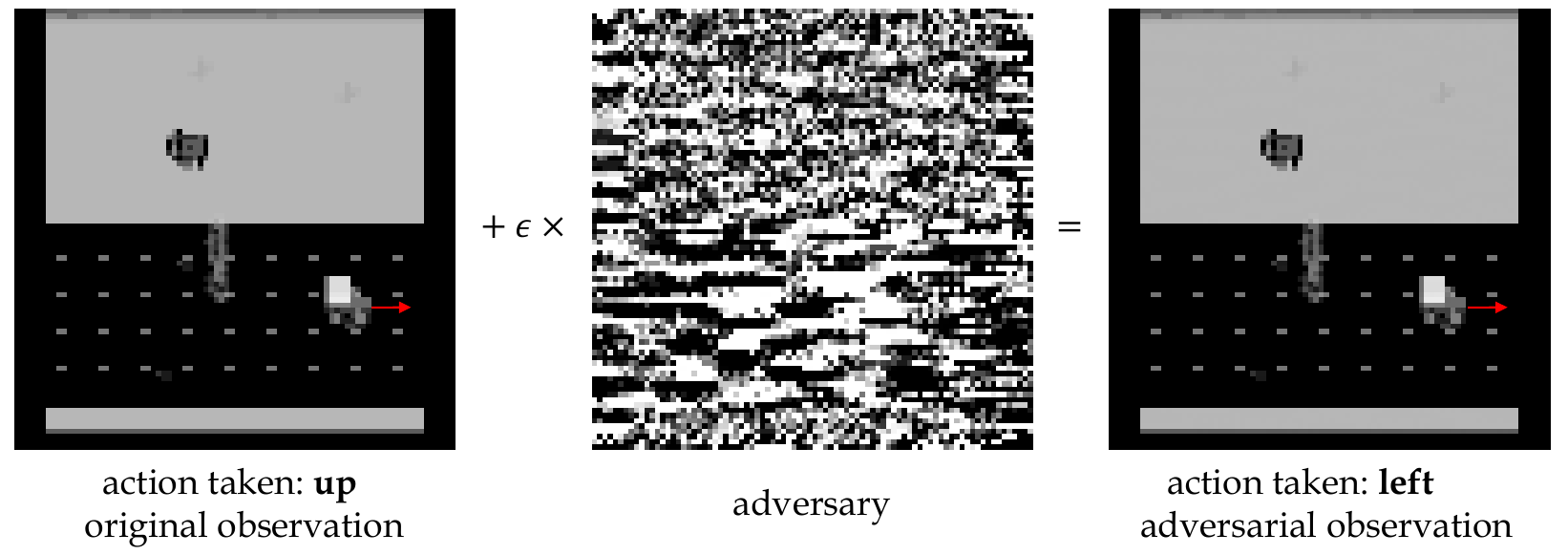}
\includegraphics[width=0.75\textwidth]{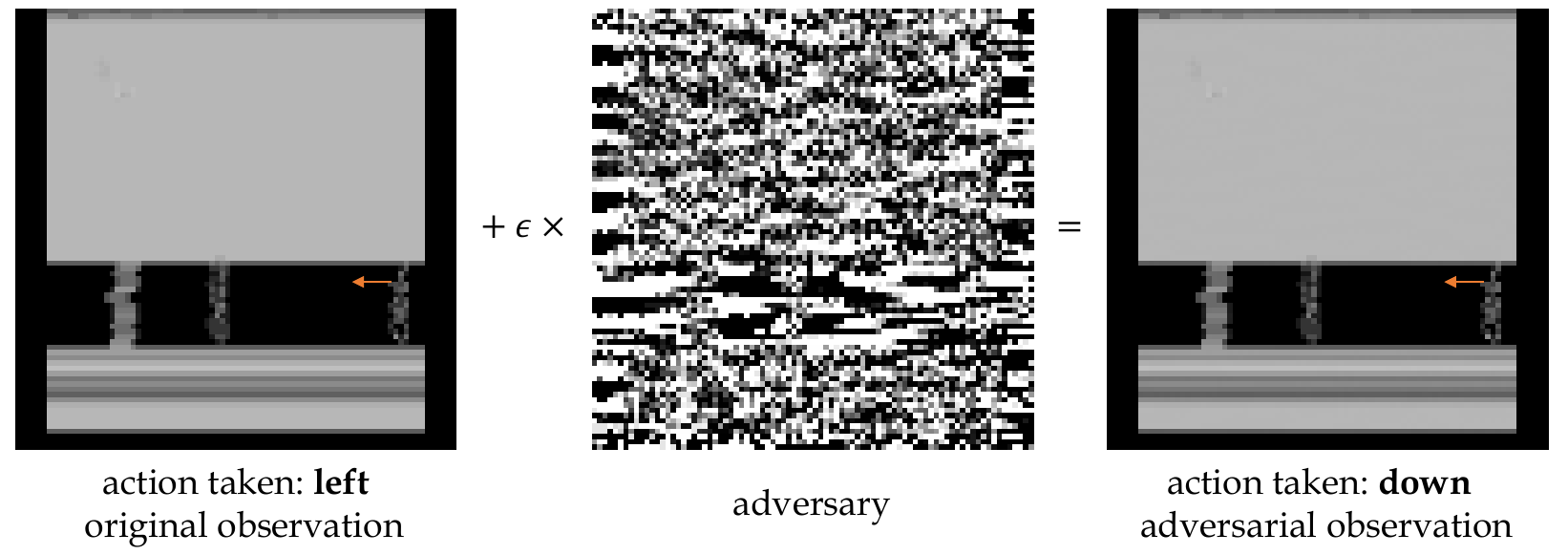}
    \caption{Examples of intrinsic states in RoadRunner games. The directions for the movement of trucks and competitors are marked.}
    \label{app fig:intrinsic states road}
\end{figure}

\section{Theorems and Proofs of Optimal Adversarial Robustness}
\subsection{Characterization of the Sparse Difference Between Intrinsic and Standard State Neighborhoods} \label{app: reasonable of C assumption}

\begin{theorem}\label{continuous a.e.}
    For any MDP $\mathcal{M}$, let $\mathcal{S}_{nu}$ denote the state set where the optimal action is not unique, i.e. $\mathcal{S}_{nu} = \left\{ s\in\mathcal{S} | \mathop{\arg\max}_a Q^*(s, a) \text{ is not a singleton} \right\}$. Then we have the following conclusions:
        \begin{itemize}
            \item Let $\mathcal{S}' =\{s \in S|\exists\  a \in A, \text{s.t. } Q^*(s, a)\  \text{is not continuous at}\  s\}$, then for all $ s \in \left(\mathcal{S} \setminus \mathcal{S}_{nu} \right) \setminus  \mathcal{S}^{\prime} $, there exists $\epsilon > 0$ such that $B_\epsilon(s) = B^*_\epsilon(s)$.
            \item Given $\epsilon > 0$, let $\mathcal{S}_{nin}$ denote the set of states where \textit{the intrinsic state $\epsilon$-neighourhood} is not the same as \textit{the $\epsilon$-neighourhood}, i.e. $\mathcal{S}_{nin} = \left\{ s\in\mathcal{S} |  B_\epsilon(s) \neq B^*_\epsilon(s) \right\}$. Then, we have $\mathcal{S}_{nin} \subseteq \mathcal{S}_{nu} \cup \mathcal{S}_{0} + B_\epsilon $, where $\mathcal{S}_{0}$ is the set of discontinuous points that cause the optimal action to change, i.e. $\mathcal{S}_{0}=\{s \in S|\forall\  \epsilon_1 > 0, \exists \  s' \in B_{\epsilon_1}(s), \text{s.t. } \arg\max_a Q^*(s', a) \neq \arg\max_a Q^*(s, a)\} \cap \mathcal{S}'$.
        \end{itemize}
\end{theorem}
\begin{proof}
   (1)  Let $\mathcal{S}^{\prime}=\{ s \in \mathcal{S}  |  \exists a \in \mathcal{A},  \text{s.t. $Q^*(s,a)$  is not continuous at } s \}$, 
   then we have that $Q^*(s,a)$ is continuous in $\mathcal{S} \setminus \left(\mathcal{S}_{nu}  \cup  \mathcal{S}^{\prime}\right) $. 
   
   Because $ \arg\max_{a}Q^*(s,a) $ is a singleton  for $\mathcal{S} \setminus \left(\mathcal{S}_{nu}  \cup  \mathcal{S}^{\prime}\right) $, we can define $\arg\max_{a}Q^*(s,a) = \{a_s^*\} $ for any $s\in \left(\mathcal{S} \setminus \mathcal{S}_{nu} \right) \setminus  \mathcal{S}^{\prime} $.
   Then, we have that 
   $$Q^*(s,a_s^*) > Q^*(s,a) \text{ for a fixed } a\in \mathcal{A} \setminus \{a_{s}^*\}.$$
   According to continuity of $Q^*(\cdot,a)$ for all $a\in \mathcal{A}$,
   there exists $ \epsilon_a >0$, such that 
   $$Q^*(s^{\prime},a_{s}^*) > Q^*(s^{\prime},a) \text{ for all }s^{\prime} \in  B_{\epsilon_a}(s).$$
   Because $\mathcal{A}$ is a finite discrete set, let $\epsilon=\min_{a\in \mathcal{A} \setminus {a_s^*}}\{\epsilon_a\}$, then we have that
   $$Q^*(s^{\prime},a_{s^{\prime}}^*) > Q^*(s^{\prime},a) \qquad \forall s^{\prime} \in  B_\epsilon(s) \text{ and } a\in \mathcal{A} \setminus \{a_{s^{\prime}}^*\},$$
   which indicates that we have
   $$B_\epsilon(s)=B_\epsilon^*(s).$$

   (2) Let 
    $\mathcal{S}_n=\{ s \in \mathcal{S}  | \forall \epsilon_1>0, \exists s^{\prime} \in B_{\epsilon_1}(s), \ \text{s.t. }    \arg\max_{a}Q^*(s^{\prime},a) \neq \arg\max_{a}Q^*(s,a) \}$ 
    and 
    $\mathcal{S}_0=\mathcal{S}_n \cap \mathcal{S}^{\prime}$. Then $S_0$ is the set of discontinuous points that cause the optimal action to change.  
    And $\mu(\mathcal{S}_0)=\mu(\mathcal{S}_n \cap \mathcal{S}^{\prime})=0$ because $\mu(\mathcal{S}^{\prime})=0$. 
    
    For any $s\in \mathcal{S}_{nin} = \left\{ s\in\mathcal{S} |  B_\epsilon(s) \neq B^*_\epsilon(s) \right\}$, we have the following two cases.

    \textbf{Case 1}. $\exists s^\prime \in B_\epsilon(s)$ s.t. $s^\prime \in \mathcal{S}_{nu}$, then $s \in \mathcal{S}_{nu}+B_\epsilon$, i.e. 
    \begin{equation}
        s \in \mathcal{S}_{nu} \cup \mathcal{S}_{0} + B_\epsilon.
        \label{th2_case1}
        \notag
    \end{equation}

    \textbf{Case 2}. $\forall s^\prime \in B_\epsilon(s) , s^\prime \notin \mathcal{S}_{nu}$, which means that  $\arg\max_{a}Q^*(s^\prime,a)$ is a singleton for all $s^\prime\in B_{\epsilon}(s)$. Define $\arg\max_{a}Q^*(s^{\prime},a) = \{a_{s^\prime}^*\} $ for any $s^\prime\in B_\epsilon(s)$. 
     
     Because $s \in \mathcal{S}_{nin}$, there exist a $s^\prime \in B_\epsilon(s)$ such that $a_{s^\prime}^* \neq a_s^*$. Let $s_1$ be the point that closest to $s$  satisfing $a_{s_1}^* \neq a_s^*$, then $s_1 \in B_\epsilon(s)$. We have 
     \begin{equation}
         s_1 \in \mathcal{S}_n.
         \label{th2_case2_1}
     \end{equation}
     
     Otherwise $s_1 \notin \mathcal{S}_n$ means that $\exists \epsilon_1 >0, \forall s^\prime \in B_{\epsilon_1}(s_1)$, $a_{s^\prime}^*=a_{s_1}^*$, then $s_1$ is not the point that closest to $s$  satisfing $a_{s_1}^* \neq a_s^*$, which is a contradiction. We also have 
     \begin{equation}
         s_1\in \mathcal{S}^\prime.
         \label{th2_case2_2}
     \end{equation}
     
     Otherwise $s_1\notin \mathcal{S}^\prime$ means that $\forall a \in \mathcal{A}$, $Q^*(\cdot,a) $ is continuous in $s_1$. First, we have 
     \begin{equation}
         Q^*(s_1,a_{s_1}^*) > Q^*(s_1,a), \forall a \in \mathcal{A} \setminus \{a_{s_1}^*\}. 
         \notag
     \end{equation}

     Then $\exists \ \epsilon_2 >0$, $\forall s \in B_{\epsilon_2}(s_1)$, s.t.  
     \begin{equation}
         Q^*(s,a_{s_1}^*) > Q^*(s,a), \forall a \in \mathcal{A} \setminus \{a_{s_1}^*\}. 
         \notag
     \end{equation}
     because of the continuity of point $s_1$. This contradicts the definition of $s_1$. 

     According to (\ref{th2_case2_1}) and (\ref{th2_case2_2}), we have 
     $s_1\in \mathcal{S}^{\prime} \cap \mathcal{S}_n$ i.e. $s_1 \in \mathcal{S}_0$. Then $s\in \mathcal{S}_0 + B_\epsilon(s)$, i.e.
     \begin{equation}
          s \in \mathcal{S}_{nu} \cup \mathcal{S}_{0} + B_\epsilon.
        \label{th2_case2_3} 
        \notag
     \end{equation}

     Thus 
     \begin{equation}
        \mathcal{S}_{nin} \subseteq \mathcal{S}_{nu} \cup \mathcal{S}_{0} + B_\epsilon.
        \notag
     \end{equation}
     Therefore, the proof of the theorem is concluded.
\end{proof}

\begin{remark}
    In practical complex tasks, we can view $\mathcal{S}_{nu}$ as an empty set.
\end{remark}

\begin{remark}
    Except for the smooth environment, many tasks can be modeled as environments with sparse rewards. Further, the value and action-value functions in these environments are almost everywhere continuous, indicating the set $\mathcal{S}_0$ is a zero-measure set.
\end{remark}

\begin{remark}
    According to the above characteristics, we know that $\mathcal{S}_0$ is a set of special discontinuous points and its elements are rare in practical complex environments. 
\end{remark}

Further, we can get the following corollary in the setting of continuous functions and there are better conclusions.
\begin{corollary}\label{Th1}
    For any MDP $\mathcal{M}$, let $\mathcal{S}_{nu}$ denote the state set where the optimal action is not unique, i.e. $\mathcal{S}_{nu} = \left\{ s\in\mathcal{S} | \mathop{\arg\max}_a Q^*(s, a) \text{ is not a singleton} \right\}$. If $Q^*(\cdot,a)$ is continuous in $\mathcal{S}$ for all $a\in\mathcal{A}$, we have the following conclusions:
        \begin{itemize}
            \item For $s\in \mathcal{S} \setminus \mathcal{S}_{nu}$, there exists $\epsilon > 0$ such that $B_\epsilon(s) = B^*_\epsilon(s)$.
            \item Given $\epsilon > 0$, let $\mathcal{S}_{nin}$ denote the set of states where \textit{the intrinsic state $\epsilon$-neighourhood} is not the same as \textit{the $\epsilon$-neighourhood}, i.e. $\mathcal{S}_{nin} = \left\{ s\in\mathcal{S} |  B_\epsilon(s) \neq B^*_\epsilon(s) \right\}$. Then, we have $\mathcal{S}_{nin} \subseteq \mathcal{S}_{nu} + B_\epsilon = \left\{ s_1 + s_2 | s_1 \in \mathcal{S}_{nu},\ \|s_2\| \le \epsilon \right\}$. Especially, when $\mathcal{S}_{nu}$ is a finite set, we have $\mu\left( \mathcal{S}_{nin} \right) \le \left| \mathcal{S}_{nu} \right| \mu\left( B_\epsilon \right) = C_d \left| \mathcal{S}_{nu} \right| \epsilon^d $, where $C_d$ is a constant with respect to dimension $d$ and norm.
        \end{itemize}
\end{corollary}
    \begin{proof}
        Corollary \ref{Th1} can be derived from Theorem \ref{continuous a.e.}
        because we have the following conclusion in continuous case.
        \begin{equation}
            \mathcal{S}_{0} \subseteq \{ s \in \mathcal{S}  |  \exists a \in \mathcal{A},  \text{s.t. $Q^*(s,a)$  is not continuous at } s \}=\emptyset.
            \notag
        \end{equation}
        Therefore, the theorem is concluded.
    \end{proof}
\begin{remark}
    Certain natural environments show smooth reward function and transition dynamics, especially in continuous control tasks where the transition dynamics come from some physical laws. Further, the value and action-value functions in these environments are continuous, making $\mathcal{S}_0$ an empty set.
\end{remark}

\subsection{Consistent Optimal Robust Policy and CAR Operator}
Define the consistent adversarial robust operator for adversarial action-value function:
\begin{equation} 
    \left( \mathcal{T}_{car} Q \right) (s,a) = r(s,a)  + \gamma \mathbb{E}_{ s^\prime \sim \mathbb{P}(\cdot|s,a)} \left[ \min _{s^\prime_\nu \in B_\epsilon(s^\prime)} Q \left(s^\prime,\mathop{\arg\max}_{a_{s^\prime_\nu}} Q\left(s^\prime_\nu, a_{s^\prime_\nu}\right)\right) \right].
    \notag
\end{equation}

\subsubsection{CAR Operator is Not a Contraction}\label{app: not a contraction}

\begin{theorem}
    $\mathcal{T}_{car}$ is \textit{not} a contraction.
\end{theorem}
\begin{proof}
     Let $\mathcal{S}=[-1,1]$, $\mathcal{A}=\{a_1,a_2\}$, $0<\epsilon \ll 1$ and dynamic transition $\mathbb{P}(\cdot|s,a)$ be a determined function. Let $n > \max\{ \frac{\delta}{\gamma},2\delta\}$, $\delta > 0$ and
     \begin{equation}
      \begin{aligned}
         Q_1(s,a_1)&=2n \cdot \mathbbm{1}_{\left\{ s\in \left[ -1,0\right) \right\}}
         +\left[ 2n-\frac{2n-2\delta}{\frac{1}{8}\epsilon}s\right] \cdot \mathbbm{1}_{\left\{ s\in \left[ 0,\frac{1}{8}\epsilon \right) \right\}}
         +2\delta \cdot \mathbbm{1}_{\left\{ s\in \left[ \frac{1}{8}\epsilon,\frac{3}{8}\epsilon \right) \right\} }\\
         &+\left[ 2\delta+ \frac{n-2\delta}{\frac{1}{8}\epsilon}\left(  s-\frac{3\epsilon}{8} \right) \right] \cdot \mathbbm{1}_{\left\{ s\in \left[ \frac{3}{8}\epsilon,\frac{1}{2}\epsilon \right) \right\}}
         +n \cdot \mathbbm{1}_{\left\{ s\in \left[ \frac{1}{2}\epsilon,1\right] \right\}},
         \notag
     \end{aligned}
    \end{equation}
    
    \begin{equation}
      \begin{aligned}
         Q_1(s,a_2)&=n \cdot \mathbbm{1}_{\left\{ s\in \left[ -1,0\right) \right\}}
         +\left[ n-\frac{n-\delta}{\frac{1}{8}\epsilon}s\right] \cdot \mathbbm{1}_{\left\{ s\in \left[ 0,\frac{1}{8}\epsilon \right) \right\}}
         +\delta \cdot \mathbbm{1}_{\left\{ s\in \left[ \frac{1}{8}\epsilon,\frac{3}{8}\epsilon \right) \right\} }\\
         &+\left[ \delta+ \frac{2n-\delta}{\frac{1}{8}\epsilon}\left(  s-\frac{3\epsilon}{8} \right) \right] \cdot \mathbbm{1}_{\left\{ s\in \left[ \frac{3}{8}\epsilon,\frac{1}{2}\epsilon \right) \right\}}
         +2n \cdot \mathbbm{1}_{\left\{ s\in \left[ \frac{1}{2}\epsilon,1\right] \right\}},
         \notag
     \end{aligned}
    \end{equation}

    \begin{equation}
      \begin{aligned}
         Q_2(s,a_1)&=2n \cdot \mathbbm{1}_{\left\{ s\in \left[ -1,0\right) \right\}}
         +\left[ 2n-\frac{2n-\delta}{\frac{1}{8}\epsilon}s\right] \cdot \mathbbm{1}_{\left\{ s\in \left[ 0,\frac{1}{8}\epsilon \right) \right\}}
         +\delta \cdot \mathbbm{1}_{\left\{ s\in \left[ \frac{1}{8}\epsilon,\frac{3}{8}\epsilon \right) \right\} }\\
         &+\left[ \delta+ \frac{n-\delta}{\frac{1}{8}\epsilon}\left(  s-\frac{3\epsilon}{8} \right) \right] \cdot \mathbbm{1}_{\left\{ s\in \left[ \frac{3}{8}\epsilon,\frac{1}{2}\epsilon \right) \right\}}
         +n \cdot \mathbbm{1}_{\left\{ s\in \left[ \frac{1}{2}\epsilon,1\right] \right\}},
         \notag
     \end{aligned}
    \end{equation}

    \begin{equation}
      \begin{aligned}
         Q_2(s,a_2)&=n \cdot \mathbbm{1}_{\left\{ s\in \left[ -1,0\right) \right\}}
         +\left[ n-\frac{n-2\delta}{\frac{1}{8}\epsilon}s\right] \cdot \mathbbm{1}_{\left\{ s\in \left[ 0,\frac{1}{8}\epsilon \right) \right\}}
         +2\delta \cdot \mathbbm{1}_{\left\{ s\in \left[ \frac{1}{8}\epsilon,\frac{3}{8}\epsilon \right) \right\} }\\
         &+\left[ 2\delta+ \frac{2n-2\delta}{\frac{1}{8}\epsilon}\left(  s-\frac{3\epsilon}{8} \right) \right] \cdot \mathbbm{1}_{\left\{ s\in \left[ \frac{3}{8}\epsilon,\frac{1}{2}\epsilon \right) \right\}}
         +2n \cdot \mathbbm{1}_{\left\{ s\in \left[ \frac{1}{2}\epsilon,1\right] \right\}}.
         \notag
     \end{aligned}
    \end{equation}

     Then 
     \begin{equation}
         \|Q_1-Q_2\|_{L^\infty\left( \mathcal{S}\times\mathcal{A} \right)}=\delta.
         \notag
     \end{equation}
     We have
     \begin{equation}
     \begin{aligned}
         \mathcal{T}_{car} Q_1(s,a) - \mathcal{T}_{car} Q_2(s,a)&= \gamma \mathbb{E}_{ s^\prime \sim \mathbb{P}(\cdot|s,a)} \left[ \min _{s^1_\nu \in B_\epsilon(s^\prime)} Q_1 \left(s^\prime,\mathop{\arg\max}_{a_{s^1_\nu}} Q_1\left(s^1_\nu, a_{s^1_\nu}\right)\right) - \right.\\
         &\left. \min _{s^2_\nu \in B_\epsilon(s^\prime)} Q_2 \left(s^\prime,\mathop{\arg\max}_{a_{s^2_\nu}} Q_2\left(s^2_\nu, a_{s^2_\nu}\right)\right)  \right].
         \notag
     \end{aligned}
     \end{equation}
     Let  $\mathbb{P}(s^\prime=-\frac{\epsilon}{2}|s,a)=1$ and $s^\prime=-\frac{\epsilon}{2}$, then 
     \begin{equation}
         \begin{aligned}
             \min _{s^1_\nu \in B_\epsilon(s^\prime)} Q_1 \left(s^\prime,\mathop{\arg\max}_{a_{s^1_\nu}} Q_1\left(s^1_\nu, a_{s^1_\nu}\right)\right)=Q_1(s^{\prime},a_1),
             \notag
         \end{aligned}
     \end{equation}
     \begin{equation}
         \begin{aligned}
             \min _{s^1_\nu \in B_\epsilon(s^\prime)} Q_2 \left(s^\prime,\mathop{\arg\max}_{a_{s^2_\nu}} Q_2\left(s^2_\nu, a_{s^2_\nu}\right)\right)=Q_2(s^{\prime},a_2).
             \notag
         \end{aligned}
     \end{equation}
     Thus 
     \begin{equation}
     \begin{aligned}
        \mathcal{T}_{car} Q_1(s,a) - \mathcal{T}_{car} Q_2(s,a)=\gamma \left [ Q_1(s^{\prime},a_1)-Q_2(s^{\prime},a_2)\right ]
        =\gamma n
        >\delta,
        \notag
     \end{aligned}
     \end{equation}
     which means that
     \begin{align}
         \|\mathcal{T}_{car} Q_1 - \mathcal{T}_{car} Q_2\|_{L^\infty\left( \mathcal{S}\times\mathcal{A} \right)} >\|Q_1-Q_2\|_{L^\infty\left( \mathcal{S}\times\mathcal{A} \right)}.
         \notag
     \end{align}
     Therefore, $\mathcal{T}_{car}$ is not a contraction. 
\end{proof}

\subsubsection{Fixed Point of CAR Operator} \label{app: fixed point}

\begin{lemma}[Bellman equations for fixed $\pi$ and $\nu$ in SA-MDP, \citet{zhang2020robust}]\label{lem:bellman equations in samdp}
    Given $\pi: \mathcal{S}\rightarrow\Delta(\mathcal{A})$ and $\nu: \mathcal{S}\rightarrow\mathcal{S}$, we have 
    \begin{align*}
        V^{\pi\circ \nu}(s)&= \mathbb{E}_{a\sim \pi\left(\cdot|{\nu\left(s\right)}\right)}  Q^{\pi\circ \nu}(s,a)\\
        &= \mathbb{E}_{a\sim \pi\left(\cdot|{\nu\left(s\right)}\right)} \left[ r(s,a)  + \gamma \mathbb{E}_{ s^\prime \sim \mathbb{P}(\cdot|s,a)} V^{\pi\circ \nu}(s^\prime) \right],\\
        Q^{\pi\circ \nu}(s,a) &= r(s,a)  + \gamma \mathbb{E}_{ s^\prime \sim \mathbb{P}(\cdot|s,a)} V^{\pi\circ \nu}(s^\prime) \\
        &= r(s,a)  + \gamma \mathbb{E}_{ s^\prime \sim \mathbb{P}(\cdot|s,a), a^\prime\sim \pi\left(\cdot|{\nu\left(s^\prime\right)}\right)} Q^{\pi\circ \nu}(s^\prime,a^\prime).
    \end{align*}
\end{lemma}
\begin{lemma}[Bellman equation for strongest adversary $\nu^*$ in SA-MDP, \citet{zhang2020robust}]\label{lem:bellman equation for strongest adversary}
    \begin{equation}
        V^{\pi\circ \nu^*(\pi)}(s) = \min_{\nu(s)\in B(s)} \mathbb{E}_{a\sim \pi\left(\cdot|{\nu\left(s\right)}\right)}Q^{\pi\circ \nu^*(\pi)}(s,a).
        \notag
    \end{equation}
\end{lemma}

\begin{definition}
    Define the linear functional $\mathcal{L}^{\pi\circ \nu}: L^p\left( \mathcal{S}\times\mathcal{A} \right) \rightarrow L^p\left( \mathcal{S}\times\mathcal{A} \right)$ for fixed $\pi$ and $\nu$:
    \begin{equation}
        \left(\mathcal{L}^{\pi\circ \nu}Q\right)(s,a):=\mathbb{E}_{ s^\prime \sim \mathbb{P}(\cdot|s,a), a^\prime\sim \pi\left(\cdot|{\nu\left(s^\prime\right)}\right)} Q(s^\prime,a^\prime).
        \notag
    \end{equation}
\end{definition}
Then, by lemma \ref{lem:bellman equations in samdp}, we have that
\begin{equation}
    Q^{\pi\circ \nu} = r + \gamma \mathcal{L}^{\pi\circ \nu}Q^{\pi\circ \nu}.
    \notag
\end{equation}

\begin{lemma}\label{lem: bounded inverse functional}
    $\mathcal{T}:\mathcal{X} \rightarrow \mathcal{X}$ is a linear functional where $\mathcal{X}$ are normed vector space. If there exists $ m >0$ such that
    \begin{equation}
        \|\mathcal{T}x\| \ge m\|x\| \quad \forall x\in \mathcal{X},
        \notag
    \end{equation}
    then $\mathcal{T}$ has a bounded inverse operator $\mathcal{T}^{-1}$.
\end{lemma}
\begin{proof}
    If  $\mathcal{T}x_1=\mathcal{T}x_2$, then $\mathcal{T}(x_1-x_2)=0$. While $0=\|\mathcal{T}(x_1-x_2)\| \geq m\|x_1-x_2\|$, thus $x_1=x_2$. Then $\mathcal{T}$ is a bijection and thus the inverse operator of $\mathcal{T}$ exists.

    For any $y\in\mathcal{X}$, $\mathcal{T}^{-1}y\in\mathcal{X}$. We have that
    \begin{align}
        \|y\| = \| \mathcal{T} \left( \mathcal{T}^{-1}y \right) \| \ge m \|  \mathcal{T}^{-1}y \|.
        \notag
    \end{align}
    Thus, we attain that
    \begin{equation}
        \|  \mathcal{T}^{-1}y \| \le \frac{1}{m} \|y\|,\quad \forall y\in\mathcal{X},
        \notag
    \end{equation}
    which shows that $\mathcal{T}^{-1}$ is bounded.
\end{proof}

\begin{lemma}\label{lem:invertible lemma}
    $I-\gamma \mathcal{L}^{\pi\circ \nu}$ is invertible and thus we have that
    \begin{equation}
        Q^{\pi\circ \nu} = \left(I-\gamma \mathcal{L}^{\pi\circ \nu}\right)^{-1} r.
        \notag
    \end{equation}
\end{lemma}
\begin{proof}
    Firstly, for all $(s,a)\in \mathcal{S}\times\mathcal{A}$, we have
    \begin{align*}
        \left(\mathcal{L}^{\pi\circ \nu}Q\right)(s,a)&=\mathbb{E}_{ s^\prime \sim \mathbb{P}(\cdot|s,a), a^\prime\sim \pi\left(\cdot|{\nu\left(s^\prime\right)}\right)} Q(s^\prime,a^\prime) \\
        &\le \mathbb{E}_{ s^\prime \sim \mathbb{P}(\cdot|s,a), a^\prime\sim \pi\left(\cdot|{\nu\left(s^\prime\right)}\right)} \left\| Q\right\|_{L^\infty\left( \mathcal{S}\times\mathcal{A} \right)} \\
        &= \left\| Q\right\|_{L^\infty\left( \mathcal{S}\times\mathcal{A} \right)}
    \end{align*}
    Thus, we have that
    \begin{equation}\label{eq:nonexpansion of pi nu bellman operator}
        \left\|\mathcal{L}^{\pi\circ \nu} Q \right\|_{L^\infty\left( \mathcal{S}\times\mathcal{A} \right)} \le \left\| Q\right\|_{L^\infty\left( \mathcal{S}\times\mathcal{A} \right)}.
    \end{equation}
    For any $Q\in L^p\left( \mathcal{S}\times\mathcal{A} \right)$, we have
    \begin{align*}
        \left\| \left(I-\gamma \mathcal{L}^{\pi\circ \nu}\right) Q \right\|_{L^\infty\left( \mathcal{S}\times\mathcal{A} \right)} &= \left\| Q-\gamma \mathcal{L}^{\pi\circ \nu} Q \right\|_{L^\infty\left( \mathcal{S}\times\mathcal{A} \right)}\\
        &\ge \left\| Q\right\|_{L^\infty\left( \mathcal{S}\times\mathcal{A} \right)}-\gamma \left\|\mathcal{L}^{\pi\circ \nu} Q \right\|_{L^\infty\left( \mathcal{S}\times\mathcal{A} \right)} \\
        &\ge \left\| Q\right\|_{L^\infty\left( \mathcal{S}\times\mathcal{A} \right)}-\gamma \left\| Q\right\|_{L^\infty\left( \mathcal{S}\times\mathcal{A} \right)} \\
        &= \left(1-\gamma\right)\left\| Q\right\|_{L^\infty\left( \mathcal{S}\times\mathcal{A} \right)},
    \end{align*}
    where the first inequality comes from the triangle inequality and the second inequality comes from \eqref{eq:nonexpansion of pi nu bellman operator}. Then, according to lemma \ref{lem: bounded inverse functional}, we attain that $I-\gamma \mathcal{L}^{\pi\circ \nu}$ is invertible.
\end{proof}

\begin{lemma} \label{lem:nonegative}
    If $Q>0$ for all $(s,a)\in \mathcal{S}\times\mathcal{A}$, then we have that $\left(I-\gamma \mathcal{L}^{\pi\circ \nu}\right)^{-1} Q > 0$ for all $(s,a)\in \mathcal{S}\times\mathcal{A}$.
\end{lemma}
\begin{proof}
    At first, we have
    \begin{align*}
        &\quad \left(I-\gamma \mathcal{L}^{\pi\circ \nu}\right) \left( \sum_{t=0}^\infty \gamma^t \left(\mathcal{L}^{\pi\circ \nu}\right)^t \right) \\
        &=\sum_{t=0}^\infty \gamma^t \left(\mathcal{L}^{\pi\circ \nu}\right)^t - \sum_{t=1}^\infty \gamma^t \left(\mathcal{L}^{\pi\circ \nu}\right)^t\\
        &= I.
    \end{align*}
    Thus, we get that
    \begin{equation}
        \left(I-\gamma \mathcal{L}^{\pi\circ \nu}\right)^{-1} = \sum_{t=0}^\infty \gamma^t \left(\mathcal{L}^{\pi\circ \nu}\right)^t.
        \notag
    \end{equation}
    If $Q(s,a)>0$ for all $(s,a)\in \mathcal{S}\times\mathcal{A}$, then for all $(s,a)\in \mathcal{S}\times\mathcal{A}$, we have
    \begin{equation}
        \left(\mathcal{L}^{\pi\circ \nu}Q\right)(s,a)=\mathbb{E}_{ s^\prime \sim \mathbb{P}(\cdot|s,a), a^\prime\sim \pi\left(\cdot|{\nu\left(s^\prime\right)}\right)} Q(s^\prime,a^\prime) \ge 0.
        \notag
    \end{equation}
    Further, we have that $\left(\left(\mathcal{L}^{\pi\circ \nu}\right)^k Q\right) (s,a)> 0$ for all $k\in\mathbb{N}$ and $(s,a)\in \mathcal{S}\times\mathcal{A}$. Thus, we have 
    \begin{align*}
        &\quad \left(I-\gamma \mathcal{L}^{\pi\circ \nu}\right)^{-1} Q (s,a) \\
        &=\sum_{t=0}^\infty \gamma^t \left(\left(\mathcal{L}^{\pi\circ \nu}\right)^t Q\right) (s,a) \\
        &>0.
    \end{align*}
    Therefore, the proof of the lemma is concluded.
\end{proof}

\begin{theorem}
    If the optimal adversarial action-value function under the strongest adversary $Q_0(s,a):=\max_\pi \min_\nu Q^{\pi\circ \nu}(s,a)$ exists for all $s\in\mathcal{S}$ and $a\in\mathcal{A}$, then it is the fixed point of CAR operator.
\end{theorem}
\begin{proof}
    Denote $V_0(s) = \max_\pi \min_\nu V^{\pi\circ \nu}(s)$.
    For all $s\in\mathcal{S}$ and $a\in\mathcal{A}$, we have 
    \begin{align*}
        Q_0(s,a) &= \max_\pi \min_\nu Q^{\pi\circ \nu}(s,a) \\
        &= r(s,a)  + \gamma \max_\pi \min_\nu \mathbb{E}_{ s^\prime \sim \mathbb{P}(\cdot|s,a)} V^{\pi\circ \nu}(s^\prime) \\
        &= r(s,a)  + \gamma \mathbb{E}_{ s^\prime \sim \mathbb{P}(\cdot|s,a)} V_0(s^\prime) \\
        &=r(s,a)  + \gamma \mathbb{E}_{ s^\prime \sim \mathbb{P}(\cdot|s,a)} \min_{\nu(s)\in B(s)} \max_\pi \mathbb{E}_{a\sim \pi\left(\cdot|{\nu\left(s\right)}\right)}Q_0(s,a)\\
        &= r(s,a)  + \gamma \mathbb{E}_{ s^\prime \sim \mathbb{P}(\cdot|s,a)} \left[ \min _{s^\prime_\nu \in B(s^\prime)} Q_0 \left(s^\prime,\mathop{\arg\max}_{a_{s^\prime_\nu}} Q_0\left(s^\prime_\nu, a_{s^\prime_\nu}\right)\right) \right]\\
        &= \left( \mathcal{T}_{car} Q \right)(s,a),
    \end{align*}
    where the fourth equation comes from lemma \ref{lem:bellman equation for strongest adversary}. This completes the proof.
\end{proof}

\begin{theorem}\label{app thm: fixed point}
    Within the ISA-MDP, $Q^*$ is the fixed point of the CAR operator. Further, $Q^*$ is the optimal adversarial action-value function under the strongest adversary, i.e. $Q^*(s, a) = \max_\pi \min_\nu Q^{\pi\circ \nu}(s, a)$, for all $s\in\mathcal{S}$ and $a\in\mathcal{A}$.
\end{theorem}
\begin{proof}
    \begin{align*}
        \left(\mathcal{T}_{car} Q^*\right)(s,a) &= r(s,a)  + \gamma \mathbb{E}_{ s^\prime \sim \mathbb{P}(\cdot|s,a)} \left[ \min _{s^\prime_\nu \in B^*(s^\prime)} Q^* \left(s^\prime,\mathop{\arg\max}_{a_{s^\prime_\nu}} Q^*\left(s^\prime_\nu, a_{s^\prime_\nu}\right)\right) \right]\\
        &=  r(s,a)  + \gamma \mathbb{E}_{ s^\prime \sim \mathbb{P}(\cdot|s,a)} \left[ \min _{s^\prime_\nu \in B^*(s^\prime)} \max_{a^\prime} Q^* \left(s^\prime,a^\prime \right) \right]\\
        &= r(s,a)  + \gamma \mathbb{E}_{ s^\prime \sim \mathbb{P}(\cdot|s,a)} \left[ \max_{a^\prime} Q^* \left(s^\prime,a^\prime \right) \right]\\
        &= Q^*(s,a),
    \end{align*}
    where the second equality utilizes the definition of $B^*(s^\prime)$. Thus, $Q^*$ is a fixed point of the CAR operator.
    
    Define $\pi$ and $\nu$ as the following:
    \begin{align}
        \pi(s) &:= \mathop{\arg\max}_a Q^*(s,a), \label{eq:pi}\\ 
        \nu(s) &:= \mathop{\arg\min}_{s_\nu \in B(s)} Q^* \left(s,\mathop{\arg\max}_{a_{s_\nu}} Q^*\left(s_\nu, a_{s_\nu}\right)\right). \label{eq:nu}
    \end{align}
    Then, we have 
    \begin{align*}
        \left( \mathcal{T}_{car} Q^* \right) (s,a) &= r(s,a)  + \gamma \mathbb{E}_{ s^\prime \sim \mathbb{P}(\cdot|s,a)} \left[ \min _{s^\prime_\nu \in B(s^\prime)} Q^* \left(s^\prime,\mathop{\arg\max}_{a_{s^\prime_\nu}} Q^*\left(s^\prime_\nu, a_{s^\prime_\nu}\right)\right) \right] \\
        &= r(s,a)  + \gamma \mathbb{E}_{ s^\prime \sim \mathbb{P}(\cdot|s,a)} \left[  Q^* \left(s^\prime,\mathop{\arg\max}_{a_{\nu(s^\prime)}} Q^*\left(\nu(s^\prime), a_{\nu(s^\prime)}\right)\right) \right] \\
        &= r(s,a)  + \gamma \mathbb{E}_{ s^\prime \sim \mathbb{P}(\cdot|s,a)} \left[  Q^* \left(s^\prime, \pi\left(\nu(s^\prime)\right)\right) \right] \\
        &=  r(s,a) + \gamma\left(\mathcal{L}^{\pi\circ \nu}Q^*\right)(s,a).
    \end{align*}
    Thus, we have 
    \begin{equation}
        Q^* = \left(I-\gamma \mathcal{L}^{\pi\circ \nu}\right)^{-1} r = Q^{\pi\circ \nu},
        \notag
    \end{equation}
    where equations comes from lemma \ref{lem:invertible lemma}. Further, according to the definition of ISA-MDP, we attain $Q^{\pi\circ \nu}(s,a)= Q^{\pi\circ \nu^*(\pi)}$.
    This shows that $Q^*$ is the action-value adversarial function of policy $\pi$ under the strongest adversary $\nu=\nu^*(\pi)$.

    According to the definition of $B^*$, within the ISA-MDP, we have that 
    \begin{equation}\label{eq: consistency}
        \pi(\nu(s)) = \pi(s), \quad \forall s\in\mathcal{S}.
    \end{equation}
    Then, for any stationary policy $\pi^\prime$, we have that
    \begin{equation}
    \begin{aligned}
        &\quad \left[\left(\mathcal{L}^{\pi\circ \nu} - \mathcal{L}^{\pi^\prime\circ \nu^*(\pi^\prime)} \right) Q^{\pi\circ \nu}\right] (s,a)\\
        &= \mathbb{E}_{ s^\prime \sim \mathbb{P}(\cdot|s,a)} \left[  Q^{\pi\circ \nu} \left(s^\prime, \pi\left(\nu(s^\prime)\right)\right) - \mathbb{E}_{  a^\prime\sim \pi^\prime\left(\cdot|{\nu^*\left(s^\prime; \pi^\prime\right)}\right)} Q^{\pi\circ \nu}(s^\prime,a^\prime) \right] \\
        &= \mathbb{E}_{ s^\prime \sim \mathbb{P}(\cdot|s,a)} \left[  Q^{\pi\circ \nu} \left(s^\prime, \pi\left(s^\prime\right)\right) - \mathbb{E}_{  a^\prime\sim \pi^\prime\left(\cdot|{\nu^*\left(s^\prime; \pi^\prime\right)}\right)} Q^{\pi\circ \nu}(s^\prime,a^\prime) \right] \\
        &= \mathbb{E}_{ s^\prime \sim \mathbb{P}(\cdot|s,a),   a^\prime\sim \pi^\prime\left(\cdot|{\nu^*\left(s^\prime; \pi^\prime\right)}\right)} \left[  Q^{\pi\circ \nu} \left(s^\prime, \pi\left(s^\prime\right)\right) - Q^{\pi\circ \nu}(s^\prime,a^\prime) \right] \\
        &\ge 0, \label{eq: nonegative 2}
    \end{aligned}
    \end{equation}
    where the second equality comes from \eqref{eq: consistency} and the last inequality comes from \eqref{eq:pi}.

    Further, we have that
    \begin{align*}
        Q^* - Q^{\pi^\prime\circ \nu^*(\pi^\prime)} &= Q^{\pi\circ \nu} - Q^{\pi^\prime\circ \nu^*(\pi^\prime)} \\
        &= Q^{\pi\circ \nu} - \left(I-\gamma \mathcal{L}^{\pi^\prime\circ \nu^*(\pi^\prime)}\right)^{-1} r \\
        &= Q^{\pi\circ \nu} - \left(I-\gamma \mathcal{L}^{\pi^\prime\circ \nu^*(\pi^\prime)}\right)^{-1} \left(I-\gamma \mathcal{L}^{\pi\circ \nu}\right) Q^{\pi\circ \nu} \\
        &= \left(I-\gamma \mathcal{L}^{\pi^\prime\circ \nu^*(\pi^\prime)}\right)^{-1} \left( \left(I-\gamma \mathcal{L}^{\pi^\prime\circ \nu^*(\pi^\prime)}\right) - \left(I-\gamma \mathcal{L}^{\pi\circ \nu}\right)  \right) Q^{\pi\circ \nu} \\
        &= \gamma \left(I-\gamma \mathcal{L}^{\pi^\prime\circ \nu^*(\pi^\prime)}\right)^{-1} \left( \mathcal{L}^{\pi\circ \nu} - \mathcal{L}^{\pi^\prime\circ \nu^*(\pi^\prime)} \right) Q^{\pi\circ \nu} \\
        &\ge 0,
    \end{align*}
    where the last inequality comes from \eqref{eq: nonegative 2} and lemma \ref{lem:nonegative}. 
    
    Thus, we have that $Q^{\pi\circ \nu} = Q^* \ge Q^{\pi^\prime\circ \nu^*(\pi^\prime)}$ for all policy $\pi^\prime$ which shows that $\pi$ is the optimal robust policy under strongest adversary.
\end{proof}

\begin{corollary}
    Within the ISA-MDP, there exists a deterministic and stationary policy $\pi^*$ which satisfies $V^{\pi^*\circ \nu^*(\pi^*)}(s) \ge V^{\pi\circ \nu^*(\pi)}(s)$ and $Q^{\pi^*\circ \nu^*(\pi^*)}(s, a) \ge Q^{\pi\circ \nu^*(\pi)}(s, a)$ for all $\pi\in\Pi$, $s\in\mathcal{S}$ and $a\in\mathcal{A}$.
\end{corollary}
\begin{proof}
    According to theorem \ref{app thm: fixed point}, we have that $Q^*(s,a) = \max_\pi \min_\nu Q^{\pi\circ \nu}(s,a)$, for all $s\in\mathcal{S}$ and $a\in\mathcal{A}$.
    Define $\pi^*$ and $\nu^*$ as the following:
    \begin{align*}
        \pi^*(s) &:= \mathop{\arg\max}_a Q^*(s,a),\\ 
        \nu^*(s) &:= \mathop{\arg\min}_{s_\nu \in B(s)} Q^* \left(s,\mathop{\arg\max}_{a_{s_\nu}} Q^*\left(s_\nu, a_{s_\nu}\right)\right). 
    \end{align*}
    Then, we have that 
    \begin{align*}
        \left( \mathcal{T}_{car} Q^* \right) (s,a) &= r(s,a)  + \gamma \mathbb{E}_{ s^\prime \sim \mathbb{P}(\cdot|s,a)} \left[ \min _{s^\prime_\nu \in B(s^\prime)} Q^* \left(s^\prime,\mathop{\arg\max}_{a_{s^\prime_\nu}} Q^*\left(s^\prime_\nu, a_{s^\prime_\nu}\right)\right) \right] \\
        &= r(s,a)  + \gamma \mathbb{E}_{ s^\prime \sim \mathbb{P}(\cdot|s,a)} \left[  Q^* \left(s^\prime,\mathop{\arg\max}_{a_{\nu(s^\prime)}} Q^*\left(\nu^*(s^\prime), a_{\nu(s^\prime)}\right)\right) \right] \\
        &= r(s,a)  + \gamma \mathbb{E}_{ s^\prime \sim \mathbb{P}(\cdot|s,a)} \left[  Q^* \left(s^\prime, \pi^*\left(\nu^*(s^\prime)\right)\right) \right] \\
        &=  r(s,a) + \gamma\left(\mathcal{L}^{\pi^*\circ \nu^*}Q^*\right)(s,a).
    \end{align*}
    Thus, we have 
    \begin{equation}
        Q^* = \left(I-\gamma \mathcal{L}^{\pi\circ \nu}\right)^{-1} r = Q^{\pi^*\circ \nu^*},
        \notag
    \end{equation}
    where equations comes from lemma \ref{lem:invertible lemma}. Further, according to the definition of ISA-MDP, we attain $Q^{\pi^*\circ \nu^*}(s,a)= Q^{\pi^*\circ \nu^*(\pi^*)}$.
    This shows that $Q^*$ is the action-value adversarial function of policy $\pi^*$ under the strongest adversary $\nu*=\nu^*(\pi^*)$. Thus, we have that 
    \begin{equation}\label{eq:pi* nu* Q}
        Q^{\pi^*\circ \nu^*(\pi^*)}(s,a) \ge Q^{\pi\circ \nu^*(\pi)}(s,a),\quad \forall s\in\mathcal{S}, a\in\mathcal{A}.  
    \end{equation}
    
    For any policy $\pi$ and $s\in\mathcal{S}$, we have that 
    \begin{align*}
        V^{\pi^*\circ \nu^*(\pi^*)}(s)&= \mathbb{E}_{a\sim \pi^*\left(\cdot|{\nu^*\left(s;\pi^*\right)}\right)}  Q^{\pi^*\circ \nu^*(\pi^*)}(s,a) \\
        &= \max_a Q^{\pi^*\circ \nu^*(\pi^*)}(s,a) \\
        &\ge \mathbb{E}_{a\sim \pi\left(\cdot|{\nu^*\left(s;\pi\right)}\right)}  Q^{\pi^*\circ \nu^*}(s,a) \\
        &\ge \mathbb{E}_{a\sim \pi\left(\cdot|{\nu^*\left(s;\pi\right)}\right)}  Q^{\pi\circ \nu^*(\pi)}(s,a) \\
        &= V^{\pi\circ \nu^*(\pi)}(s),
    \end{align*}
    where the first and last equations come from lemma \ref{lem:bellman equations in samdp} and the last inequality comes from~\eqref{eq:pi* nu* Q}. Therefore, the proof of the corollary is concluded.
\end{proof}

\subsubsection{Convergence of CAR operator}\label{app:convergence}
In this section, we prove a conclusion for convergence of the fixed point iterations of the CAR operator under the $\left(L_r, L_{\mathbb{P}}\right)$-smooth environment assumption.
\begin{definition}[\cite{bukharin2024robust}]
    Let $\mathcal{S} \subseteq \mathbb{R}^d$. We say the environment is $\left(L_r, L_{\mathbb{P}}\right)$-smooth, if the reward function $r: \mathcal{S} \times \mathcal{A} \rightarrow \mathbb{R}$, and the transition dynamics $\mathbb{P}: \mathcal{S} \times \mathcal{A} \rightarrow \Delta\left(\mathcal{S}\right)$ satisfy
    $$
    \left|r(s, a)-r\left(s^{\prime}, a\right)\right| \leq L_r\left\|s-s^{\prime}\right\| \text { and }\left\|\mathbb{P}(\cdot \mid s, a)-\mathbb{P}\left(\cdot \mid s^{\prime}, a\right)\right\|_{L^1\left(\mathcal{S}\right)} \leq L_{\mathbb{P}}\left\|s-s^{\prime}\right\|,
    $$
    for $\left(s, s^{\prime}, a\right) \in \mathcal{S} \times \mathcal{S} \times \mathcal{A}$. $\|\cdot\|$ denotes a metric on $\mathbb{R}^d$. 
    
\end{definition}
The definition is motivated by observations that certain natural environments show smooth reward function and transition dynamics, especially in continuous control tasks where the transition dynamics come from some physical laws.

The following lemma shows that $\mathcal{T}_{car}^{k} Q$ is uniformly bounded.
\begin{lemma} \label{lem: uniform bound}
    Suppose $Q$ and $r$ are uniformly bounded, i.e. $\exists\ M_Q,M_r >0$ such that $\left|Q(s,a)\right| \le M_Q,\ \left|r(s,a)\right| \le M_r\ \forall s\in\mathcal{S}, a\in\mathcal{A}$. Then $\mathcal{T}_{car} Q(\cdot,a)$ is uniformly bounded, i.e.
    \begin{equation}
        \left|\mathcal{T}_{car} Q(s,a) \right| \le C_Q,\ \forall s\in\mathcal{S}, a\in\mathcal{A},
        \notag
    \end{equation}
    where $C_Q = \max\left\{ M_Q, \frac{M_r}{1-\gamma} \right\}$. Further, for any $k\in\mathbb{N}$, $\mathcal{T}_{car}^{k} Q(\cdot,a)$ has the same uniform bound as $\mathcal{T}_{car} Q(\cdot,a)$, i.e.
    \begin{equation}\label{eq: uniform bound}
        \left|\mathcal{T}_{car}^{k} Q(s,a) \right| \le C_Q,\ \forall s\in\mathcal{S}, a\in\mathcal{A}.
    \end{equation}
\end{lemma}
\begin{proof}
    \begin{align*}
        \left|\mathcal{T}_{car} Q(s,a) \right| &= \left|r(s,a)  + \gamma \mathbb{E}_{ s^\prime \sim \mathbb{P}(\cdot|s,a)} \left[ \min _{s^\prime_\nu \in B(s^\prime)} Q \left(s^\prime,\mathop{\arg\max}_{a_{s^\prime_\nu}} Q\left(s^\prime_\nu, a_{s^\prime_\nu}\right)\right) \right]\right| \\
        &\le \left|r(s,a) \right| + \gamma\mathbb{E}_{ s^\prime \sim \mathbb{P}(\cdot|s,a)} \left|\min _{s^\prime_\nu \in B(s^\prime)} Q \left(s^\prime,\mathop{\arg\max}_{a_{s^\prime_\nu}} Q\left(s^\prime_\nu, a_{s^\prime_\nu}\right)\right)\right| \\
        &\le M_r + \gamma M_Q \\
        &\le \max\left\{ M_Q, \frac{M_r}{1-\gamma} \right\}, \qquad \forall s\in\mathcal{S}, a\in\mathcal{A}.
    \end{align*}

    Let $C_Q = \max\left\{ M_Q, \frac{M_r}{1-\gamma} \right\}$. Suppose the inequality (\ref{eq: uniform bound}) holds for $k=n$. Then, for $k=n+1$, we have
    \begin{align*}
        \left|\mathcal{T}_{car}^{n+1} Q(s,a) \right| &= \left|r(s,a)  + \gamma \mathbb{E}_{ s^\prime \sim \mathbb{P}(\cdot|s,a)} \left[ \min _{s^\prime_\nu \in B(s^\prime)} \mathcal{T}_{car}^{n} Q \left(s^\prime,\mathop{\arg\max}_{a_{s^\prime_\nu}} \mathcal{T}_{car}^{n} Q\left(s^\prime_\nu, a_{s^\prime_\nu}\right)\right) \right]\right| \\
        &\le \left|r(s,a) \right| + \gamma\mathbb{E}_{ s^\prime \sim \mathbb{P}(\cdot|s,a)}  \left|\min _{s^\prime_\nu \in B(s^\prime)} \mathcal{T}_{car}^{n} Q \left(s^\prime,\mathop{\arg\max}_{a_{s^\prime_\nu}} \mathcal{T}_{car}^{n} Q\left(s^\prime_\nu, a_{s^\prime_\nu}\right)\right)\right| \\
        &\le M_r + \gamma C_Q \\
        &\le (1-\gamma) C_Q + \gamma C_Q \\
        &= C_Q.
    \end{align*}
    By induction, we have $\left|\mathcal{T}_{car}^{k} Q(s,a) \right| \le C_Q,\ \forall s\in\mathcal{S}, a\in\mathcal{A}, k\in\mathbb{N}$.
\end{proof}

The following lemma shows that $\mathcal{T}_{car}^{k} Q$ is uniformly Lipschitz continuous in the $\left(L_r, L_{\mathbb{P}}\right)$-smooth environment.
\begin{lemma} \label{lem: lip}
    Suppose the environment is $\left(L_r, L_{\mathbb{P}}\right)$-smooth and suppose $Q$ and $r$ are uniformly bounded, i.e. $\exists\ M_Q,M_r >0$ such that $\left|Q(s,a)\right| \le M_Q,\ \left|r(s,a)\right| \le M_r\ \forall s\in\mathcal{S}, a\in\mathcal{A}$. Then $\mathcal{T}_{car} Q(\cdot,a)$ is Lipschitz continuous, i.e.
        \begin{equation}
            \left| \mathcal{T}_{car} Q(s,a) - \mathcal{T}_{car} Q(s^\prime,a) \right| \le L_{\mathcal{T}_{car}} \|s - s^\prime\|,
            \notag
        \end{equation}
        where $L_{\mathcal{T}_{car}} =  L_r + \gamma C_Q L_{\mathbb{P}}$ and $C_Q = \max\left\{ M_Q, \frac{M_r}{1-\gamma} \right\}$. Further, for any $k\in\mathbb{N}$, $\mathcal{T}_{car}^{k} Q(\cdot,a)$ is Lipschitz continuous and has the same Lipschitz constant as $\mathcal{T}_{car} Q(\cdot,a)$, i.e.
        \begin{equation}
            \left| \mathcal{T}_{car}^{k} Q(s,a) - \mathcal{T}_{car}^{k} Q(s^\prime,a) \right| \le L_{\mathcal{T}_{car}} \|s - s^\prime\|.
            \notag
        \end{equation}
\end{lemma}
\begin{proof}
    For all $s_1,s_2 \in \mathcal{S}$, we have
    \begin{align*}
        &\quad \mathcal{T}_{car} Q(s_1,a) - \mathcal{T}_{car} Q(s_2,a) \\
        &= r(s_1,a)  + \gamma \mathbb{E}_{ s^\prime \sim \mathbb{P}(\cdot|s_1,a)} \left[ \min _{s^\prime_\nu \in B(s^\prime)} Q \left(s^\prime,\mathop{\arg\max}_{a_{s^\prime_\nu}} Q\left(s^\prime_\nu, a_{s^\prime_\nu}\right)\right) \right] \\
        &\quad - r(s_2,a) - \gamma \mathbb{E}_{ s^\prime \sim \mathbb{P}(\cdot|s_2,a)} \left[ \min _{s^\prime_\nu \in B(s^\prime)} Q \left(s^\prime,\mathop{\arg\max}_{a_{s^\prime_\nu}} Q\left(s^\prime_\nu, a_{s^\prime_\nu}\right)\right) \right] \\
        &= \left(r(s_1,a) - r(s_2,a)\right) \\
        &\quad + \gamma \int_{s^\prime} \left( \mathbb{P}(s^\prime|s_1,a) - \mathbb{P}(s^\prime|s_2,a)\right) \min _{s^\prime_\nu \in B(s^\prime)} Q \left(s^\prime,\mathop{\arg\max}_{a_{s^\prime_\nu}} Q\left(s^\prime_\nu, a_{s^\prime_\nu}\right)\right) ds^\prime.
    \end{align*}

    Then, we have 
    \begin{align*}
        &\quad \left| \mathcal{T}_{car} Q(s_1,a) - \mathcal{T}_{car} Q(s_2,a)\right| \\
        &\le \left| \left(r(s_1,a) - r(s_2,a)\right) \right| \\
        &\quad + \left| \gamma \int_{s^\prime} \left( \mathbb{P}(s^\prime|s_1,a) - \mathbb{P}(s^\prime|s_2,a)\right) \min _{s^\prime_\nu \in B(s^\prime)} Q \left(s^\prime,\mathop{\arg\max}_{a_{s^\prime_\nu}} Q\left(s^\prime_\nu, a_{s^\prime_\nu}\right)\right) ds^\prime \right| \\    
        &\le L_r \|s_1 - s_2\| \\
        &\quad + \gamma \int_{s^\prime} \left| \mathbb{P}(s^\prime|s_1,a) - \mathbb{P}(s^\prime|s_2,a)\right| \left| \min _{s^\prime_\nu \in B(s^\prime)} Q \left(s^\prime,\mathop{\arg\max}_{a_{s^\prime_\nu}} Q\left(s^\prime_\nu, a_{s^\prime_\nu}\right)\right) \right| ds^\prime\\
        &\le  L_r \|s_1 - s_2\| + \gamma C_Q \int_{s^\prime} \left| \mathbb{P}(s^\prime|s_1,a) - \mathbb{P}(s^\prime|s_2,a)\right| ds^\prime\\
        &\le  L_r \|s_1 - s_2\| + \gamma C_Q L_{\mathbb{P}} \|s_1 - s_2\| \\
        &= \left( L_r + \gamma C_Q L_{\mathbb{P}} \right) \|s_1 - s_2\|.
    \end{align*}
    The second inequality comes from the Lipschitz property of $r$. The third inequality comes from the uniform boundedness of $Q$ and the last inequality utilizes the Lipschitz property of $\mathbb{P}$.

    Note that $\mathcal{T}_{car}^{k}$ and $\mathcal{T}_{car}$ have the same uniform boundedness $C_Q$. Then, due to lemma \ref{lem: uniform bound}, we can extend the above proof to $\mathcal{T}_{car}^{k}$.
\end{proof}

\begin{remark}
    Note that if replace the operator $\mathcal{T}_{car}$ in the Lemma \ref{lem: uniform bound} and Lemma \ref{lem: lip} with Bellman optimality operator $\mathcal{T}_B$, these lemmas still hold.
\end{remark}

The following lemma shows that the fixed point iteration has a property close to contraction.
\begin{lemma} \label{lem: near contraction}
    Suppose $Q$ and $r$ are uniformly bounded, i.e. $\exists\ M_Q,M_r >0$ such that $\left|Q(s,a)\right| \le M_Q,\ \left|r(s,a)\right| \le M_r\ \forall s\in\mathcal{S}, a\in\mathcal{A}$. Let $Q^*$ denote the Bellman optimality Q-function. Within the ISA-MDP, we have
    \begin{equation}
        \|\mathcal{T}_{car} Q - \mathcal{T}_{car} Q^*\|_\infty \le \gamma \left( \| Q - Q^* \|_\infty + 2 \max_s \max_{s_\nu \in B^*(s)} \max_{a} \left| Q \left(s,a\right) - Q \left(s_\nu,a\right) \right| \right).
        \notag
    \end{equation}
    Let $\epsilon$ denote the diameter of $B$, i.e., $\epsilon = \max_{s\in\mathcal{S}}\max_{s_1,s_2\in B(s)} \| s_1 - s_2 \|$.
    Further, if $Q(\cdot, a)$ is $L$-Lipschitz continuous with respect to $s\in\mathcal{S}$, i.e
    \begin{equation}
        \left| Q(s,a) - Q(s^\prime,a) \right| \le L \|s-s^\prime\|,\quad \forall s,s^\prime \in \mathcal{S},\ a\in \mathcal{A},
        \notag
    \end{equation}
    we have
    \begin{equation}
        \|\mathcal{T}_{car} Q - \mathcal{T}_{car} Q^*\|_\infty \le \gamma \| Q - Q^* \|_\infty + 2 \gamma L \epsilon .
        \notag
    \end{equation}
\end{lemma}
\begin{proof}
    Denote $a_{s^\prime_\nu,Q}^*=\mathop{\arg\max}_{a} Q\left(s^\prime_\nu, a\right)$ and $ s^{\prime,*}_\nu = \arg\min _{s^\prime_\nu \in B^*(s^\prime)} Q \left(s^\prime,a_{s^\prime_\nu,Q}^*\right)  $. If $\mathcal{T}_{car} Q > \mathcal{T}_{car} Q^*$, we have
    \begin{align*}
        &\quad \left(\mathcal{T}_{car} Q\right)(s,a) - \left(\mathcal{T}_{car} Q^*\right)(s,a) \\
        &= \gamma \mathbb{E}_{ s^\prime \sim \mathbb{P}(\cdot|s,a)} \left[ \min _{s^\prime_\nu \in B^*(s^\prime)} Q \left(s^\prime,a_{s^\prime_\nu,Q}^*\right) - \min _{s^\prime_\nu \in B^*(s^\prime)} Q^* \left(s^\prime,a_{s^\prime_\nu,Q^*}^*\right) \right]\\
        &= \gamma \mathbb{E}_{ s^\prime \sim \mathbb{P}(\cdot|s,a)} \left[ Q \left(s^\prime,a_{s^{\prime,*}_\nu,Q}^*\right) - Q^* \left(s^\prime,a_{s^\prime,Q^*}^*\right) \right]\\
        &= \gamma \mathbb{E}_{ s^\prime \sim \mathbb{P}(\cdot|s,a)} \left[ Q \left(s^\prime,a_{s^{\prime,*}_\nu,Q}^*\right) - Q^* \left(s^\prime,a_{s^{\prime,*}_\nu,Q}^*\right) + Q^* \left(s^\prime,a_{s^{\prime,*}_\nu,Q}^*\right) - Q^* \left(s^\prime,a_{s^\prime,Q^*}^*\right) \right]\\
        &\le  \gamma \mathbb{E}_{ s^\prime \sim \mathbb{P}(\cdot|s,a)} \left[ Q \left(s^\prime,a_{s^{\prime,*}_\nu,Q}^*\right) - Q^* \left(s^\prime,a_{s^{\prime,*}_\nu,Q}^*\right) \right]\\
        &\le \gamma \mathbb{E}_{ s^\prime \sim \mathbb{P}(\cdot|s,a)} \left[ \max_{a^\prime} \left( Q \left(s^\prime,a^\prime\right) - Q^* \left(s^\prime,a^\prime\right) \right) \right]\\
        &\le \gamma \| Q - Q^* \|_\infty,
    \end{align*}
    where the second equality utilize the definition of $B^*(s^\prime)$ and the first inequality comes from the optimality of $a_{s^\prime,Q^*}^*$.
    If $\mathcal{T}_{car} Q < \mathcal{T}_{car} Q^*$, we have
    \begin{align}
        &\quad \left(\mathcal{T}_{car} Q^*\right)(s,a) - \left(\mathcal{T}_{car} Q\right)(s,a) \notag\\
        &= \gamma \mathbb{E}_{ s^\prime \sim \mathbb{P}(\cdot|s,a)} \left[  \min _{s^\prime_\nu \in B^*(s^\prime)} Q^* \left(s^\prime,a_{s^\prime_\nu,Q^*}^*\right) - \min _{s^\prime_\nu \in B^*(s^\prime)} Q \left(s^\prime,a_{s^\prime_\nu,Q}^*\right)  \right] \notag\\
        &= \gamma \mathbb{E}_{ s^\prime \sim \mathbb{P}(\cdot|s,a)} \left[ Q^* \left(s^\prime,a_{s^\prime,Q^*}^*\right) - Q \left(s^\prime,a_{s^{\prime,*}_\nu,Q}^*\right) \right] \notag\\
        &= \gamma \mathbb{E}_{ s^\prime \sim \mathbb{P}(\cdot|s,a)} \left[ Q^* \left(s^\prime,a_{s^\prime,Q^*}^*\right) - Q\left(s^\prime,a_{s^\prime,Q^*}^*\right)\right] \label{eq: convergence 1}\\
        &\quad + \gamma \mathbb{E}_{ s^\prime \sim \mathbb{P}(\cdot|s,a)}\left[Q\left(s^\prime,a_{s^\prime,Q^*}^*\right) - Q\left( s^{\prime,*}_\nu, a_{s^\prime,Q^*}^* \right) \right] \label{eq: convergence 2}\\
        &\quad + \gamma \mathbb{E}_{ s^\prime \sim \mathbb{P}(\cdot|s,a)}\left[Q\left( s^{\prime,*}_\nu, a_{s^\prime,Q^*}^* \right) - Q \left(s^\prime,a_{s^{\prime,*}_\nu,Q}^*\right) \right]. \label{eq: convergence 3}
    \end{align}
    
    We will separately analyze the items \ref{eq: convergence 1}, \ref{eq: convergence 2} and \ref{eq: convergence 3}. Firstly, we can bound the item \ref{eq: convergence 1} with $\| Q - Q^* \|_\infty$.
    \begin{align*}
        &\quad \mathbb{E}_{ s^\prime \sim \mathbb{P}(\cdot|s,a)} \left[ Q^* \left(s^\prime,a_{s^\prime,Q^*}^*\right) - Q\left(s^\prime,a_{s^\prime,Q^*}^*\right)\right] \\
        & \le \mathbb{E}_{ s^\prime \sim \mathbb{P}(\cdot|s,a)} \left[ \max_{a^\prime} \left( Q \left(s^\prime,a^\prime\right) - Q^* \left(s^\prime,a^\prime\right) \right) \right] \\
        &\le \| Q - Q^* \|_\infty.
    \end{align*}

    For the item \ref{eq: convergence 2}, we have
    \begin{align*}
        &\quad \mathbb{E}_{ s^\prime \sim \mathbb{P}(\cdot|s,a)}\left[Q\left(s^\prime,a_{s^\prime,Q^*}^*\right) - Q\left( s^{\prime,*}_\nu, a_{s^\prime,Q^*}^* \right) \right]  \\
        &\le \mathbb{E}_{ s^\prime \sim \mathbb{P}(\cdot|s,a)} \left[ \max_{a^\prime} \left( Q \left(s^\prime,a^\prime\right) - Q \left(s^{\prime,*}_\nu,a^\prime\right) \right) \right] \\
        &\le \mathbb{E}_{ s^\prime \sim \mathbb{P}(\cdot|s,a)} \left[ \max_{s^\prime_\nu \in B^*(s^\prime)} \max_{a^\prime} \left| Q \left(s^\prime,a^\prime\right) - Q \left(s^\prime_\nu,a^\prime\right) \right| \right] \\
        &\le \max_s \max_{s_\nu \in B^*(s)} \max_{a} \left| Q \left(s,a\right) - Q \left(s_\nu,a\right) \right| .
    \end{align*}

    Due to $a_{s^{\prime,*}_\nu,Q}^*=\mathop{\arg\max}_{a} Q\left(s^{\prime,*}_\nu, a\right)$, we have $Q\left(s^{\prime,*}_\nu, a\right) \le Q\left(s^{\prime,*}_\nu, a_{s^{\prime,*}_\nu,Q}^*\right),\ \forall a$. Then, for the item \ref{eq: convergence 3}, we have
    \begin{align*}
        &\quad \mathbb{E}_{ s^\prime \sim \mathbb{P}(\cdot|s,a)}\left[Q\left( s^{\prime,*}_\nu, a_{s^\prime,Q^*}^* \right) - Q \left(s^\prime,a_{s^{\prime,*}_\nu,Q}^*\right) \right] \\
        &\le \mathbb{E}_{ s^\prime \sim \mathbb{P}(\cdot|s,a)}\left[Q\left(s^{\prime,*}_\nu, a_{s^{\prime,*}_\nu,Q}^*\right) - Q \left(s^\prime,a_{s^{\prime,*}_\nu,Q}^*\right) \right] \\
        &\le \mathbb{E}_{ s^\prime \sim \mathbb{P}(\cdot|s,a)} \left[ \max_{a^\prime} \left| Q \left(s^\prime,a^\prime\right) - Q \left(s^{\prime,*}_\nu,a^\prime\right) \right| \right]\\
        &\le \mathbb{E}_{ s^\prime \sim \mathbb{P}(\cdot|s,a)} \left[ \max_{s^\prime_\nu \in B^*(s^\prime)} \max_{a^\prime} \left| Q \left(s^\prime,a^\prime\right) - Q \left(s^\prime_\nu,a^\prime\right) \right| \right] \\
        &\le \max_s \max_{s_\nu \in B^*(s)} \max_{a} \left| Q \left(s,a\right) - Q \left(s_\nu,a\right) \right| .
    \end{align*}

    Thus, we have 
    \begin{align*}
        &\quad \left(\mathcal{T}_{car} Q^*\right)(s,a) - \left(\mathcal{T}_{car} Q\right)(s,a) \\
        &\le \gamma \left( \| Q - Q^* \|_\infty + 2 \max_s \max_{s_\nu \in B^*(s)} \max_{a} \left| Q \left(s,a\right) - Q \left(s_\nu,a\right) \right| \right).
    \end{align*}
    In a summary, we get 
    \begin{equation}
        \|\mathcal{T}_{car} Q - \mathcal{T}_{car} Q^*\|_\infty \le \gamma \left( \| Q - Q^* \|_\infty + 2 \max_s \max_{s_\nu \in B^*(s)} \max_{a} \left| Q \left(s,a\right) - Q \left(s_\nu,a\right) \right| \right).
        \notag
    \end{equation}

    Further, when $Q(\cdot, a)$ is $L$-Lipschitz continuous, i.e
    \begin{equation}
        \left| Q(s,a) - Q(s^\prime,a) \right| \le L \|s-s^\prime\|,\quad \forall s,s^\prime \in \mathcal{S},\ a\in \mathcal{A},
        \notag
    \end{equation}
    we have 
    \begin{align*}
        &\quad \max_s \max_{s_\nu \in B^*(s)} \max_{a} \left| Q \left(s,a\right) - Q \left(s_\nu,a\right) \right| \\
        &\le  \max_s \max_{s_\nu \in B^*(s)} L \|s-s_\nu\| \\
        &\le L \epsilon.
    \end{align*}
    Then, we have
    \begin{equation}
        \|\mathcal{T}_{car} Q - \mathcal{T}_{car} Q^*\|_\infty \le \gamma \left( \| Q - Q^* \|_\infty + 2 L \epsilon \right).
        \notag
    \end{equation}
    Therefore, the proof of the lemma is concluded.
\end{proof}
\begin{remark}
    We can relax the Lipschitz condition to local Lipschitz continuous in the $B^*(s)$.
\end{remark}

We prove that the fixed point iterations of $\mathcal{T}_{car}$ at least converge to a sub-optimal solution close to $Q^*$ in the $\left(L_r, L_{\mathbb{P}}\right)$-smooth environment.
\begin{theorem}
    Suppose the environment is $\left(L_r, L_{\mathbb{P}}\right)$-smooth and suppose $Q_0$ and $r$ are uniformly bounded, i.e. $\exists\ M_{Q_0},M_r >0$ such that $\left|Q_0(s,a)\right| \le M_{Q_0},\ \left|r(s,a)\right| \le M_r,\ \forall s\in\mathcal{S}, a\in\mathcal{A}$. Let $Q^*$ denote the Bellman optimal Q-function and $Q_{k+1} = \mathcal{T}_{car} Q_{k} = \mathcal{T}_{car}^{k+1} Q_0$ for all $k\in\mathbb{N}$ and let $\epsilon$ denote the diameter of $B$, i.e., $\epsilon = \max_{s\in\mathcal{S}}\max_{s_1,s_2\in B(s)} \| s_1 - s_2 \|$. Within the ISA-MDP, we have that
    \begin{equation}
        \|Q_{k+1} - Q^*\|_\infty \le  \gamma^{k+1} \| Q_0 - Q^*\|_\infty + \gamma^{k+1} D_{Q_0} + \frac{2 \gamma \epsilon }{1-\gamma}L_{\mathcal{T}_{car}},
        \notag
    \end{equation}
    where $D_{Q_0} = 2 \max_{s,a} \max_{s_\nu \in B(s)} \max_{a} \left| Q_0 \left(s,a\right) - Q_0 \left(s_\nu,a\right) \right|$ is a constant relating to the local continuity of initial $Q_0$, $L_{\mathcal{T}_{car}} =  L_r + \gamma C_{Q_0} L_{\mathbb{P}}$ and $C_{Q_0} = \max\left\{ M_{Q_0}, \frac{M_r}{1-\gamma} \right\}$.
\end{theorem}
\begin{proof}
    For any $k\in\mathbb{N}$, we have
    \begin{align*}
        &\quad \|Q_{k+1} - Q^*\|_\infty \\
        &= \|\mathcal{T}_{car}^{k+1} Q_0 - \mathcal{T}_{car}^{k+1} Q^*\|_\infty \\
        &\le \gamma \|\mathcal{T}_{car}^{k} Q_0 - \mathcal{T}_{car}^{k} Q^*\|_\infty  + 2 \gamma  L_{\mathcal{T}_{car}} \epsilon \\
        &\le \gamma \left( \gamma \|\mathcal{T}_{car}^{k-1} Q_0 - \mathcal{T}_{car}^{k-1} Q^*\|_\infty  + 2 \gamma  L_{\mathcal{T}_{car}} \epsilon \right) + 2 \gamma  L_{\mathcal{T}_{car}} \epsilon \\
        &= \gamma^2 \|\mathcal{T}_{car}^{k-1} Q_0 - \mathcal{T}_{car}^{k-1} Q^*\|_\infty + 2 \epsilon L_{\mathcal{T}_{car}}  \sum_{l=1}^{2} \gamma^l \\
        &\le \cdots \\
        &\le \gamma^k \|\mathcal{T}_{car} Q_0 - \mathcal{T}_{car} Q^*\|_\infty + 2 \epsilon L_{\mathcal{T}_{car}}  \sum_{l=1}^{k} \gamma^l \\
        &\le \gamma^{k+1} \| Q_0 - Q^*\|_\infty + 2 \gamma^{k+1} \max_s \max_{s_\nu \in B^*(s)} \max_{a} \left| Q_0 \left(s,a\right) - Q_0 \left(s_\nu,a\right) \right|  + 2 \epsilon L_{\mathcal{T}_{car}}  \sum_{l=1}^{k} \gamma^l \\
        &\le \gamma^{k+1} \| Q_0 - Q^*\|_\infty + 2 \gamma^{k+1} \max_s \max_{s_\nu \in B^*(s)} \max_{a} \left| Q_0 \left(s,a\right) - Q_0 \left(s_\nu,a\right) \right| + \frac{2 \gamma \epsilon }{1-\gamma}L_{\mathcal{T}_{car}}.
    \end{align*}
    The first and second inequalities come from Lemma \ref{lem: lip} and Lemma \ref{lem: near contraction}. The penultimate inequality comes from Lemma \ref{lem: near contraction}. Therefore, the proof of the theorem is concluded.
\end{proof}

\section{Theorems and Proofs of Policy Robustness under Bellman p-error in Action-value Function Space}

\textbf{Banach Space} is a complete normed space $\left(X, \|\cdot\| \right)$, consisting of a vector space $X$ together with a norm $\|\cdot\| :X\rightarrow \mathbb{R}^+$. In this paper, we consider the setting where the continuous state space $\mathcal{S} \subset \mathbb{R}^d$ is a compact set and the action space $\mathcal{A}$ is a finite set. We discuss in the Banach space $\left(L^p\left( \mathcal{S}\times\mathcal{A} \right), \|\cdot\|_{p} \right),\ 1\le p\le\infty$. Define $L^p\left( \mathcal{S}\times\mathcal{A} \right) = \left\{ f| \|f\|_{p} <\infty \right\}$, where $\|f\|_{p} = \left( \int_{\mathcal{S}}\sum_{a\in\mathcal{A}} \left| f(s,a) \right|^p d\mu(s) \right)^{\frac{1}{p}}$ for $\ 1\le p <\infty$, $\mu$ is the measure over $\mathcal{S}$ and $\|f\|_{\infty} = \inf \left\{ M\in\mathbb{R}_{\ge 0}|\left| f(s,a) \right|\le M \text{ for almost every } (s,a)  \right\}$. For simplicity, we refer to this Banach space as $L^p\left( \mathcal{S}\times\mathcal{A} \right)$.

\subsection{Infinity Norm Space is Necessary for Adversarial Robustness in Action-value Function Space} \label{app: infinity is necessary in value function space}
\begin{theorem}
    There exists an MDP instance $\mathcal{M}$ such that the following statements hold. Given a function $Q$ and adversary perturbation budget $\epsilon$, let $\mathcal{S}^Q_{sub}$ denote the set of states where the greedy policy according to $Q$ is suboptimal, i.e. $\mathcal{S}^Q_{sub} = \left\{ s | Q^*(s,\mathop{\arg\max}_a Q(s, a)) < \max_a Q^*(s, a) \right\}$ and let $\mathcal{S}^{Q,\epsilon}_{adv}$ denote the set of states in whose $\epsilon$-neighbourhood there exists the adversarial state, i.e. $\mathcal{S}^{Q,\epsilon}_{adv} = \left\{ s | \exists s_\nu \in B_\epsilon(s),\ \text{s.t. } Q^*(s,\mathop{\arg\max}_a Q(s_\nu, a)) < \max_a Q^*(s, a) \right\}$, where $Q^*$ is the Bellman optimal $Q$-function.
        \begin{itemize}
            \item For any $1\le p<\infty$ and $\delta>0$, there exists a function $Q\in L^p\left( \mathcal{S}\times\mathcal{A} \right)$ satisfying $\|Q-Q^*\|_{L^p\left( \mathcal{S}\times\mathcal{A} \right)} \leq \delta$ such that $m\left(\mathcal{S}^Q_{sub}\right) = O(\delta)$ yet $m\left( \mathcal{S}^{Q,\epsilon}_{adv} \right) =m \left(\mathcal{S}\right)$.
            \item There exists a $\bar{\delta}>0$ such that for any $0< \delta \le \bar{\delta}$, for any function $Q\in L^\infty\left( \mathcal{S}\times\mathcal{A} \right)$ satisfying $\|Q-Q^*\|_{L^\infty\left( \mathcal{S}\times\mathcal{A} \right)} \leq  \delta$, we have that $m\left(\mathcal{S}^Q_{sub}\right) = O(\delta)$ and $m\left( \mathcal{S}^{Q,\epsilon}_{adv} \right) =2 \epsilon + O\left( \delta \right)$.
        \end{itemize}
\end{theorem}
\begin{proof}
   Given a MDP instance $\mathcal{M}$ such that $\mathcal{S}=[-1,1]$, $\mathcal{A}=\{a_1,a_2\}$ and 
\begin{equation}
    \mathbb{P}(s^\prime |s,a_1)=\left\{ \begin{aligned}
        \mathbbm{1}_{\left\{s^\prime=s-\epsilon_1\right\}}, &\quad s\in [-1+\epsilon_1,1]\\
        \mathbbm{1}_{\left\{s^\prime=-1\right\}}, &\quad s\in[-1,-1+\epsilon_1)
    \end{aligned} \right .
    \notag
\end{equation}
\begin{equation}
    \mathbb{P}(s^\prime |s,a_2)=\left\{ \begin{aligned}
        \mathbbm{1}_{\left\{s^\prime=s+\epsilon_1\right\}}, &\quad s\in [-1,1-\epsilon_1]\\
        \mathbbm{1}_{\left\{s^\prime=1\right\}},&\quad s\in (1-\epsilon_1,1]
    \end{aligned} \right.
    \notag
\end{equation}
\begin{equation}
    \begin{aligned}
        r(s,a_1)&=-ks,\\
        r(s,a_2)&=ks.
    \end{aligned} 
    \notag
\end{equation}
where $\mathbb{P}$ is the transition dynamic, $r$ is the reward function, $k>0$, $0<\epsilon_1 \ll 1$ and $\mathbbm{1}_{\left\{\cdot\right\}}$ is the indicator function. Let $\gamma$ be the discount factor.

First, we prove that equation (\ref{eq: optimal policy in Th7}) is the optimal policy.
\begin{equation}
    \pi^{*}(s)=\arg\max_{a}Q^*(s,a)=\left\{ \begin{aligned}
        \{a_2\},s>0\\
        \{a_1\},s<0\\
        \{a_1,a_2\},s=0
    \end{aligned} \right.
    \label{eq: optimal policy in Th7}
\end{equation}

 Define $s_t^{\pi}$ is state rollouted by policy $\pi$ in  time step $t$.
 
 Let $s_0>0$, then 
\begin{itemize}
    \item If $s_t^{\pi^*}=1$ , while $s_t^{\pi} \in [-1,1]$, then $s_t^{\pi^*} \geq |s_t^{\pi}|$ hold for any policy $\pi$.
    \item If $s_t^{\pi^*}<1$ . First, we have 
    \begin{equation}
        0< s_t^{\pi^*}=s_0+t\epsilon_1 <1,
        \notag
    \end{equation}
     \begin{equation}
        -1<s_0-t\epsilon_1 .
        \notag
    \end{equation}
    Then for any policy $\pi$, we have the following equation by  definition of transition,
    \begin{equation}
        s_t^{\pi}=s_0+\sum_{i=1}^{t}x_i \epsilon_1, 
        \notag
    \end{equation}
    where $x_i \in \{-1,1\} , i=1,...,t$. Then 
    \begin{equation}
        s_t^{\pi^*} \geq |s_t^{\pi}|.
        \notag
    \end{equation}
\end{itemize}

Then for any policy $\pi$ ,$s_0>0$ and $t\geq 0$, we have 
\begin{equation}
    s_t^{\pi^*} \geq |s_t^{\pi}|,
    \notag
\end{equation}
and 
\begin{equation}
    \begin{aligned}
        &\quad \pi(a_2|s_t^{\pi})r(s_t^{\pi},a_2)+\pi(a_1|s_t^{\pi})r(s_t^{\pi},a_1)\\
        & \leq \max\{ \pi(a_2|s_t^{\pi})(ks_t^{\pi}) +\pi(a_1|s_t^{\pi})(-ks_t^{\pi}),\pi(a_2|s_t^{\pi})(-ks_t^{\pi}) +\pi(a_1|s_t^{\pi})(ks_t^{\pi})\}\\
        &=\left|ks_t^{\pi}\left(\pi(a_2|s_t^{\pi})-\pi(a_1|s_t^{\pi})\right)\right|\\
        &\leq \left| ks_t^{\pi} \right|\\
        &\leq ks_t^{\pi^{*}}\\
        &= r(s_t^{\pi^{*}},a_2).
        \label{Th7-0}
    \end{aligned}
\end{equation}

Let $s_0$ be the initial state, $\tau=(s_0,...)$ be the trajectory of policy $\pi$. Define $J(\pi,s_0)=\mathbb{E}_{\tau}\sum_{t} \gamma^t r(s_t,a_t) $ is expected reward in initial state $s_0$ about policy $\pi$. Then
\begin{align}
    J(\pi^*,s_0)-J(\pi,s_0)&=\mathbb{E}_{s_t^{\pi^*},a_t\sim \pi^*(\cdot|s_t^{\pi^*}) }\sum_{t} \gamma^t r(s_t^{\pi^*},a_t)- \mathbb{E}_{s_t^{\pi},a_t\sim \pi(\cdot|s_t^{\pi}) }\sum_{t} \gamma^t r(s_t^{\pi},a_t) \notag\\
    \label{Th7-1}
    &=\sum_{t} \gamma^t r(s_t^{\pi^*},a_2)-\sum_{t}  \gamma^t \mathbb{E}_{s_t^{\pi},a_t\sim \pi(\cdot|s_t^{\pi}) } r(s_t^{\pi},a_t) \notag\\
    &=\sum_{t} \gamma^t r(s_t^{\pi^*},a_2)-\sum_{t}  \gamma^t \mathbb{E}_{s_t^{\pi}} \left[\pi(a_2|s_t^{\pi})r(s_t^{\pi},a_2)+\pi(a_1|s_t^{\pi})r(s_t^{\pi},a_1)\right] \notag\\
    &=\sum_{t} \gamma^t \mathbb{E}_{s_t^{\pi}} \left[r(s_t^{\pi^*},a_2)-\left[\pi(a_2|s_t^{\pi})r(s_t^{\pi},a_2)+\pi(a_1|s_t^{\pi})r(s_t^{\pi},a_1)\right] \right]\\
    \label{Th7-2}
    & \geq 0.
\end{align}

For (\ref{Th7-1}), the policy $\pi^*$ and dynamic transition $\mathbb{P}$ are deterministic.  For (\ref{Th7-2}), We use property (\ref{Th7-0}).

Then for $s >0$, we get that the optimal policy is $\pi(\cdot|s)=a_2$.
By symmetry, we can also get that the optimal policy is $\pi(\cdot|s)=a_1$ for $s<0$ and  $a_1$,$a_2$ are also optimal action for $ s=0$. Thus we have proved equation (\ref{eq: optimal policy in Th7}) is the optimal policy.

First,we have the following equation according to (\ref{eq: optimal policy in Th7}) 
\begin{align}
        Q^*(0,a_2)=Q^*(0,a_1) .  
\end{align}

For $s> 0$, we have
\begin{equation}
    \begin{aligned}
         &\quad\ Q^*(s,a_2)\\
         &=ks+\gamma k(s+\epsilon_1)+\gamma^2 k(s+2\epsilon_1)+...+\gamma^{t_s}k(s+t_s\epsilon_1)+\sum_{n=1}^{\infty}\gamma^{t_s+n}k\times 1\\
         &=ks+k \left[ \sum_{t=1}^{t_s}\gamma^t\left(s+t\epsilon_1 \right) + \sum_{t=t_s+1}^{\infty}\gamma^{t}\right],
         \label{eq: Th8 Q(s,a_2)}
    \end{aligned}
\end{equation}
where $s+t_s\epsilon_1 \in (1-\epsilon_1,1]$, i.e. $t_s = \lfloor  \frac{1-s}{\epsilon_1} \rfloor$.

For $s\geq \epsilon_1$, we have 
\begin{equation}
    \begin{aligned}
    &\quad\ Q^*(s,a_1)\\
    &=-ks+\gamma k(s-\epsilon_1)+\gamma^2 ks+\dots+\gamma^{t_s+2}k(s+t_s\epsilon_1)+\sum_{n=1}^{\infty}\gamma^{t_s+2+n}k\times 1\\
       &=-ks+k \left[ \sum_{t=1}^{t_s+2}\gamma^t\left(s+(t-2)\epsilon_1 \right) + \sum_{t=t_s+3}^{\infty}\gamma^{t}\right].
       \label{eq: Th8 Q(s,a_1)_1}
\end{aligned}
\end{equation}

For $0<  s< \epsilon_1$, we have
\begin{equation}
    \begin{aligned}
    &\quad\ Q^*(s,a_1)\\
    &=-ks+\gamma(-k)(s-\epsilon_1)+\gamma^2(-k)(s-2\epsilon_1)+\dots+\gamma^{q_s}(-k)(s-q_s\epsilon_1)+\sum_{n=1}^{\infty}\gamma^{q_s+n}(-k)(-1)\\
    &=-ks+k \left[ \sum_{t=1}^{q_s}\gamma^t\left(t\epsilon_1-s \right) + \sum_{t=q_s+1}^{\infty}\gamma^{t}\right].
    \label{eq: Th8 Q(s,a_1)_2}
\end{aligned}
\end{equation}
where  $s-q_s\epsilon_1 \in [-1,-1+\epsilon_1)$ ,i.e. $q_s = \lfloor  \frac{1+s}{\epsilon_1} \rfloor >t_s$. 

According to (\ref{eq: Th8 Q(s,a_2)}), (\ref{eq: Th8 Q(s,a_1)_1}), (\ref{eq: Th8 Q(s,a_1)_2}) and $q_s>t_s$, we have
\begin{equation}
    Q^*(s,a_2) -Q^*(s,a_1) > 2ks ,s> 0. \label{Qgap1}
\end{equation}

By symmetry, we can also get 
\begin{equation}
    Q^*(s,a_1) -Q^*(s,a_2) > -2ks ,s< 0. \label{Qgap2}
\end{equation}

(1) First, we have 
\begin{equation}
    0 < Q^{*}(s,a) <\sum_{t=0}^{\infty} \gamma^t =\frac{1}{1-\gamma}.
    \label{Th8 Q_bound}
\end{equation}

For any $1\leq p < \infty$ , let $n> \max \left\{\frac{1}{\epsilon},\left(\frac{1}{1-\gamma}\right)^p,\delta^p,\delta^{p-1}\right\}$ , $n\in\mathbb{N}$  and
\begin{align*}
    Q(s,a_2)&=\left\{ \begin{array}{l}
        Q^*(s,a_2)-n^{\frac{1}{p}} ,\quad s\in [\frac{k}{n},\frac{k}{n}+\frac{\delta^p}{n^2}] , k=0,1,\dots ,n-1\\ 
        Q^*(s,a_2),  \quad\text{otherwise}\\
    \end{array}\right.\\
    Q(s,a_1)&=\left\{ \begin{array}{l}
        Q^*(s,a_1)-n^{\frac{1}{p}} ,\quad s\in [-\frac{k+1}{n} ,-\frac{k+1}{n}+\frac{\delta^p}{n^2} ],k=0,1,\dots ,n-1\\ 
        Q^*(s,a_1),  \quad\text{otherwise}\\
    \end{array}\right.
\end{align*}
Then, we have that
\begin{align*}
    &\quad\ \|Q(s,a_1)-Q^*(s,a_1)\|_{L^p\left( \mathcal{S} \right)}\\
    &=\|Q(s,a_2)-Q^*(s,a_2)\|_{L^p\left( \mathcal{S} \right)}\\
    &=\left[n* \frac{\delta^p}{n^2}*\left(n^{\frac{1}{p}}\right)^p\right]^{\frac{1}{p}}\\
    &\leq \delta.
\end{align*}
And 
\begin{align*}
    &\quad\ \|Q(s,a)\|_{L^p\left( \mathcal{S} \right)}\\
    &=\|Q(s,a)-Q^*(s,a)+Q^*(s,a)\|_{L^p\left( \mathcal{S} \right)}\\
    &\leq \|Q(s,a)-Q^*(s,a)\|_{L^p\left( \mathcal{S} \right)}+\|Q^*(s,a)\|_{L^p\left( \mathcal{S} \right)}\\
    &< \infty.
\end{align*}
which means $Q\in L^p\left( \mathcal{S}\times\mathcal{A} \right)$.

We have the following two inequalities because  $n>\left(\frac{1}{1-\gamma}\right)^p$ and (\ref{Th8 Q_bound}),
\begin{equation}
    Q^*(s,a_2)-n^{\frac{1}{p}} <Q^*(s,a_1),  
    \notag
\end{equation}
\begin{equation}   
     Q^*(s,a_1)-n^{\frac{1}{p}} <Q^*(s,a_2).
     \notag
\end{equation}

Then, we have that
\begin{equation}
    \mathcal{S}^Q_{sub}=\bigcup_{k=-n}^{n-1}  \left [\frac{k}{n},\frac{k}{n}+\frac{\delta^p}{n^2}\right]
    \label{eq: ThB1_defination}
\end{equation}
and 
\begin{align}
    m\left( \mathcal{S}^Q_{sub} \right) = 2n*\frac{\delta^p}{n^2} <2\delta= O\left(\delta\right)
    \notag
\end{align}
because of $n>\delta^{p-1}$.

According to  (\ref{eq: ThB1_defination}), the distance between any two adjacent intervals of $\mathcal{S}^Q_{sub}$ is less than $\epsilon$. For any $s\in\mathcal{S}$, $\exists\ k\in \{-n,-n+1,...,n-1\} $ s.t. $s\in [\frac{k}{n},\frac{k+1}{n}]$. Because $n>\frac{1}{\epsilon}$
(i.e. $\frac{1}{n} < \epsilon$), then we have that
$$d(s,\frac{k}{n})<\epsilon, \text{ i.e. } d(s,\mathcal{S}^Q_{sub}) <\epsilon,$$
where $d(\cdot,\cdot)$ is Euclid distance. According to the definition of $\mathcal{S}^{Q,\epsilon}_{adv}$, we have $\mathcal{S}^{Q,\epsilon}_{adv}=\mathcal{S}$ and  
\begin{equation}
    m\left( \mathcal{S}^{Q,\epsilon}_{adv} \right)=m\left( \mathcal{S}\right).
    \notag
\end{equation}

(2) Let $\bar{\delta}\in(0,k]$, for any $0< \delta \le \bar{\delta}$, for any state-action value function $Q\in L^\infty\left( \mathcal{S}\times\mathcal{A} \right)$ satisfying $\|Q-Q^*\|_{L^\infty\left( \mathcal{S}\times\mathcal{A} \right)} \leq  \delta$, we can get the following two inequalities by (\ref{Qgap1}) and (\ref{Qgap2}).
\begin{equation}
    Q(s,a_2)\geq Q^*(s,a_2) -\delta >  Q^*(s,a_1) +\delta \geq Q(s,a_1) ,s \in (\frac{\delta}{k},1],
    \notag
\end{equation}
\begin{equation}
    Q(s,a_1)\geq Q^*(s,a_1) -\delta >  Q^*(s,a_2) +\delta \geq Q(s,a_2) ,s \in [-1,-\frac{\delta}{k}).
    \notag
\end{equation}
Then
\begin{equation}
    m\left( \mathcal{S}^Q_{sub} \right) \leq \frac{2\delta}{k}=O\left( \delta \right),
    \notag
\end{equation}
\begin{equation}
    m\left( \mathcal{S}^{Q,\epsilon}_{adv} \right) \leq \frac{2\delta}{k}+2\epsilon =2\epsilon+ O\left( \delta \right).
    \notag
\end{equation}
Therefore, the proof of the theorem is concluded.
\end{proof}

\subsection{Stability of Bellman Optimality Equations} \label{app: stability of BOE}

We propose the following concept of stability drawing on relevant research in the field of partial differential equations \citep{wang20222}.
\begin{definition}
    Given two Banach spaces $\mathcal{B}_1$ and $\mathcal{B}_2$, if there exist $\delta>0$ and $C>0$ such that for all $Q\in \mathcal{B}_1 \cap \mathcal{B}_2$ satisfying $\|\mathcal{T}Q - Q\|_{\mathcal{B}_1} < \delta$, we have that $\|Q - Q^*\|_{\mathcal{B}_2} < C \|\mathcal{T}Q - Q\|_{\mathcal{B}_1}$, where $Q^*$ is the exact solution of this functional equation. Then, we say that a nonlinear functional equation $\mathcal{T}Q = Q$ is $\left( \mathcal{B}_1, \mathcal{B}_2 \right)$-stable.
\end{definition}

\begin{remark}
    This definition indicates that if $\mathcal{T}Q = Q$ is $\left( \mathcal{B}_1, \mathcal{B}_2 \right)$-stable, then $\|Q - Q^*\|_{\mathcal{B}_2} = O\left(  \|\mathcal{T}Q - Q\|_{\mathcal{B}_1} \right)$, as $ \|\mathcal{T}Q - Q\|_{\mathcal{B}_1} \longrightarrow 0$, $\forall Q\in \mathcal{B}_1 \cap \mathcal{B}_2$.    
\end{remark}

\begin{lemma}\label{lem: basic ineq 1}
    For any functions $f, g: \mathcal{X}\rightarrow \mathbb{R}$, we have
    \begin{equation}
        \max_{x\in\mathcal{X}} f(x) - \max_{x\in\mathcal{X}} g(x)\le \max_{x\in\mathcal{X}} \left( f(x) - g(x) \right).
        \notag
    \end{equation}
\end{lemma}
\begin{proof}
    \begin{equation}
        \max_{x\in\mathcal{X}} f(x) - \max_{x\in\mathcal{X}} g(x) = f(x_f^*) - \max_{x\in\mathcal{X}} g(x) \le f(x_f^*) - g(x_f^*) \le \max_{x\in\mathcal{X}} \left( f(x) - g(x) \right),
        \notag
    \end{equation}
    where $x_f^*$ is the maximizer of function $f$, i.e. $x_f^* = \arg\max_{x\in\mathcal{X}} f(x)$.
\end{proof}

\begin{lemma}\label{lem: basic ineq 2}
    For any functions $f, g: \mathcal{X}\rightarrow \mathbb{R}$, we have
    \begin{equation}
        \left| \max_{x\in\mathcal{X}} \left( f+g \right)(x) - \max_{x\in\mathcal{X}} f(x) \right| \le \max_{x\in\mathcal{X}} \left| g(x) \right|.
        \notag
    \end{equation}
\end{lemma}
\begin{proof}
    If $\max_{x\in\mathcal{X}} \left( f+g \right)(x) \ge \max_{x\in\mathcal{X}} f(x)$, we have
    \begin{align*}
        &\quad \max_{x\in\mathcal{X}} \left( f+g \right)(x) - \max_{x\in\mathcal{X}} f(x) \\
        &\le \max_{x\in\mathcal{X}} f(x) + \max_{x\in\mathcal{X}} g(x) - \max_{x\in\mathcal{X}} f(x) \\
        &= \max_{x\in\mathcal{X}} g(x) \\
        &\le \max_{x\in\mathcal{X}} \left| g(x) \right|.
    \end{align*}
    If $\max_{x\in\mathcal{X}} \left( f+g \right)(x) < \max_{x\in\mathcal{X}} f(x)$, we have
    \begin{equation}
        \max_{x\in\mathcal{X}} f(x) - \max_{x\in\mathcal{X}} \left( f+g \right)(x) \le \max_{x\in\mathcal{X}} \left( -g(x) \right) \le \max_{x\in\mathcal{X}} \left| g(x) \right|,
        \notag
    \end{equation}
    where the first inequality comes from Lemma \ref{lem: basic ineq 1}.
\end{proof}
\begin{theorem}
    For any MDP $\mathcal{M}$, let $C_{\mathbb{P},p}:= \sup_{(s,a)\in\mathcal{S}\times \mathcal{A}} \left\| \mathbb{P}(\cdot \mid s, a) \right\|_{L^{\frac{p}{p-1}}\left( \mathcal{S} \right)}$. Assume $p$ and $q$ satisfy the following conditions:
    \begin{equation}
        C_{\mathbb{P},p}< \frac{1}{\gamma};\quad
        p \ge \max\left\{1, \frac{\log \left( \left| \mathcal{A}\right|\right) + \log \left( \mu\left( \mathcal{S} \right) \right)}{\log \frac{1}{\gamma C_{\mathbb{P},p}} } \right\}; \quad p \le q \le \infty.
        \notag
    \end{equation}
    Then, Bellman optimality equation $\mathcal{T}_B Q = Q$ is $\left(  L^q\left( \mathcal{S}\times\mathcal{A} \right),  L^p\left( \mathcal{S}\times\mathcal{A} \right) \right)$-stable.
\end{theorem}

\begin{proof}
    For any $1 \le p \le q \le \infty$ and $Q \in L^p\left( \mathcal{S}\times\mathcal{A} \right) \cap L^q\left( \mathcal{S}\times\mathcal{A} \right)$, denote that 
    \begin{align*}
        \mathcal{L}_0 Q(s,a)&:=\gamma \mathbb{E}_{s^{\prime} \sim \mathbb{P}(\cdot \mid s, a)}\left[\max _{a^{\prime} \in \mathcal{A}} Q\left(s^{\prime}, a^{\prime}\right)\right],\\
        \mathcal{L} Q &:= \mathcal{T}_B Q - Q = r + \mathcal{L}_0 Q - Q.
    \end{align*}
    
    Let $Q^*$ denote the Bellman optimality Q-function. Note that $\mathcal{T}_B Q^* = Q^* $ and $\mathcal{L} Q^*=0$. Define 
    \begin{align*}
        w &= w_Q := Q - Q^*, \\
        f &= f_Q := \mathcal{L} Q = \mathcal{L} Q - \mathcal{L} Q^*.
    \end{align*}
    Based on the above notations, we have
    \begin{align*}
        f & = \mathcal{L} Q - \mathcal{L} Q^* \\
        &= \mathcal{L}_0 Q - Q - \mathcal{L}_0 Q^* + Q^* \\
        &= \left( \mathcal{L}_0 Q - \mathcal{L}_0 Q^* \right) - \left( Q - Q^* \right) \\
        &= -w + \mathcal{L}_0 \left( Q^* + w \right) - \mathcal{L}_0 Q^*.
    \end{align*}
    Then, we have
    \begin{align*}
        \left| w(s,a) \right| &=  \left|-f + \mathcal{L}_0 \left( Q^* + w \right) - \mathcal{L}_0 Q^* \right| \bigg|_{(s,a)} \\
        &\le \left| f \right| + \left| \mathcal{L}_0 \left( Q^* + w \right) - \mathcal{L}_0 Q^* \right| \bigg|_{(s,a)}.
    \end{align*}
    Thus, we obtain 
    \begin{align*}
        \left\| w \right\|_{L^p\left( \mathcal{S}\times\mathcal{A} \right)} &\le \left\| \left| f \right| + \left| \mathcal{L}_0 \left( Q^* + w \right) - \mathcal{L}_0 Q^* \right| \right\|_{L^p\left( \mathcal{S}\times\mathcal{A} \right)} \\
        &\le \left\| f \right\|_{L^p\left( \mathcal{S}\times\mathcal{A} \right)} + \left\| \mathcal{L}_0 \left( Q^* + w \right) - \mathcal{L}_0 Q^* \right\|_{L^p\left( \mathcal{S}\times\mathcal{A} \right)},
    \end{align*}
    where the last inequality comes from the Minkowski's inequality. In the following, we analyze the relation between $\left\| \mathcal{L}_0 \left( Q^* + w \right) - \mathcal{L}_0 Q^* \right\|_{L^p\left( \mathcal{S}\times\mathcal{A} \right)}$ and $\left\| w \right\|_{L^p\left( \mathcal{S}\times\mathcal{A} \right)}$.
    \begin{align*}
        &\quad \left| \mathcal{L}_0 \left( Q^* + w \right) - \mathcal{L}_0 Q^* \right| \bigg|_{(s,a)} \\
        &= \left| \gamma \mathbb{E}_{s^{\prime} \sim \mathbb{P}(\cdot \mid s, a)}\left[\max _{a^{\prime} \in \mathcal{A}} \left( Q^*\left(s^{\prime}, a^{\prime}\right) + w\left(s^{\prime}, a^{\prime}\right) \right)- \max _{a^{\prime} \in \mathcal{A}} Q^*\left(s^{\prime}, a^{\prime}\right)\right] \right| \\
        &\le \gamma \mathbb{E}_{s^{\prime} \sim \mathbb{P}(\cdot \mid s, a)}\left|\max _{a^{\prime} \in \mathcal{A}} \left( Q^*\left(s^{\prime}, a^{\prime}\right) + w\left(s^{\prime}, a^{\prime}\right) \right)- \max _{a^{\prime} \in \mathcal{A}} Q^*\left(s^{\prime}, a^{\prime}\right)\right| \\
        &\le \gamma \mathbb{E}_{s^{\prime} \sim \mathbb{P}(\cdot \mid s, a)}\left[\max _{a^{\prime} \in \mathcal{A}} \left|  w\left(s^{\prime}, a^{\prime}\right) \right|\right] \\
        &= \gamma \int_{s^\prime } \max _{a^{\prime} \in \mathcal{A}} \left|  w\left(s^{\prime}, a^{\prime}\right) \right| \mathbb{P}(s^\prime \mid s, a) ds^\prime \\
        &\le \gamma \left\| \max _{a \in \mathcal{A}} \left|  w\left(s, a\right) \right| \right\|_{L^p\left( \mathcal{S} \right)} \left( \int_{s^\prime } \left(\mathbb{P}(s^\prime \mid s, a) \right)^{\frac{p}{p-1}} ds^\prime \right)^{1-\frac{1}{p}} \\
        &=  \gamma \left\| \mathbb{P}(\cdot \mid s, a) \right\|_{L^{\frac{p}{p-1}}\left( \mathcal{S} \right)} \left\| \max _{a \in \mathcal{A}} \left|  w\left(s, a\right) \right| \right\|_{L^p\left( \mathcal{S} \right)}. 
    \end{align*}
    where the second inequality comes from Lemma \ref{lem: basic ineq 2} and the last inequality comes from the Holder's inequality. Let $C_{\mathbb{P},p}:= \sup_{(s,a)\in\mathcal{S}\times \mathcal{A}} \left\| \mathbb{P}(\cdot \mid s, a) \right\|_{L^{\frac{p}{p-1}}\left( \mathcal{S} \right)} $. Then, we have 
    \begin{align*}
        &\quad \left\| \mathcal{L}_0 \left( Q^* + w \right) - \mathcal{L}_0 Q^* \right\|_{L^p\left( \mathcal{S}\times\mathcal{A} \right)} \\
        &\le \left( \int_{ \mathcal{S}\times\mathcal{A} } \left( \gamma \left\| \mathbb{P}(\cdot \mid s, a) \right\|_{L^{\frac{p}{p-1}}\left( \mathcal{S} \right)}  \left\| \max _{a \in \mathcal{A}} \left|  w\left(s, a\right) \right| \right\|_{L^p\left( \mathcal{S} \right)}  \right)^p d\mu(s,a) \right)^{\frac{1}{p}}\\
        &\le \left( \int_{ \mathcal{S}\times\mathcal{A} } 1 d\mu(s,a) \right)^{\frac{1}{p}}   \gamma C_{\mathbb{P},p} \left\| \max _{a \in \mathcal{A}} \left|  w\left(s, a\right) \right| \right\|_{L^p\left( \mathcal{S} \right)}  \\
        &= \left( \mu\left( \mathcal{S}\times\mathcal{A} \right) \right)^{\frac{1}{p}} \gamma C_{\mathbb{P},p} \left\| \max _{a \in \mathcal{A}} \left|  w\left(s, a\right) \right| \right\|_{L^p\left( \mathcal{S} \right)} \\
        &= \gamma C_{\mathbb{P},p} \left( \left| \mathcal{A} \right| \mu\left( \mathcal{S} \right) \right)^{\frac{1}{p}} \left\| \max _{a \in \mathcal{A}} \left|  w\left(s, a\right) \right| \right\|_{L^p\left( \mathcal{S} \right)} \\
        &\le \gamma C_{\mathbb{P},p} \left( \left| \mathcal{A} \right| \mu\left( \mathcal{S} \right) \right)^{\frac{1}{p}} \left\|   w  \right\|_{L^p\left( \mathcal{S}\times\mathcal{A} \right)},
    \end{align*}
    where the last inequality comes from $\left\|w\right\|_{l^\infty\left( \mathcal{A}\right)}\le \left\|w\right\|_{l^p\left( \mathcal{A}\right)}$. Thus, when $C_{\mathbb{P},p} < \frac{1}{\gamma}$ and $p\ge \frac{\log \left( \left| \mathcal{A}\right|\right) + \log \left( \mu\left( \mathcal{S} \right) \right)}{\log \frac{1}{\gamma C_{\mathbb{P},p}} }$ and $q\ge p$, we have 
    \begin{equation} \label{eq: stability}
        \left\| w \right\|_{L^p\left( \mathcal{S}\times\mathcal{A} \right)} \le \frac{1}{1-\gamma C_{\mathbb{P},p} \left( \left| \mathcal{A} \right| \mu\left( \mathcal{S} \right) \right)^{\frac{1}{p}}}  \left\| f \right\|_{L^p\left( \mathcal{S}\times\mathcal{A} \right)} \le \frac{\left( \left| \mathcal{A} \right| \mu\left( \mathcal{S} \right) \right)^{\frac{1}{p} - \frac{1}{q}}}{1-\gamma C_{\mathbb{P},p}\left( \left| \mathcal{A} \right| \mu\left( \mathcal{S} \right) \right)^{\frac{1}{p}}}  \left\| f \right\|_{L^q\left( \mathcal{S}\times\mathcal{A} \right)},
    \end{equation}
    where the last inequality comes from $\left\| f \right\|_{L^p\left( \mathcal{S}\times\mathcal{A} \right)} \le  \mu\left( \mathcal{S}\times\mathcal{A} \right)^{\frac{1}{p} - \frac{1}{q}}  \left\| f \right\|_{L^q\left( \mathcal{S}\times\mathcal{A} \right)}$.
\end{proof}
\begin{remark}
    Note that we have proved a stronger conclusion than stability because the equation (\ref{eq: stability}) holds for all $Q$ rather than for $Q$ satisfying $ \|\mathcal{T}Q - Q\|_{\mathcal{B}_1} \longrightarrow 0$.
\end{remark}

\begin{remark}
    When $\mathbb{P}(\cdot \mid s, a)$ is a probability mass function, then we have that $C_{\mathbb{P},p} \le 1 < \frac{1}{\gamma}$ holds for all $1<p\le \infty$. Generally, note that $\lim_{p\rightarrow\infty} C_{\mathbb{P},p} =1$ and as a consequence, when $p$ is large enough, $C_{\mathbb{P},p} < \frac{1}{\gamma}$ holds.
\end{remark}

\subsection{Instability of Bellman Optimality Equations} \label{app: instability of BOE}

\begin{theorem}
    There exists a MDP $\mathcal{M}$ such that Bellman optimality equation $\mathcal{T}_B Q = Q$ is not $\left(  L^p\left( \mathcal{S}\times\mathcal{A} \right),  L^\infty \left( \mathcal{S}\times\mathcal{A} \right) \right)$-stable, for $1 \le p < \infty$.
\end{theorem}
Generally, we have the following theorem.
\begin{theorem}\label{thm: unstable of bellman opt eq}
    There exists an MDP $\mathcal{M}$ such that Bellman optimality equation $\mathcal{T}_B Q = Q$ is not $\left(  L^q\left( \mathcal{S}\times\mathcal{A} \right),  L^p \left( \mathcal{S}\times\mathcal{A} \right) \right)$-stable, for all $1 \le q < p\le \infty$.    
\end{theorem}
\begin{proof}
    In order to show a Bellman optimality equation $\mathcal{T}_B Q = Q$ is not $\left(  L^q\left( \mathcal{S}\times\mathcal{A} \right),  L^p \left( \mathcal{S}\times\mathcal{A} \right) \right)$-stable, it is sufficient and necessary to prove $\forall n\in\mathbb{N}, \forall \delta>0, \exists Q(s,a), \text{ such that }\|\mathcal{T}_BQ-Q\|_{L^q\left( \mathcal{S}\times\mathcal{A} \right)}<\delta, \text{ but }  \|Q-Q^*\|_{L^p\left( \mathcal{S}\times\mathcal{A} \right)} \ge n {\|\mathcal{T}_BQ-Q\|_{L^q\left( \mathcal{S}\times\mathcal{A} \right)}}.$
    
    Define an MDP $\mathcal{M}$ where $\mathcal{S}=[-1,1], \mathcal{A}=\{a_1,a_2\},$
	\begin{equation*}
			\mathbb{P}(s^\prime |s,a_1)=\left\{ \begin{aligned}
				&\mathbbm{1}_{\left\{s^\prime=s-0.1\right\}}, \quad && s\in[-0.9,1] \\
				&\mathbbm{1}_{\left\{s^\prime=s \right\}}    ,      && else 
			\end{aligned}\right.,\quad		
			\mathbb{P}(s^\prime |s,a_2)=\left\{ \begin{aligned}
				&\mathbbm{1}_{\left\{s^\prime=s+0.1\right\}}, \quad && s\in[-1,0.9] \\
				&\mathbbm{1}_{\left\{s^\prime=s \right\}}    ,      && else 
			\end{aligned}\right.,
	\end{equation*}
	$r(s,a_i)=k_is,\ k_2\ge k_1> 0$. The transition function is essentially a deterministic transition dynamic and for convenience, we denote that
        \begin{equation*}
            p(s,a_1)=\left\{ \begin{aligned}
                &s-0.1, \quad && s\in[-0.9,1] \\
                &s     ,      && else 
            \end{aligned}\right.,\quad		
            p(s,a_2)=\left\{ \begin{aligned}
                &s+0.1, \quad && s\in[-1,0.9] \\
                &s     ,      && else 
            \end{aligned}\right..
       \end{equation*}
	
	Let $Q^*(s,a)=Q^{\pi^*}(s,a)$ be the optimal Q-function, where $\pi^*$ is the optimal policy. 
 
 We have $Q^*(s,a_2)\geq Q^*(s,a_1),\ \forall s\ge 0$. To prove this, we define $\bar{\pi}(s)\equiv a_2,\forall s\ge 0$, and thus\begin{equation}\label{equ:barQ}
     Q^{\bar{\pi}}(s,a_2)=\sum_{i=0}^{\infty}\gamma^i r(s_i,a_2),
 \end{equation} where $s_0=s\ge 0,\  s_{i}=p(s_{i-1},a_2)\ge 0,\,i\ge 1.$ Consider Q-function of any policy $\pi$ 
 \begin{equation}\label{equ:Qpi}
     Q^{\pi}(s,\alpha_0)=\sum_{i=0}^{\infty}\gamma^i r(\tilde{s}_i,\pi(\tilde{s}_i)),
 \end{equation}
 where $\pi(\tilde{s}_0)=\alpha_0 \in \mathcal{A},\ \tilde{s}_0=s\ge 0,\ \tilde{s}_{i+1}=p(\tilde{s}_i,\pi(\tilde{s}_i))$ and $r(\tilde{s}_i,\pi(\tilde{s}_i)) = \pi(a_1|\tilde{s}_i) r(\tilde{s}_i,a_1) + \pi(a_2|\tilde{s}_i) r(\tilde{s}_i,a_2)$.
	
 We first notice that all $s_i$ and $\tilde{s}_i$ lie on the grid points $\mathcal{S}\cap \{s + 0.1z:\,z\in \mathbb{Z} \}$, actually, $-\lfloor \frac{1+s}{0.1} \rfloor\le z\le\lfloor \frac{1-s}{0.1} \rfloor$. In the following, we prove $s_{i}\geq \tilde{s}_{i},\ \forall i$. We consider the recursion method and suppose $s_i\geq \tilde{s}_i.$ Then, we have the following two cases. If $s_i\le 0.9$, we obtain
 $$s_{i+1}=s_i+0.1\ge \tilde{s}_i+0.1\ge \tilde{s}_{i+1}.$$
 If $s_i> 0.9$, it follows from $z\le\lfloor \frac{1-s}{0.1} \rfloor$ that $$s_{i+1}=s_i=s+0.1 \lfloor \frac{1-s}{0.1} \rfloor \ge \tilde{s}_{i+1}.$$ 
 Thus, we have $s_{i+1}\geq \tilde{s}_{i+1},\ \forall i$. Note that $s_0=\tilde{s}_0=s$, and by recursion, it can be obtained that $s_{i}\geq \tilde{s}_{i}$ holds for all $i$.
 
 Noticing that the reward $r(s,a)$ is an increasing function in terms of $s$ and satisfies $r(s,a_2)\ge r(s,a_1), \forall s\ge 0,$ we have $$r(s_{i},a_2)\ge r(s_{i},\alpha_{i})\ge r(\tilde{s}_{i},\alpha_{i}),\quad \forall\alpha_{i}\in\mathcal{A},i=0,1,2,\dots ,$$
 where the second inequality is due to $s_i\ge \tilde{s}_i$. As a consequence, \begin{equation}\label{equ:reward}
     r(s_{i},a_2)\ge r(\tilde{s}_i,\pi(\tilde{s}_i)),\quad\forall i=0,1,2,\dots.
 \end{equation}

Combining \eqref{equ:barQ}, \eqref{equ:Qpi}, and \eqref{equ:reward}, we obtain that $Q^{\bar{\pi}}(s,a_2)\ge Q^{\pi}(s,\alpha_0)$. Further, with $\alpha_0=a_2$, we derive $\bar{\pi}(s)={\pi^*}(s)$ on $s>0$. With $\alpha_0=a_1$, we derive $Q^*(s,a_2)=Q^{\bar{\pi}}(s,a_2)\geq Q^*(s,a_1),\ \forall s\ge 0.$
	
We then prove that given $1\le q<p$, $\forall n\in \mathbb{N}, \delta>0$, there exists $Q(s,a)$ with $\|\mathcal{T}_BQ-Q\|_{L^q\left( \mathcal{S}\times\mathcal{A} \right)}\leq \delta$, such that $\|Q-Q^*\|_{L^p\left( \mathcal{S}\times\mathcal{A} \right)}\ge n\|\mathcal{T}_BQ-Q\|_{L^q\left( \mathcal{S}\times\mathcal{A} \right)}$. Let $Q(s,a_1)=Q^*(s,a_1),$ 
\begin{equation*}
	Q(s,a_2)=Q^*(s,a_2)+h\cdot\mathbbm{1}_{(\frac{1}{4}\epsilon, \frac{3}{4}\epsilon)}+\frac{4h}{\epsilon}s\cdot\mathbbm{1}_{(0,\frac{1}{4}\epsilon]}+\left(-\frac{4h}{\epsilon}s+4h\right)\cdot\mathbbm{1}_{[\frac{3}{4}\epsilon,\epsilon)},
\end{equation*}
	where $h>0$, $\epsilon=\min\left\{\left(\frac{\delta}{3h}\right)^q,\left(3n\cdot2^{\frac{1}{p}}\right)^{-\frac{pq}{p-q}}\right\}$ and $\mathbbm{1}_A(s)=\left\{ \begin{aligned}
		&1, &&s\in A\\
		&0, && else
	\end{aligned}\right.$ denotes the indicator function. It can be seen from the definition that \begin{equation}\label{equ:Q-UL}
	    Q^*(s,a_2)\le Q^*(s,a_2)+h\cdot\mathbbm{1}_{(\frac{1}{4}\epsilon, \frac{3}{4}\epsilon)} \leq Q(s,a_2)\leq Q^*(s,a_2)+h\cdot\mathbbm{1}_{(0,\epsilon)}.
	\end{equation}

	We consider the following cases.
	\begin{itemize}
		\item When $s\in (-0.1,-0.1+\epsilon),a=a_2$, 
            \begin{align*}
              \mathcal{T}_BQ(s,a_2) =r(s,a_2)+\gamma\max_{a_i}Q(s+0.1,a_i)=r(s,a_2)+\gamma Q(s+0.1,a_2),
            \end{align*}
              Together with \eqref{equ:Q-UL}, we have
            \begin{align*}
               & Q(s,a_2)=Q^*(s,a_2)= r(s,a_2)+\gamma Q^*(s+0.1,a_2)\\
               \le & \mathcal{T}_BQ(s,a_2) \leq r(s,a_2)+\gamma[Q^*(s+0.1,a_2)+h]\\
               =& Q^*(s,a_2)+h\gamma= Q(s,a_2)+h\gamma,
            \end{align*}
               thus $|\mathcal{T}_BQ(s,a_2)-Q(s,a_2)|\leq h\gamma.$
		\item When $s\in (0,\epsilon),a=a_2$, 
            \begin{align*}                \mathcal{T}_BQ(s,a_2)&=r(s,a_2)+\gamma\max_{a_i}Q(s+0.1,a_i)=r(s,a_2)+\gamma Q(s+0.1,a_2)\\
              &=r(s,a_2)+\gamma Q^*(s+0.1,a_2)=Q^*(s,a_2),
            \end{align*}
               Again from \eqref{equ:Q-UL}, there is
               \begin{align*}
                   |\mathcal{T}_BQ(s,a_2)-Q(s,a_2)|=|Q^*(s,a_2)-Q(s,a_2)|\leq h.
               \end{align*} 
		\item When $s\in (0.1,0.1+\epsilon),a=a_1$, 
            \begin{align*}
              \mathcal{T}_BQ(s,a_1) =r(s,a_1)+\gamma\max_{a_i}Q(s-0.1,a_i)=r(s,a_1)+\gamma Q(s-0.1,a_2),
            \end{align*}
              Utilizing \eqref{equ:Q-UL}, we have
            \begin{align*}
               & Q(s,a_1)=Q^*(s,a_1)= r(s,a_1)+\gamma Q^*(s-0.1,a_2)\\
               \le & \mathcal{T}_BQ(s,a_1) \leq r(s,a_1)+\gamma[Q^*(s-0.1,a_2)+h]\\
               =& Q^*(s,a_1)+h\gamma= Q(s,a_1)+h\gamma,
            \end{align*}
               thus $|\mathcal{T}_BQ(s,a_1)-Q(s,a_1)|\leq h\gamma.$

		\item Otherwise, $$\mathcal{T}_BQ(s,a_i)=r(s,a_i)+\gamma Q\left(p(s,a_i),\pi^*(p(s,a_i))\right)=Q^*(s,a_i),$$
		also note that $Q(s,\cdot)=Q^*(s,\cdot)$ for $s\notin(0,\epsilon)$, thus $$|\mathcal{T}_BQ(s,a)-Q(s,a)|=|Q^*(s,a)-Q(s,a)|=0.$$
	\end{itemize}
From the analysis above, we have
	\begin{equation}\label{equ:TQ-Q}
		\|\mathcal{T}_BQ-Q\|_{L^q\left( \mathcal{S}\times\mathcal{A} \right)}\leq (2h\gamma+h)\epsilon^{\frac{1}{q}}\leq 3h\epsilon^{\frac{1}{q}} \le\delta,
	\end{equation}
and
	\begin{equation}\label{equ:Q-Q*}
		\|Q-Q^*\|_{L^p\left( \mathcal{S}\times\mathcal{A} \right)}\geq\|(Q-Q^*)\mathbbm{1}_{( \frac{1}{4}\epsilon, \frac{3}{4}\epsilon )}\|_{L^p\left( \mathcal{S}\times\mathcal{A} \right)}\geq h\left(\frac{\epsilon}{2}\right)^{\frac{1}{p}} \ge n \|\mathcal{T}_BQ-Q\|_{L^q\left( \mathcal{S}\times\mathcal{A} \right)}.
	\end{equation}
Inequality \eqref{equ:TQ-Q} and \eqref{equ:Q-Q*} come from $\epsilon=\min\left\{\left(\frac{\delta}{3h}\right)^q,\left(3n\cdot2^{\frac{1}{p}}\right)^{-\frac{pq}{p-q}}\right\}$, which prove the desired property.
\end{proof}

\section{Theorems and Proofs of Stability Analysis of DQN}
In practical DQN training, we use the following loss:
\begin{equation}
    \mathcal{L}(\theta)=\frac{1}{\left| \mathcal{B} \right|} \sum_{(s,a,r,s^\prime)\in \mathcal{B}} \left|r + \gamma \max_{a^\prime} Q(s^\prime,a^\prime; \bar{\theta}) - Q(s,a; \theta)  \right|,
    \notag
\end{equation}
where $\mathcal{B}$ represents a batch of transition pairs sampled from the replay buffer and $\bar{\theta}$ is the parameter of target network.

$\mathcal{L}(\theta)$ is a approximation of the following objective:
    \begin{align*}
        \mathcal{L}(Q;\pi)&=\mathbb{E}_{s\sim d_{\mu_0}^\pi(\cdot)} \mathbb{E}_{a\sim \pi(\cdot|s)} \left| \mathcal{T}_B Q(s,a) - Q(s,a)  \right| \\
        &=\mathbb{E}_{(s,a)\sim d_{\mu_0}^\pi(\cdot,\cdot)} \left| \mathcal{T}_B Q(s,a) - Q(s,a)  \right|,
    \end{align*}
where $d_{\mu_0}^\pi(s)=\mathbb{E}_{s_0\sim \mu_0} \left[ (1-\gamma) \sum_{t=0}^{\infty} \gamma^t \operatorname{Pr}^\pi(s_t=s|s_0) \right]$ is the state visitation distribution
and $d_{\mu_0}^\pi(s,a)=\mathbb{E}_{s_0\sim \mu_0} \left[ (1-\gamma) \sum_{t=0}^{\infty} \gamma^t \operatorname{Pr}^\pi(s_t=s, a_t=a|s_0) \right]$ is the state-action visitation distribution.

\subsection{Definition and Properties of Seminorm} \label{app:seminorm}
\begin{definition}
    Given a policy $\pi$, for any function $f:\mathcal{S}\times\mathcal{A}\rightarrow \mathbb{R}$ and $1\le p\le\infty$, we define the seminorm $L^{p,d_{\mu_0}^\pi}$.
        \begin{itemize}
            \item If $d_{\mu_0}^\pi$ is a probability density function, we define
            \begin{equation}
                \begin{aligned}
                    \left\| f \right\|_{L^{p,d_{\mu_0}^\pi}\left( \mathcal{S}\times\mathcal{A} \right)} &:= \left\| d_{\mu_0}^\pi f \right\|_{L^{p}\left( \mathcal{S}\times\mathcal{A} \right)} \\
                    &= \left(\int_{(s,a)\in\mathcal{S}\times\mathcal{A}} \left| d_{\mu_0}^\pi(s,a)  f(s,a) \right|^p d \mu(s,a) \right)^{\frac{1}{p}}.
                \end{aligned}
                \notag
            \end{equation}
            \item If $d_{\mu_0}^\pi$ is a probability mass function, we define
            \begin{equation}
                    \left\| f \right\|_{L^{p,d_{\mu_0}^\pi}\left( \mathcal{S}\times\mathcal{A} \right)} := \left(\sum_{(s,a)\in\mathcal{S}\times\mathcal{A}} \left| d_{\mu_0}^\pi(s,a)  f(s,a) \right|^p  \right)^{\frac{1}{p}}.  
                    \notag
            \end{equation}
        \end{itemize}
\end{definition}

\begin{remark}
    Note that $\mathcal{L}(Q;\pi) = \left\|\mathcal{T}_B Q - Q \right\|_{L^{1,d_{\mu_0}^\pi}\left( \mathcal{S}\times\mathcal{A} \right)}$.
\end{remark}
\begin{theorem}
    For any $d_{\mu_0}^\pi(s,a)$ and $1\le p \le \infty$, $L^{p,d_{\mu_0}^\pi}$ is a seminorm.
\end{theorem}
\begin{proof}
    Firstly, we show that $L^{p,d_{\mu_0}^\pi}$ satisfies the absolute homogeneity. For any function $f$ and $\lambda\in\mathbb{R}$, we have
    \begin{align*}
        \left\| \lambda f \right\|_{L^{p,d_{\mu_0}^\pi}\left( \mathcal{S}\times\mathcal{A} \right)} = \left\| d_{\mu_0}^\pi \lambda f \right\|_{L^{p}\left( \mathcal{S}\times\mathcal{A} \right)} = \left|\lambda \right| \left\| d_{\mu_0}^\pi f \right\|_{L^{p}\left( \mathcal{S}\times\mathcal{A} \right)} = \left|\lambda \right|  \left\| f \right\|_{L^{p,d_{\mu_0}^\pi}\left( \mathcal{S}\times\mathcal{A} \right)}.
    \end{align*}

    Next, we show that the triangle inequality holds. For any functions $f$ and $g$, we have
    \begin{align*}
         \left\| f + g \right\|_{L^{p,d_{\mu_0}^\pi}\left( \mathcal{S}\times\mathcal{A} \right)} &= \left\| d_{\mu_0}^\pi (f+g) \right\|_{L^{p}\left( \mathcal{S}\times\mathcal{A} \right)} \\
         &\le \left\| d_{\mu_0}^\pi f \right\|_{L^{p}\left( \mathcal{S}\times\mathcal{A} \right)} + \left\| d_{\mu_0}^\pi g \right\|_{L^{p}\left( \mathcal{S}\times\mathcal{A} \right)} \\
         &= \left\| f \right\|_{L^{p,d_{\mu_0}^\pi}\left( \mathcal{S}\times\mathcal{A} \right)} + \left\| g \right\|_{L^{p,d_{\mu_0}^\pi}\left( \mathcal{S}\times\mathcal{A} \right)},
    \end{align*}
    where the inequality comes from the triangle inequality of $L^{p}\left( \mathcal{S}\times\mathcal{A} \right)$.
\end{proof}
\begin{theorem}
    If $d_{\mu_0}^\pi(s,a)>0$ for almost everywhere $(s,a)\in\mathcal{S}\times\mathcal{A}$, then $L^{p,d_{\mu_0}^\pi}\left( \mathcal{S}\times\mathcal{A} \right) := \left\{f \left| \left\| f \right\|_{L^{p,d_{\mu_0}^\pi}\left( \mathcal{S}\times\mathcal{A} \right)} \le \infty \right. \right\}$ is a Banach space, for $1\le p \le \infty$.
\end{theorem}
\begin{proof}
    Firstly, we show the $L^{p,d_{\mu_0}^\pi}$ is positive definite. If $\left\| f \right\|_{L^{p,d_{\mu_0}^\pi}\left( \mathcal{S}\times\mathcal{A} \right)} = 0$, we have that $d_{\mu_0}^\pi(s,a) f(s,a) = 0$, for almost everywhere $(s,a) \in \mathcal{S}\times\mathcal{A}$. Due to the nonnegativity of $d_{\mu_0}^\pi(s,a)$, we have $ f(s,a) = 0$, for almost everywhere $(s,a) \in \mathcal{S}\times\mathcal{A}$.

    We show the completeness of $L^{p,d_{\mu_0}^\pi}\left( \mathcal{S}\times\mathcal{A} \right)$ in the following. For any Cauchy sequence $\left\{ f_i \right\} \subset L^{p,d_{\mu_0}^\pi}\left( \mathcal{S}\times\mathcal{A} \right)$, $\left\{ d_{\mu_0}^\pi f_i \right\}$ is a Cauchy sequence in $L^{p}\left( \mathcal{S}\times\mathcal{A} \right)$. Then, due to the completeness of $L^{p}\left( \mathcal{S}\times\mathcal{A} \right)$, there exists $g\in L^{p}\left( \mathcal{S}\times\mathcal{A} \right)$ such that
    $$\lim_{i\rightarrow \infty} \left\| d_{\mu_0}^\pi f_i \right\|_{L^{p}\left( \mathcal{S}\times\mathcal{A} \right)} = \left\| g \right\|_{L^{p}\left( \mathcal{S}\times\mathcal{A} \right)}.$$
    Let $f = \frac{g}{d_{\mu_0}^\pi}$ for almost everywhere $(s,a)\in\mathcal{S}\times\mathcal{A}$, we have $\left\| f \right\|_{L^{p,d_{\mu_0}^\pi}\left( \mathcal{S}\times\mathcal{A} \right)} =  \left\| g \right\|_{L^{p}\left( \mathcal{S}\times\mathcal{A} \right)} \le \infty$ and 
    $$ \lim_{i\rightarrow \infty}\left\| f_i \right\|_{L^{p,d_{\mu_0}^\pi}\left( \mathcal{S}\times\mathcal{A} \right)} = \left\| f \right\|_{L^{p,d_{\mu_0}^\pi}\left( \mathcal{S}\times\mathcal{A} \right)} .$$
    Thus, $L^{p,d_{\mu_0}^\pi}\left( \mathcal{S}\times\mathcal{A} \right)$ is a Banach space.
\end{proof}
We analyze the properties of $L^{p,d_{\mu_0}^\pi}\left( \mathcal{S}\times\mathcal{A} \right)$ in the following lemma.
\begin{lemma}\label{lem: norm property}
    Given a policy $\pi$, for any function $f:\mathcal{S}\times\mathcal{A}\rightarrow \mathbb{R}$, then we have the following properties.
    \begin{itemize}
        \item If $M_{d_{\mu_0}^\pi}:=\sup_{(s,a)\in\mathcal{S}\times\mathcal{A}} d_{\mu_0}^\pi(s,a) <\infty$, then 
        $$ \left\| f \right\|_{L^{p,d_{\mu_0}^\pi}\left( \mathcal{S}\times\mathcal{A} \right)} \le M_{d_{\mu_0}^\pi} \left\|  f \right\|_{L^{p}\left( \mathcal{S}\times\mathcal{A} \right)} ,\ \forall 1\le p\le\infty.$$
        \item If $C_{d_{\mu_0}^\pi}:=\inf_{(s,a)\in\mathcal{S}\times\mathcal{A}} d_{\mu_0}^\pi(s,a)  >0$, then we have 
        $$  C_{d_{\mu_0}^\pi} \left\|  f \right\|_{L^{p}\left( \mathcal{S}\times\mathcal{A} \right)} \le \left\| f \right\|_{L^{p,d_{\mu_0}^\pi}\left( \mathcal{S}\times\mathcal{A} \right)} ,\ \forall 1\le p\le\infty.$$
        \item $\left\| f \right\|_{L^{1,d_{\mu_0}^\pi}\left( \mathcal{S}\times\mathcal{A} \right)}\le \left\| d_{\mu_0}^\pi \right\|_{L^{\frac{p}{p-1}}\left( \mathcal{S}\times\mathcal{A} \right)} \left\|  f \right\|_{L^{p}\left( \mathcal{S}\times\mathcal{A} \right)},\ \forall 1\le p\le\infty$.
        \item If $d_{\mu_0}^\pi(s,a) \neq 0$ for almost everywhere $(s,a)\in\mathcal{S}\times\mathcal{A}$, then we have 
        $$ \left\|  f \right\|_{L^{1}\left( \mathcal{S}\times\mathcal{A} \right)} \le C_{d_{\mu_0}^\pi,p} \left\| f \right\|_{L^{p,d_{\mu_0}^\pi}\left( \mathcal{S}\times\mathcal{A} \right)},\ \forall 1< p<\infty,$$ 
        where $C_{d_{\mu_0}^\pi,p}=\left( \int_{(s,a)\in\mathcal{S}\times\mathcal{A}} \left| d_{\mu_0}^\pi(s,a)  \right|^{-\frac{p}{p-1}} d \mu(s,a) \right)^{\frac{p-1}{p}}$.
        \item If $d_{\mu_0}^\pi(s,a) \neq 0$ for almost everywhere $(s,a)\in\mathcal{S}\times\mathcal{A}$, then we have 
        $$ \left\|  f \right\|_{L^{p}\left( \mathcal{S}\times\mathcal{A} \right)} \le C_{d_{\mu_0}^\pi,p} \left\| f \right\|_{L^{p^2,d_{\mu_0}^\pi}\left( \mathcal{S}\times\mathcal{A} \right)},\ \forall 1< p<\infty,$$ 
        where $C_{d_{\mu_0}^\pi,p}=\left( \int_{(s,a)\in\mathcal{S}\times\mathcal{A}} \left| d_{\mu_0}^\pi(s,a)  \right|^{-\frac{p^2}{p-1}} d \mu(s,a) \right)^{\frac{p-1}{p^2}}$.
    \end{itemize}
\end{lemma}
\begin{proof}
(1)
    If $M_{d_{\mu_0}^\pi}:=\sup_{(s,a)\in\mathcal{S}\times\mathcal{A}} d_{\mu_0}^\pi(s,a) <\infty$, we have
    \begin{align*}
        \left\| f \right\|_{L^{p,d_{\mu_0}^\pi}\left( \mathcal{S}\times\mathcal{A} \right)} &= \left(\int_{(s,a)\in\mathcal{S}\times\mathcal{A}} \left| d_{\mu_0}^\pi(s,a) f(s,a) \right|^p d \mu(s,a) \right)^{\frac{1}{p}} \\
        &\le M_{d_{\mu_0}^\pi} \left(\int_{(s,a)\in\mathcal{S}\times\mathcal{A}} \left| f(s,a) \right|^p d \mu(s,a) \right)^{\frac{1}{p}} \\
        &=  M_{d_{\mu_0}^\pi} \left\|  f \right\|_{L^{p}\left( \mathcal{S}\times\mathcal{A} \right)}.
    \end{align*}
    
(2)
    If $C_{d_{\mu_0}^\pi}:=\inf_{(s,a)\in\mathcal{S}\times\mathcal{A}} d_{\mu_0}^\pi(s,a) >0$, we have
    \begin{align*}
        \left\| f \right\|_{L^{p,d_{\mu_0}^\pi}\left( \mathcal{S}\times\mathcal{A} \right)} &= \left(\int_{(s,a)\in\mathcal{S}\times\mathcal{A}} \left| d_{\mu_0}^\pi(s,a) f(s,a) \right|^p d \mu(s,a) \right)^{\frac{1}{p}} \\
        &\ge C_{d_{\mu_0}^\pi} \left\|  f \right\|_{L^{p}\left( \mathcal{S}\times\mathcal{A} \right)}.
    \end{align*}
    
(3)
    \begin{align*}
        \left\| f \right\|_{L^{1,d_{\mu_0}^\pi}\left( \mathcal{S}\times\mathcal{A} \right)} &= \int_{(s,a)\in\mathcal{S}\times\mathcal{A}} \left| d_{\mu_0}^\pi(s,a) f(s,a) \right| d \mu(s,a) \\
        &\le \left\|  f \right\|_{L^{p}\left( \mathcal{S}\times\mathcal{A} \right)} \left(\int_{(s,a)\in\mathcal{S}\times\mathcal{A}} \left| d_{\mu_0}^\pi(s,a) \right|^{\frac{p}{p-1}} d \mu(s,a) \right)^{1-\frac{1}{p}} \\
        &= \left\| d_{\mu_0}^\pi \right\|_{L^{\frac{p}{p-1}}\left( \mathcal{S}\times\mathcal{A} \right)} \left\|  f \right\|_{L^{p}\left( \mathcal{S}\times\mathcal{A} \right)},
    \end{align*}
    where the first inequality comes from the Holder's inequality.
    
(4)
    For $1<p<\infty$, we have
    \begin{align*}
         \left\| f \right\|_{L^{p,d_{\mu_0}^\pi}\left( \mathcal{S}\times\mathcal{A} \right)}^p &= \int_{(s,a)\in\mathcal{S}\times\mathcal{A}} \left| d_{\mu_0}^\pi(s,a) f(s,a) \right|^p d \mu(s,a)  \\
         &\ge \left\|  f \right\|_{L^{1}\left( \mathcal{S}\times\mathcal{A} \right)}^p \left( \int_{(s,a)\in\mathcal{S}\times\mathcal{A}} \left| d_{\mu_0}^\pi(s,a) \right|^{-\frac{p}{p-1}} d \mu(s,a) \right)^{-(p-1)},
    \end{align*}
    where the inequality comes from reverse Holder's inequality.

(5)
    Further, we have
    \begin{align*}
        \left\| f \right\|_{L^{{p^2},d_{\mu_0}^\pi}\left( \mathcal{S}\times\mathcal{A} \right)}^{p^2} &= \int_{(s,a)\in\mathcal{S}\times\mathcal{A}} \left| d_{\mu_0}^\pi(s,a) f(s,a) \right|^{p^2} d \mu(s,a)  \\
        &\ge \left\|  f \right\|_{L^{p}\left( \mathcal{S}\times\mathcal{A} \right)}^{p^2} \left( \int_{(s,a)\in\mathcal{S}\times\mathcal{A}} \left| d_{\mu_0}^\pi(s,a) \right|^{-\frac{p^2}{p-1}} d \mu(s,a) \right)^{-(p-1)},
    \end{align*}
    where the inequality comes from reverse Holder's inequality.
\end{proof}

\begin{remark}
    Note that in a practical Q-learning scheme, we take the $\epsilon$-greedy policy for exploration and as a result, for any state-action pair $(s, a)$, we can visit it with positive probability, i.e. $d_{\mu_0}^\pi(s,a) >0$. Furthermore, the condition, $d_{\mu_0}^\pi(s,a)\neq 0$ for almost everywhere $(s,a)\in\mathcal{S}\times\mathcal{A}$, always holds.
\end{remark}

\subsection{Stability of DQN: the Good} \label{app: Stability of DQN: the Good}

\begin{theorem}
    For any MDP $\mathcal{M}$ and fixed policy $\pi$, assume $C_{d_{\mu_0}^\pi}:=\inf_{(s,a)\in\mathcal{S}\times\mathcal{A}} d_{\mu_0}^\pi(s,a)  >0$ and let $C_{\mathbb{P},p}:= \sup_{(s,a)\in\mathcal{S}\times \mathcal{A}} \left\| \mathbb{P}(\cdot \mid s, a) \right\|_{L^{\frac{p}{p-1}}\left( \mathcal{S} \right)}$. Assume $p$ and $q$ satisfy the following conditions:
    \begin{equation}
        C_{\mathbb{P},p}< \frac{1}{\gamma};\quad
        p \ge \max\left\{1, \frac{\log \left( \left| \mathcal{A}\right|\right) + \log \left( \mu\left( \mathcal{S} \right) \right)}{\log \frac{1}{\gamma C_{\mathbb{P},p}} } \right\}; \quad p \le q \le \infty.
        \notag
    \end{equation}
    Then, Bellman optimality equation $\mathcal{T}_B Q = Q$ is $\left(  L^{q,d_{\mu_0}^\pi}\left( \mathcal{S}\times\mathcal{A} \right),  L^p\left( \mathcal{S}\times\mathcal{A} \right) \right)$-stable.
\end{theorem}
\begin{proof}
    For any $1 \le p \le q \le \infty$ and $Q \in L^p\left( \mathcal{S}\times\mathcal{A} \right) \cap L^q\left( \mathcal{S}\times\mathcal{A} \right)$, denote that 
    \begin{align*}
        \mathcal{L}_0 Q(s,a)&:=\gamma \mathbb{E}_{s^{\prime} \sim \mathbb{P}(\cdot \mid s, a)}\left[\max _{a^{\prime} \in \mathcal{A}} Q\left(s^{\prime}, a^{\prime}\right)\right],\\
        \mathcal{L} Q &:= \mathcal{T}_B Q - Q = r + \mathcal{L}_0 Q - Q.
    \end{align*}
    
    Let $Q^*$ denote the Bellman optimality Q-function. Note that $\mathcal{T}_B Q^* = Q^* $ and $\mathcal{L} Q^*=0$. Define 
    \begin{align*}
        w &= w_Q := Q - Q^*, \\
        f &= f_Q := \mathcal{L} Q = \mathcal{L} Q - \mathcal{L} Q^*.
    \end{align*}
    Based on the above notations, we have
    \begin{align}
        f & = \mathcal{L} Q - \mathcal{L} Q^* \\
        &= -w + \mathcal{L}_0 \left( Q^* + w \right) - \mathcal{L}_0 Q^*.
    \end{align}
    
    According to the inequality (\ref{eq: stability}), we have that when $C_{\mathbb{P},p} < \frac{1}{\gamma}$ and $p\ge \frac{\log \left( \left| \mathcal{A}\right|\right) + \log \left( \mu\left( \mathcal{S} \right) \right)}{\log \frac{1}{\gamma C_{\mathbb{P},p}} }$ and $q\ge p$, we have 
    \begin{equation} 
        \left\| w \right\|_{L^p\left( \mathcal{S}\times\mathcal{A} \right)} \le \frac{\left( \left| \mathcal{A} \right| \mu\left( \mathcal{S} \right) \right)^{\frac{1}{p} - \frac{1}{q}}}{1-\gamma C_{\mathbb{P},p}\left( \left| \mathcal{A} \right| \mu\left( \mathcal{S} \right) \right)^{\frac{1}{p}}}  \left\| f \right\|_{L^q\left( \mathcal{S}\times\mathcal{A} \right)}.
        \notag
    \end{equation}
    According to Lemma \ref{lem: norm property}, when $C_{d_{\mu_0}^\pi}:=\inf_{(s,a)\in\mathcal{S}\times\mathcal{A}} d_{\mu_0}^\pi(s,a)  >0$, we have
    \begin{equation}
        \left\| w \right\|_{L^p\left( \mathcal{S}\times\mathcal{A} \right)} \le \frac{\left( \left| \mathcal{A} \right| \mu\left( \mathcal{S} \right) \right)^{\frac{1}{p} - \frac{1}{q}}}{ C_{d_{\mu_0}^\pi} \left(1-\gamma C_{\mathbb{P},p}\left( \left| \mathcal{A} \right| \mu\left( \mathcal{S} \right) \right)^{\frac{1}{p}} \right) }  \left\| f \right\|_{L^{q,d_{\mu_0}^\pi}\left( \mathcal{S}\times\mathcal{A} \right)}.
        \notag
    \end{equation}
    Therefore, the proof of the theorem is concluded.
\end{proof}
\begin{remark}
    Note that in a practical Q-learning scheme, we take the $\epsilon$-greedy policy for exploration and as a result, for any state-action pair $(s, a)$, we can visit it with positive probability, and thus the condition $C_{d_{\mu_0}^\pi}>0$ is fulfilled.
\end{remark}
We also demonstrate a theorem with better bound yet stronger conditions. 

\begin{theorem}
    For any MDP $\mathcal{M}$ and fixed policy $\pi$, assume $C_{d_{\mu_0}^\pi}:=\inf_{(s,a)\in\mathcal{S}\times\mathcal{A}} d_{\mu_0}^\pi(s,a)  >0$. Assume $p$, $q$ and $\gamma$ satisfy the following conditions:
    \begin{equation}
        C_{d_{\mu_0}^\pi, \mathbb{P},p}:= \frac{\left\| d_{\mu_0}^\pi \right\|_{L^{p^2}\left( \mathcal{S}\times\mathcal{A} \right)} C_{\mathbb{P},p}}{C_{d_{\mu_0}^\pi}}<\frac{1}{\gamma} ;\quad p\ge \frac{\log \left( \left| \mathcal{A}\right|\right) + \log \left( \mu\left( \mathcal{S} \right) \right)}{\log \frac{1}{\gamma C_{d_{\mu_0}^\pi, \mathbb{P},p}} } - 1; \quad q\ge p^2.
        \notag
    \end{equation}
    Then, Bellman optimality equation $\mathcal{T}_B Q = Q$ is $\left(  L^{q,d_{\mu_0}^\pi}\left( \mathcal{S}\times\mathcal{A} \right),  L^p\left( \mathcal{S}\times\mathcal{A} \right) \right)$-stable.
\end{theorem}
\begin{proof}
    For any $1 \le p \le q \le \infty$ and $Q \in L^p\left( \mathcal{S}\times\mathcal{A} \right) \cap L^q\left( \mathcal{S}\times\mathcal{A} \right)$, denote that 
    \begin{align*}
        \mathcal{L}_0 Q(s,a)&:=\gamma \mathbb{E}_{s^{\prime} \sim \mathbb{P}(\cdot \mid s, a)}\left[\max _{a^{\prime} \in \mathcal{A}} Q\left(s^{\prime}, a^{\prime}\right)\right],\\
        \mathcal{L} Q &:= \mathcal{T}_B Q - Q = r + \mathcal{L}_0 Q - Q.
    \end{align*}
    
    Let $Q^*$ denote the Bellman optimality Q-function. Note that $\mathcal{T}_B Q^* = Q^* $ and $\mathcal{L} Q^*=0$. Define 
    \begin{align*}
        w &= w_Q := Q - Q^*, \\
        f &= f_Q := \mathcal{L} Q = \mathcal{L} Q - \mathcal{L} Q^*.
    \end{align*}
    Based on the above notations, we have
    \begin{align*}
        f & = \mathcal{L} Q - \mathcal{L} Q^* \\
        &= -w + \mathcal{L}_0 \left( Q^* + w \right) - \mathcal{L}_0 Q^*.
    \end{align*}
    Then, we have
    \begin{align*}
        \left| w(s,a) \right| &=  \left|-f + \mathcal{L}_0 \left( Q^* + w \right) - \mathcal{L}_0 Q^* \right| \bigg|_{(s,a)} \\
        &\le \left| f \right| + \left| \mathcal{L}_0 \left( Q^* + w \right) - \mathcal{L}_0 Q^* \right| \bigg|_{(s,a)}.
    \end{align*}
    Thus, we obtain 
    \begin{align*}
        \left\| w \right\|_{L^{p^2,d_{\mu_0}^\pi}\left( \mathcal{S}\times\mathcal{A} \right)} &\le \left\|d_{\mu_0}^\pi \left| f \right| + d_{\mu_0}^\pi \left| \mathcal{L}_0 \left( Q^* + w \right) - \mathcal{L}_0 Q^* \right| \right\|_{L^{p^2}\left( \mathcal{S}\times\mathcal{A} \right)} \\
        &\le \left\| f \right\|_{L^{p^2,d_{\mu_0}^\pi}\left( \mathcal{S}\times\mathcal{A} \right)} + \left\| \mathcal{L}_0 \left( Q^* + w \right) - \mathcal{L}_0 Q^* \right\|_{L^{p^2,d_{\mu_0}^\pi}\left( \mathcal{S}\times\mathcal{A} \right)},
    \end{align*}
    where the last inequality comes from the Minkowski's inequality. Owing to Lemma \ref{lem: norm property}, we have
    \begin{align*}
        \left\| w \right\|_{L^{p^2,d_{\mu_0}^\pi}\left( \mathcal{S}\times\mathcal{A} \right)} \ge \frac{1}{C_{d_{\mu_0}^\pi,p}} \left\|  w \right\|_{L^{p}\left( \mathcal{S}\times\mathcal{A} \right)},
    \end{align*}
    where $C_{d_{\mu_0}^\pi,p}=\left( \int_{(s,a)\in\mathcal{S}\times\mathcal{A}} \left| d_{\mu_0}^\pi(s,a)  \right|^{-\frac{p^2}{p-1}} d \mu(s,a) \right)^{\frac{p-1}{p^2}}$.
    In the following, we analyze the relation between $\left\| \mathcal{L}_0 \left( Q^* + w \right) - \mathcal{L}_0 Q^* \right\|_{L^{p^2,d_{\mu_0}^\pi}\left( \mathcal{S}\times\mathcal{A} \right)}$ and $\left\| w \right\|_{L^p\left( \mathcal{S}\times\mathcal{A} \right)}$.
    \begin{align*}
        &\quad \left| \mathcal{L}_0 \left( Q^* + w \right) - \mathcal{L}_0 Q^* \right| \bigg|_{(s,a)} \\
        &= \left| \gamma \mathbb{E}_{s^{\prime} \sim \mathbb{P}(\cdot \mid s, a)}\left[\max _{a^{\prime} \in \mathcal{A}} \left( Q^*\left(s^{\prime}, a^{\prime}\right) + w\left(s^{\prime}, a^{\prime}\right) \right)- \max _{a^{\prime} \in \mathcal{A}} Q^*\left(s^{\prime}, a^{\prime}\right)\right] \right| \\
        &\le \gamma \mathbb{E}_{s^{\prime} \sim \mathbb{P}(\cdot \mid s, a)}\left|\max _{a^{\prime} \in \mathcal{A}} \left( Q^*\left(s^{\prime}, a^{\prime}\right) + w\left(s^{\prime}, a^{\prime}\right) \right)- \max _{a^{\prime} \in \mathcal{A}} Q^*\left(s^{\prime}, a^{\prime}\right)\right| \\
        &\le \gamma \mathbb{E}_{s^{\prime} \sim \mathbb{P}(\cdot \mid s, a)}\left[\max _{a^{\prime} \in \mathcal{A}} \left|  w\left(s^{\prime}, a^{\prime}\right) \right|\right] \\
        &= \gamma \int_{s^\prime } \max _{a^{\prime} \in \mathcal{A}} \left|  w\left(s^{\prime}, a^{\prime}\right) \right| \mathbb{P}(s^\prime \mid s, a) ds^\prime \\
        &\le \gamma \left\| \max _{a \in \mathcal{A}} \left|  w\left(s, a\right) \right| \right\|_{L^p\left( \mathcal{S} \right)} \left( \int_{s^\prime } \left(\mathbb{P}(s^\prime \mid s, a) \right)^{\frac{p}{p-1}} ds^\prime \right)^{1-\frac{1}{p}} \\
        &= \gamma \left\| \mathbb{P}(\cdot \mid s, a) \right\|_{L^{\frac{p}{p-1}}\left( \mathcal{S} \right)} \left\| \max _{a \in \mathcal{A}} \left|  w\left(s, a\right) \right| \right\|_{L^p\left( \mathcal{S} \right)}, 
    \end{align*}
    where the second inequality comes from Lemma \ref{lem: basic ineq 2} and the third inequality comes from the Holder's inequality. let $C_{\mathbb{P},p}:= \sup_{(s,a)\in\mathcal{S}\times \mathcal{A}} \left\| \mathbb{P}(\cdot \mid s, a) \right\|_{L^{\frac{p}{p-1}}\left( \mathcal{S} \right)}$ and then, we have 
    \begin{align*}
        &\quad \left\| \mathcal{L}_0 \left( Q^* + w \right) - \mathcal{L}_0 Q^* \right\|_{L^{p^2,d_{\mu_0}^\pi}\left( \mathcal{S}\times\mathcal{A} \right)} \\
        &\le \left( \int_{\mathcal{S}\times\mathcal{A} } \left( \gamma \left\| \mathbb{P}(\cdot \mid s, a) \right\|_{L^{\frac{p}{p-1}}\left( \mathcal{S} \right)} \left\| \max _{a \in \mathcal{A}} \left|  w\left(s, a\right) \right| \right\|_{L^p\left( \mathcal{S} \right)} d_{\mu_0}^\pi(s,a) \right)^{p^2} d\mu(s,a) \right)^{\frac{1}{p^2}}\\
        &= \left( \int_{ \mathcal{S}\times\mathcal{A} } \left(d_{\mu_0}^\pi(s,a)\right)^{p^2} d\mu(s,a) \right)^{\frac{1}{p^2}}   \gamma  C_{\mathbb{P},p} \left\| \max _{a \in \mathcal{A}} \left|  w\left(s, a\right) \right| \right\|_{L^p\left( \mathcal{S} \right)}  \\
        &\le \gamma \left\| d_{\mu_0}^\pi \right\|_{L^{p^2}\left( \mathcal{S}\times\mathcal{A} \right)}    C_{\mathbb{P},p} \left\| \max _{a \in \mathcal{A}} \left|  w\left(s, a\right) \right| \right\|_{L^p\left( \mathcal{S} \right)} \\
        &\le \gamma \left\| d_{\mu_0}^\pi \right\|_{L^{p^2}\left( \mathcal{S}\times\mathcal{A} \right)}   C_{\mathbb{P},p} \left\|   w  \right\|_{L^p\left( \mathcal{S}\times\mathcal{A} \right)},
    \end{align*}
    where  the last inequality comes from $\left\|w\right\|_{l^\infty\left( \mathcal{A}\right)}\le \left\|w\right\|_{l^p\left( \mathcal{A}\right)}$. 
    Then, we have that
    \begin{align*}
        \left(\frac{1}{C_{d_{\mu_0}^\pi,p}} - \gamma \left\| d_{\mu_0}^\pi \right\|_{L^{p^2}\left( \mathcal{S}\times\mathcal{A} \right)}   C_{\mathbb{P},p} \right) \left\|   w  \right\|_{L^p\left( \mathcal{S}\times\mathcal{A} \right)} \le \left\| f \right\|_{L^{p^2,d_{\mu_0}^\pi}\left( \mathcal{S}\times\mathcal{A} \right)} .
    \end{align*}
    When $C_{d_{\mu_0}^\pi}:=\inf_{(s,a)\in\mathcal{S}\times\mathcal{A}} d_{\mu_0}^\pi(s,a)  >0$, we have 
    \begin{align*}
        C_{d_{\mu_0}^\pi,p} \le \frac{\left( \left| \mathcal{A} \right| \mu\left( \mathcal{S} \right) \right)^{\frac{p-1}{p^2}}}{C_{d_{\mu_0}^\pi}} 
    \end{align*}
    Thus, when the following conditions hold
    $$C_{d_{\mu_0}^\pi, \mathbb{P},p}:= \frac{\left\| d_{\mu_0}^\pi \right\|_{L^{p^2}\left( \mathcal{S}\times\mathcal{A} \right)} C_{\mathbb{P},p}}{C_{d_{\mu_0}^\pi}}<\frac{1}{\gamma} ;\quad p\ge \frac{\log \left( \left| \mathcal{A}\right|\right) + \log \left( \mu\left( \mathcal{S} \right) \right)}{\log \frac{1}{\gamma C_{d_{\mu_0}^\pi, \mathbb{P},p}} } - 1; \quad q\ge p^2,$$ 
    we have 
    \begin{align*} 
        \left\| w \right\|_{L^p\left( \mathcal{S}\times\mathcal{A} \right)} &\le \frac{\left( \left| \mathcal{A} \right| \mu\left( \mathcal{S} \right) \right)^{\frac{p-1}{p^2}}}{C_{d_{\mu_0}^\pi} - \gamma \left\| d_{\mu_0}^\pi \right\|_{L^{p^2}\left( \mathcal{S}\times\mathcal{A} \right)}   C_{\mathbb{P},p} \left( \left| \mathcal{A} \right| \mu\left( \mathcal{S} \right) \right)^{\frac{p-1}{p^2}}}   \left\| f \right\|_{L^{p^2,d_{\mu_0}^\pi}\left( \mathcal{S}\times\mathcal{A} \right)} \\
        &\le \frac{\left( \left| \mathcal{A} \right| \mu\left( \mathcal{S} \right) \right)^{\frac{1}{p} - \frac{1}{q}}}{C_{d_{\mu_0}^\pi} - \gamma \left\| d_{\mu_0}^\pi \right\|_{L^{p^2}\left( \mathcal{S}\times\mathcal{A} \right)}   C_{\mathbb{P},p} \left( \left| \mathcal{A} \right| \mu\left( \mathcal{S} \right) \right)^{\frac{p-1}{p^2}}}  \left\| f \right\|_{L^{q,d_{\mu_0}^\pi}\left( \mathcal{S}\times\mathcal{A} \right)},
    \end{align*}
    where the last inequality comes from $\left\| f \right\|_{L^{p^2,d_{\mu_0}^\pi}\left( \mathcal{S}\times\mathcal{A} \right)} \le  \mu\left( \mathcal{S}\times\mathcal{A} \right)^{\frac{1}{p^2} - \frac{1}{q}}  \left\| f \right\|_{L^{q,d_{\mu_0}^\pi}\left( \mathcal{S}\times\mathcal{A} \right)}$.
\end{proof}
\begin{remark}
    The conditions are not satisfactory. The result implicitly adds the constrain for $\gamma$ because $\lim_{p\rightarrow\infty} C_{d_{\mu_0}^\pi, \mathbb{P},p} = \frac{1}{C_{d_{\mu_0}^\pi}}$, which indicates $\gamma$ may be very small, i.e. $\gamma < C_{d_{\mu_0}^\pi}$. 
\end{remark}
In the following theorem, we describe the instablility of DQN.

\begin{theorem}
    There exists a MDP $\mathcal{M}$ such that for all $\pi$ satisfying $M_{d_{\mu_0}^\pi}:=\sup_{(s,a)\in\mathcal{S}\times\mathcal{A}} d_{\mu_0}^\pi(s,a) <\infty$, Bellman optimality equation $\mathcal{T}_B Q = Q$ is not $\left(  L^{q,d_{\mu_0}^\pi}\left( \mathcal{S}\times\mathcal{A} \right),  L^p \left( \mathcal{S}\times\mathcal{A} \right) \right)$-stable, for all $1 \le q < p\le \infty$.
\end{theorem}
\begin{proof}
    According to the proof of Theorem \ref{thm: unstable of bellman opt eq}, for $1 \le q < p\le \infty$, there exists a MDP $\mathcal{M}$ satisfying the following statement. For all $\delta>0$ and $n\in\mathbb{N}$, there exists a $Q\in L^p \left( \mathcal{S}\times\mathcal{A} \right) \cap L^q \left( \mathcal{S}\times\mathcal{A} \right)$ satisfying $\left\| \mathcal{T}_B Q - Q \right\|_{L^q \left( \mathcal{S}\times\mathcal{A} \right)} \le \frac{\delta}{M_{d_{\mu_0}^\pi}}$ such that $\left\| Q - Q^* \right\|_{L^p \left( \mathcal{S}\times\mathcal{A} \right)} > n\left\| \mathcal{T}_B Q - Q \right\|_{L^q \left( \mathcal{S}\times\mathcal{A} \right)}$. 

    According to Lemma \ref{lem: norm property}, if $M_{d_{\mu_0}^\pi}:=\sup_{(s,a)\in\mathcal{S}\times\mathcal{A}} d_{\mu_0}^\pi(s,a) <\infty$, we have
    \begin{equation}
        \left\| Q \right\|_{L^{q,d_{\mu_0}^\pi}\left( \mathcal{S}\times\mathcal{A} \right)} \le M_{d_{\mu_0}^\pi} \left\| Q \right\|_{L^q \left( \mathcal{S}\times\mathcal{A} \right)} < \infty.
        \notag
    \end{equation}
    Thus, we have $Q\in L^{q,d_{\mu_0}^\pi}\left( \mathcal{S}\times\mathcal{A} \right) \cap L^p \left( \mathcal{S}\times\mathcal{A} \right)$. For the same reason, we have
    \begin{equation}
        \left\| \mathcal{T}_B Q - Q \right\|_{L^{q,d_{\mu_0}^\pi}\left( \mathcal{S}\times\mathcal{A} \right)} \le M_{d_{\mu_0}^\pi} \left\| \mathcal{T}_B Q - Q \right\|_{L^q \left( \mathcal{S}\times\mathcal{A} \right)} \le \delta.
        \notag
    \end{equation}
    Hence, we get that For all $\delta>0$ and $n\in\mathbb{N}$, there exists a $Q\in L^{q,d_{\mu_0}^\pi}\left( \mathcal{S}\times\mathcal{A} \right) \cap L^p \left( \mathcal{S}\times\mathcal{A} \right)$ satisfying $\left\| \mathcal{T}_B Q - Q \right\|_{L^{q,d_{\mu_0}^\pi}\left( \mathcal{S}\times\mathcal{A} \right)} \le \delta$ such that $\left\| Q - Q^* \right\|_{L^p \left( \mathcal{S}\times\mathcal{A} \right)} > n \left\| \mathcal{T}_B Q - Q \right\|_{L^q \left( \mathcal{S}\times\mathcal{A} \right)}$.
    Therefore, the proof of the theorem is concluded.
\end{proof}

\begin{remark}
    If $M_{d_{\mu_0}^\pi} = \infty$, $d_{\mu_0}^\pi$ degenerates to the discrete probability distribution. Then, $L^{p,d_{\mu_0}^\pi}$ can be considered as a norm defined on a finite dimension space. In this setting, we also have $C_{d_{\mu_0}^\pi}=0$ and Bellman optimality equation $\mathcal{T}_B Q = Q$ is not $\left(  L^{q,d_{\mu_0}^\pi}\left( \mathcal{S}\times\mathcal{A} \right),  L^p \left( \mathcal{S}\times\mathcal{A} \right) \right)$-stable, for any $p$ and $q$.
\end{remark}

According to the above theorems and remarks, we have the following corollary in the DQN procedure.
\begin{corollary}
     In practical DQN procedure, the Bellman optimality equations $\mathcal{T}_B Q = Q$ is $(  L^{\infty,d_{\mu_0}^\pi}( \mathcal{S}\times\mathcal{A}),  L^p( \mathcal{S}\times\mathcal{A}))$-stable for all $1\le p \le \infty$, while it is not $(  L^{q,d_{\mu_0}^\pi}( \mathcal{S}\times\mathcal{A}),  L^p( \mathcal{S}\times\mathcal{A}))$-stable for all $1 \le q < p\le \infty$. 
\end{corollary}

\subsection{Stability of DQN: the Bad} \label{app:Stability of DQN: the Bad}
\begin{theorem}
    There exists an MDP $\mathcal{M}$ such that for all $\pi$ satisfying $ d_{\mu_0}^\pi$ is a discrete probability distribution, Bellman optimality equation $\mathcal{T}_B Q = Q$ is not $\left(  L^{q,d_{\mu_0}^\pi}\left( \mathcal{S}\times\mathcal{A} \right),  L^p \left( \mathcal{S}\times\mathcal{A} \right) \right)$-stable, for any $p$ and $q$.
\end{theorem}
\begin{proof}
    We only need to show there exists an MDP such that $\forall n\in\mathbb{N}, \forall \delta>0, \exists Q(s,a),$ such that $\|\mathcal{T}_BQ-Q\|_{L^{q,d_{\mu_0}^\pi}\left( \mathcal{S}\times\mathcal{A} \right)}<\delta, \text{ but }  \|Q-Q^*\|_{L^p\left( \mathcal{S}\times\mathcal{A} \right)} \ge n.$
        
        Consider an MDP $\mathcal{M}$ where $\mathcal{S}=[-1,1], \mathcal{A}=\{a_1,a_2\},$
	\begin{equation*}
			\mathbb{P}(s^\prime |s,a_1)=\left\{ \begin{aligned}
				&\mathbbm{1}_{\left\{s^\prime=s-0.1\right\}}, \quad && s\in[-0.9,1] \\
				&\mathbbm{1}_{\left\{s^\prime=s \right\}}    ,      && else 
			\end{aligned}\right.,\quad		
			\mathbb{P}(s^\prime |s,a_2)=\left\{ \begin{aligned}
				&\mathbbm{1}_{\left\{s^\prime=s+0.1\right\}}, \quad && s\in[-1,0.9] \\
				&\mathbbm{1}_{\left\{s^\prime=s \right\}}    ,      && else 
			\end{aligned}\right.,
	\end{equation*}
	$r(s,a_i)=k_is,\ k_2\ge k_1> 0$. The transition function is essentially a deterministic transition dynamic and for convenience, we denote that
        \begin{equation*}
            p(s,a_1)=\left\{ \begin{aligned}
                &s-0.1, \quad && s\in[-0.9,1] \\
                &s     ,      && else 
            \end{aligned}\right.,\quad		
            p(s,a_2)=\left\{ \begin{aligned}
                &s+0.1, \quad && s\in[-1,0.9] \\
                &s     ,      && else 
            \end{aligned}\right..
       \end{equation*}
	
    Let $Q^*(s,a)=Q^{\pi^*}(s,a)$ be the optimal Q-function, where $\pi^*$ is the optimal policy. 
 
    Define $B_0=\{ s\in\mathcal{S}:\exists a\in\mathcal{A}, s.t.\, d^{\pi}_{\mu_0}(s,a)\neq 0 \}$, which contains all the states that can be explored. Let $B=\left\{B_0\cup\{B_0+0.1\}\cup\{B_0-0.1\}\right\}\cap\mathcal{S}$, then $\forall s\in B,\ p(s,a)\in B.$ Since $d_{\mu_0}^\pi$ is a discrete probability distribution, $\mu (B)=\mu (B_0)=0.$
    
    Let $D=[-1,1]\setminus B$, and $Q(s,a)=Q^*(s,a)+h\cdot \mathbbm{1}_D$, where $h=\frac{n}{2^{\frac{1}{p}}}$. We have that $Q(s,a)=Q^*(s,a),\forall s\notin D.$
    We then have for any $s\in B$ that
    \begin{align*}
        \mathcal{T}_B Q(s,a) & = r(s,a)+\gamma\cdot\max_{a^{\prime}\in\mathcal{A}}Q(p(s,a),a^{\prime})\\
        & = r(s,a)+\gamma\cdot\max_{a^{\prime}\in\mathcal{A}}Q^*(p(s,a),a^{\prime})\\
        & = Q^*(s,a)\\
        & = Q(s,a).
    \end{align*}
    For any $s\notin B$, $d_{\mu_0}^\pi(s,a)=0$ for all $a\in\mathcal{A}$. Hence, $\|\mathcal{T}_B Q(s,a)-Q(s,a)\|_{L^{q,d^{\pi}_{\mu_0}}(\mathcal{S}\times\mathcal{A})}=0 < \delta.$

    However, we find that $$\|Q(s,a)-Q^*(s,a)\|_{L^p(\mathcal{S}\times\mathcal{A})} = h\cdot \mu(D)^{\frac{1}{p}}=h\cdot 2^{\frac{1}{p}} \ge n,$$
    which completes the proof.
\end{proof}

\begin{remark}
    If $d_{\mu_0}^\pi$ is a discrete probability distribution, $L^{\infty,d_{\mu_0}^\pi}\left( \mathcal{S}\times\mathcal{A} \right)$ is not a good choice. However, the sample process should be considered in practical reinforcement learning algorithms and as a consequence, we have to apply the space $L^{\infty,d_{\mu_0}^{\pi}}\left( \mathcal{S}\times\mathcal{A} \right)$ rather than $L^{\infty}\left( \mathcal{S}\times\mathcal{A} \right)$.
\end{remark}

\section{Theorems and Proofs of Policy Robustness under k-measurement Error in Probability Space} \label{app: infinity measurement is necessary in probability space}

\subsection{Vulnerability of Non-infinity Measurement Errors}

\begin{lemma} \label{70}
    \citep{TongZhang2023Algorithms} Consider two probability measures $P$ and $Q$ that are absolutely continuous with respect to a reference measure $\mu$ with density $p$ and $q$, then \begin{equation}
        2\|P-Q\|^2_{TV} \leq \operatorname{KL}(P\|Q) \leq 2(3+\sup_z\ln\frac{p(z)}{q(z)})\|P-Q\|_{TV}
        \label{299}
    \end{equation}
\end{lemma}

\begin{theorem}[Vulnerability of Non-infinity Measurement Errors]
    There exists an MDP such that the following statement holds. Let $\varphi$ be a policy from the policy family $\mathcal{F} = \left\{ \varphi \left| \mathop{\arg\max}_a \varphi(a|s) = \mathop{\arg\max}_a \pi^*(a|s) \right. \right\}$. For any $\epsilon>0$, $1 \le k < \infty$, $\delta > 0$ and any state distribution $\mu$, there exists a policy $\pi$ satisfying $\mathcal{D}_{k, \operatorname{KL}}^{\mu} \left( \varphi \| \pi \right) \le \delta$ or $\mathcal{D}_{k, \operatorname{KL}}^{\mu} \left( \pi \| \varphi \right) \le \delta$, such that $m\left( \mathcal{S}_{sub}^\pi \right) = O(\delta)$ but $m\left( \mathcal{S}_{adv}^{\pi,\epsilon} \right) = m(\mathcal{S})$.
\end{theorem}

\begin{proof}
   Consider this given MDP: state space $\mathcal{S}=\left[-\frac{1}{2}, \frac{1}{2}\right]$, action space $\mathcal{A}=\left\{a_1, a_2\right\}$, and the optimal policy $\pi^*$ defined as follows:
    \begin{equation}
         \pi^*(a_1|s) = \begin{cases}
        1, & \text{if} \  s \in [0,\frac{1}{2}] \\
        0, & \text{if} \  s \in \left[-\frac{1}{2},0\right)
        \end{cases}
        \notag
    \end{equation}
    \begin{equation}
        \pi^*(a_2|s) = 1 - \pi^*(a_1|s).
        \notag
    \end{equation}

   Without loss of generality, let's assume that $\mu(s) = 1$. Let $\varphi$ represent any policy such that $\mathop{\arg\max}_a \varphi(a|s) = \mathop{\arg\max}_a \pi^*(a|s)$. Then we define $\tilde{\varphi}$: \begin{equation}
       \tilde{\varphi}(a_i|s) = \begin{cases}
        \varphi(a_i|s) , & \text{if} \  t < \varphi(a_i|s) < 1-t\\
        t , & \text{if}\  \varphi(a_i|s) \leq t \\
        1-t , & \text{if}\  \varphi(a_i|s) \ge 1-t
        \end{cases}
        \notag
   \end{equation} where $t$ is a number that satisfies $0 < t < \frac{1}{2}$.
   
   Let's define the following set:

   \begin{equation}
       D_{\epsilon}= \cap_{i=0}^{n} (B_i \cap B_i'),
       \notag
   \end{equation} where $B_i = \left[2i\epsilon,2i\epsilon+l \right]$, $B_i' = \left[-2i\epsilon-l ,-2i\epsilon \right]$, $n=\left[\frac{1-2l}{4\epsilon}\right]$, and $l=\min\{\frac{\delta^k}{2(n+1)^2(12+4\ln\frac{1-t}{t})^k},\epsilon \}$.

    We introduce the policy $\pi$ as shown: for $s$ in $\left[0,\frac{1}{2}\right]$, $\pi(a_1|s)$ is formulated as:\begin{equation}
        \pi(a_1|s) = \begin{cases}
        \tilde{\varphi}(a_1|s) , & \text{if} \ s \notin D_\epsilon \\
        \frac{t - \tilde{\varphi}(a_1|2i\epsilon)}{l/2}(s-2i\epsilon) + \tilde{\varphi}(a_1|2i\epsilon) , & \text{if}\  s \in \left[2i\epsilon,2i\epsilon+\frac{l}{2}\right] \\
        -\frac{t - \tilde{\varphi}(a_1|2i\epsilon+l)}{l/2}(s-2i\epsilon-\frac{l}{2}) + t , & \text{if}\  s \in \left[2i\epsilon+\frac{l}{2},2i\epsilon+l \right]
        \end{cases}
        \notag
    \end{equation} for $s$ in $\left[-\frac{1}{2},0 \right)$, $\pi(a_2|s)$ is formulated as: \begin{equation}
        \pi(a_2|s) = \begin{cases}
    \tilde{\varphi}(a_2|s) , & \text{if} \ s \notin D_\epsilon \\
    \frac{t - \tilde{\varphi}(a_2|-2i\epsilon-l)}{l/2}(s+2i\epsilon+l) + \tilde{\varphi}(a_2|-2i\epsilon-l) , & \text{if}\  s \in \left[-2i\epsilon-l,-2i\epsilon-\frac{l}{2}\right] \\
    -\frac{t - \tilde{\varphi}(a_2|-2i\epsilon-\frac{l}{2})}{l/2}(s+2i\epsilon+\frac{l}{2}) + t , & \text{if}\  s \in \left[-2i\epsilon-\frac{l}{2},-2i\epsilon \right]
    \end{cases}
    \notag
    \end{equation} where $i \in \{0,1,\dots ,n\}$ .

    Then we can calculate $ \mathcal{D}_{k, \operatorname{KL}}^{\mu} \left( \varphi \| \pi \right)$:
\begin{equation}
        \begin{aligned}
            &\quad
            \mathcal{D}_{k, \operatorname{KL}}^{\mu} \left( \varphi \| \pi \right) = \left\|\mu\operatorname{KL}\left( \varphi \| \pi \right)\right\|_{k}\\
            &\leq \mathcal{D}_{k, \operatorname{KL}}^{\mu} \left( \varphi \| \tilde{\varphi} \right)+\mathcal{D}_{k, \operatorname{KL}}^{\mu} \left( \tilde{\varphi} \| \pi \right)+\left(\int_{s \in \mathcal{S}} \left| \sum_{i=1}^2\left(\varphi(a_i|s)-\tilde{\varphi}(a_i|s)\right)\ln\frac{\tilde{\varphi}(a_i|s)}{\pi(a_i|s)} \right|^k d\mu(s)\right)^{\frac{1}{k}}.
        \end{aligned}
        \notag
    \end{equation} Similarly, we get the same conclusion when $s<0$.

    When $s \ge 0$ and $\varphi \ne \tilde{\varphi}$, we have $\tilde{\varphi}(a_1|s) = 1-t$, $\tilde{\varphi}(a_2|s) = t$, $1-t \leq \varphi(a_1|s) \leq 1$, $0 \leq \varphi(a_2|s) \leq t$. Therefore, through calculation, we know that: \begin{equation}
        \left|\varphi(a_1|s)\ln\frac{\varphi(a_1|s)}{\tilde{\varphi}(a_1|s)}\right| \leq \ln \frac{1}{1-t} , \
        \left|\varphi(a_2|s)\ln\frac{\varphi(a_2|s)}{\tilde{\varphi}(a_2|s)}\right| \leq \frac{t}{e}.
        \notag
    \end{equation}

    Therefore, \begin{equation}
        \begin{aligned}
            &\quad
            \mathcal{D}_{k, \operatorname{KL}}^{\mu} \left( \varphi \| \tilde{\varphi} \right) \\
            &=\left(\int_{s \in \mathcal{S}} \left| \sum_{i=1}^2\varphi(a_i|s)\ln\frac{\varphi(a_i|s)}{\tilde{\varphi}(a_i|s)} \right|^k ds\right)^{\frac{1}{k}} \\
            &=\left(\int_{\tilde{\varphi} \ne \varphi} \left| \sum_{i=1}^2\varphi(a_i|s)\ln\frac{\varphi(a_i|s)}{\tilde{\varphi}(a_i|s)} \right|^k ds\right)^{\frac{1}{k}} \\
            &\leq \left(\int_{\tilde{\varphi} \ne \varphi} \left| \ln \frac{1}{1-t} + \frac{t}{e}\right|^k ds\right)^{\frac{1}{k}}\\
            & \leq \ln \frac{1}{1-t} + \frac{t}{e}.
        \end{aligned}
        \notag
    \end{equation}

    From the definition of $\pi$ and $\tilde{\varphi}$, we know that $t \leq \tilde{\varphi}(a_i|s) \leq 1-t$, $t \leq \pi(a_i|s) \leq 1-t$, $|\varphi(a_i|s)-\tilde{\varphi}(a_i|s)| \leq t$ where $i=1,2.$ So we have :
    \begin{equation}
        \begin{aligned}
            &\quad
            \left(\int_{s \in \mathcal{S}} \left| \sum_{i=1}^2\left(\varphi(a_i|s)-\tilde{\varphi}(a_i|s)\right)\ln\frac{\tilde{\varphi}(a_i|s)}{\pi(a_i|s)} \right|^k ds\right)^{\frac{1}{k}} \\
            &\leq \left(\int_{s \in \mathcal{S}} \left| 2t \ln\frac{1-t}{t} \right|^k ds\right)^{\frac{1}{k}} \\
            &\leq 2t \ln\frac{1-t}{t}.
        \end{aligned}
        \notag
    \end{equation}

    We can easily verify that \begin{equation}
        \lim_{t \rightarrow 0} \left(\ln \frac{1}{1-t} + \frac{t}{e} + 2t \ln\frac{1-t}{t}\right) = 0
        \label{309}
    \end{equation} 

    So we can fix a $t$ s.t. $\ln \frac{1}{1-t} + \frac{t}{e} + 2t \ln\frac{1-t}{t} \leq \frac{\delta}{2}$, i.e. \begin{equation}\label{310}
         \mathcal{D}_{k, \operatorname{KL}}^{\mu} \left( \varphi \| \tilde{\varphi} \right) + \left(\int_{s \in \mathcal{S}} \left| \sum_{i=1}^2\left(\varphi(a_i|s)-\tilde{\varphi}(a_i|s)\right)\ln\frac{\tilde{\varphi}(a_i|s)}{\pi(a_i|s)} \right|^k ds \right)^{\frac{1}{k}} \leq \frac{\delta}{2}
    \end{equation}

    Start calculating $\mathcal{D}_{k, \operatorname{KL}}^{\mu} \left( \tilde{\varphi} \| \pi \right)$ from the right side of inequality \eqref{299} in lemma \ref{70}:

    \begin{equation}
        \begin{aligned}
            &\quad
            \mathcal{D}_{k, \operatorname{KL}}^{\mu} \left( \tilde{\varphi} \| \pi \right) = \left\|\mu\operatorname{KL}\left( \tilde{\varphi} \| \pi \right)\right\|_{k}\\ 
            &= \left( \int_{s \in \mathcal{S}} \left| \mu(s)\operatorname{KL}\left( \tilde{\varphi}(\cdot|s) \| \pi(\cdot|s) \right) \right|^k ds \right)^{\frac{1}{k}}\\
            &\leq \left( \int_{s \in D_\epsilon} \left((3+ \sup_a |\ln \frac{\tilde{\varphi}(a|s)}{\pi(a|s)})|\sum_{i=1}^2 \left| \tilde{\varphi}(a_i|s) - \pi(a_i|s)\right|\right)^k ds \right)^{\frac{1}{k}}\\
            &\leq \left(  \int_{s \in D_\epsilon} \left(2(3+ \ln\frac{1-t}{t}) \left| \tilde{\varphi}(a_1|s) - \pi(a_1|s)\right|\right)^k ds \right)^{\frac{1}{k}}\\
            &\leq \left( (2(3+ \ln\frac{1-t}{t})^k  \int_{s \in D_\epsilon} \left( \left| \tilde{\varphi}(a_i|s) - \pi(a_i|s)\right|\right)^k ds \right)^{\frac{1}{k}}\\
            &\leq \left( (2(3+ \ln\frac{1-t}{t})^k 2(n+1)l\right)^{\frac{1}{k}}\\ 
            &\leq \frac{\delta}{2}.
            \label{311}
        \end{aligned}
    \end{equation}

    By combining formula \eqref{310} with formula \eqref{311}, it can be inferred that \begin{equation}
        \mathcal{D}_{k, \operatorname{KL}}^{\mu}\left( \varphi \| \pi \right) \leq \delta.
        \notag
    \end{equation}

    Let's now consider the reverse KL. In this situation, We have to assume that $\varphi$ is stochastic because of the restriction that the denominator cannot be zero. Then, there must be a number $t$ such that $0<t<\min_{a \in \mathcal{A},s \in \mathcal{S}}\varphi(a|s)$, and we fix this $t$. Thus, $\varphi = \tilde{\varphi}$. We only need to consider $\mathcal{D}_{k, \operatorname{KL}}^{\mu} \left( \pi \| \tilde{\varphi} \right)$.

    Obviously, the equation \eqref{309} is symmetric with respect to $\pi$ and $\varphi$, so we also have : \begin{equation}
        \mathcal{D}_{k, \operatorname{KL}}^{\mu} \left( \pi \| \varphi \right) = \mathcal{D}_{k, \operatorname{KL}}^{\mu} \left( \pi \| \tilde{\varphi} \right) \leq \delta.
        \notag
    \end{equation}

    If $s \in \mathcal{S}_{sub}^\pi$ , then there must hold: $\pi(a_1|s) \leq \pi(a_2|s)
    $ when $s \ge 0$ and $\pi(a_1|s) > \pi(a_2|s)$ when $s<0$. According to the definition of $\pi$, in order to meet the above conditions, there must exist $s \in D_\epsilon$. Therefore:\begin{equation}
        \mathcal{S}_{sub}^\pi \subseteq D_\epsilon.
        \notag
    \end{equation} Hence, \begin{equation}
        m\left(\mathcal{S}_{sub}^\pi \right) \leq m\left(D_\epsilon \right) = 2(n+1)l \leq \frac{\delta^k}{(n+1)M^k} = O(\delta^k).
        \notag
    \end{equation}

    Let $s_i = 2i\epsilon + \frac{l}{2}, \ i = 0,1,\dots,n$. From the defination of $\pi$, we know that $\pi(a_1|s_i) = t < \pi(a_2|s_i)$. This is used to represent the points on the positive real axis where $\pi(a_1|\cdot)$ takes on the value of $t$. Combined with the continuity of the function $\pi$ at $s_i$, there must exist a neighborhood $B_{\epsilon_i}(s_i)$, s.t. $\pi(a_1|s) < \pi(a_2|s)$ for $\forall \ s \in B_{\epsilon_i}(s_i)$, where $i = 0,1,\dots,n$.

    We now show $\left[s_i, s_{i+1}\right] \subseteq \mathcal{S}_{adv}^{\pi, \epsilon}$, where $i=0,1,\dots,n-1.$ For $\forall \ s \in \left[s_i ,  s_{i+1}\right]$, it is obvious that \begin{equation}
        B_{\epsilon}(s) \cap \left(\left[s_i,s_i+\epsilon_i\right) \cup \left(s_{i+1}-\epsilon_{i+1},s_{i+1}\right]\right) \ne \emptyset,
        \notag
    \end{equation} because $s_{i+1} - s_i = 2\epsilon$, where $i = 0,1,\dots,n-1$. 
    
    Therefore, $\exists \ s_\nu \in B_{\epsilon}(s) \cap \left(\left[s_i,s_i+\epsilon_i\right) \cup \left(s_{i+1}-\epsilon_{i+1},s_{i+1}\right]\right) \subseteq B_{\epsilon}$, s.t. \begin{equation}
        Q^*(s,\arg\max_a \pi(a|s_\nu)) = Q^*(s,a_2) < \max_a Q(s,a).
        \notag
    \end{equation} So we have \begin{equation}
        \left[s_i ,  s_{i+1}\right] \subseteq \mathcal{S}_{adv}^{\pi,\epsilon},
        \notag
    \end{equation} where $i = 0,1, \dots , n-1$.

    As $\pi(a_1|\frac{l}{2})=t_2 \leq \pi(a_2|\frac{l}{2})$, it follows that $\arg \max_a \pi(a|\frac{l}{2}) = a_2$, thus $ \left[0,s_0\right] \subseteq \mathcal{S}_{adv}^{\pi,\epsilon}$, and similarly $\left[s_n, \frac{1}{2}\right] \subseteq \mathcal{S}_{adv}^{\pi,\epsilon}$.

    In conclusion, we can get $\left[0,\frac{1}{2}\right] \subseteq \mathcal{S}_{adv}^{\pi,\epsilon}$. The proof of $\left[-\frac{1}{2},0\right) \subseteq \mathcal{S}_{adv}^{\pi,\epsilon}$ follows a similar process and does not need further elaboration.

    Therefore, \begin{equation}
         \mathcal{S} \subseteq\mathcal{S}_{adv}^{\pi,\epsilon}  . 
         \notag
    \end{equation} So we have : \begin{equation}
        m\left( \mathcal{S}_{adv}^{\pi,\epsilon} \right) = m(\mathcal{S}).
        \notag
    \end{equation}
    Therefore, the proof of the theorem is concluded.
\end{proof}

\subsection{Robustness Guarantee under Infinity Measurement Error}

\begin{theorem}[Robustness Guarantee under Infinity Measurement Error] 
    For any MDP and any state distribution $\mu$ satisfying $\mu(s) > 0 $ for any $s \in \mathcal{S}$, let $S_{\delta} = \{s \mid \exists \ a,a' \in \mathcal{A},\ \text{s.t.}\ |\varphi(a|s)-\varphi(a'|s)| 
    \leq 2\sqrt{\frac{2\delta}{\mu(s)}}  \}$ and $h(\delta) = \mu(S_{\delta})$. Then, $h(\delta)$ is a monotonic function with $h(0) = 0$. Let $\varphi$ be a policy from the policy family $\mathcal{F} = \left\{ \varphi \mid \mathop{\arg\max}_a \varphi(a|s) = \mathop{\arg\max}_a \pi^*(a|s)  \right\}$. If $S_\delta $ is the union of finite connected subsets, then for any $\delta > 0$ and any policy $\pi$ satisfying $\mathcal{D}_{\infty, \operatorname{KL}}^{\mu} \left( \varphi \| \pi \right) \le \delta$ or $\mathcal{D}_{\infty, \operatorname{KL}}^{\mu} \left( \pi \| \varphi \right) \le \delta$, we have that $m\left( \mathcal{S}_{sub}^\pi \right) = O(h(\delta))$ and $m\left( \mathcal{S}_{adv}^{\pi,\epsilon} \right) = 2\epsilon +  O(h(\delta))$.
\end{theorem}
\begin{proof}
    From $\mathcal{D}_{\infty, \operatorname{KL}}^{\mu} \left( \varphi \| \pi \right) = \|\mu(s)\operatorname{KL}(\varphi(\cdot|s)\|\pi(\cdot|s))\|_{\mathcal{L}^{\infty}(\mathcal{S}) \times \operatorname{KL} (\mathcal{A})} \leq \delta$, we can deduce that \begin{equation}
        \mu(s)\operatorname{KL}(\varphi(\cdot|s)\|\pi(\cdot|s)) \leq \delta, \quad\forall \ s \in \mathcal{S}.
        \notag
    \end{equation} Then according to inequality \eqref{299} , we can derive that \begin{equation}
        \begin{aligned}
             & \quad \max_{a \in \mathcal{A}} |\varphi(a|s) - \pi(a|s)| \leq \sum_{a \in \mathcal{A}} |\varphi(a|s) - \pi(a|s)|\\
             & = \|\varphi(\cdot|s)-\pi(\cdot|s)\|_{\mathcal{L}^1(\mathcal{A})} = 2\| \varphi(\cdot|s) - \pi(\cdot|s)) \|_{{TV}}\\
             & \leq \sqrt{2\operatorname{KL}(\varphi(\cdot|s)\|\pi(\cdot|s))} \leq \sqrt{\frac{2\delta}{\mu(s)}}.
             \label{322}
        \end{aligned}
    \end{equation}

Then, as long as $|\varphi(a|s)-\varphi(a'|s)|>2\sqrt{\frac{2\delta}{\mu(s)}}$, for $\forall \ a,a' \in \mathcal{A}$, we can deduce $s \notin \mathcal{S}_{sub}^\pi$. Hence, we have that \begin{equation}
    \mathcal{S}_{sub}^\pi \subseteq S_\delta, \  \mathcal{S}_{adv}^{\pi,\epsilon} \subseteq S_\delta + B_\epsilon.
    \notag
\end{equation} 

Therefore, \begin{equation}
    m(\mathcal{S}_{sub}^\pi) \leq m(S_\delta) = h(\delta),
    \notag
\end{equation} and \begin{equation}
    \begin{aligned}
        & \quad m(\mathcal{S}_{adv}^{\pi,\epsilon}) \leq m(S_\delta + B_\epsilon)\leq \sum_{i=1}^n m(S_{\delta_i} + B_\epsilon)\\
        &  \leq m(S_\delta) + \sum_{i=1}^n(m(S_{\delta_i} + B_\epsilon)-m(B_\epsilon))\\
        & = h(\delta) + O(\epsilon).
    \end{aligned}
    \notag
\end{equation}

From the derivation of equation \eqref{322}, we can see that the whole process is symmetrical with respect to $\pi$ and $\varphi$, so this conclusion also applies to reverse $\operatorname{KL}$, i.e. when policy $\pi$ satisfying $\mathcal{D}_{\infty, \operatorname{KL}}^{\mu} \left( \pi \| \varphi \right) \le \delta$.
\end{proof}

\section{Derivation of CAR-RL}

In this section, we show the derivation of CAR-DQN and CAR-PPO loss functions.

\subsection{Derivation of CAR-DQN}\label{app:derive car-dqn}

Our theory motivates us to use the following objective:
\begin{align*}
    \mathcal{L}_{car}(\theta) &:= \left\|\mathcal{T}_B Q_\theta - Q_\theta \right\|_{L^{\infty,d_{\mu_0}^{\pi_\theta}}\left( \mathcal{S}\times\mathcal{A} \right)} \\
    &=\sup_{(s,a)\in\mathcal{S}\times\mathcal{A}} d_{\mu_0}^{\pi_\theta}(s,a) \left| \mathcal{T}_B Q_\theta(s,a) - Q_\theta(s,a) \right| \\
    &=\sup_{(s,a)\in\mathcal{S}\times\mathcal{A}} d_{\mu_0}^{\pi_\theta}(s,a) \left| r(s, a)+\gamma \mathbb{E}_{s^{\prime} \sim \mathbb{P}(\cdot \mid s, a)}\left[\max _{a^{\prime} \in \mathcal{A}} Q_\theta\left(s^{\prime}, a^{\prime}\right)\right] - Q_\theta(s,a) \right| .
\end{align*}
However, the objective is intractable in a model-free setting, due to the unknown environment, i.e. unknown reward function and unknown transition function.
\begin{remark}
    We apply the space $L^{\infty,d_{\mu_0}^{\pi_\theta}}\left( \mathcal{S}\times\mathcal{A} \right)$ rather than $L^{\infty}\left( \mathcal{S}\times\mathcal{A} \right)$ because the sampling process should be considered in practical reinforcement learning algorithms.
\end{remark}

\subsubsection{Surrogate Objective} \label{app: surrogate objective of car-dqn}
We can derive that
\begin{align*}
    \mathcal{L}_{car}(\theta) &= \sup_{s\in\mathcal{S}} \max_{a\in\mathcal{A}} d_{\mu_0}^{\pi_\theta}(s,a) \left| \mathcal{T}_{B} Q_\theta(s,a) - Q_\theta(s,a) \right|  \\
    &= \sup_{s\in\mathcal{S}} \max_{s_\nu \in B_\epsilon(s)} \max_{a\in\mathcal{A}} d_{\mu_0}^{\pi_\theta}(s,a) \left| \mathcal{T}_{B} Q_\theta(s_\nu,a) - Q_\theta(s_\nu,a) \right| \\
    &= \sup_{(s,a)\in\mathcal{S}\times\mathcal{A}} d_{\mu_0}^{\pi_\theta}(s,a) \max_{s_\nu \in B_\epsilon(s)}  \left| \mathcal{T}_{B} Q_\theta(s_\nu,a) - Q_\theta(s_\nu,a) \right|. 
\end{align*}
However, in a practical reinforcement learning setting, we cannot directly get the estimation of $\mathcal{T}_{B} Q_\theta(s_\nu, a)$.

\begin{theorem}
    Let $\mathcal{L}_{car}^{train}(\theta) := \sup_{(s,a)\in\mathcal{S}\times\mathcal{A}} d_{\mu_0}^{\pi_\theta}(s,a) \max_{s_\nu \in B_\epsilon(s)} \left| \mathcal{T}_{B} Q_\theta(s,a) - Q_\theta(s_\nu,a) \right| $ and $\mathcal{L}_{car}^{diff}(\theta) := \sup_{(s,a)\in\mathcal{S}\times\mathcal{A}} d_{\mu_0}^{\pi_\theta}(s,a) \max_{s_\nu \in B_\epsilon(s)} \left| \mathcal{T}_{B} Q_\theta(s_\nu,a) - \mathcal{T}_{B}Q_\theta(s,a) \right|$. We have that 
    \begin{equation}
        \left| \mathcal{L}_{car}^{train}(\theta) -  \mathcal{L}_{car}^{diff}(\theta)\right| \le \mathcal{L}_{car}(\theta) \le \mathcal{L}_{car}^{train}(\theta) + \mathcal{L}_{car}^{diff}(\theta).
        \notag
    \end{equation}
\end{theorem}
\begin{proof}
    On one hand, we have
    \begin{align*}
        \mathcal{L}_{car}(\theta) 
        &= \sup_{(s,a)\in\mathcal{S}\times\mathcal{A}} d_{\mu_0}^{\pi_\theta}(s,a) \max_{s_\nu \in B_\epsilon(s)}  \left| \mathcal{T}_{B} Q_\theta(s_\nu,a) - Q_\theta(s_\nu,a) \right| \\
        &\le \sup_{(s,a)\in\mathcal{S}\times\mathcal{A}} d_{\mu_0}^{\pi_\theta}(s,a) \max_{s_\nu \in B_\epsilon(s)}  \left( \left| \mathcal{T}_{B} Q_\theta(s,a) - Q_\theta(s_\nu,a) \right| + \left| \mathcal{T}_{B} Q_\theta(s_\nu,a) - \mathcal{T}_{B}Q_\theta(s,a) \right| \right) \\
        &\le  \sup_{(s,a)\in\mathcal{S}\times\mathcal{A}} d_{\mu_0}^{\pi_\theta}(s,a) \max_{s_\nu \in B_\epsilon(s)} \left| \mathcal{T}_{B} Q_\theta(s_\nu,a) - \mathcal{T}_{B}Q_\theta(s,a) \right|  \\
        &\quad + \sup_{(s,a)\in\mathcal{S}\times\mathcal{A}} d_{\mu_0}^{\pi_\theta}(s,a) \max_{s_\nu \in B_\epsilon(s)} \left| \mathcal{T}_{B} Q_\theta(s,a) - Q_\theta(s_\nu,a) \right|. 
    \end{align*}
    On the other hand, we have
    \begin{align*}
        \mathcal{L}_{car}(\theta) &= \sup_{(s,a)\in\mathcal{S}\times\mathcal{A}} d_{\mu_0}^{\pi_\theta}(s,a) \max_{s_\nu \in B_\epsilon(s)}  \left| \mathcal{T}_{B} Q_\theta(s_\nu,a) - Q_\theta(s_\nu,a) \right| \\
        &\ge \sup_{(s,a)\in\mathcal{S}\times\mathcal{A}} d_{\mu_0}^{\pi_\theta}(s,a) \max_{s_\nu \in B_\epsilon(s)}  \left| \left| \mathcal{T}_{B} Q_\theta(s_\nu,a) - \mathcal{T}_{B}Q_\theta(s,a) \right| - \left| \mathcal{T}_{B} Q_\theta(s,a) - Q_\theta(s_\nu,a) \right| \right| \\
        &\ge \left| \sup_{(s,a)\in\mathcal{S}\times\mathcal{A}} d_{\mu_0}^{\pi_\theta}(s,a) \max_{s_\nu \in B_\epsilon(s)} \left| \mathcal{T}_{B} Q_\theta(s_\nu,a) - \mathcal{T}_{B}Q_\theta(s,a) \right|  \right.\\
        &\quad \left. - \sup_{(s,a)\in\mathcal{S}\times\mathcal{A}} d_{\mu_0}^{\pi_\theta}(s,a) \max_{s_\nu \in B_\epsilon(s)} \left| \mathcal{T}_{B} Q_\theta(s,a) - Q_\theta(s_\nu,a) \right| \right|,
    \end{align*}
    where the second inequality comes from Lemma \ref{lem: basic ineq 2}.
\end{proof}

It is hard to calculate or estimate $\mathcal{L}_{car}^{diff}(\theta)$ in practice. Fortunately, we think $\mathcal{L}_{car}^{diff}(\theta)$ should be small in practice and we give a constant upper bound of $\mathcal{L}_{car}^{diff}(\theta)$ in the smooth environment.

\begin{lemma} \label{lem: uniform bound of bellman operator}
    Suppose $Q$ and $r$ are uniformly bounded, i.e. $\exists\ M_Q,M_r >0$ such that $\left|Q(s,a)\right| \le M_Q,\ \left|r(s,a)\right| \le M_r\ \forall s\in\mathcal{S}, a\in\mathcal{A}$. Then $\mathcal{T}_{B} Q(\cdot,a)$ is uniformly bounded, i.e.
    \begin{equation}
        \left|\mathcal{T}_{B} Q(s,a) \right| \le C_Q,\ \forall s\in\mathcal{S}, a\in\mathcal{A},
        \notag
    \end{equation}
    where $C_Q = \max\left\{ M_Q, \frac{M_r}{1-\gamma} \right\}$. Further, for any $k\in\mathbb{N}$, $\mathcal{T}_{B}^{k} Q(\cdot,a)$ has the same uniform bound as $\mathcal{T}_{B} Q(\cdot,a)$, i.e.
    \begin{equation}\label{eq: uniform bound of bellman operator}
        \left|\mathcal{T}_{B}^{k} Q(s,a) \right| \le C_Q,\ \forall s\in\mathcal{S}, a\in\mathcal{A}.
        \notag
    \end{equation}
\end{lemma}
\begin{proof}
    \begin{align*}
        \left|\mathcal{T}_{B} Q(s,a) \right| &= \left|r(s,a)  + \gamma \mathbb{E}_{ s^\prime \sim \mathbb{P}(\cdot|s,a)} \left[ \max _{a^{\prime} \in \mathcal{A}} Q\left(s^{\prime}, a^{\prime}\right) \right]\right| \\
        &\le \left|r(s,a) \right| + \gamma\mathbb{E}_{ s^\prime \sim \mathbb{P}(\cdot|s,a)} \left|\max _{a^{\prime} \in \mathcal{A}} Q\left(s^{\prime}, a^{\prime}\right)\right| \\
        &\le M_r + \gamma M_Q \\
        &\le \max\left\{ M_Q, \frac{M_r}{1-\gamma} \right\}, \qquad \forall s\in\mathcal{S}, a\in\mathcal{A}.
    \end{align*}

    Let $C_Q = \max\left\{ M_Q, \frac{M_r}{1-\gamma} \right\}$. Suppose the inequality (\ref{eq: uniform bound of bellman operator}) holds for $k=n$. Then, for $k=n+1$, we have
    \begin{align*}
        \left|\mathcal{T}_{B}^{n+1} Q(s,a) \right| &= \left|r(s,a)  + \gamma \mathbb{E}_{ s^\prime \sim \mathbb{P}(\cdot|s,a)} \left[  \max _{a^{\prime} \in \mathcal{A}} \mathcal{T}_{B}^{n} Q\left(s^{\prime}, a^{\prime}\right)\right]\right| \\
        &\le \left|r(s,a) \right| + \gamma\mathbb{E}_{ s^\prime \sim \mathbb{P}(\cdot|s,a)}  \left| \max _{a^{\prime} \in \mathcal{A}} \mathcal{T}_{B}^{n} Q\left(s^{\prime}, a^{\prime}\right)\right| \\
        &\le M_r + \gamma C_Q \\
        &\le (1-\gamma) C_Q + \gamma C_Q \\
        &= C_Q.
    \end{align*}
    By induction, we have $\left|\mathcal{T}_{B}^{k} Q(s,a) \right| \le C_Q,\ \forall s\in\mathcal{S}, a\in\mathcal{A}, k\in\mathbb{N}$.
\end{proof}

\begin{lemma} \label{lem: lip of bellman operator}
    Suppose the environment is $\left(L_r, L_{\mathbb{P}}\right)$-smooth and suppose $Q$ and $r$ are uniformly bounded, i.e. $\exists\ M_Q,M_r >0$ such that $\left|Q(s,a)\right| \le M_Q,\ \left|r(s,a)\right| \le M_r\ \forall s\in\mathcal{S}, a\in\mathcal{A}$. Then $\mathcal{T}_{B} Q(\cdot,a)$ is Lipschitz continuous, i.e.
        \begin{equation}
            \left| \mathcal{T}_{B} Q(s,a) - \mathcal{T}_{B} Q(s^\prime,a) \right| \le L_{\mathcal{T}_{B}} \|s - s^\prime\|,
            \notag
        \end{equation}
        where $L_{\mathcal{T}_{B}} =  L_r + \gamma C_Q L_{\mathbb{P}}$ and $C_Q = \max\left\{ M_Q, \frac{M_r}{1-\gamma} \right\}$. Further, for any $k\in\mathbb{N}$, $\mathcal{T}_{B}^{k} Q(\cdot,a)$ is Lipschitz continuous and has the same Lipschitz constant as $\mathcal{T}_{B} Q(\cdot,a)$, i.e.
        \begin{equation}
            \left| \mathcal{T}_{B}^{k} Q(s,a) - \mathcal{T}_{B}^{k} Q(s^\prime,a) \right| \le L_{\mathcal{T}_{B}} \|s - s^\prime\|.
            \notag
        \end{equation}
\end{lemma}
\begin{proof}
    For all $s_1,s_2 \in \mathcal{S}$, we have
    \begin{align*}
        &\quad \mathcal{T}_{B} Q(s_1,a) - \mathcal{T}_{B} Q(s_2,a) \\
        &= r(s_1,a)  + \gamma \mathbb{E}_{ s^\prime \sim \mathbb{P}(\cdot|s_1,a)} \left[ \max _{a^{\prime} \in \mathcal{A}} Q\left(s^{\prime}, a^{\prime}\right) \right] - r(s_2,a) - \gamma \mathbb{E}_{ s^\prime \sim \mathbb{P}(\cdot|s_2,a)} \left[ \max _{a^{\prime} \in \mathcal{A}} Q\left(s^{\prime}, a^{\prime}\right) \right] \\
        &= \left(r(s_1,a) - r(s_2,a)\right) + \gamma \int_{s^\prime} \left( \mathbb{P}(s^\prime|s_1,a) - \mathbb{P}(s^\prime|s_2,a)\right) \max _{a^{\prime} \in \mathcal{A}} Q\left(s^{\prime}, a^{\prime}\right) ds^\prime.
    \end{align*}

    Then, we have 
    \begin{align*}
        &\quad \left| \mathcal{T}_{B} Q(s_1,a) - \mathcal{T}_{B} Q(s_2,a)\right| \\
        &\le \left| \left(r(s_1,a) - r(s_2,a)\right) \right| + \left| \gamma \int_{s^\prime} \left( \mathbb{P}(s^\prime|s_1,a) - \mathbb{P}(s^\prime|s_2,a)\right) \max _{a^{\prime} \in \mathcal{A}} Q\left(s^{\prime}, a^{\prime}\right) ds^\prime \right| \\    
        &\le L_r \|s_1 - s_2\| + \gamma \int_{s^\prime} \left| \mathbb{P}(s^\prime|s_1,a) - \mathbb{P}(s^\prime|s_2,a)\right| \left| \max _{a^{\prime} \in \mathcal{A}} Q\left(s^{\prime}, a^{\prime}\right) \right| ds^\prime\\
        &\le  L_r \|s_1 - s_2\| + \gamma C_Q \int_{s^\prime} \left| \mathbb{P}(s^\prime|s_1,a) - \mathbb{P}(s^\prime|s_2,a)\right| ds^\prime\\
        &\le  L_r \|s_1 - s_2\| + \gamma C_Q L_{\mathbb{P}} \|s_1 - s_2\| \\
        &= \left( L_r + \gamma C_Q L_{\mathbb{P}} \right) \|s_1 - s_2\|.
    \end{align*}
    The second inequality comes from the Lipschitz property of $r$. The third inequality comes from the uniform boundedness of $Q$ and the last inequality utilizes the Lipschitz property of $\mathbb{P}$.

    Note that the uniform boundedness used in the above proof is $C_Q$ rather than $M_Q$. Then, due to lemma \ref{lem: uniform bound of bellman operator}, we can extend the above proof to $\mathcal{T}_{B}^{k}$.
\end{proof}

\begin{theorem}
    Suppose the environment is $\left(L_r, L_{\mathbb{P}}\right)$-smooth and suppose $Q_\theta$ and $r$ are uniformly bounded, i.e. $\exists\ M_Q,M_r >0$ such that $\left|Q_\theta(s,a)\right| \le M_Q,\ \left|r(s,a)\right| \le M_r\ \forall s\in\mathcal{S}, a\in\mathcal{A}$. If $M:=\sup_{\theta,(s,a)\in\mathcal{S}\times\mathcal{A}} d_{\mu_0}^{\pi_\theta}(s,a) <\infty$, then we have
    \begin{equation}
        \mathcal{L}_{car}^{diff}(\theta) \le C_{\mathcal{T}_{B}} \epsilon,
        \notag
    \end{equation}
    where $C_{\mathcal{T}_{B}}=L_{\mathcal{T}_{B}} M$, $L_{\mathcal{T}_{B}} =  L_r + \gamma C_Q L_{\mathbb{P}}$ and $C_Q = \max\left\{ M_Q, \frac{M_r}{1-\gamma} \right\}$. 
\end{theorem}
\begin{proof}
    \begin{align*}
        \mathcal{L}_{car}^{diff}(\theta) &= \sup_{(s,a)\in\mathcal{S}\times\mathcal{A}} d_{\mu_0}^{\pi_\theta}(s,a) \max_{s_\nu \in B_\epsilon(s)} \left| \mathcal{T}_{B} Q_\theta(s_\nu,a) - \mathcal{T}_{B}Q_\theta(s,a) \right| \\
        & \le \sup_{(s,a)\in\mathcal{S}\times\mathcal{A}} d_{\mu_0}^{\pi_\theta}(s,a) \max_{s_\nu \in B_\epsilon(s)} \left( L_r + \gamma C_Q L_{\mathbb{P}} \right) \|s_\nu - s\| \\
        & \le \left( L_r + \gamma C_Q L_{\mathbb{P}} \right) \epsilon \sup_{(s,a)\in\mathcal{S}\times\mathcal{A}} d_{\mu_0}^{\pi_\theta}(s,a) \\
        & \le M \left( L_r + \gamma C_Q L_{\mathbb{P}} \right) \epsilon,
    \end{align*}
    where the first inequality comes from Lemma \ref{lem: lip of bellman operator} and the last inequality comes from the uniform boundedness of $d_{\mu_0}^{\pi_\theta}$.
\end{proof}

\subsubsection{Approximate Objective} \label{app: soft objective of car-dqn}

\begin{lemma} \label{lem: soft}
    For any function $f: \Omega \rightarrow \mathbb{R}$ and $\lambda>0$, we have
    \begin{equation}
        \max_{p\in \Delta(\Omega)} \left[ \mathbb{E}_{\omega\sim p} f(\omega) - \lambda \operatorname{KL}(p\| p_0) \right]=\lambda \ln{\mathbb{E}_{\omega\sim p_0} \left[ e^{\frac{f(\omega)}{\lambda}} \right]},
        \notag
    \end{equation}
    where $ \Delta(\Omega)$ denotes the set of probability distributions on $\Omega$. And the solution is achieved by the following distribution $q$:
    \begin{equation}
         q(\omega) = \frac{ p_0(\omega)e^{\frac{f(\omega)}{\lambda}}}{\int_{\omega\in\Omega} p_0(\omega) e^{\frac{f(\omega)}{\lambda}} d\mu(\omega)} = \frac{1}{C} p_0(\omega)e^{\frac{f(\omega)}{\lambda}}.
         \notag
    \end{equation}
\end{lemma}
\begin{proof}
    Let $$C:=  \mathbb{E}_{\omega\sim p_0} \left[ e^{\frac{f(\omega)}{\lambda}} \right] = \int_{\omega\in\Omega} p_0(\omega) e^{\frac{f(\omega)}{\lambda}} d\mu(\omega), $$
    $$ q(\omega) = \frac{ p_0(\omega)e^{\frac{f(\omega)}{\lambda}}}{\int_{\omega\in\Omega} p_0(\omega) e^{\frac{f(\omega)}{\lambda}} d\mu(\omega)} = \frac{1}{C} p_0(\omega)e^{\frac{f(\omega)}{\lambda}} .$$
    Then, we have
    \begin{align*}
        &\mathbb{E}_{\omega\sim p} f(\omega) - \lambda \operatorname{KL}(p\| p_0) \\
        =&\mathbb{E}_{\omega\sim p} \left[ \lambda \ln{e^{\frac{f(\omega)}{\lambda}}} - \lambda \ln{\frac{p(\omega)}{p_0(\omega)}} \right]\\
        =&\lambda \mathbb{E}_{\omega\sim p} \left[  \ln{ \frac{e^{\frac{f(\omega)}{\lambda}}p_0(\omega)}{p(\omega)}} \right]\\
        =& \lambda \mathbb{E}_{\omega\sim p} \left[  \ln{ \frac{C q(\omega)}{p(\omega)}} \right]\\
        =& \lambda \left[ \ln{C} - \operatorname{KL}(p \| q) \right]\\
        \le & \lambda \ln{C}\\
        =& \lambda \ln{\mathbb{E}_{\omega\sim p_0} \left[ e^{\frac{f(\omega)}{\lambda}} \right]}.
    \end{align*}
    Note that the equal sign holds if and only if $p=q$. Thus, we get
    $$ q \in \arg\max_{p\in \Delta(\Omega)} \left[ \mathbb{E}_{\omega\sim p} f(\omega) - \lambda \operatorname{KL}(p\| p_0) \right] .  $$ 
    Therefore, the proof of the lemma is concluded.
\end{proof}

We get the following approximate objective of $\mathcal{L}_{car}^{train}(\theta)$:
\begin{equation}
    \mathcal{L}_{car}^{app}(\theta)=\max_{(s,a,r,s^\prime)\in \mathcal{B}} \frac{1}{\left| \mathcal{B} \right|} \max_{s_\nu \in B_\epsilon(s)} \left|r + \gamma \max_{a^\prime} Q_{\bar{\theta}}(s^\prime,a^\prime) - Q_{\theta}(s_\nu,a) \right|.
    \notag
\end{equation}

Denote 
\begin{equation}
    f_i = f(s_i,a_i,r_i,s_i^\prime):= \max_{s_\nu \in B_\epsilon(s_i)} \left|r_i + \gamma \max_{a^\prime} Q_{\bar{\theta}}(s_i^\prime,a^\prime) - Q_{\theta}(s_\nu,a_i)   \right|.   
    \notag
\end{equation}
To fully utilize each sample in the batch, we derive the soft version of the above objective:
\begin{align}
    \mathcal{L}_{car}^{app}(\theta)&=\max_{(s,a,r,s^\prime)\in \mathcal{B}} \frac{1}{\left| \mathcal{B} \right|} f(s,a,r,s^\prime) \notag\\
    &= \frac{1}{\left| \mathcal{B} \right|} \max_{p\in\Delta\left(\mathcal{B}\right)} \sum_{(s_i,a_i,r_i,s_i^\prime)\in \mathcal{B}} p_if(s_i,a_i,r_i,s_i^\prime) \label{obj: max} \\
    &\ge \frac{1}{\left| \mathcal{B} \right|} \max_{p\in\Delta\left(\mathcal{B}\right)} \left( \sum_{(s_i,a_i,r_i,s_i^\prime)\in \mathcal{B}} p_i f(s_i,a_i,r_i,s_i^\prime) - \lambda \operatorname{KL}\left( p \| U\left(\mathcal{B}\right) \right) \right), \label{obj: soft max}
\end{align}
where $U\left(\mathcal{B}\right)$ represents the uniform distribution over $\mathcal{B}$. According to Lemma \ref{lem: soft}, the optimal solution of the maximization problem (\ref{obj: soft max}) is $p^*$:
\begin{equation}
    p_i^* = \frac{e^{\frac{1}{\lambda} f_i}}{\sum_{i\in \mathcal{\left|B\right|}} e^{\frac{1}{\lambda} f_i}}.
    \notag
\end{equation}
The maximization problem (\ref{obj: soft max}) is the lower bound of the maximization problem (\ref{obj: max}) so $p^*$ is a proper approximation of the optimal solution of the maximization problem (\ref{obj: max}). Thus, we get the soft version of the CAR objective:
\begin{equation}
    \mathcal{L}_{car}^{soft}(\theta) = \frac{1}{\left| \mathcal{B} \right|} \sum_{(s_i,a_i,r_i,s_i^\prime)\in \mathcal{B}} \frac{e^{\frac{1}{\lambda} f_i}}{\sum_{i\in \mathcal{\left|B\right|}} e^{\frac{1}{\lambda} f_i}} \max_{s_\nu \in B_\epsilon(s_i)} \left|r_i + \gamma \max_{a^\prime} Q_{\bar{\theta}}(s_i^\prime,a^\prime) - Q_{\theta}(s_\nu,a_i)  \right|.
    \notag
\end{equation}

\subsection{Derivation of CAR-PPO} \label{app: derivation of car-ppo}

Our theoretical results motivate us to optimize the following objective:
\begin{align*}
    &\quad\ \min_\theta \quad\mathcal{L}_{car}(\theta) \\
    &= \min_\theta \quad \mathcal{D}_{\infty, \operatorname{KL}}^{\mu_t} \left( \operatorname{clip}(\pi_\theta) \| \varphi_t \right) \\
    &= \min_\theta \quad \sup_{s\in\mathcal{S}}\ \mu_t(s) \operatorname{KL} \left( \operatorname{clip}(\pi_\theta(\cdot \mid s)) \| \varphi_t(\cdot \mid s) \right) \\
    &= \min_\theta \quad  \sup_{s\in\mathcal{S}}\ \mu_t(s) \left( -\frac{1}{\beta} \mathcal{H}\left( \pi_\theta(\cdot \mid s) \right) - \sum_{a\in\mathcal{A}} \operatorname{clip}(\pi_\theta (a \mid s) ) A^{\pi_t}(a \mid s) \right) \\
    & \le \min_\theta \quad  \sup_{s\in\mathcal{S}}\ \mu_t(s) \left( -\frac{1}{\beta} \mathcal{H}\left( \pi_\theta(\cdot \mid s) \right) - \min_{s_\nu \in B(s)} \sum_{a\in\mathcal{A}} \operatorname{clip}(\pi_\theta (a \mid s_\nu) ) A^{\pi_t}(a \mid s) \right)  \\
    & = \min_\theta \quad  \sup_{s\in\mathcal{S}}\ \mu_t(s) \left( -\frac{1}{\beta} \mathcal{H}\left( \pi_\theta(\cdot \mid s) \right) - \min_{s_\nu \in B(s)} \mathbb{E}_{a \sim \pi_t(\cdot \mid s)} \operatorname{clip}\left( \frac{\pi_\theta (a \mid s_\nu)}{\pi_t (a \mid s)} \right) A^{\pi_t}(a \mid s) \right) \\
    &:= \min_\theta \quad  \mathcal{L}_{car}^{train}(\theta),
\end{align*}
where $\mathcal{H}(\cdot)$ represents the entropy.

Furthermore, we get the following approximate objective of $\mathcal{L}_{car}^{train}(\theta)$ considering the practice sampling process:
\begin{equation}
    \mathcal{L}_{car}^{app}(\theta) = \frac{1}{|\mathcal{B}|} \max_{(s,a) \in \mathcal{B}} \left( -\frac{1}{\beta} \mathcal{H}\left( \pi_\theta(\cdot | s) \right) - \min_{s_\nu \in B(s)} \operatorname{clip}\left( \frac{\pi_\theta (a | s_\nu)}{\pi_t (a | s)} \right) A^{\pi_t}(a | s) \right).
    \notag
\end{equation}
Denote
\begin{equation}
    f_i = f(s_i, a_i) := -\frac{1}{\beta} \mathcal{H}\left( \pi_\theta(\cdot | s_i) \right) - \min_{s_\nu \in B(s_i)} \operatorname{clip}\left( \frac{\pi_\theta (a_i | s_\nu)}{\pi_t (a_i | s_i)} \right) A^{\pi_t}(a_i | s_i).
    \notag
\end{equation}
To effectively use each sample in a batch, we introduce a soft version of the CAR objective:
\begin{align}
    \mathcal{L}_{car}^{app}(\theta) &= \frac{1}{|\mathcal{B}|} \max_{(s,a) \in \mathcal{B}} f(s_i, a_i) \notag\\
    &= \frac{1}{|\mathcal{B}|} \max_{p\in\Delta\left(\mathcal{B}\right)} \sum_{(s_i, a_i) \in \mathcal{B}} p_i f(s_i, a_i) \label{obj: max ppo} \\
    &\ge \frac{1}{|\mathcal{B}|} \max_{p\in\Delta\left(\mathcal{B}\right)} \left( \sum_{(s_i, a_i) \in \mathcal{B}} p_i f(s_i, a_i) - \lambda \operatorname{KL}\left( p \| U\left(\mathcal{B}\right) \right) \right), \label{obj: soft max ppo}
\end{align}
where $U\left(\mathcal{B}\right)$ represents the uniform distribution over $\mathcal{B}$. According to Lemma \ref{lem: soft}, the optimal solution of the maximization problem (\ref{obj: soft max ppo}) is $p^*$:
\begin{equation}
    p_i^* = \frac{e^{\frac{1}{\lambda} f_i}}{\sum_{i\in \mathcal{\left|B\right|}} e^{\frac{1}{\lambda} f_i}}.
    \notag
\end{equation}
The maximization problem (\ref{obj: soft max ppo}) is the lower bound of the maximization problem (\ref{obj: max ppo}) so $p^*$ is a proper approximation of the optimal solution of the maximization problem (\ref{obj: max ppo}). Thus, we get the soft version of the CAR objective:
\begin{equation}
    \mathcal{L}_{car}^{soft}(\theta) = \frac{1}{|\mathcal{B}|} \sum_{(s_i, a_i) \in \mathcal{B}} \frac{e^{\frac{1}{\lambda} f_i}}{\sum_{i\in \mathcal{\left|B\right|}} e^{\frac{1}{\lambda} f_i}} \left( -\frac{1}{\beta} \mathcal{H}\left( \pi_\theta(\cdot | s_i) \right) - \min_{s_\nu \in B(s_i)} \operatorname{clip}\left( \frac{\pi_\theta (a_i | s_\nu)}{\pi_t (a_i | s_i)} \right) A^{\pi_t}(a_i | s_i) \right).
    \notag
\end{equation}

\section{Additional Algorithm Details}\label{app: additional algorithm details}

\subsection{Overall CAR-DQN and CAR-PPO Algorithm} 
We present the CAR-DQN training algorithm in Algorithm \ref{alg:car-dqn} and the CAR-PPO in Algorithm~\ref{alg:car-ppo}.

\begin{algorithm}[H]
    \caption{Consistent Adversarial Robust Deep Q-network (CAR-DQN).}
    \label{alg:car-dqn}
\begin{algorithmic}[1]
    \REQUIRE Number of iterations $T$, target network update frequency $M$, a schedule $\beta_t$ for the exploration probability $\beta$, a schedule $\epsilon_t$ for the perturbation radius $\epsilon$, soft coefficient~$\lambda$ for soft CAR objective.
    \STATE Initialize current Q network $Q(s,a)$ with parameters $\theta$ and target Q network $Q^\prime(s,a)$ with parameters $\theta^\prime \leftarrow \theta$. 
    \STATE Initialize replay buffer $\mathcal{B}$.
    \FOR{$t=1$ {\bfseries to} $T$}
    \STATE With probability $\beta_t$ select random action $a_t$, otherwise select the greedy action $a_t = \mathop{\arg\max}_a Q(s_t, a;\theta)$.
    \STATE Execute action $a_t$ in environment and observe reward $r_t$ and the next state $s_{t+1}$.
    \STATE Store transition pair $\left\{ s_t, a_t, r_t, s_{t+1} \right\}$ in $\mathcal{B}$.
    \STATE Randomly sample a minibatch of $N$ transition pairs $\left\{ s_i, a_i, r_i, s_{i+1} \right\}$ from $\mathcal{B}$.
    \STATE Set $y_i = r_i + \gamma Q^\prime(s_{i+1}, \mathop{\arg\max}_{a^\prime} Q(s_i, a^\prime;\theta);\theta^\prime)$ for non-terminal $s_i$, and $y_i = r_i$ for terminal $s_i$.
    \STATE Option 1: Use projected gradient descent (PGD) to solve $\mathcal{L}_{car}^{soft}(\theta)$.
    \STATE \qquad For every $i\in \mathcal{\left|B\right|}$, run PGD to solve: $$s_{i,\nu}=\mathop{\arg\max}_{s_\nu \in B_{\epsilon_t}(s_i)}\left|y_i - Q(s_\nu,a_i;\theta)  \right|. $$
    \STATE \qquad Compute the approximation of $\mathcal{L}_{car}^{soft}(\theta)$:
        $$\mathcal{L}_{car}(\theta) = \sum_{i\in \mathcal{\left|B\right|}} \alpha_i \left|y_i - Q(s_{i,\nu},a_i;\theta)  \right|,$$
        \qquad where $ \alpha_i = \frac{e^{\frac{1}{\lambda}  \left|y_i - Q(s_{i,\nu},a_i;\theta)  \right|}}{\sum_{i\in \mathcal{\left|B\right|}} e^{\frac{1}{\lambda}  \left|y_i - Q(s_{i,\nu},a_i;\theta)  \right|}}$.
    \STATE Option 2: Use convex relaxations of neural networks to solve a surrogate loss for~$\mathcal{L}_{car}^{soft}(\theta)$.
    \STATE \qquad For every $i\in \mathcal{\left|B\right|}$, obtain upper and lower bounds on $Q(s, a_i;\theta)$ for all $s \in B_{\epsilon_t}(s_i)$:
        $$u_i(\theta)=\operatorname{ConvexRelaxUB}\left( Q(s, a_i;\theta), \theta,  s \in B_{\epsilon_t}(s_i)\right),$$
        $$l_i(\theta)=\operatorname{ConvexRelaxLB}\left( Q(s, a_i;\theta), \theta,  s \in B_{\epsilon_t}(s_i)\right).$$
    \STATE \qquad Compute the surrogate loss for $\mathcal{L}_{car}^{soft}(\theta)$:
        $$\mathcal{L}_{car}(\theta) = \sum_{i\in \mathcal{\left|B\right|}} \alpha_i \max\left\{ \left| y_i - u_i(\theta)\right|, \left| y_i - l_i(\theta)\right| \right\}, $$
        \qquad where $ \alpha_i = \frac{e^{\frac{1}{\lambda}  \max\left\{ \left| y_i - u_i(\theta)\right|, \left| y_i - l_i(\theta)\right| \right\}}}{\sum_{i\in \mathcal{\left|B\right|}} e^{\frac{1}{\lambda}  \max\left\{ \left| y_i - u_i(\theta)\right|, \left| y_i - l_i(\theta)\right| \right\}}}$.
    \STATE Update the Q network by performing a gradient descent step to minimize $\mathcal{L}_{car}(\theta)$.
    \STATE Update target Q network every $M$ steps: $\theta^\prime \leftarrow \theta$.
    \ENDFOR
\end{algorithmic}
    
\end{algorithm}

\begin{algorithm}[H]
    \caption{Consistent Adversarial Robust Proximal Policy Optimization (CAR-PPO).}
    \label{alg:car-ppo}
\begin{algorithmic}[1]
    \REQUIRE Number of iterations $T$, a schedule $\epsilon_t$ for the perturbation radius $\epsilon$, robustness weighting $\kappa$, soft coefficient $\lambda$ for soft CAR objective.
    \STATE Initialize policy network $\pi_{\theta}(a|s)$ and value network $V_{\theta_V}(s)$.
    \FOR{$t=1$ {\bfseries to} $T$}
    \STATE Run $\pi_{\theta_t}$ in the environment to collect a set of trajectories $\mathcal{D} = \{ \tau_k \}$ containing $|\mathcal{D}|$ episodes, each $\tau_k$ is a trajectory containing $|\tau_k|$ samples, $\tau_k := \{ (s_{k,i}, a_{k,i}, r_{k,i}, s_{k,i+1}) \}, i\in [|\tau_k|]$.
    \STATE Compute reward-to-go $\hat{R}_{k,i}$ for each step $i$ in every episode $k$ using the trajectories and discount factor $\gamma$.
    \STATE Update value network $V_{\theta_V}(s)$ by regression on the mean-square error:
    $$
    \theta_V \leftarrow \underset{\theta_V}{\arg \min }\ \frac{1}{\sum_k\left|\tau_k\right|} \sum_{\tau_k \in D} \sum_{i=0}^{\left|\tau_k\right|}\left(V\left(s_{k, i}\right)-\hat{R}_{k, i}\right)^2.
    $$
    \STATE Estimate advantage $\hat{A}_{k,i}$ for each step $i$ in every episode $k$ using generalized advantage estimation~(GAE) and the current value function $V_{\theta_V}(s)$.
    \STATE Solve $\mathcal{L}_{car}^{soft}(\theta)$ using projected gradient descent (PGD) or stochastic gradient Langevin dynamics~(SGLD):
    \STATE \qquad For all $k, i$, run PGD or SGLD to solve (the objective can be solved in a batch):
    $$s_{i,\nu}=\mathop{\arg\min}_{s_\nu \in B_{\epsilon_t}(s_{k, i})} g\left(\frac{\pi_\theta (a_{k, i} | s_\nu)}{\pi_{\theta_t} (a_{k, i} | s_{k, i})}, \hat{A}_{k, i} \right). $$
    \STATE \qquad Compute the approximation of $\mathcal{L}_{car}^{soft}(\theta)$:
        $$\mathcal{L}_{car}(\theta) = \frac{1}{\sum_k\left|\tau_k\right|} \sum_{\tau_k \in D} \sum_{i=0}^{\left|\tau_k\right|} \alpha_i \left( -\frac{1}{\beta} \mathcal{H}\left( \pi_\theta(\cdot | s_{k, i}) \right) - g\left(\frac{\pi_\theta (a_{k, i} | s_{i,\nu})}{\pi_{\theta_t} (a_{k, i} | s_{k, i})}, \hat{A}_{k, i} \right) \right),$$
        \qquad where
        $$
        \alpha_i = \frac{e^{\frac{1}{\lambda} \left( -\frac{1}{\beta} \mathcal{H}\left( \pi_\theta(\cdot | s_{k, i}) \right) - g\left(\frac{\pi_\theta (a_{k, i} | s_{i,\nu})}{\pi_{\theta_t} (a_{k, i} | s_{k, i})}, \hat{A}_{k, i} \right) \right)}}{\sum_{\tau_k \in D} \sum_{i=0}^{\left|\tau_k\right|} e^{\frac{1}{\lambda} \left( -\frac{1}{\beta} \mathcal{H}\left( \pi_\theta(\cdot | s_{k, i}) \right) - g\left(\frac{\pi_\theta (a_{k, i} | s_{i,\nu})}{\pi_{\theta_t} (a_{k, i} | s_{k, i})}, \hat{A}_{k, i} \right) \right)}}.
        $$
    \STATE Update the policy network by minimizing the vanilla PPO objective and the CAR objective (the minimization is solved using ADAM):
    $$
    \theta_{t+1} \leftarrow \underset{\theta}{\arg \min }\ \frac{1}{\sum_k\left|\tau_k\right|} \sum_{\tau_k \in D} \sum_{i=0}^{\left|\tau_k\right|} g\left(\frac{\pi_\theta (a_{k, i} | s_{k,i})}{\pi_{\theta_t} (a_{k, i} | s_{k, i})}, \hat{A}_{k, i} \right) + \kappa \cdot \mathcal{L}_{car}(\theta).
    $$
    \ENDFOR
\end{algorithmic}
    
\end{algorithm}

\subsection{Additional Implementation Details}

\paragraph{Pre-process in Atari Games.} 
We pre-process the input images into $84\times 84$ grayscale images and normalize pixel values to the range $[0,1]$. In each environment, agents execute an action every 4 frames, skipping the other frames without frame stacking, and all rewards are clipped to the range $[-1,1]$.

\paragraph{DQN Architecture.}
We implement Dueling network architectures~\cite{wang2016dueling} and the same architecture following \cite{zhang2020robust, oikarinen2021robust} which has 3 convolutional layers and a two-head fully connected layers. The first convolutional layer has $8\times 8$ kernel, stride 4, and 32 channels. The second convolutional layer has $4\times 4$ kernel, stride 2, and 64 channels. The third convolutional layer has $3\times 3$ kernel, stride 1, and 64 channels and is then flattened. The fully connected layers have 512 units for both heads wherein one head outputs a state value and the other outputs advantages of each action. The ReLU activation function applies to every middle layer.

\paragraph{PPO Details.}
We implement all PPO agents with the same fully connected (MLP) structure as~\cite{zhang2020robust, oikarinen2021robust}. In the Ant environment, we choose the best regularization parameter $\kappa$ in $\{0.01, 0.03,0.1,0.3,1.0 \}$ for SA-PPO, RADIAL-PPO, and WocaR-PPO to achieve better robustness. For fair and comparable agent selection, we conduct multiple experiments for each setup, repeating each 17 times to account for the inherent performance variability in RL. After training, we attack all agents using random, critic, MAD and RS attacks. Then, we select the median agent by considering natural and these robust returns as our final agent. This chosen agent is then attacked using the SA-RL and PA-AD attacks to further robustness evaluation, because these attacks involve quite high computational costs.

\section{Additional Experiment Results}

\subsection{Additional Comparative Results}
\label{app: add exp}

\textbf{Additional Comparisons.} Comparative results for baselines and CAR-DQN with increasing perturbation budgets are shown in Table~\ref{app table: compare}.  

\textbf{Training Stability.} We also observe that there are some instability phenomena in the training of CAR, RADIAL, and WocaR in Figure \ref{fig: natural rewards during training} and \ref{fig: robustness rewards during training}. We conjecture that the occasional instability in CAR training comes from the unified loss combining natural and robustness objectives which may cause undesirable optimization direction under a batch of special samples. The instability of RADIAL is particularly evident in the robustness curve on the BankHeist environment and natural curve on the Freeway environment and it may be from the larger batch size (=128) setting during the RADIAL training while CAR, SA-DQN and WocaR set the batch size as 32 or 16. The worst-case estimation of WocaR may be inaccurate in some states and WocaR also uses a small batch of size 16. The combination of these two can lead to instability, especially in complex environments such as RoadRunner and BankHeist. Another possible reason is that CAR, RADIAL, and WocaR all use the cheap relaxation method leading to a loose bound while SA-DQN utilizes a tighter relaxation.

\begin{table*}[t]
\caption{Average episode rewards $\pm$ standard error of the mean over 50 episodes on baselines and CAR-DQN on Atari games. The best results of the algorithm with the same type of solver are highlighted in bold. CAR-DQN with the PGD solver outperforms SA-DQN with the PGD solver in all metrics and achieves remarkably better robustness ($110\%$ higher reward) on RoadRunner. CAR-DQN with the convex relaxation solver outperforms baselines in a majority of cases.}
\label{table: compare}
\vskip 0.15in
\resizebox{\textwidth}{!}{%
\begin{tabular}{cc|cccc||cccc}
\hline\hline
\multicolumn{2}{c|}{\multirow{3}{*}{Model}}                                                                         & \multicolumn{4}{c||}{\textbf{Pong}}                                                                                                                                    & \multicolumn{4}{c}{\textbf{BankHeist}}                                                                                                                                    \\ \cline{3-10} 
\multicolumn{2}{c|}{}                                                                                               & \multicolumn{1}{c|}{\multirow{2}{*}{\begin{tabular}[c]{@{}c@{}}Natural\\ Reward\end{tabular}}} & PGD                     & MinBest                 & ACR     & \multicolumn{1}{c|}{\multirow{2}{*}{\begin{tabular}[c]{@{}c@{}}Natural\\ Reward\end{tabular}}} & PGD                       & MinBest                   & ACR     \\ \cline{4-6} \cline{8-10} 
\multicolumn{2}{c|}{}                                                                                               & \multicolumn{1}{c|}{}                                                                          & \multicolumn{3}{c||}{$\epsilon=1/255$}                       & \multicolumn{1}{c|}{}                                                                          & \multicolumn{3}{c}{$\epsilon=1/255$}                            \\ \hline
\multicolumn{1}{c|}{Standard}                                                                      & DQN            & \multicolumn{1}{c|}{$21.0 \pm 0.0$}                                                            & $-21.0 \pm 0.0$         & $-21.0 \pm 0.0$         & $0$     & \multicolumn{1}{c|}{$1317.2 \pm 4.2$}                                                          & $22.2 \pm 1.9$            & $0.0 \pm 0.0$             & $0$     \\ \hline
\multicolumn{1}{c|}{\multirow{2}{*}{PGD}}                                                          & SA-DQN         & \multicolumn{1}{c|}{$21.0 \pm 0.0$}                                                            & $21.0 \pm 0.0$          & $21.0 \pm 0.0$          & $0$     & \multicolumn{1}{c|}{$1248.8 \pm 1.4$}                                                          & $965.8 \pm 35.9$          & $1118.0 \pm 6.3$          & $0$     \\
\multicolumn{1}{c|}{}                                                                              & CAR-DQN (Ours) & \multicolumn{1}{c|}{$21.0 \pm 0.0$}                                                            & $21.0 \pm 0.0$          & $21.0 \pm 0.0$          & $0$     & \multicolumn{1}{c|}{\textbf{$\bf 1307.0 \pm 6.1$}}                                                 & \textbf{$\bf 1243.2 \pm 7.4$} & \textbf{$\bf 1242.6 \pm 8.4$} & $0$     \\ \hline
\multicolumn{1}{c|}{\multirow{4}{*}{\begin{tabular}[c]{@{}c@{}}Convex \\ Relaxation\end{tabular}}} & SA-DQN         & \multicolumn{1}{c|}{$21.0 \pm 0.0$}                                                            & $21.0 \pm 0.0$          & $21.0 \pm 0.0$          & $1.000$ & \multicolumn{1}{c|}{$1236.0 \pm 1.4$}                                                          & $1232.2 \pm 2.5$          & $1232.2 \pm 2.5$          & $0.991$ \\
\multicolumn{1}{c|}{}                                                                              & RADIAL-DQN     & \multicolumn{1}{c|}{$21.0 \pm 0.0$}                                                            & $21.0 \pm 0.0$          & $21.0 \pm 0.0$          & $0.898$ & \multicolumn{1}{c|}{$1341.8 \pm 3.8$}                                                          & $1341.8 \pm 3.8$          & $1341.8 \pm 3.8$          & $0.982$ \\
\multicolumn{1}{c|}{}                                                                              & WocaR-DQN      & \multicolumn{1}{c|}{$21.0 \pm 0.0$}                                                            & $21.0 \pm 0.0$          & $21.0 \pm 0.0$          & $0.979$ & \multicolumn{1}{c|}{$1315.0 \pm 6.1$}                                                          & $1312.0 \pm 6.1$          & $1312.0 \pm 6.1$          & $0.987$ \\
\multicolumn{1}{c|}{}                                                                              & CAR-DQN (Ours) & \multicolumn{1}{c|}{$21.0 \pm 0.0$}                                                            & $21.0 \pm 0.0$          & $21.0 \pm 0.0$          & $0.986$ & \multicolumn{1}{c|}{\textbf{$\bf 1349.6 \pm 3.0$}}                                                 & \textbf{$\bf 1347.6 \pm 3.6$} & \textbf{$\bf 1347.4 \pm 3.6$} & $0.974$ \\ \hline
\hline
\multicolumn{2}{c|}{\multirow{3}{*}{Model}}                                                                         & \multicolumn{4}{c||}{\textbf{Freeway}}                                                                                                                                 & \multicolumn{4}{c}{\textbf{RoadRunner}}                                                                                                                                   \\ \cline{3-10} 
\multicolumn{2}{c|}{}                                                                                               & \multicolumn{1}{c|}{\multirow{2}{*}{\begin{tabular}[c]{@{}c@{}}Natural\\ Reward\end{tabular}}} & PGD                     & MinBest                 & ACR     & \multicolumn{1}{c|}{\multirow{2}{*}{\begin{tabular}[c]{@{}c@{}}Natural\\ Reward\end{tabular}}} & PGD                       & MinBest                   & ACR     \\ \cline{4-6} \cline{8-10} 
\multicolumn{2}{c|}{}                                                                                               & \multicolumn{1}{c|}{}                                                                          & \multicolumn{3}{c||}{$\epsilon=1/255$}                       & \multicolumn{1}{c|}{}                                                                          & \multicolumn{3}{c}{$\epsilon=1/255$}                            \\ \hline
\multicolumn{1}{c|}{Standard}                                                                      & DQN            & \multicolumn{1}{c|}{$33.9 \pm 0.0$}                                                            & $0.0 \pm 0.0$           & $0.0 \pm 0.0$           & $0$     & \multicolumn{1}{c|}{$41492 \pm 903$}                                                           & $0 \pm 0$                 & $0 \pm 0$                 & $0$     \\ \hline
\multicolumn{1}{c|}{\multirow{2}{*}{PGD}}                                                          & SA-DQN         & \multicolumn{1}{c|}{$33.6 \pm 0.1$}                                                            & $23.4 \pm 0.2$          & $21.1 \pm 0.2$          & $0.250$ & \multicolumn{1}{c|}{$33380 \pm 611$}                                                           & $20482 \pm 1087$          & $24632 \pm 812$           & $0$     \\
\multicolumn{1}{c|}{}                                                                              & CAR-DQN (Ours) & \multicolumn{1}{c|}{\textbf{$\bf 34.0 \pm 0.0$}}                                                   & \textbf{$\bf 33.7 \pm 0.1$} & \textbf{$\bf 33.7 \pm 0.1$} & $0$     & \multicolumn{1}{c|}{\textbf{$\bf 49700 \pm 1015$}}                                                 & \textbf{$\bf 43286 \pm 801$}  & \textbf{$\bf 48908 \pm 1107$} & $0$     \\ \hline
\multicolumn{1}{c|}{\multirow{4}{*}{\begin{tabular}[c]{@{}c@{}}Convex \\ Relaxation\end{tabular}}} & SA-DQN         & \multicolumn{1}{c|}{$30.0 \pm 0.0$}                                                            & $30.0 \pm 0.0$          & $30.0 \pm 0.0$          & $1.000$ & \multicolumn{1}{c|}{$46372 \pm 882$}                                                           & $44960\pm 1152$           & $45226 \pm 1102$          & $0.819$ \\
\multicolumn{1}{c|}{}                                                                              & RADIAL-DQN     & \multicolumn{1}{c|}{$33.1 \pm 0.1$}                                                            & \textbf{$\bf 33.3 \pm 0.1$} & \textbf{$\bf 33.3 \pm 0.1$} & $0.998$ & \multicolumn{1}{c|}{$46224\pm 1133$}                                                           & $45990 \pm 1112$          & $46082 \pm 1128$          & $0.994$ \\
\multicolumn{1}{c|}{}                                                                              & WocaR-DQN      & \multicolumn{1}{c|}{$30.8 \pm 0.1$}                                                            & $31.0 \pm 0.0$          & $31.0 \pm 0.0$          & $0.992$ & \multicolumn{1}{c|}{$43686 \pm 1608$}                                                          & $45636 \pm 706$           & $45636 \pm 706$           & $0.956$ \\
\multicolumn{1}{c|}{}                                                                              & CAR-DQN (Ours) & \multicolumn{1}{c|}{\textbf{$\bf 33.2 \pm 0.1$}}                                                   & $33.2 \pm 0.1$          & $33.2 \pm 0.1$          & $0.981$ & \multicolumn{1}{c|}{\textbf{$\bf 49398 \pm 1106$}}                                                 & \textbf{$\bf 49456 \pm 992$}  & \textbf{$\bf 47526 \pm 1132$} & $0.760$ \\ \hline \hline
\end{tabular}%
}
\end{table*}

\begin{table*}[t]
\caption{Average episode rewards $\pm$ standard error of the mean over 50 episodes on baselines and CAR-DQN. The best results of the algorithm with the same type of solver are highlighted in bold.}
\label{app table: compare}
\vskip 0.15in
\resizebox{\textwidth}{!}{%
\begin{tabular}{c|cc|c|ccc|ccc|ccc}
\hline \hline
\multirow{2}{*}{Environment} & \multicolumn{2}{c|}{\multirow{2}{*}{Model}}                                                                         & \multirow{2}{*}{Natural Return} & \multicolumn{3}{c|}{PGD}                                               & \multicolumn{3}{c|}{MinBest}                                                      & \multicolumn{3}{c}{ACR}                               \\
                             & \multicolumn{2}{c|}{}                                                                                               &                                 & $\epsilon=1/255$          & $\epsilon=3/255$          & $\epsilon=5/255$          & $\epsilon=1/255$          & $\epsilon=3/255$          & $\epsilon=5/255$          & $\epsilon=1/255$ & $\epsilon=3/255$ & $\epsilon=5/255$ \\ \hline
\multirow{7}{*}{\textbf{Pong}}        & \multicolumn{1}{c|}{Standard}                                                                      & DQN            & $21.0 \pm 0.0$                  & $-21.0 \pm 0.0$           & $-21.0 \pm 0.0$           & $-20.8 \pm 0.1$           & $-21.0 \pm 0.0$           & $-21.0 \pm 0.0$           & $-21.0 \pm 0.0$           & $0$              & $0$              & $0$              \\ \cline{2-13} 
                             & \multicolumn{1}{c|}{\multirow{2}{*}{PGD}}                                                          & SA-DQN         & \textbf{$\bf 21.0 \pm 0.0$}         & \textbf{$\bf 21.0 \pm 0.0$}   & $-19.4 \pm 0.3$           & $-21.0 \pm 0.0$           & \textbf{$\bf 21.0 \pm 0.0$}   & $-19.4 \pm 0.2$           & $-21.0 \pm 0.0$           & $0$              & $0$              & $0$              \\
                             & \multicolumn{1}{c|}{}                                                                              & CAR-DQN (Ours) & \textbf{$\bf 21.0 \pm 0.0$}         & \textbf{$\bf 21.0 \pm 0.0$}   & $\bf 16.8 \pm 0.7$            & $-21.0 \pm 0.0$           & \textbf{$\bf 21.0 \pm 0.0$}   & $\bf 20.7 \pm 0.1$            & $\bf -0.8 \pm 2.8$            & $0$              & $0$              & $0$              \\ \cline{2-13} 
                             & \multicolumn{1}{c|}{\multirow{4}{*}{\begin{tabular}[c]{@{}c@{}}Convex \\ Relaxation\end{tabular}}} & SA-DQN         & \textbf{$\bf 21.0 \pm 0.0$}         & \textbf{$\bf 21.0 \pm 0.0$}   & \textbf{$\bf 21.0 \pm 0.0$}   & $-19.6 \pm 0.1$           & \textbf{$\bf 21.0 \pm 0.0$}   & \textbf{$\bf 21.0 \pm 0.0$}   & $-9.5 \pm 1.3$            & $1.000$          & $0$              & $0$              \\
                             & \multicolumn{1}{c|}{}                                                                              & RADIAL-DQN     & \textbf{$\bf 21.0 \pm 0.0$}         & \textbf{$\bf 21.0 \pm 0.0$}   & \textbf{$\bf 21.0 \pm 0.0$}   & \textbf{$\bf 21.0 \pm 0.0$}   & \textbf{$\bf 21.0 \pm 0.0$}   & \textbf{$\bf 21.0 \pm 0.0$}   & $4.9 \pm 0.6$             & $0.898$          & $0$              & $0$              \\
                             & \multicolumn{1}{c|}{}                                                                              & WocaR-DQN      & \textbf{$\bf 21.0 \pm 0.0$}         & \textbf{$\bf 21.0 \pm 0.0$}   & $20.5 \pm 0.1$            & $20.6 \pm 0.1$            & \textbf{$\bf 21.0 \pm 0.0$}   & $20.7 \pm 0.1$            & $20.9 \pm 0.1$            & $0.979$          & $0$              & $0$              \\
                             & \multicolumn{1}{c|}{}                                                                              & CAR-DQN (Ours) & \textbf{$\bf 21.0 \pm 0.0$}         & \textbf{$\bf 21.0 \pm 0.0$}   & \textbf{$\bf 21.0 \pm 0.0$}   & \textbf{$\bf 21.0 \pm 0.0$}   & \textbf{$\bf 21.0 \pm 0.0$}   & \textbf{$\bf 21.0 \pm 0.0$}   & \textbf{$\bf 21.0 \pm 0.0$}   & $0.986$          & $0$              & $0$              \\ \hline \hline
\multirow{7}{*}{\textbf{Freeway}}     & \multicolumn{1}{c|}{Standard}                                                                      & DQN            & $33.9 \pm 0.0$                  & $0.0 \pm 0.0$             & $0.0 \pm 0.0$             & $0.0 \pm 0.0$             & $0.0 \pm 0.0$             & $0.0 \pm 0.0$             & $0.0 \pm 0.0$             & $0$              & $0$              & $0$              \\ \cline{2-13} 
                             & \multicolumn{1}{c|}{\multirow{2}{*}{PGD}}                                                          & SA-DQN         & $33.6 \pm 0.1$                  & $23.4 \pm 0.2$            & $20.6 \pm 0.3$            & \textbf{$\bf 7.6 \pm 0.3$}    & $21.1 \pm 0.2$            & $21.3 \pm 0.2$            & $21.8 \pm 0.3$            & $0.250$          & $0.275$          & $0.275$          \\
                             & \multicolumn{1}{c|}{}                                                                              & CAR-DQN (Ours) & \textbf{$\bf 34.0 \pm 0.0$}         & \textbf{$\bf 33.7 \pm 0.1$}   & \textbf{$\bf 25.8 \pm 0.2$}   & $3.8 \pm 0.2$             & \textbf{$\bf 33.7 \pm 0.1$}   & \textbf{$\bf 30.0 \pm 0.3$}   & \textbf{$\bf 26.2 \pm 0.2$}   & $0$              & $0$              & $0$              \\ \cline{2-13} 
                             & \multicolumn{1}{c|}{\multirow{4}{*}{\begin{tabular}[c]{@{}c@{}}Convex \\ Relaxation\end{tabular}}} & SA-DQN         & $30.0 \pm 0.0$                  & $30.0 \pm 0.0$            & $30.2 \pm 0.1$            & $27.7 \pm 0.1$            & $30.0 \pm 0.0$            & $30.0 \pm 0.0$            & $29.2 \pm 0.1$            & $1.000$          & $0.912$          & $0$              \\
                             & \multicolumn{1}{c|}{}                                                                              & RADIAL-DQN     & $33.1 \pm 0.1$                  & \textbf{$\bf 33.3 \pm 0.1$}   & \textbf{$\bf 33.3 \pm 0.1$}   & \textbf{$\bf 29.0 \pm 0.1$}   & \textbf{$\bf 33.3 \pm 0.1$}   & \textbf{$\bf 33.3 \pm 0.1$}   & \textbf{$\bf 31.2 \pm 0.2$}   & $0.998$          & $0$              & $0$              \\
                             & \multicolumn{1}{c|}{}                                                                              & WocaR-DQN      & $30.8 \pm 0.1$                  & $31.0 \pm 0.0$            & $30.6 \pm 0.1$            & $29.0 \pm 0.2$            & $31.0 \pm 0.0$            & $31.1 \pm 0.1$            & $29.0 \pm 0.2$            & $0.992$          & $0.150$          & $0$              \\
                             & \multicolumn{1}{c|}{}                                                                              & CAR-DQN (Ours) & \textbf{$\bf 33.2 \pm 0.1$}         & $33.2 \pm 0.1$            & $32.3 \pm 0.2$            & $27.6 \pm 0.3$            & $33.2 \pm 0.1$            & $32.8 \pm 0.2$            & $31.0 \pm 0.1$            & $0.981$          & $0$              & $0$              \\ \hline \hline
\multirow{7}{*}{\textbf{BankHeist}}   & \multicolumn{1}{c|}{Standard}                                                                      & DQN            & $1317.2 \pm 4.2$                & $22.2 \pm 1.9$            & $0.0 \pm 0.0$             & $0.0 \pm 0.0$             & $0.0 \pm 0.0$             & $0.0 \pm 0.0$             & $0.0 \pm 0.0$             & $0$              & $0$              & $0$              \\ \cline{2-13} 
                             & \multicolumn{1}{c|}{\multirow{2}{*}{PGD}}                                                          & SA-DQN         & $1248.8 \pm 1.4$                & $965.8 \pm 35.9$          & $35.6 \pm 3.4$            & $0.6 \pm 0.3$             & $1118.0 \pm 6.3$          & $50.8 \pm 2.5$            & $4.8 \pm 0.7$             & $0$              & $0$              & $0$              \\
                             & \multicolumn{1}{c|}{}                                                                              & CAR-DQN (Ours) & \textbf{$\bf 1307.0 \pm 6.1$}       & \textbf{$\bf 1243.2 \pm 7.4$} & \textbf{$\bf 908.2 \pm 17.0$} & \textbf{$\bf 83.0 \pm 2.2$}   & \textbf{$\bf 1242.6 \pm 8.4$} & \textbf{$\bf 970.8 \pm 9.6$}  & \textbf{$\bf 819.4 \pm 9.0$}  & $0$              & $0$              & $0$              \\ \cline{2-13} 
                             & \multicolumn{1}{c|}{\multirow{4}{*}{\begin{tabular}[c]{@{}c@{}}Convex \\ Relaxation\end{tabular}}} & SA-DQN         & $1236.0 \pm 1.4$                & $1232.2 \pm 2.5$          & $1208.8 \pm 1.7$          & $1029.8 \pm 34.6$         & $1232.2 \pm 2.5$          & $1214.8 \pm 2.6$          & $1051.0 \pm 35.5$         & $0.991$          & $0.409$          & $0$              \\
                             & \multicolumn{1}{c|}{}                                                                              & RADIAL-DQN     & $1341.8 \pm 3.8$                & $1341.8 \pm 3.8$          & \textbf{$\bf 1346.4 \pm 3.2$} & $1092.6 \pm 37.8$         & $1341.8 \pm 3.8$          & $1328.6 \pm 5.4$          & $732.6 \pm 11.5$          & $0.982$          & $0$              & $0$              \\
                             & \multicolumn{1}{c|}{}                                                                              & WocaR-DQN      & $1315.0 \pm 6.1$                & $1312.0 \pm 6.1$          & $1323.4 \pm 2.2$          & $1094.0 \pm 10.2$         & $1312.0 \pm 6.1$          & $1301.6 \pm 3.9$          & $1041.4 \pm 17.4$         & $0.987$          & $0.093$          & $0$              \\
                             & \multicolumn{1}{c|}{}                                                                              & CAR-DQN (Ours) & \textbf{$\bf 1349.6 \pm 3.0$}       & \textbf{$\bf 1347.6 \pm 3.6$} & $1332.0 \pm 7.3$          & \textbf{$\bf 1191.0 \pm 9.0$} & \textbf{$\bf 1347.4 \pm 3.6$} & \textbf{$\bf 1338.0 \pm 2.9$} & \textbf{$\bf 1233.6 \pm 5.0$} & $0.974$          & $0$              & $0$              \\ \hline \hline
\multirow{7}{*}{\textbf{RoadRunner}}  & \multicolumn{1}{c|}{Standard}                                                                      & DQN            & $41492 \pm 903$                 & $0 \pm 0$                 & $0 \pm 0$                 & $0 \pm 0$                 & $0 \pm 0$                 & $0 \pm 0$                 & $0 \pm 0$                 & $0$              & $0$              & $0$              \\ \cline{2-13} 
                             & \multicolumn{1}{c|}{\multirow{2}{*}{PGD}}                                                          & SA-DQN         & $33380 \pm 611$                 & $20482 \pm 1087$          & $0 \pm 0$                 & $0 \pm 0$                 & $24632 \pm 812$           & $614 \pm 72$              & $214 \pm 26$              & $0$              & $0$              & $0$              \\
                             & \multicolumn{1}{c|}{}                                                                              & CAR-DQN (Ours) & \textbf{$\bf 49700 \pm 1015$}       & \textbf{$\bf 43286 \pm 801$}  & \textbf{$\bf 25740 \pm 1468$} & \textbf{$\bf 2574 \pm 261$}   & \textbf{$\bf 48908 \pm 1107$} & \textbf{$\bf 35882 \pm 904$}  & \textbf{$\bf 23218 \pm 698$}  & $0$              & $0$              & $0$              \\ \cline{2-13} 
                             & \multicolumn{1}{c|}{\multirow{4}{*}{\begin{tabular}[c]{@{}c@{}}Convex \\ Relaxation\end{tabular}}} & SA-DQN         & $46372 \pm 882$                 & $44960\pm 1152$           & $20910 \pm 827$           & $3074 \pm 179$            & $45226 \pm 1102$          & $25548 \pm 737$           & $12324 \pm 529$           & $0.819$          & $0$              & $0$              \\
                             & \multicolumn{1}{c|}{}                                                                              & RADIAL-DQN     & $46224\pm 1133$                 & $45990 \pm 1112$          & \textbf{$\bf 42162 \pm 1147$} & \textbf{$\bf 23248 \pm 499$}  & $46082 \pm 1128$          & \textbf{$\bf 42036 \pm 1048$} & \textbf{$\bf 25434 \pm 756$}  & $0.994$          & $0$              & $0$              \\
                             & \multicolumn{1}{c|}{}                                                                              & WocaR-DQN      & $43686 \pm 1608$                & $45636 \pm 706$           & $19386 \pm 721$           & $6538 \pm 464$            & $45636 \pm 706$           & $21068 \pm 1026$          & $15050 \pm 683$           & $0.956$          & $0$              & $0$              \\
                             & \multicolumn{1}{c|}{}                                                                              & CAR-DQN (Ours) & \textbf{$\bf 49398 \pm 1106$}       & \textbf{$\bf 49456 \pm 992$}  & $28588 \pm 1575$          & $15592 \pm 885$           & \textbf{$\bf 47526 \pm 1132$} & $32878 \pm 1898$          & $21102 \pm 1427$          & $0.760$          & $0$              & $0$              \\ \hline \hline
\end{tabular}%
}
\end{table*}

\begin{figure}[t]
    \centering
    \includegraphics[width=0.23\textwidth]{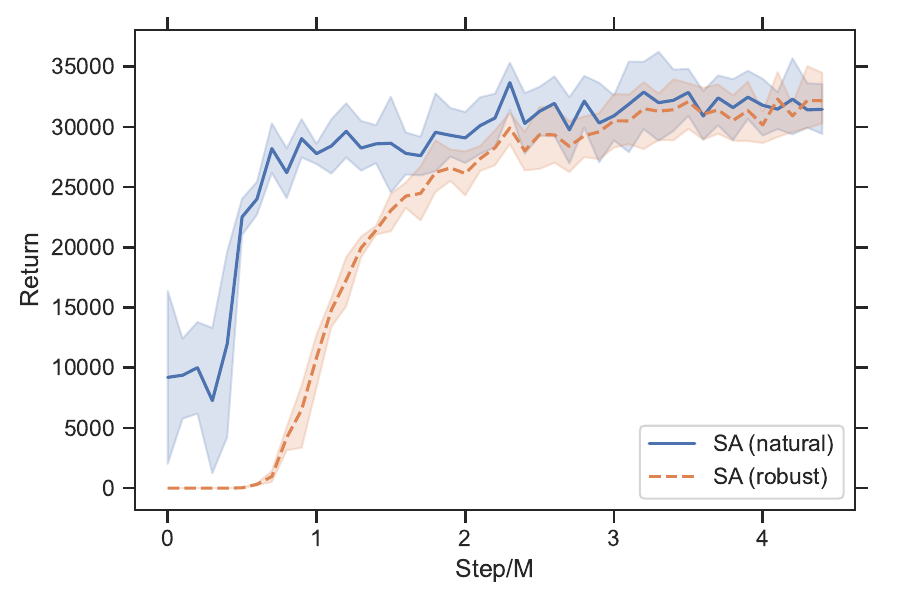}
    \includegraphics[width=0.23\textwidth]{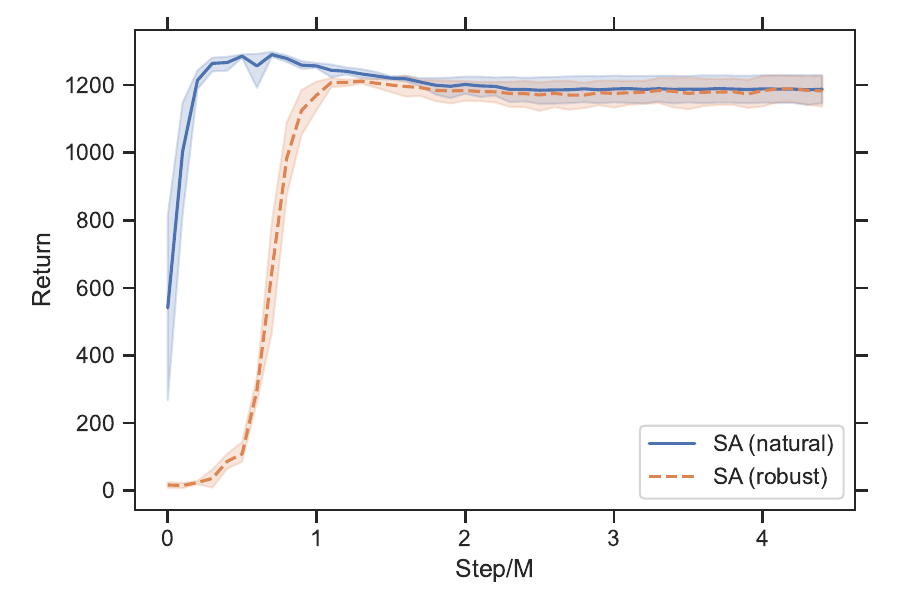}
    \includegraphics[width=0.23\textwidth]{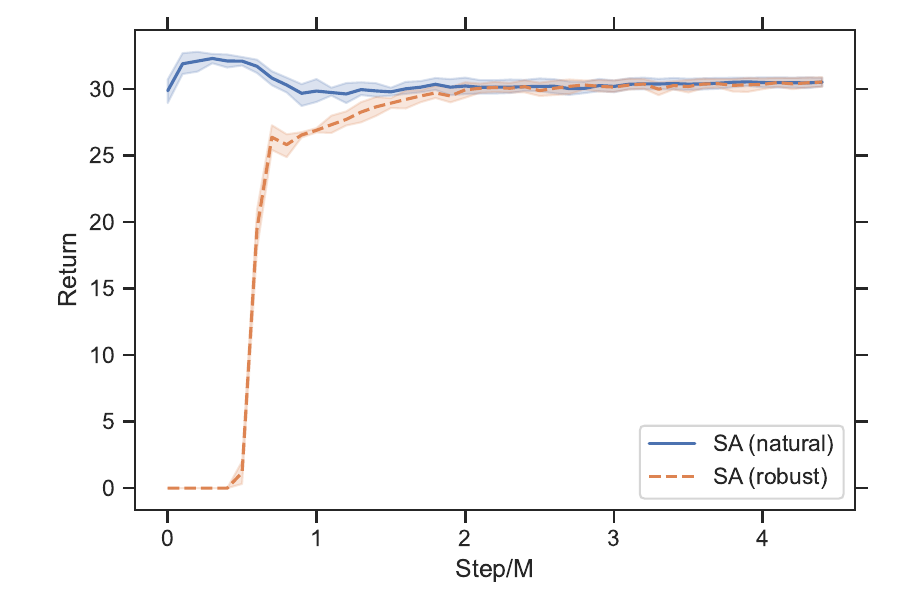}
    \includegraphics[width=0.23\textwidth]{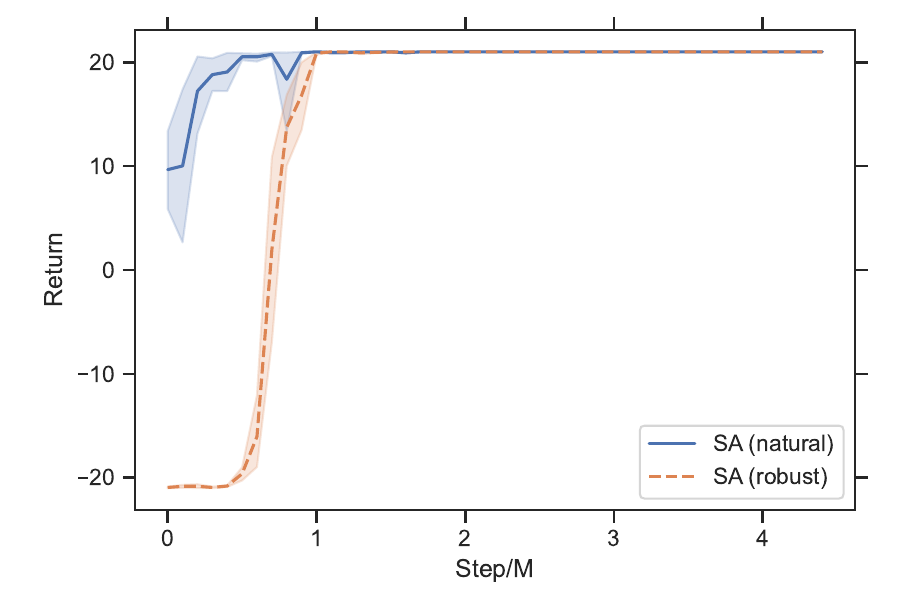}
    \caption{Natural and robustness performance exhibited
by SA-DQN agents during the training process on 4 Atari games.}
    \label{app fig:sa NRreturn}
\end{figure}

\begin{figure}[t]
    \centering
\includegraphics[width=0.23\textwidth]{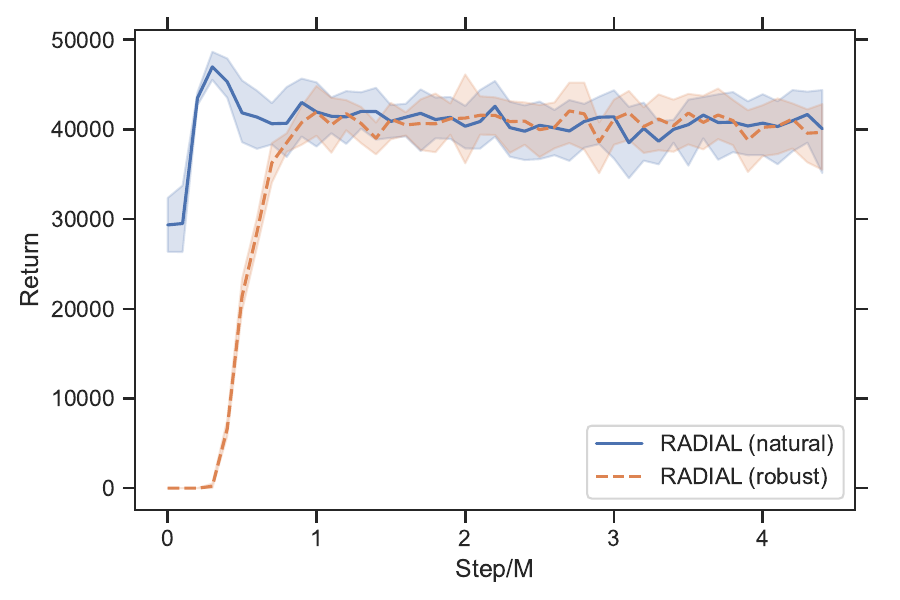}
\includegraphics[width=0.23\textwidth]{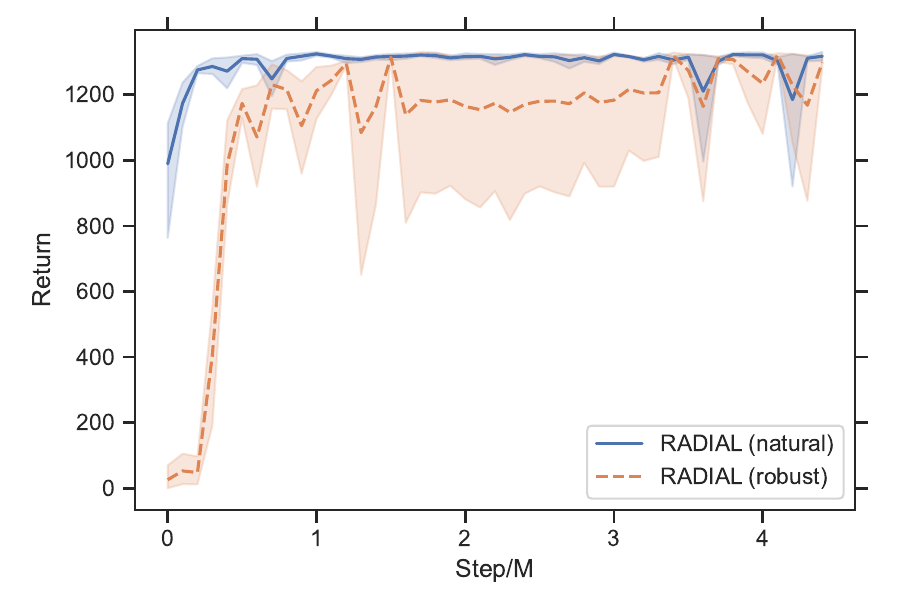}
\includegraphics[width=0.23\textwidth]{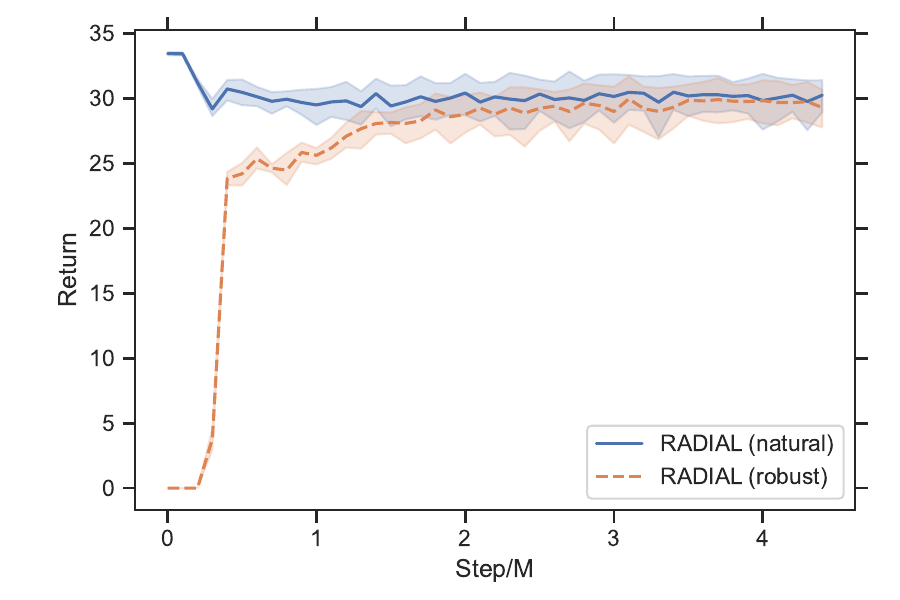}
\includegraphics[width=0.23\textwidth]{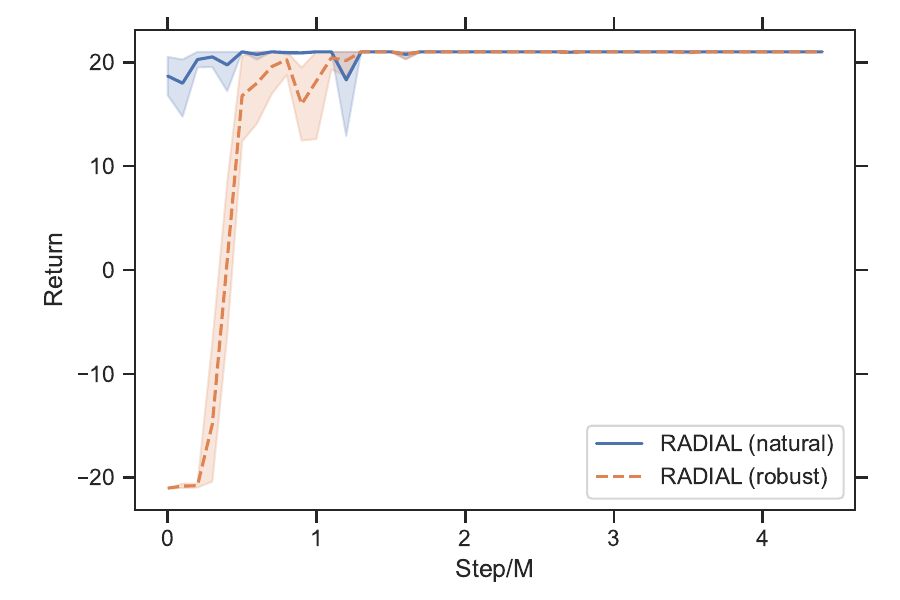}
    \caption{Natural and robustness performance exhibited
by RADIAL-DQN agents during the training process on 4 Atari games.}
    \label{app fig:ra NRreturn}
\end{figure}

\begin{figure}[t]
    \centering
\includegraphics[width=0.23\textwidth]{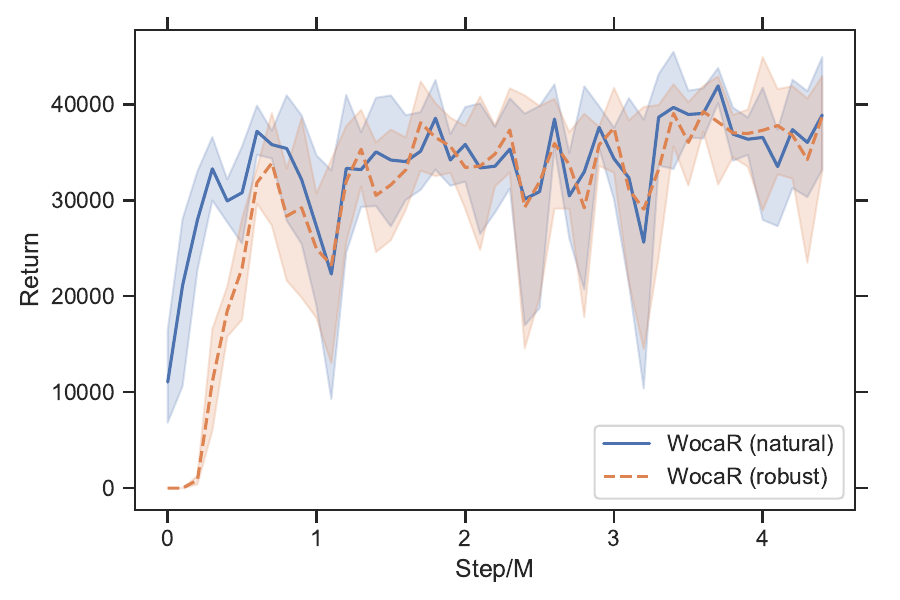}
\includegraphics[width=0.23\textwidth]{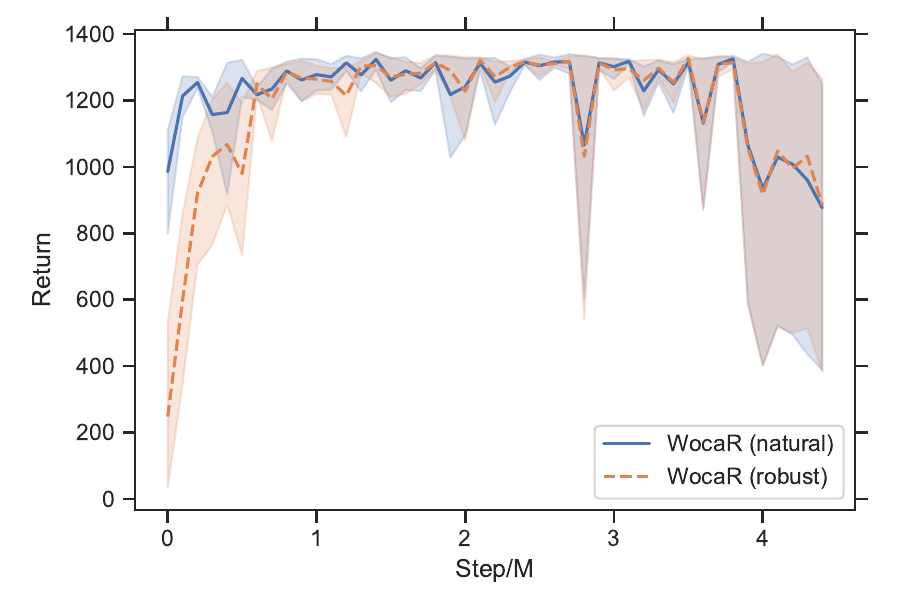}
\includegraphics[width=0.23\textwidth]{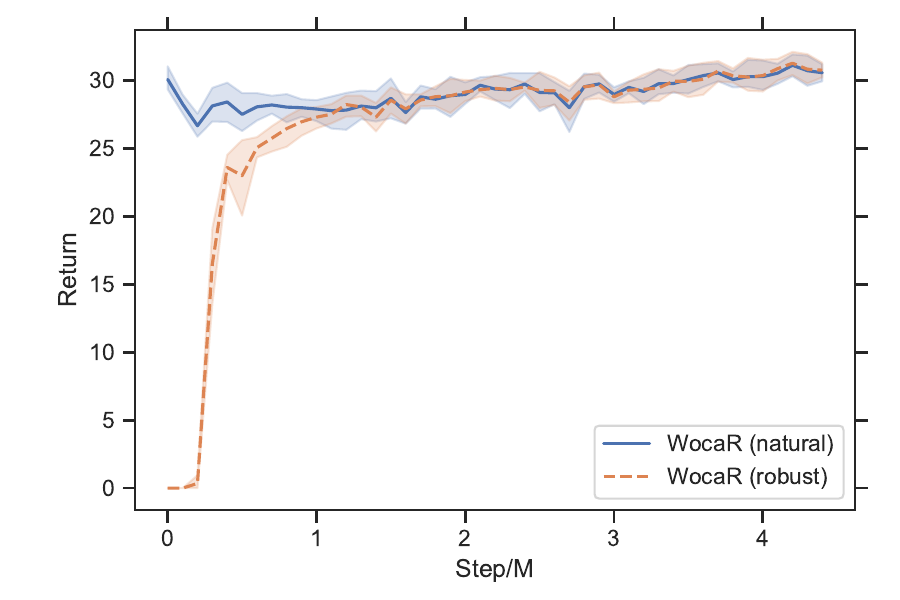}
\includegraphics[width=0.23\textwidth]{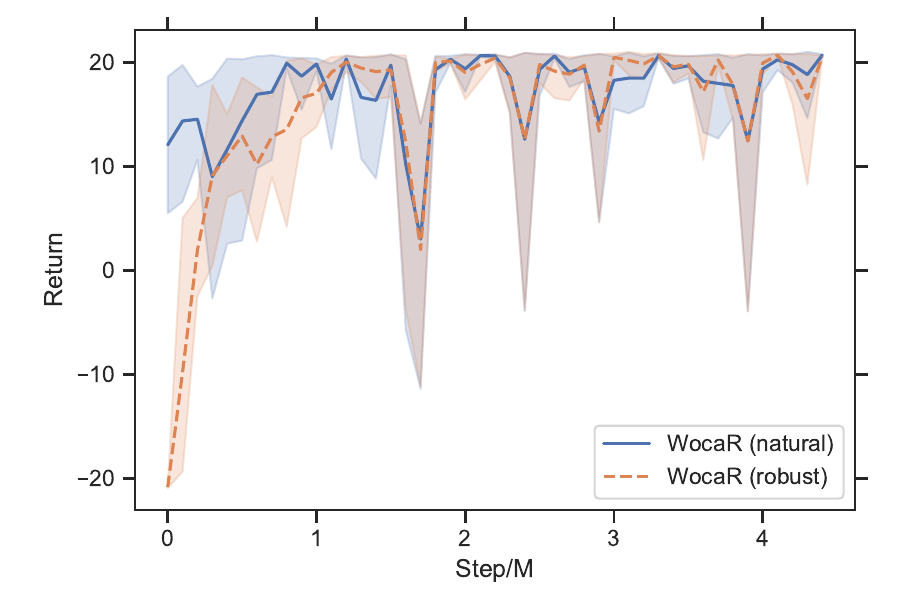}
    \caption{Natural and robustness performance exhibited
by WocaR-DQN agents during the training process on 4 Atari games.}
    \label{app fig:wocar NRreturn}
\end{figure}

\begin{table}[t]
\caption{Performance of PPO with different $k$-measurement errors.}
\label{app table: k measurement}
\vskip 0.15in
\resizebox{\columnwidth}{!}{%
\begin{tabular}{c|c|c|ccccccc}
\hline
\multirow{2}{*}{\textbf{Env}}                                                            & \multirow{2}{*}{\textbf{$k$-measurement}} & \multirow{2}{*}{\textbf{\begin{tabular}[c]{@{}c@{}}Natural\\ Reward\end{tabular}}} & \multicolumn{7}{c}{\textbf{Attack Reward}}                                                                           \\
                                                                                         &                                           &                                                                                    & \textbf{Random} & \textbf{Critic} & \textbf{MAD}  & \textbf{RS}   & \textbf{SA-RL} & \textbf{PA-AD} & \textbf{Worst} \\ \hline
\multirow{3}{*}{\begin{tabular}[c]{@{}c@{}}Hopper\\ ($\epsilon$=0.075)\end{tabular}}     & $k=1$                                     & 3081                                                                               & 2923            & 2035            & 1763          & 756           & 79             & 823            & 79             \\
                                                                                         & $k=2$                                     & 3483                                                                               & 3461            & 1506            & 2472          & 573           & 375            & 402            & 375            \\
                                                                                         & $k=\infty$                                & \textbf{3711}                                                                      & \textbf{3702}   & \textbf{3692}   & \textbf{3473} & \textbf{1652} & \textbf{2430}  & \textbf{2640}  & \textbf{1652}  \\ \hline
\multirow{3}{*}{\begin{tabular}[c]{@{}c@{}}Walker2d\\ ($\epsilon$=0.05)\end{tabular}}    & $k=1$                                     & 4622                                                                               & 4628            & 4584            & 4507          & 1062          & 719            & 336            & 336            \\
                                                                                         & $k=2$                                     & 4738                                                                               & 4651            & 4620            & 4121          & 923           & 1571           & 424            & 424            \\
                                                                                         & $k=\infty$                                & \textbf{4755}                                                                      & \textbf{4848}   & \textbf{5044}   & \textbf{4637} & \textbf{4379} & \textbf{4307}  & \textbf{4303}  & \textbf{4303}  \\ \hline
\multirow{3}{*}{\begin{tabular}[c]{@{}c@{}}Halfcheetah\\ ($\epsilon$=0.15)\end{tabular}} & $k=1$                                     & 5048                                                                               & 4463            & 3281            & 918           & 1049          & -213           & -69            & -213           \\
                                                                                         & $k=2$                                     & 4370                                                                               & 3857            & 3295            & 956           & 441           & -160           & -192           & -192           \\
                                                                                         & $k=\infty$                                & \textbf{5053}                                                                      & \textbf{5058}   & \textbf{5065}   & \textbf{5051} & \textbf{5140} & \textbf{4860}  & \textbf{4942}  & \textbf{4860}  \\ \hline
\multirow{3}{*}{\begin{tabular}[c]{@{}c@{}}Ant\\ ($\epsilon$=0.15)\end{tabular}}         & $k=1$                                     & 5381                                                                               & \textbf{5329}   & 4696            & 1768          & 1097          & -1398          & -3107          & -3107          \\
                                                                                         & $k=2$                                     & \textbf{5485}                                                                      & 5036            & 4705            & 1199          & 1001          & -1470          & -1108          & -1470          \\
                                                                                         & $k=\infty$                                & 5029                                                                               & 5006            & \textbf{4786}   & \textbf{4549} & \textbf{3553} & \textbf{3099}  & \textbf{3911}  & \textbf{3099} \\ \hline
\end{tabular}%
}
\end{table}

\begin{figure}[t]
    \centering
\includegraphics[width=0.65\columnwidth]{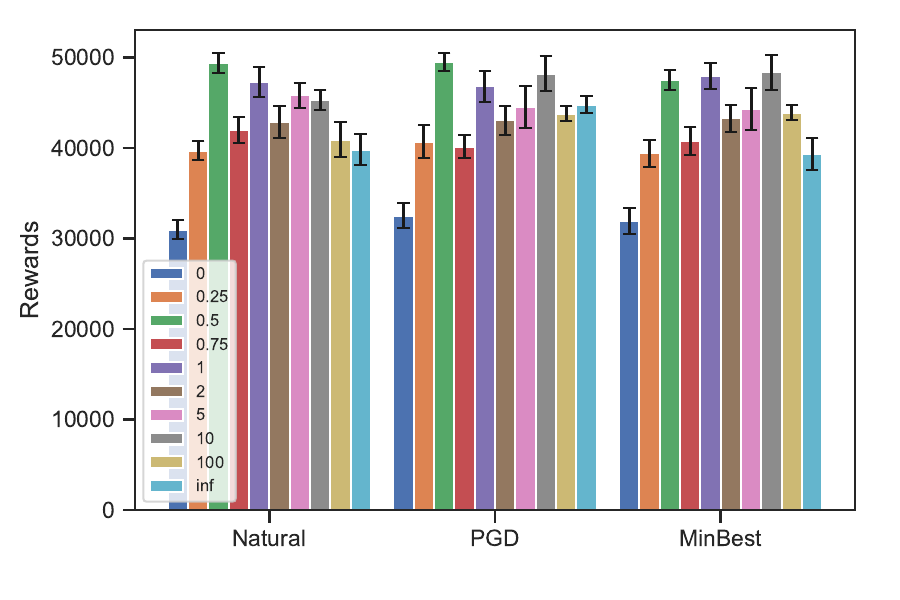}
    \vskip -0.1in
    \caption{Natural, PGD attack and MinBest attack rewards of CAR-DQN with different soft coefficients on RoadRunner.}
    \label{app fig:soft roadrunner}
\end{figure}

\begin{figure}[t]
    \centering
    \begin{minipage}{0.4\textwidth}
        \centering
        \setlength{\parindent}{0.5em}
        \quad \scriptsize{RoadRunner}
    \end{minipage}
    \begin{minipage}{0.4\textwidth}
        \centering
        \setlength{\parindent}{1em}
        \quad \scriptsize{BankHeist}
    \end{minipage}
    \\
    \rotatebox{90}{\tiny{\qquad \qquad \quad Robust}}
    \begin{subfigure}
        \centering
        \includegraphics[width=0.4\textwidth]{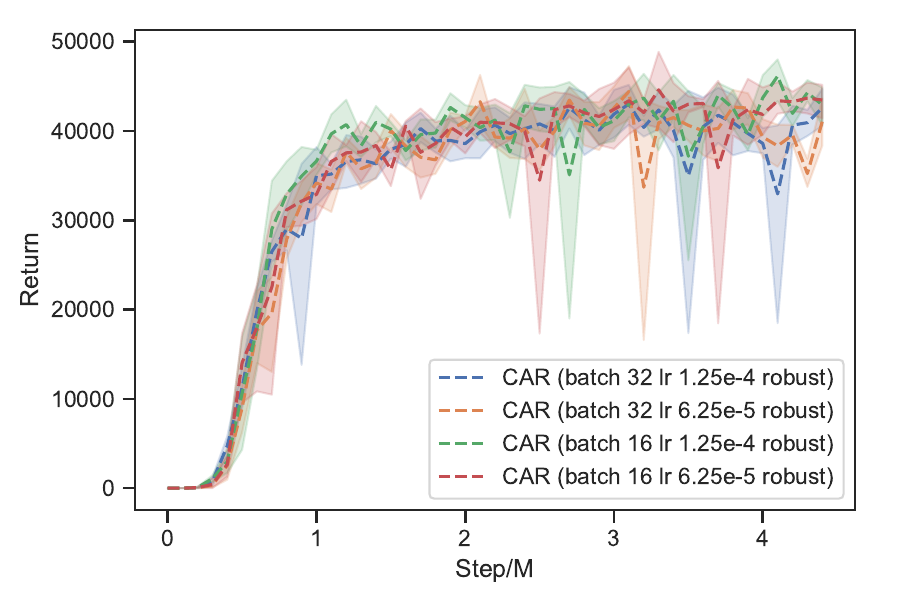}        
    \end{subfigure}
    \begin{subfigure}
        \centering
        \includegraphics[width=0.4\textwidth]{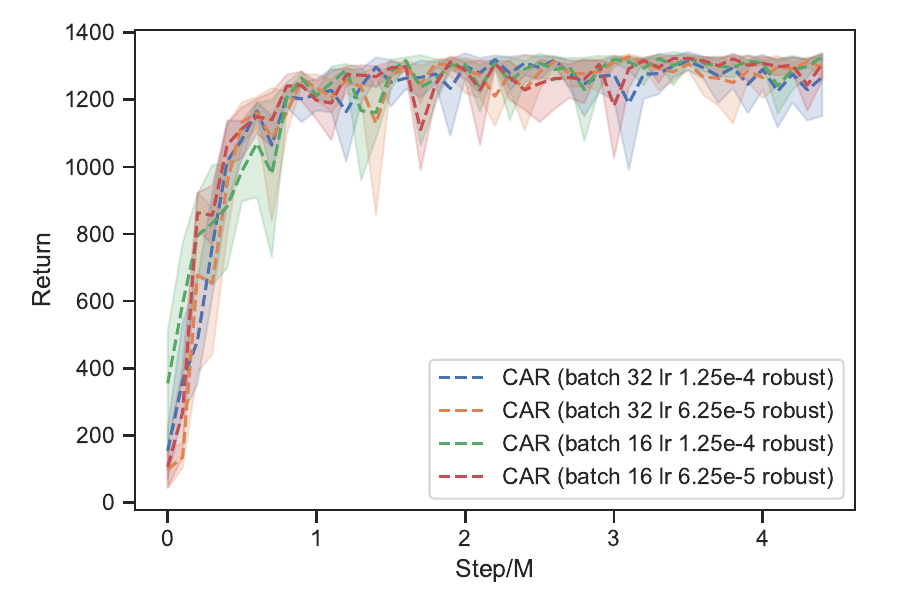}
    \end{subfigure}
    \\
    \rotatebox{90}{\tiny{\qquad \qquad \quad Natural}}
    \begin{subfigure}
        \centering
        \includegraphics[width=0.4\textwidth]{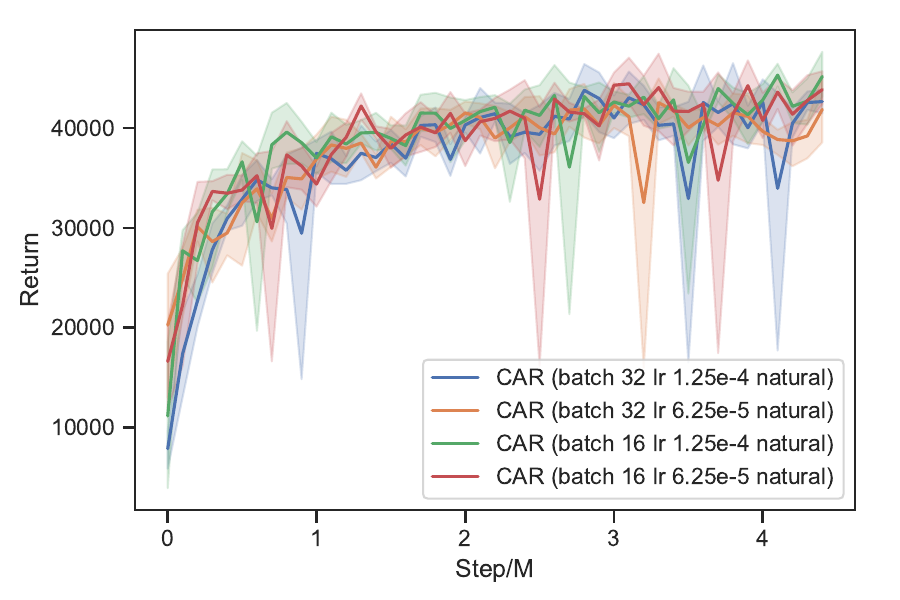}
    \end{subfigure}
    \begin{subfigure}
        \centering
        \includegraphics[width=0.4\textwidth]{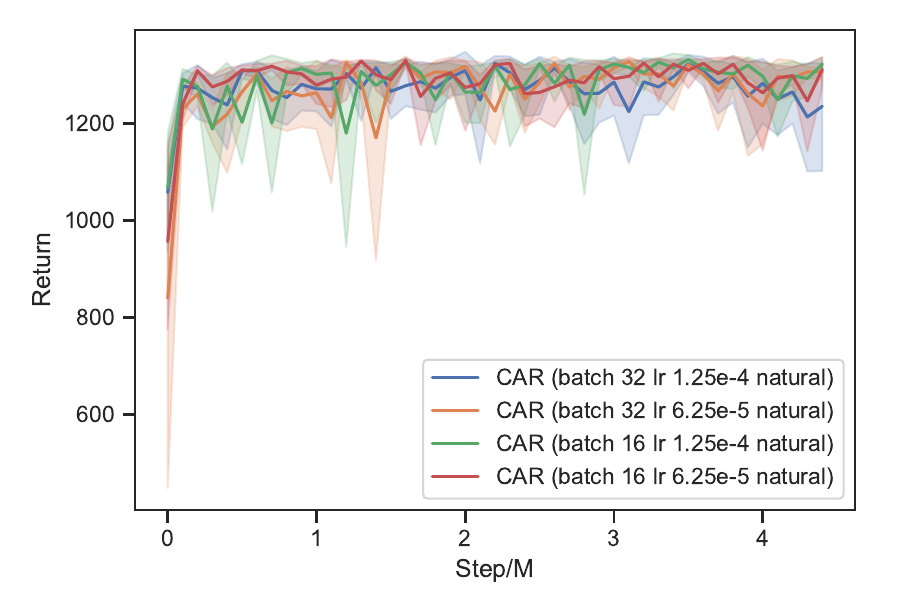}
    \end{subfigure}
    \caption{Episode rewards of CAR-DQN with different batch sizes and learning rates during training on RoadRunner and BankHeist with and without PGD attack.}
    \label{fig:batch and lr}
\end{figure}

\textbf{Insensitivity of Learning Rate and Batch Size.} 
We compare the performance of CAR-DQN with different small batch size $(16, 32)$ and learning rate $(1.25\times 10^{-4}, 6.25\times 10^{-5})$ which are respectively used by~\citet{zhang2020robust, liang2022efficient}. As shown in Figure \ref{fig:batch and lr}, we can see CAR-DQN is insensitive to these parameters.

\subsection{Comparisons and Analysis between RADIAL-DQN and CAR-DQN with Increasing Perturbation Radius} \label{app: exp with increasing perturbation radius}

The core at RADIAL-DQN is a heuristic robust regularization that minimizes the overlap between bounds of perturbed Q values of the current action and others:
$$
\mathcal{L}_{radial}\left(\theta\right)=\mathbb{E}_{\left(s, a, s^{\prime}, r\right)}\left[\sum_y Q_{\text {diff}}(s, y) \cdot {Ovl}(s, y, \epsilon)\right],
$$
where
$
Q_{\text {diff}}(s, y)=\max (0, Q(s, a)-Q(s, y)), \ O v l(s, y, \epsilon)=\max (0, \bar{Q}(s, y, \epsilon)-\underline{Q}(s, a, \epsilon)+\eta)
$ and $\eta=c\cdot Q_{\text {diff}}(s, y), c=0.5$. $Q_{\text {diff}}$ is treated as a constant for the optimization. 
We consider RADIAL-DQN could perform better than CAR-DQN with increasing perturbation radius since $\mathcal{L}_{radial}\left(\theta\right)$ is a stronger regularization to enhance robustness while compromising natural rewards. The stronger robust constraint is mainly reflected in two aspects:
\begin{itemize}
    \item The loose bounds. RADIAL-DQN uses the cheap but loose convex relaxation method (IBP) to estimate $\bar{Q}(s, y, \epsilon)$ and $\underline{Q}(s, a, \epsilon)$.
    \item The positive margin $\eta$.  
\end{itemize}
They both result in ${Ovl}(s, y, \epsilon)$ a stronger constraint for representing the overlap of perturbed Q values.
However, RADIAL-DQN has the following weaknesses:
\begin{itemize}
    \item $\mathcal{L}_{radial}\left(\theta\right)$ will harm natural rewards. As shown in Figure~\ref{fig: natural rewards during training}, the natural rewards curve of RADIAL-DQN on RoadRunner distinctly tends to decrease, especially around 0.5 million steps. In contrast, the natural curves of our CAR-DQN showcase more stable upward trends in all environments. Besides, as shown in Table~\ref{table: car dqn trained with larger radius}, RADIAL-DQN training with a larger radius attains lower natural rewards which also restricts robustness according to our theory, while CAR-DQN keeps a better and consistent natural and robust performance.
    \item Heuristic implementation lacks theoretical guarantees and introduces sensitive hyperparameter $c$. We conduct additional experiments on RoadRunner with different $c$ and observe the sensitivity of RADIAL-DQN to the choice of $c$, as shown in Table~\ref{app table: sensitive radial}. Larger $c$ could cause poor performance because the robustness constraint is too strict and thus the policy degrades to some simple policy with lower rewards. Smaller $c$ may result in much weaker robustness. By contrast, our CAR-DQN is developed based on the theory of optimal robust policy and stability. Although we also introduce a hyperparameter $\lambda$, our ablation studies in Figure~\ref{fig:soft roadrunner} show that our algorithm is insensitive to the soft coefficient $\lambda$. 

    \begin{table}[ht]
    \caption{Performance of RADIAL-DQN sensitive to different positive margins $c\cdot Q_{\text {diff}}(s, y)$.}
    \label{app table: sensitive radial}
    \resizebox{\textwidth}{!}{%
    \begin{tabular}{c|c|ccc|ccc}
    \hline
    \multirow{2}{*}{Model} & \multirow{2}{*}{Natural Return} & \multicolumn{3}{c|}{PGD}                                & \multicolumn{3}{c}{MinBest}                            \\ 
                           &                                 & $\epsilon=1/255$ & $\epsilon=3/255$  & $\epsilon=5/255$ & $\epsilon=1/255$ & $\epsilon=3/255$  & $\epsilon=5/255$ \\ \hline
    RADIAL-DQN ($c=0.25$)  & $14678 \pm 329$                 & $ 14836\pm 314$  & $ 13670 \pm 466$  & $ 13512 \pm 617$ & $14712 \pm 309$  & $14804 \pm 457$   & $13226 \pm 351$  \\ \hline
    RADIAL-DQN ($c=0.5$)   & $\bf 46224 \pm 1133$                & $\bf 45990 \pm 1112$ & $\bf  42162 \pm 1147$ & $\bf  23248 \pm 499$ & $\bf 46082 \pm 1128$ & $\bf  42036 \pm 1048$ & $ \bf 25434 \pm 756$ \\ \hline
    RADIAL-DQN ($c=0.75$)  & $3992 \pm 482$                  & $3992 \pm 482$   & $3992 \pm 482$    & $3992 \pm 482$   & $3992 \pm 482$   & $3992 \pm 482$    & $3992 \pm 482$   \\ \hline
    \end{tabular}%
    }
    \end{table}
    \item Depending on the currently learned optimal action. $\mathcal{L}_{radial}\left(\theta\right)$ essentially takes the currently learned action as a robust label which may produce a wrong direction for robustness training if the learned action is not optimal. In contrast, our CAR-DQN seeks optimal robust policies with theoretical guarantees and does not utilize the learned action for robustness training, simultaneously improving natural and robust performance.
\end{itemize}

The main motivation of CAR-DQN based on our theory is to improve natural and robust performance concurrently which makes sense in real-world scenarios where strong adversarial attacks are relatively rare. 
Our training loss can guarantee robustness under the attack with the same perturbation radius as the training. We also think it is a very significant problem whether and how we can design an algorithm training with little epsilon and theoretically guarantee robustness for larger epsilon. However, this is beyond the scope of our paper and we will consider this problem in subsequent work. 

Moreover, as shown in Table~\ref{app table: compare}, CAR-DQN also achieves the top performance in larger perturbation radiuses on Pong and BankHeist and matches the RADIAL-DQN on Freeway. To show the superiority of CAR-DQN further, we also train CAR-DQN with a perturbation radius of 3/255 and 5/255 on RoadRunner for 4.5 million steps (see Table~\ref{table: car dqn trained with larger radius}).
\begin{table}[ht]
\caption{Performance of CAR-DQN and RADIAL-DQN trained with different perturbation radiuses on the RoadRunner environment. The best results of the algorithm with the same training radius are highlighted in bold.}
\label{table: car dqn trained with larger radius}
\resizebox{\textwidth}{!}{%
\begin{tabular}{c|c|ccc|ccc}
\hline \hline
\multirow{2}{*}{Model}        & \multirow{2}{*}{Natural Return} & \multicolumn{3}{c|}{PGD}                                  & \multicolumn{3}{c}{MinBest}                               \\ 
                              &                                 & $\epsilon=1/255$  & $\epsilon=3/255$  & $\epsilon=5/255$  & $\epsilon=1/255$   & $\epsilon=3/255$  & $\epsilon=5/255$  \\ \hline
RADIAL-DQN ($\epsilon=1/255$) & $46224 \pm 1133$                & $45990 \pm 1112$  & $\bf  42162 \pm 1147$ & $\bf  23248 \pm 499$  & $46082 \pm 1128$   & $ \bf 42036 \pm 1048$ & $\bf  25434 \pm 756$  \\ \hline
CAR-DQN ($\epsilon=1/255$)    & $\bf 49398 \pm 1106$               & $\bf  49456 \pm 992$  & $28588 \pm 1575$  & $15592 \pm 885$   & $\bf  47526 \pm 1132$  & $32878 \pm 1898$  & $21102 \pm 1427$  \\ \hline \hline
RADIAL-DQN ($\epsilon=3/255$) & $34656 \pm 1104$                & $35094 \pm 1277$  & $ 35082 \pm 948 $ & $\bf  32770 \pm 1062$ & $ 35096 \pm 1277 $ & $ 34374 \pm 996$  & $\bf  27926 \pm 881$  \\ \hline
CAR-DQN ($\epsilon=3/255$)    & $\bf 47348 \pm 1305$                & $\bf 46284 \pm 1114 $ & $\bf  43578 \pm 1315$ & $ 27060 \pm 1117$ & $\bf  46286 \pm 1122 $ & $\bf  42602 \pm 1336$ & $ 24862 \pm 1195$ \\ \hline \hline
RADIAL-DQN ($\epsilon=5/255$) & $35160 \pm 1157$                & $36158 \pm 1104$  & $36732 \pm 1076$  & $34826 \pm 913$   & $36158 \pm 1104$   & $36732 \pm 1076$  & $34592 \pm 913$   \\ \hline
CAR-DQN ($\epsilon=5/255$)    & $\bf 42545 \pm 2028$                & $\bf 43230 \pm 1468 $ & $\bf  37845 \pm 2344$ & $\bf  39235 \pm 1519$ & $\bf  43645 \pm 1531 $ & $\bf  37535 \pm 2112$ & $\bf  38150 \pm 1316$ \\ \hline \hline
\end{tabular}%
}
\end{table}
We can see that CAR-DQN still attains superior natural and robust performance training with larger attack radiuses while RADIAL-DQN markedly degrades its natural performance due to the too-strong robustness constraint. CAR-DQN always has a higher robust return on the training radius than RADIAL-DQN.

\vskip 0.2in
\bibliography{car_rl}

\end{document}